\documentclass[final,3p,times]{elsarticle}
\usepackage{amsmath, amssymb, amsthm}
\usepackage{geometry}
\usepackage{color}
\usepackage{bm}
\usepackage{bbold}
\usepackage{algorithm}
\usepackage{algorithmic}
\usepackage{caption}
\usepackage{subcaption}
\usepackage{float}
\usepackage{tikz}
\usetikzlibrary{arrows.meta,positioning,fit,calc}
\usepackage{accents}
\newcommand{\thinoverline}[1]{\accentset{\rule{1.2em}{0.3pt}}{#1}}

\font\ppppppcarac=ptmr8y at 4pt
\font\pppppcarac=ptmr8y at 5pt
\font\ppppcarac=ptmr8y at 6pt
\font\pppcarac=ptmr8y at 7pt
\font\ppcarac=ptmr8y at 8pt
\font\pcarac=ptmr8y at 9pt
\font\Pcarac=ptmr8y at 9pt

\font\bf=ptmb8y at 10pt

\newcommand{\bfA}{{\textbf{A}}}
\newcommand{\bfB}{{\textbf{B}}}

\newcommand{\bfF}{{\textbf{F}}}

\newcommand{\bfH}{{\textbf{H}}}

\newcommand{\bfS}{{\textbf{S}}}

\newcommand{\bfU}{{\textbf{U}}}
\newcommand{\bfV}{{\textbf{V}}}
\newcommand{\bfW}{{\textbf{W}}}
\newcommand{\bfX}{{\textbf{X}}}
\newcommand{\bfY}{{\textbf{Y}}}
\newcommand{\bfZ}{{\textbf{Z}}}

\newcommand{\bfa}{{\textbf{a}}}
\newcommand{\bfb}{{\textbf{b}}}
\newcommand{\bfc}{{\textbf{c}}}

\newcommand{\bfe}{{\textbf{e}}}
\newcommand{\bff}{{\textbf{f}}}

\newcommand{\bfh}{{\textbf{h}}}

\newcommand{\bfm}{{\textbf{m}}}
\newcommand{\bfn}{{\textbf{n}}}

\newcommand{\bfr}{{\textbf{r}}}

\newcommand{\bfu}{{\textbf{u}}}
\newcommand{\bfv}{{\textbf{v}}}
\newcommand{\bfw}{{\textbf{w}}}
\newcommand{\bfx}{{\textbf{x}}}
\newcommand{\bfy}{{\textbf{y}}}
\newcommand{\bfz}{{\textbf{z}}}

\newcommand{\bfun}{{ \hbox{\bf 1} }}

\newcommand{\bfzero}{{ \hbox{\bf 0} }}

\newcommand{\bfalpha}{{\bm{\alpha}}}
\newcommand{\bfbeta}{{\bm{\beta}}}

\newcommand{\bfeta}{{\bm{\eta}}}
\newcommand{\bftheta}{{\bm{\theta}}}

\newcommand{\bfnu}{{\bm{\nu}}}

\newcommand{\bfvarphi}{{\bm{\varphi}}}
\newcommand{\bfpsi}{{\bm{\psi}}}

\newcommand{\HH}{{\mathbb{H}}}

\newcommand{\MM}{{\mathbb{M}}}
\newcommand{\NN}{{\mathbb{N}}}

\newcommand{\RR}{{\mathbb{R}}}
\newcommand{\SSeclair}{{\mathbb{S}}}
\newcommand{\TT}{{\mathbb{T}}}
\newcommand{\UU}{{\mathbb{U}}}
\newcommand{\VV}{{\mathbb{V}}}

\newcommand{\XX}{{\mathbb{X}}}
\newcommand{\YY}{{\mathbb{Y}}}

\newcommand{\bb}{{\mathbb{b}}}

\newcommand{\hh}{{\mathbb{h}}}

\newcommand{\pp}{{\mathbb{p}}}
\newcommand{\qq}{{\mathbb{q}}}
\newcommand{\uu}{{\mathbb{u}}}
\newcommand{\vv}{{\mathbb{v}}}
\newcommand{\xx}{{\mathbb{x}}}
\newcommand{\yy}{{\mathbb{y}}}

\newcommand{\ioneone}{{{\mathbb{1}}}}

\newcommand{\curB}{{\mathcal{B}}}
\newcommand{\curD}{{\mathcal{D}}}

\newcommand{\curH}{{\mathcal{H}}}
\newcommand{\curI}{{\mathcal{I}}}
\newcommand{\curJ}{{\mathcal{J}}}

\newcommand{\curL}{{\mathcal{L}}}
\newcommand{\curM}{{\mathcal{M}}}
\newcommand{\curO}{{\mathcal{O}}}
\newcommand{\curP}{{\mathcal{P}}}

\newcommand{\curT}{{\mathcal{T}}}

\newcommand{\curV}{{\mathcal{V}}}

\usepackage{mathrsfs}
\newcommand{\curC}{{\mathscr{C}}}       
\newcommand{\curN}{{\mathscr{N}}}       
\newcommand{\curS}{{\mathscr{S}}}       

\newcommand{\bfcurH}{{\boldsymbol{\mathcal{H}}}}

\font\vcarac=ptmr8y at 10pt

\newcommand{\vc}{{\hbox{\vcarac{v}}}}    
\newcommand{\uc}{{\hbox{\vcarac{u}}}}    

\newcommand{\ad}{{{\hbox{{\pppcarac ad}}}}}
\newcommand{\adp}{{{\hbox{{\ppppcarac ad}}}}}

\newcommand{\CI}{\hbox{{\pcarac CI}}}

\newcommand{\criter}{\hbox{{\ppcarac Cr}}}

\newcommand{\CRPS}{{\hbox{{\ppcarac CRPS}}}}

\newcommand{\CRPSpp}{{\hbox{{\pppppcarac CRPS}}}}

\newcommand{\pelem}{{\hbox{{\pppcarac elem}}}}

\newcommand{\ESpp}{{\hbox{\pppppcarac ES}}}

\newcommand{\pfull}{{\hbox{{\ppppcarac full}}}}

\newcommand{\pgrid}{{\hbox{{\pppcarac grid}}}}
\newcommand{\ppgrid}{{\hbox{{\pppppcarac grid}}}}

\newcommand{\hi}{{\hbox{{\ppppcarac hi}}}}

\newcommand{\pin}{{\hbox{{\pppcarac in}}}}
\newcommand{\ppin}{{\hbox{{\pppppcarac in}}}}

\newcommand{\pinc}{{\hbox{{\pppcarac inc}}}}

\newcommand{\pint}{{\hbox{{\pppcarac int}}}}

\newcommand{\pinit}{{\hbox{{\pppcarac init}}}}

\newcommand{\IOC}{{\hbox{\pcarac IOC}}}
\newcommand{\IOCp}{{\hbox{\ppppcarac IOC}}}

\newcommand{\KLg}{{\hbox{{\Pcarac KL}}}}

\newcommand{\LBp}{{\hbox{{\pppppcarac LB}}}}

\newcommand{\MSP}{{\hbox{{\pppppcarac MSP}}}}

\newcommand{\NLL}{{\hbox{\ppcarac NLL}}}

\newcommand{\opt}{{\,\hbox{{\pppcarac opt}}}}
\newcommand{\optp}{{\,\hbox{{\ppppcarac opt}}}}

\newcommand{\pout}{{\hbox{{\pppcarac out}}}}
\newcommand{\ppout}{{\hbox{{\pppppcarac out}}}}

\newcommand{\Obsp}{{\hbox{{\ppppcarac Obs}}}}
\newcommand{\Obspp}{{\hbox{{\pppppcarac Obs}}}}

\newcommand{\obsp}{{\hbox{{\ppppcarac obs}}}}

\newcommand{\prelax}{{\hbox{{\ppppcarac relax}}}}

\newcommand{\pprc}{{{\hbox{{\ppppcarac prc}}}}}

\newcommand{\pred}{{\hbox{{\ppppcarac red}}}}

\newcommand{\pref}{{\hbox{{\ppppcarac ref}}}}

\newcommand{\rc}{{\hbox{{\ppppcarac rc}}}}

\newcommand{\SBVP}{{\hbox{\pcarac SBVP}}}
\newcommand{\SBVPp}{{\hbox{\ppppcarac SBVP}}}
\newcommand{\SBVPpp}{{\hbox{\pppppcarac SBVP}}}

\newcommand{\SNN}{{\hbox{\pcarac SNN}}}
\newcommand{\SNNp}{{\hbox{\ppppcarac SNN}}}
\newcommand{\SNNpp}{{\hbox{\pppppcarac SNN}}}

\newcommand{\SPINNpp}{{\hbox{\pppppcarac SPINN}}}

\newcommand{\pSB}{{\hbox{{\ppppppcarac SB}}}}

\newcommand{\scaledpp}{{\hbox{{\ppppcarac scaled}}}}

\newcommand{\psim}{{\hbox{{\ppppcarac sim}}}}

\newcommand{\psp}{{{\hbox{{\ppppcarac sp}}}}}

\newcommand{\targetp}{{\hbox{{\ppppcarac targ}}}}

\newcommand{\test}{{\hbox{{\pppcarac test}}}}
\newcommand{\testp}{{\hbox{{\ppppcarac test}}}}
\newcommand{\testpp}{{\hbox{{\pppppcarac test}}}}

\newcommand{\trainingp}{{\hbox{{\pppcarac train}}}}
\newcommand{\trainingpp}{{\hbox{{\ppppcarac train}}}}

\newcommand{\ptrial}{{\hbox{{\pppcarac trial}}}}
\newcommand{\pptrial}{{\hbox{{\ppppcarac trial}}}}

\newcommand{\TRIAL}{{\hbox{{\pppcarac TRIAL}}}}
\newcommand{\PGD}{{\hbox{{\pppcarac PGD}}}}

\newcommand{\Unit}{\mathbf{1}}

\newcommand{\rchi}{%
  \mathord{%
    \mathchoice
      {\raisebox{0.25ex}{$\displaystyle\chi$}}
      {\raisebox{0.25ex}{$\textstyle\chi$}}
      {\raisebox{0.15ex}{$\scriptstyle\chi$}}
      {\raisebox{0.10ex}{$\scriptscriptstyle\chi$}}
  }%
}

\newtheorem{proposition}{Proposition}

\journal{CMAME}
\numberwithin{equation}{section}

\title{Physics-informed random neural architectures on compact boundaryless multiply connected manifolds
for random-operator models in machine learning and uncertainty quantification}

\author[1]{Christian Soize \corref{cor1}}
\ead{christian.soize@univ-eiffel.fr}
\cortext[cor1]{Corresponding author: C. Soize, christian.soize@univ-eiffel.fr}
\address[1]{Universit\'e Gustave Eiffel, MSME UMR 8208, 5 bd Descartes, 77454 Marne-la-Vall\'ee, France}

\date{September 2025}

\begin{document}
\begin{frontmatter}

\begin{abstract}
We introduce stochastic physics-informed neural networks (SPINNs) for stochastic boundary-value problems (SBVPs) with random operators. The method enforces physics through a structural probabilistic coupling: a common Gaussian stochastic germ drives both the SBVP random operator and a manifold-based stochastic neural architecture. The stochastic neural network is non-feedforward and intrinsically random. A latent Gaussian field and an inhomogeneous Poisson point process on a compact, boundaryless, multiply connected manifold generate random neuron locations, geodesically local sparse connectivity, and a sparse random weight matrix, all governed by a low-dimensional hyperparameter. For deterministic inputs, the network output is random and can represent non-Gaussian random operators. Training is performed in observation space: SBVP simulations produce random observables, while the SNN defines a predictive conditional probability density function. The supervised objective combines data fidelity, through negative log-likelihood, with a cross-entropy alignment term relative to the SBVP-induced conditional distribution. The resulting target is a convex mixture of observation-induced and SBVP-induced distributions, with a dimensionless parameter controlling robustness to model--data mismatch. The shared stochastic germ yields common-random-numbers Monte Carlo estimators and reduced-variance gradients. The method uses only evaluations of the SBVP solution mapping, without differentiating or modifying it, and avoids PDE-residual or boundary-penalty losses. A stochastic elliptic problem on a periodic cylindrical surface illustrates the approach in SBVP-consistent and SBVP-misspecified regimes. A generalized-posterior formulation is also outlined for hyperparameter inference and posterior averaging without per-query reoptimization.
\end{abstract}


\end{frontmatter}

\section{Introduction}
\label{Section1}
{\color{black}
The generic manifold-based stochastic neural architecture used in this work was introduced in the companion manuscript~\cite{Soize2026}. That work develops the latent-field and point-process generation of random neuron locations, geodesically local connectivity, and random weights, together with the associated non-feedforward neural equation and data-only likelihood-based learning for general stochastic input--output mappings. These generic architectural components are not claimed as contributions of the present paper. Here, this architecture is used as a probabilistic hypothesis class for introducing stochastic physics-informed neural networks (SPINNs) for stochastic boundary-value problems (SBVPs) with random operators. The contribution of the present work is the structural coupling of the SBVP random operator and the stochastic neural architecture through a shared Gaussian germ, together with the resulting observation-space learning objective combining data fidelity and SBVP-induced cross-entropy alignment.
}
\paragraph{(i) Objective of the paper}
We formulate and analyze a probabilistic coupling between:
(a) a stochastic boundary-value problem (SBVP) posed on a manifold $\curS_c$, governed by a non-Gaussian random operator; and
(b) a stochastic neural network (SNN) whose architecture is defined on a compact, boundaryless, multiply connected  manifold $\curS$ and generated by a latent Gaussian field together with an inhomogeneous Poisson point process producing random neuron locations, geodesically local sparse connectivity, and random weights, all controlled by a low-dimensional hyperparameter $\bftheta$. The random architecture is induced jointly by $U(\cdot\,;\bftheta)$ and the Poisson point process, while the sparse random weight matrix is generated from the latent-field values on the sampled neurons.
The Gaussian field $U_c$ on $\curS_c$, which drives the random operator of the SBVP, and the Gaussian field $U$ on $\curS$, which generates the neural architecture of the SNN (via the inhomogeneous Poisson process on $\curS$), both depend on the same stochastic germ $\bfH$, an $\RR^m$-valued normalized Gaussian random variable.
The SBVP is defined by a PDE with a non-Gaussian coefficient field obtained as a nonlinear mapping of $U_c$, and by an input vector $\xx$ with associated observable $\bfY(\xx)=\curO(Z(\cdot\,;\xx))$, where $Z(\cdot\,;\xx)$ denotes the stochastic solution.
On the SNN side, the latent field $U_\bftheta$ on $\curS$, together with the independent inhomogeneous Poisson point process $\Pi_{\Lambda_\bftheta}$ of deterministic intensity $\Lambda_\bftheta$ depending only on the covariance function $C_U(\cdot\,;\bftheta)$, induces a random neural architecture parameterized by $\bftheta$. This architecture includes neuron locations on $\curS$, sparse connectivity determined by geodesic distances, a sparse random weight matrix, input localizations for $\xx$, and an output vector $\YY(\xx\,;\bftheta)$.
Conditional on $\bftheta$, the latent field $U(\cdot\,;\bftheta)$ and the inhomogeneous Poisson process $\Pi_{\Lambda_\bftheta}$ are independent; dependence arises solely through $\bftheta$.
The two Gaussian fields are expressed as
\begin{equation*}
U(\bfx\,;\bftheta) = [M(\bfx\,;\bftheta)]\,\bfH \, ,\,\, \bfx\in\curS \quad , \quad
U_c(\bfx) = [M_c(\bfx)]\,\bfH \, , \,\,\bfx\in\curS_c  \, ,
\end{equation*}
with deterministic matrix-valued fields $[M(\cdot\,;\bftheta)]$ on $\curS$ and $[M_c]$ on $\curS_c$. Thus, the randomness of the neural architecture arises from both $\bfH$ and the Poisson process $\Pi_{\Lambda_\bftheta}$.

{\color{black}
The objective of the paper is to formalize the SNN--SBVP coupling within a purely probabilistic framework and to define a supervised procedure for estimating the hyperparameter $\bftheta$ from observables, without introducing PDE-residual penalties. Training is performed in observation space through an objective that combines a data-fidelity term (negative log-likelihood) with a cross-entropy term aligning the SNN and SBVP conditional probability density functions.
In this work, \textit{physics-informed} therefore denotes the structural probabilistic coupling induced by the shared Gaussian germ $\bfH$, rather than a residual-based constraint. The common use of $\bfH$ couples the SBVP aleatory variability to the SNN latent structure and provides common-random-numbers estimators with reduced variance for the alignment contribution.
The intended application class consists of problems in which deterministic or controlled inputs produce random, potentially high-dimensional and strongly non-Gaussian observables because the governing operator is random. The formulation is particularly relevant when only one output realization is available for each observational input, when complete state fields are not observed, and when an SBVP simulator is available but is not to be differentiated or modified for residual-based training. It also accommodates the situation in which the nominal SBVP is only approximately consistent with the observed system.
SPINN is not proposed as a universal replacement for simpler probabilistic surrogates. Its purpose is to provide an additional modeling route when the probability law of the neural architecture and its structural relation to the SBVP random operator are part of the representation objective. The present work therefore assesses feasibility, probabilistic calibration, and generalization within this targeted setting; it does not claim universal superiority in wall-clock efficiency, scalability, or data economy over all alternative surrogate paradigms.
}
\paragraph{(ii) Related work and background}
We confine this short review to ANN-based surrogate models for PDEs/SPDEs (PINNs, physics-constrained deep learning, and operator learning); intrusive spectral/UQ solvers and non-ANN surrogates fall outside our scope.

Neural-network methods for differential equations predate the modern PINN terminology, including trial-solution constructions that enforce boundary conditions and variational formulations such as the Deep Ritz method \cite{Lagaris1998,EYu2018}.
Physics-Informed Neural Networks (PINNs) have emerged as a powerful tool for solving differential equations by embedding governing equations into the training of deep neural networks~\cite{Raissi2019,Sirignano2018,Karniadakis2021,Yuan2025}.
In the deterministic setting, PINNs approximate the unknown solution with a neural network whose loss function combines a \textit{data loss}, ensuring consistency with observations, and a physics loss, penalizing residuals of the governing differential operator. This approach has proven effective for parameterized PDEs, where incorporating the operator and boundary conditions into the learning objective can improve generalization.
When measurements are noisy, probabilistic PINN variants such as multi-output formulations provide estimates of conditional probability distributions compatible with proper scoring and related UQ criteria~\cite{Yang2022c,Gneiting2007,Psaros2023}.

Extending PINNs to nonlinear stochastic PDEs (SPDEs) remains challenging due to computational complexity, the closure problem for statistical moments, and the inadequacy of existing generative models in representing realistic high-dimensional stochastic physics. Three ANN-based strands dominate:
(1) Moment-based PINNs, which aim to learn statistical moments of the stochastic solution~\cite{Chen2021}, but require closure approximations (e.g., Gaussian closure, moment truncation) that can limit accuracy for nonlinear systems;
(2) Monte Carlo PINNs, training a separate network for each stochastic realization~\cite{Guo2022,Guo2022b}, which can suffer from prohibitive costs in high-dimensional random spaces and difficulties in treating stochastic weights; and
(3) generative PINNs (e.g., VAE/GAN-augmented) that learn solution distributions but face stability and scalability limits in complex physics~\cite{Yang2022b,Zhong2023b}. In parallel, operator-learning surrogates (e.g., DeepONet and Fourier Neural Operators) approximate solution operators directly and are increasingly used as ANN baselines for parametric PDEs~\cite{Lu2021,Li2021,Kovachki2023,Lu2022}.
Low-rank/SVD-guided augmentations have been proposed to improve DeepONet flexibility and interpretability~\cite{Venturi2023,Lu2021}, and resolution-robust implicit variants of Fourier-type neural operators have been investigated in heterogeneous mechanics applications~\cite{You2022,Li2021}.
More recently, uncertainty quantification for operator learning has been pursued via randomized priors and ensemble constructions for DeepONets~\cite{Yang2022d,Gneiting2007}, as well as Bayesian treatments that explicitly account for noisy inputs/outputs in PINNs and neural operators~\cite{Zou2025,Bulte2025}. Complementary lines of work address functional uncertainty quantification in operator learning through distribution-free calibration and score-based generative modeling in function space \cite{Ma2024,Lim2025}, while works develop resolution-independent operator-learning generative priors for Bayesian inverse problems~\cite{Jiang2024,Owhadi2015} and diffusion-model-based field reconstruction from sparse observations~\cite{Zhuang2025,Ma2024}. Complementary to these developments, the manifold-based stochastic neural-architecture construction of~\cite{Soize2026} induces probability distributions on both network topology and weights, yielding predictors with values in spaces of probability distributions that can be assessed with strictly proper scoring rules.

Bayesian formulations replace point estimates of weights with posteriors, yielding calibrated posterior variability and robustness to noise. B-PINNs combine a Bayesian neural network prior with physics-informed likelihoods and infer posteriors via HMC or VI~\cite{Yang2021b,Linka2022}. Related physics-constrained probabilistic surrogates train without labeled fields by minimizing divergences between model probability density functions and physics-induced probability density functions, offering an alternative to residual penalties~\cite{Zhu2019}. There is also a growing body of work using approximate Bayesian treatments of ANN surrogates (Laplace/SGLD/VI) and conformal or calibration layers for neural PDE solvers, but they remain largely decoupled from random-operator structure.

{\color{black}
At the architectural level, several related families must be distinguished. Bayesian neural networks and dropout-as-Bayesian-approximation introduce stochasticity primarily at the level of weights or activations within a prescribed architecture~\cite{Neal1996,Gal2016}. Stochastic-depth methods randomly bypass computational blocks during training~\cite{Huang2016}. Randomly wired networks instead sample graph connectivity from prescribed stochastic graph generators and train the resulting graph instances as task-specific neural architectures~\cite{XieRandomWired2019}. In these approaches, random wiring broadens the architecture design space, but the sampled graph is generally subsequently used as a conventional task predictor.

Sparse and evolving-connectivity methods have a different objective. Sparse evolutionary training starts from a random sparse topology and alternates weight optimization with pruning and random regrowth of connections~\cite{Mocanu2018}. Deep Rewiring maintains a strictly bounded number of connections while sampling task-adapted sparse configurations during supervised learning~\cite{Bellec2018}, and RigL updates a fixed-size sparse topology using weight magnitudes and intermittent gradient information~\cite{Evci2020}. These methods retain a prescribed set of neurons and usually a prescribed layered organization; connectivity evolution is part of the training algorithm and is primarily directed toward computational efficiency, memory reduction, and task accuracy.

Neural architecture search (NAS) treats operations and connectivity as architecture variables and searches for a high-performing architecture using reinforcement learning, stochastic relaxation, differentiable optimization, or related strategies~\cite{Zoph2016,Elsken2019,Xie2019,Chen2020}. Although stochastic variables can be used during the search, their principal role is to explore or relax a combinatorial design space. The usual output of NAS is a selected architecture, rather than a stochastic operator whose predictive output remains random for a deterministic input. Geometry-aware networks on manifolds, including geodesic and mixture-model CNNs, encode locality using intrinsic geometric information~\cite{Masci2015,Monti2017,Bronstein2017}, while local neural operators analyze receptive ranges for operator regression~\cite{Ye2024,Kovachki2023}. Probabilistic function-space surrogates such as Neural Processes model distributions over functions within architectures whose connectivity is prescribed~\cite{Garnelo2018,Kim2019}.

The topology generator employed in the present paper is inherited from~\cite{Soize2026} and is not a NAS algorithm. For a fixed $\bftheta$, the covariance structure of the latent Gaussian field $U(\cdot\,;\bftheta)$ determines the intensity $\Lambda_\bftheta$ of the inhomogeneous Poisson point process $\Pi_{\Lambda_\bftheta}$ used to sample neuron locations on $\curS$. Conditional on $\bftheta$, $U(\cdot\,;\bftheta)$ and $\Pi_{\Lambda_\bftheta}$ are statistically independent. The sampled locations define random intrinsic geodesic distances, from which the sparse connectivity mask is constructed. Realized latent-field values then determine the input and output neuron sets and modulate the random weights. Consequently, neuron locations, node roles, topology, and weights are governed by a common low-dimensional generative hyperparameter, without introducing an edge-by-edge combinatorial search.

The modeling advantages of this construction are specific to the targeted random-operator setting. The Poisson intensity provides a nonuniform spatial allocation of neurons according to the pointwise variance of the latent field. Intrinsic geodesic distances preserve locality on the compact, multiply connected manifold and avoid artificial shortcuts through the ambient Euclidean space. The anisotropic latent field induces nonhomogeneous, direction-dependent, and spatially correlated random weights. Moreover, the generated SNN is used as a distribution-valued probabilistic model rather than as a single architecture selected for deterministic prediction. In SPINN, the Gaussian germ entering this latent-field construction is additionally shared with the Gaussian germ generating the SBVP random operator. This shared-germ coupling, together with training on observables $\bfY(\xx)=\curO(Z)$, distinguishes the present formulation from randomly wired networks, dynamic sparse training, and conventional NAS. These are structural modeling properties; no universal superiority in accuracy or computational efficiency over the cited architecture-design methods is claimed.
}
\paragraph{(iii) Novelty and contribution}
{\color{black}
The generic manifold-based stochastic neural architecture employed below is inherited from~\cite{Soize2026} and is not claimed as a new contribution of the present paper. Reference~\cite{Soize2026} addresses the construction and likelihood-based training of this architecture for general stochastic input--output mappings, independently of any SBVP. The present contribution begins with the coupling of this architecture to an SBVP with a random operator. It introduces a stochastic physics-informed neural network (SPINN) in which physics is encoded through a structural probabilistic coupling and learning is performed directly at the level of observable conditional distributions. The contributions specific to the present paper are the following.
}

(1) \textit{Non-intrusive probabilistic training}.
We formulate a supervised learning strategy that combines a data-fidelity term (negative log-likelihood) with a cross-entropy alignment term between the SBVP observables $\bfY(\xx)$ and the SNN outputs $\YY(\xx\,;\bftheta)$. No PDE residuals or boundary penalties are imposed. The method requires neither access to the full state fields nor modification of the underlying numerical solver. Consequently, the approach requires only evaluations of the numerical solution operator and defines stochastic calibration directly at the level of the induced probability distribution of the observable.

{\color{black} (2) \textit{Use of an intrinsically stochastic, non-feedforward architecture as an SBVP-coupled probabilistic hypothesis class}.
The non-feedforward random neural architecture itself is inherited from~\cite{Soize2026}. Its specific role in the present formulation is to provide a distribution-valued SNN model that can be structurally coupled to the random operator of the SBVP. For a deterministic input $\xx$, the output $\YY(\xx\,;\bftheta)$ remains intrinsically random because the latent Gaussian field and the inhomogeneous Poisson process generate random neuron locations, geodesically local sparse connectivity, and a sparse random weight matrix. Within SPINN, this architecture therefore defines a low-dimensional parametrized family of conditional probability distributions capable of representing the observable variability induced by the SBVP random operator.
}

(3) \textit{Shared probabilistic germ for variance-controlled coupling}.
We introduce a probabilistic coupling in which a common Gaussian random vector $\bfH$ drives both the non-Gaussian random operator of the SBVP on $\curS_c$ and the latent stochastic structure of the SNN on $\curS$. The same germ enters both the stochastic SBVP solution $\bfY(\xx)=\curO(Z(\cdot\,;\xx))$ and the latent Gaussian field $U(\cdot\,;\bftheta)$ used by the SNN. The RNA, together with the conditional SNN output $\YY(\xx\,;\bftheta)$, is then produced jointly by this latent field and the inhomogeneous Poisson point process. This shared construction reduces the variance of the conditional cross-entropy estimator while maintaining consistency with data fidelity.

(4) \textit{Structural embedding of the random operator into the learning architecture}.
The proposed SPINN replaces residual-based physics enforcement by embedding the SBVP random-operator descriptor into the SNN generative architecture through the shared Gaussian germ. Physics is therefore not imposed as an external residual constraint, but encoded probabilistically at the level of the stochastic construction. This shifts the role of physics from a loss-level regularizer to an architectural prior on a class of predictive conditional probability distributions.
As a consequence, the learning problem is formulated directly at the level of probability distributions of observables, and the SNN family becomes a parametrized class of random operators structurally consistent with the SBVP model.

{\color{black}
(5) \textit{Model-discrepancy-aware population target}.
The finite-sample cross-entropy estimator can be evaluated as a weighted likelihood on SBVP-generated pseudo-observations. The contribution of the present formulation does not lie in this algebraic representation, but in the probabilistic target that it induces when the observations and the nominal SBVP do not have the same conditional probability distribution. Proposition~\ref{Proposition1} shows that SPINN performs an information projection of the convex mixture
$q_\lambda = (1-\lambda) p^{\Obsp}_{\bfY\vert\XX} + \lambda p^{\SBVPpp}_{\bfY\vert\XX}$
onto the SNN family. In the SBVP-consistent regime, this mixture reduces to a single target probability distribution and the alignment mainly regularizes finite-sample estimation. In the SBVP-misspecified regime, the mixture is nontrivial and defines an explicit compromise between observational evidence and the probabilistic structure supplied by the nominal SBVP.\\
}

\noindent\textit{Remark.}
Hyperparameter estimation in the proposed framework is performed via trial-based model selection followed by projected gradient descent (PGD) with Adam updates. Section~\ref{Section7.5} presents a complementary Bayesian-type formulation for the inference of $\bftheta$. This approach is not a central contribution of the present work and relies on standard methodology. Specifically, from a single training dataset, a Gibbs-type posterior measure on $\bftheta$ is constructed using a particle approximation. This formulation is included for completeness and discussion purposes. It enables posterior averaging of conditional probability distributions without repeated reoptimization of $\bftheta$ for new inputs $\xx$.
\paragraph{(iv) Illustration supporting the SPINN proposed methodology}
Concerning the illustration, we work on a nontrivial manifold setting tailored to the physics: the SBVP is posed on the open surface strip $\curS_c$ of a circular cylinder in $\RR^3$ with periodic identification of the boundary circles $\Sigma_0$ and $\Sigma_L$. The PDE is formulated with the intrinsic surface gradient/divergence on $\curS_c$, together with periodic trace and conormal flux conditions. The stochastic operator is defined by a non-Gaussian second-order positive-valued random field $K_c(\bfx)$ constructed from the Gaussian field $U_c$, defined on a probability space $(\Omega,\curT,\curP)$, through a monotone transport to an auxiliary Gamma field $V$ and an affine normalization, which ensures
(1) almost surely periodicity and continuity of $K_c$ on $\thinoverline{\curS}{\!_c}$,
(2) an explicit lower bound $K_c(\bfx)\geq k_{c,{\min}}=\varepsilon/(1+\varepsilon)>0$ a.s.  yielding coercivity.
On the functional side, the stochastic weak problem is formulated in the periodic Sobolev space $\curV \subset \curH^1(\curS_c)$ and its zero-mean subspace $\curV_0$, we invoke the Poincaré inequality to establish norm equivalence, and prove that, almost surely,  the random bilinear form $B$ of the weak formulation is continuous and coercive.
This yields existence and uniqueness of a second-order (mean-square) random solution $Z\in L^2(\Omega;\curH^1(\curS_c))$ together with an explicit energy estimate in terms of $\Vert f_{\mathrm{src}}\Vert_{L^2(\curS_c)}$ and geometric constants. The proposed numerical illustration is therefore a fully analyzable, geometry-aware stochastic elliptic problem and provides rigorous well-posedness.
\paragraph{(v) Training regimes and role of the SBVP}

The distinction between SBVP-consistent and SBVP-misspecified regimes is essential because observations generated outside the nominal SBVP model need not satisfy the SBVP-induced conditional probability distribution exactly. Two distinct regimes can therefore be identified.
In the first regime (Case~1), the training dataset is generated directly from the SBVP, so that the observations are sampled from the SBVP-induced conditional probability distribution. Consistency is therefore imposed at the level of the training target, not by an exact equality between the SBVP output $\bfY(\xx)$ and the SNN output $\YY(\xx\, ;\bftheta)$ for a finite-dimensional hyperparameterization. Here, the SBVP does not act as a corrective constraint; since it generates samples from the target conditional probability distribution, it mainly reduces the variance and stabilizes the estimators used to identify $\bftheta$.
{\color{black}
In the second regime (Case~2), the training dataset is derived from external measurements, high-fidelity calculations, or an alternative stochastic model. Such data do not, in general, satisfy the nominal SBVP in the probabilistic sense. The observation-generating and SBVP-induced conditional probability distributions are therefore distinct. In this regime, the SBVP-generated samples cannot be interpreted as ordinary augmented data drawn from the same target distribution as the observations. They represent a second, model-based probability measure against which the SNN conditional probability density function is aligned through the cross-entropy term.
Case~2 is the principal model-discrepancy setting addressed by the mixed objective. The resulting population target is a convex mixture of the observation-induced and SBVP-induced conditional probability distributions. It allows the observations to correct systematic errors in the nominal SBVP while retaining the SBVP as a probabilistic physical reference. Across both regimes, the SBVP therefore enters either as the probability distribution generating the observations in Case~1 or as a distinct reference probability distribution in Case~2.
}
\paragraph{(vi) Choice of regimes for the numerical application}
As noted in paragraph~(v) above, two regimes are considered. The first one, called Case~1, corresponds to training and test datasets generated using the SBVP. Therefore, the points in these datasets almost surely satisfy the SBVP, and the SNN defines a statistical surrogate whose conditional probability distribution is consistent with the SBVP-induced conditional probability distribution.
The second one, called Case~2, consists of using training and test datasets whose points do not almost surely satisfy the SBVP. In the numerical application, this regime is implemented by generating the datasets with an intentionally misspecified perturbation of the SBVP, denoted by $\SBVP^\MSP$, obtained by modifying the input supports and the correlation lengths. More generally, such datasets may originate from numerical simulations carried out with an alternative stochastic model or from experimental measurements. In this regime, the SBVP does not generate the data but provides the reference probability structure against which the SNN conditional probability density function is aligned through the cross-entropy term. This constrains the supervised estimation of $\bftheta$ by the SBVP-induced conditional probability distribution while preserving data fidelity.
\paragraph{(vii) Organization of the paper}
Section~\ref{Section2} provides the formal definition of the manifold-based stochastic neural architecture and of the SNN–SBVP probabilistic coupling that defines SPINN. It establishes the population target of the supervised objective and the score-form gradient identity underlying the alignment term, thereby grounding the framework in an explicit information-theoretic formulation.
The subsequent sections then examine diagnostic methodology, illustrative well-posedness, and numerical performance within this rigorously defined probabilistic framework.
Section~\ref{Section3} presents a general testing methodology together with intrinsic diagnostic criteria.
For ease of reading, Section~\ref{Section4} recalls and adapts, for the present SPINN setting, the supervised learning of random neural architectures structured by latent random fields on the manifold $\curS$, as introduced in~\cite{Soize2026}. In addition, it presents developments that relax the constraints defining the admissible sets of the optimization problems and enable the implementation of the algorithms used to solve the random nonlinear neural equation.
Section~\ref{Section5} provides a physical and abstract illustration of an SBVP with a random operator, controlled inputs, and variability, posed on a manifold $\curS_c$.
Section~\ref{Section6} presents the numerical application, implementation details for the SNN approach, and results. Section~\ref{Section7} presents a discussion in particular of the CRPS levels in the finite-parameter regime, the conditional probability density functions for selected outputs, and the construction of a probability distribution on the hyperparameters from a training dataset without per-query optimization.
Finally, Section~\ref{Section8} summarizes the main contributions and outlines perspectives.\\

\noindent {\textit{(viii) Conventions for the variables, vectors, and matrices}}.\\
\noindent $x,\eta$: lower-case Latin or Greek letters are deterministic real variables.\\
$\bfx,\bfeta$: boldface lower-case Latin or Greek letters are deterministic vectors.\\
$X$: upper-case Latin letters are real-valued random variables.\\
$\bfX$: boldface upper-case Latin letters are vector-valued random variables.\\
$[x]$: lower-case Latin letters between brackets are deterministic matrices.\\
$[\bfX]$: boldface upper-case letters between brackets are matrix-valued random variables.\\

\noindent $\xx$: blackboard-bold lower-case letter is the deterministic input vector.\\
$\XX$: blackboard-bold upper-case letter is the random input vector.\\
$\bfY(\xx)$: random output vector of the SBVP given input $\xx$.\\
$\YY(\xx)$: random output vector of the SNN given input $\xx$.\\

\noindent {\textit{(ix) Algebraic notations and acronyms}}.

\noindent $\NN$: set of natural numbers including $0$.\\
$\RR$: set of real numbers.\\
$\RR^m$: Euclidean vector space of dimension $m$.\\
\noindent $\MM_{m,n}$, $\MM_m$ : set of the $(m\times n)$, $(m\times m)$,  real matrices.\\
$\MM_m^+$, $\MM_m^{+0}$: set of  the positive-definite, positive,  $(m\times m)$ real matrices.\\
$[I_{m}]$: identity matrix in $\MM_m$.\\
$\langle \, \bfx ,\bfy\,\rangle$: usual Euclidean inner product of $\bfx$ and $\bfy$.\\
$\Vert \, \bfx\,\Vert$: Euclidean norm of $\bfx$ equal to $\langle \, \bfx ,\bfx\,\rangle^{1/2}$.\\
$\Vert\, [x]\, \Vert_F$: Frobenius norm of matrix  $[x]$.\\
i.i.d.: independent and identically distributed.\\

\noindent CE: cross-entropy.\\
DOF: degree of freedom.\\
KDE: kernel density estimation.\\
NLL: negative log-likelihood.\\
PDE: partial differential equation.\\
PDF: probability density function.\\
RNA: random neural architecture.\\
SBVP: stochastic boundary-value problem.\\
SNN: stochastic neural network.\\
SPINN: stochastic physics-informed neural network.\\

\noindent {\textit{(x) Convention used for random variables}}.
In this paper, for any finite integer $m\geq 1$, the Euclidean space $\RR^m$ is equipped with the $\sigma$-algebra $\curB_{\RR^m}$. If $\bfY$ is a $\RR^m$-valued random variable defined on the probability space $(\Omega,\curT,\curP)$, $\bfY$ is a  mapping $\omega\mapsto\bfY(\omega)$ from $\Omega$ into $\RR^m$, measurable from $(\Omega,\curT)$ into $(\RR^m,\curB_{\RR^m})$, and $\bfY(\omega)$ is a realization (sample) of $\bfY$ for $\omega\in\Omega$. The probability distribution of $\bfY$ is the probability measure $P_\bfY(d\bfy)$ on the measurable set $(\RR^m,\curB_{\RR^m})$ (we will simply say on $\RR^m$). The Lebesgue measure on $\RR^m$ is noted $d\bfy$ and when $P_\bfY(d\bfy)$ is written as $p_\bfY(\bfy)\, d\bfy$, $p_\bfY$ is the probability density function (PDF) on $\RR^m$ of $P_\bfY(d\bfy)$ with respect to $d\bfy$. We denote by $E$ the mathematical expectation operator that is such that $E\{\bfY\} = \int_{\RR^m} \bfy \, P_\bfY(d\bfy)$. Random variable $\bfY$ is a second-order random variable if $E\{\Vert\bfY\Vert^2\} < +\infty$. The set of all the second-order random variables defined on $(\Omega,\curT,\curP)$ with values in $\RR^m$ is a Hilbert space denoted as $L^2(\Omega,\RR^m)$, equipped with the inner product
$\langle \bfY ,\bfY'\rangle_\Omega  = E\{ \langle \bfY ,\bfY'\rangle \}$ and the associated norm  $\Vert\bfY\Vert_\Omega = \langle \bfY ,\bfY \rangle_\Omega^{1/2}$.
Finally, a normalized Gaussian random vector $\bfV$ is a vector whose components are all statistically independent, with each component being Gaussian, centered, and having unit variance.
\section{Definition of the physics-informed manifold-based stochastic neural architecture for boundary-value problems with random operators}
\label{Section2}
This section provides the formal probabilistic definition of SPINN as a structural coupling between a stochastic boundary-value problem (SBVP) with a random operator and a manifold-based stochastic neural architecture. We recall the SBVP structure and use the SNN construction introduced in~\cite{Soize2026}, with the adaptations and implementation details required for the present SBVP-coupled formulation given in Section~\ref{Section4}. We then derive the supervised objective, its population target, and the gradient identity underlying the SBVP$\to$SNN alignment term.
\subsection{Key elements for defining SPINN}
\label{Section2.1}
We consider

\noindent {\tiny$\bullet$} a domain $\curS_c$, which is either a bounded open subset of $\RR^3$, the bounded unit cell of a periodic medium, or a manifold embedded in $\RR^3$. We denote its closure by $\thinoverline{\curS}{\!_c}$,

\noindent {\tiny$\bullet$} a second-order, centered, Gaussian real-valued random field $\{U_c(\bfx),\bfx\in \thinoverline{\curS}{\!_c}\}$, indexed by $\thinoverline{\curS}{\!_c}$,  with almost surely continuous trajectories, defined on a probability space $(\Omega,\curT,\curP)$,

\noindent {\tiny$\bullet$} a compact, boundaryless, multiply connected manifold $\curS$ on which the SNN architecture is defined,

\noindent {\tiny$\bullet$} a hyperparameter $\bftheta \in \curC_\bftheta \subset \RR^{n_\bftheta}$ controlling the probability distribution of the latent Gaussian random field $U(\cdot\, ;\bftheta)$ indexed by $\curS$.
\paragraph{(i) Definition of the stochastic observation-based SBVP}
We consider a SBVP posed on $\curS_c$.
The partial differential equation (PDE) coefficients are stochastic not because of observational or forcing noise, but because they arise from the random field $\{U_c(\bfx),\bfx\in \thinoverline{\curS}{\!_c}\}$, which determines the coefficient $K_c(\bfx) = \xi(U_c(\bfx))$, where $\xi$ is a given continuous function from $\RR$ into $]0\, , +\infty[$, making $K_c$ a non-Gaussian random field.
For instance, $K_c$ may represent a heterogeneous material coefficient on $\thinoverline{\curS}{\!_c}$, yielding a stochastic elliptic boundary-value problem driven by a positive-valued non-Gaussian field.
The SBVP also depends on a vector-valued parameter $\xx\in\RR^{n_\ppin}$. The components of $\xx$ are parameters that control the non-Gaussian field $K_c$ and the forcing function. The unique second-order stochastic solution is denoted $\{Z(\bfx\, ;\xx),\bfx\in\curS_c\}$ and can be written as
\begin{equation} \label{eq2.1}
Z(\bfx\, ;\xx) = S^\SBVPp( \bfx\, ; \xx, U_c)\quad , \quad \bfx\in\curS_c \quad , \quad  \xx\in \RR^{n_\ppin} \, ,
\end{equation}
where $S^\SBVPp$ is the deterministic mapping associated with solving the SBVP.
We define a random output as the $\RR^{n_\ppout}$-valued random variable $\bfY(\xx)$ given by
\begin{equation}  \label{eq2.2}
\bfY(\xx) =  \curN^\SBVPp(\xx \,; U_c) \quad, \quad \xx\in \RR^{n_\ppin} \, ,
\end{equation}
where the mapping $\curN^\SBVPp$ is defined by
\begin{equation}  \label{eq2.3}
\curN^\SBVPp(\xx \,; U_c) = \curO(Z(\cdot \, ; \xx))  = \curO(S^\SBVPp( \cdot \, ; \xx, U_c)) \, ,
\end{equation}
and $\curO$ is a prescribed linear or a nonlinear observation operator.
\paragraph{(ii) Definition of the SNN associated with the stochastic observation-based SBVP}
The stochastic neural network (SNN) is defined by a random neural architecture (RNA) as explained in Section~\ref{Section1} and detailed in Section~\ref{Section4}. More precisely, the SNN is not a fixed feedforward neural network with random weights; rather it is stochastically generated by the latent random field $\{U(\bfx\, ;\bftheta), \bfx\in\curS\}$, indexed by $\curS$ \cite{Adler2007}.
The RNA of this SNN, consisting of $N$ neurons,  is generated by an inhomogeneous Poisson point process $\Pi_{\Lambda_\bftheta}$  defined on the manifold $\curS$ \cite{Daley2003,Last2018}, whose intensity function $\Lambda_\bftheta$ with respect to the intrinsic surface measure $d\sigma(\bfx)$ on $\curS$ for which the mean count is $\int_\curS \Lambda_\bftheta(\bfx)\, d\sigma(\bfx)$, depends only on the covariance operator $C_U(\cdot\,;\bftheta)$ of $U(\cdot\, ;\bftheta)$. Hence, for fixed $\bftheta$, $U(\cdot\,;\bftheta)$ and $\Pi_{\Lambda_\bftheta}$ are statistically independent; their only coupling across configurations arises through the common hyperparameter $\bftheta$.
The random  neuron locations on $\curS$ are drawn from $\Pi_{\Lambda_\bftheta}$. The sparse connectivity between neurons is random and depends on their random geodesic distances on $\curS$ \cite{Penrose2003}. The weights are random and generated from the latent field
$U(\cdot\, ;\bftheta)$. The input vector $\xx\in\RR^{n_\ppin}$ and the $\RR^{n_\ppout}$-valued random output variable
$\YY(\xx\, ;\bftheta)$ are randomly selected from the set of neurons based on both $U(\cdot\, ;\bftheta)$ and $\Pi_{\Lambda_\bftheta}$.
The random output  $\YY(\xx\, ;\bftheta)$ depends on the hyperparameter $\bftheta \in \curC_{\bftheta} \subset \RR^{n_{\bftheta}}$.
Since the SNN is not a feedforward neural network, the random activations of the neurons are obtained by solving a random nonlinear neuronal equation; the equation and the fixed-point solver are described in Sections~\ref{Section4.7} and~\ref{Section4.8}, respectively, while existence results and additional algorithmic details for the underlying SNN construction are given in~\cite{Soize2026}.
The probability law of the SNN is controlled by the hyperparameter $\bftheta$, whose optimal value can be obtained via
supervised training, formulated as the task of inferring $\bftheta$ in the generative process, rather than directly optimizing the random weights as in Bayesian feedforward neural networks. For each $\bftheta$, the loss function is defined as the negative log-likelihood (NLL), where the non-Gaussian conditional probability density function of $\YY(\xx\, ;\bftheta)$ given $\xx$ is estimated via the KDE based on $n_\psim$ realizations of the random neuron locations, random selections of neurons for the input and output, random connectivities, sparse random weight matrix, and bias vector.
The SNN model is written as
\begin{equation}  \label{eq2.4}
\YY(\xx\, ;\bftheta) = \curN^\SNNp(\xx \,; U(\cdot \, ; \bftheta),\Pi_{\Lambda_\bftheta}) \quad, \quad \xx\in \RR^{n_\ppin}
                   \quad ,  \quad \bftheta\in\curC_{\bftheta}\subset \RR^{n_{\bftheta}} \, ,
\end{equation}
where the SNN input is the vector $\xx \in \RR^{n_\ppin}$ used in the SBVP, and the SNN output is the $\RR^{n_\ppout}$-valued random variable $\YY(\xx\, ;\bftheta)$, corresponding  to the SBVP observation  $\bfY(\xx)$ of the SBVP.
\paragraph{(iii) Definition of the training dataset and scaling}

The training dataset consists of $n_d$ independent input–output pairs.
The input $\xx$ is regarded as a realization of an $\RR^{n_\ppin}$-valued random variable $\XX$ with probability distribution $P_\XX(d\xx)$ on $\RR^{n_\ppin}$.
Let $\{\xx^i, i=1,\ldots,n_d\}$ denote independent realizations of $\XX$.
For each input $\xx^i$, a single output $\yy^i \in \RR^{n_\ppout}$ is observed.
This output is one realization of an $\RR^{n_\ppout}$-valued random observation variable $\YY^\Obsp(\xx^i)$.
The latter represents the data-generating conditional probability distribution, which may coincide with the SBVP response $\bfY(\xx^i)$ (Case~1) or differ from it (Case~2) (see Section~\ref{Section1}-(v)).
Since the SBVP involves a random operator, its response $\bfY(\xx^i)$ is itself an $\RR^{n_\ppout}$-valued random variable.
Likewise, for any hyperparameter $\bftheta$, the SNN output $\YY(\xx^i;\bftheta)$ is random.
The supervised learning dataset is therefore
\begin{equation}  \label{eq2.5}
\curD_\trainingp = \{ (\xx^i,\yy^i) ,  i =1,\ldots, n_d\} \quad , \quad \xx^i\in \RR^{n_\ppin} \quad , \quad \yy^i \in  \RR^{n_\ppout} \, .
\end{equation}
Since the SNN output $\YY(\xx;\bftheta)$ takes values in $[-1,1]^{n_\ppout}$, we scale inputs and outputs componentwise to $[-1,1]^{n_\ppin}$ and $[-1,1]^{n_\ppout}$, respectively.
We introduce affine scaling maps $S_\XX$ for inputs and $S_{\YY^\Obspp}$ for observed outputs so that
$S_{\YY^\Obspp}(\YY^\Obsp(\xx)) \in [-1,1]^{n_\ppout}$.
The same output scaling is applied to $\YY(\xx;\bftheta)$ and to the SBVP response $\bfY(\xx)$ to ensure consistency in the alignment term.

All reported results are expressed in physical (unscaled) units; numerical computations are performed in the scaled domain and mapped back through the inverse transformations.
\paragraph{(iv) Formal definition of the classical supervised training for SNN}
For any given $\bftheta\in\curC_\bftheta$ and for all $\xx$ in $\RR^{n_\ppin}$, it is assumed that the conditional probability distribution $P^{_,\SNNpp}_{\YY\vert\XX}(d\yy \,\vert \,\xx\,; \bftheta)$ admits a probability density function with respect to $d\yy$ and therefore
\begin{equation}  \label{eq2.6}
P^{_,\SNNpp}_{\YY\vert\XX}(d\yy \,\vert \,\xx\, ; \bftheta) = p^{\SNNpp}_{\YY\vert\XX}(\yy \,\vert \,\xx\, ; \bftheta) \, d\yy \, ,
\end{equation}
where $\yy \mapsto p^\SNNpp_{\YY\vert\XX}(\yy \,\vert\, \xx \, ; \bftheta)$ denotes the conditional probability density function of the $\RR^{n_\ppout}$-valued random variable $\YY(\xx\, ;\bftheta)=\curN^\SNNp(\xx\, ;U(\cdot;\bftheta),\Pi_{\Lambda_\bftheta})$.
Training the vector-valued hyperparameter $\bftheta$ from network supervision consists in computing an optimal value  $\bftheta^\optp$  of the hyperparameter $\bftheta$ obtained by solving the optimization problem,
\begin{equation}  \label{eq2.7}
\bftheta^\optp = \arg \min_{\bftheta\in \curC_{\bftheta}} \curJ^\SNNp(\bftheta) \, ,
\end{equation}
where the cost function $\curJ^\SNNp(\bftheta)$ is defined
\begin{equation}  \label{eq2.8}
\curJ^\SNNp(\bftheta) = -\sum_{i=1}^{n_d} \log p^\SNNpp_{\YY\vert\XX}(\yy^i \,\vert \,\xx^i; \bftheta) \quad , \quad \bftheta\in\curC_{\bftheta} \, .
\end{equation}
For fixed $\bftheta$ and input $\xx^i$, one first draws a reference realization $(u^{(0)}(\cdot;\bftheta),\pp^{(0)}_{\Lambda_\bftheta})$ of $(U(\cdot;\bftheta),\Pi_{\Lambda_\bftheta})$, which defines one realization of the random neural architecture.
Conditioned on this architecture realization, $n_\psim$ independent realizations $u^\ell(\cdot;\bftheta)$ of the latent field are then generated to construct $n_\psim$ realizations of the corresponding sparse random weight matrix and to solve the associated random nonlinear neural equations, yielding $n_\psim$ realizations $\yy^\ell(\xx^i;\bftheta)$ of $\YY(\xx^i;\bftheta)$.
The probability density function $p^\SNNpp_{\YY\vert\XX}(\cdot \,\vert\, \xx^i; \bftheta)$ is finally estimated by Gaussian KDE \cite{Silverman1986} from $\{\yy^\ell(\xx^i;\bftheta)\}_{\ell=1}^{n_\psim}$ (see Algorithm~\ref{algorithm1} in Section~\ref{Section2.3}).
\paragraph{(v) Formal definition of the stochastic solution of the  SBVP}
For  all $\xx$ in $\RR^{n_\ppin}$, it is assumed that the conditional probability distribution $P^{_,\SBVPpp}_{\bfY\vert\XX}(d\bfy \,\vert \,\xx)$ admits a probability density function with respect to $d\bfy$ and therefore
\begin{equation}  \label{eq2.9}
P^{_,\SBVPpp}_{\bfY\vert\XX}(d\bfy \,\vert \,\xx) = p^{\SBVPpp}_{\bfY\vert\XX}(\bfy \,\vert \,\xx) \, d\bfy \, .
\end{equation}
where $\bfy \mapsto p^\SBVPpp_{\bfY\vert\XX}(\bfy \,\vert\, \xx)$ denotes the conditional probability density function of the $\RR^{n_\ppout}$-valued random variable $\bfY(\xx)=\curN^\SBVPp(\xx\, ;U_c)$.
For a given point $\xx^i \in\RR^{n_\ppin}$ in the training dataset, the SBVP is solved (using the Monte Carlo simulation method) to estimate the  conditional probability density function $\yy\mapsto p^\SBVPpp_{\bfY\vert\XX}(\yy \,\vert \,\xx^i)$ of $\bfY(\xx^i) =  \curN^\SBVPp(\xx^i \,; U_c)$ via the KDE from $n_\psim$ realizations $\bfy^\ell(\xx^i) =  \curN^\SBVPp(\xx^i \,; u_c^\ell)$ of $\bfY(\xx^i)$, where $\{u_c^\ell, \ell=1,\ldots, n_\psim\}$ are $n_\psim$ realizations of the Gaussian random field $U_c$.
\subsection{Definition of the supervised training for SPINN as the SNN-SBVP coupling}
\label{Section2.2}
There are \textit{a priori} many possible methods to achieve the SNN-SBVP coupling. For the reasons presented in the next paragraph, we propose a probabilistic coupling that guarantees the \textit{data fidelity} of the SNN while ensuring the \textit{alignment} of the SBVP with the SNN. This probabilistic approach consists of minimizing the data NLL (data fidelity) of the SNN plus a cross-entropy (CE) term (alignment) that matches the SNN conditional probability density function with the SBVP conditional probability density function,  under the same random vector germ used by the latent random field $U(\cdot\, ; \bftheta)$ and the random field $U_c$ used for generating the random operator of the SBVP. The fundamental objective of the proposed method for SPINN is to use, in both geometric and probabilistic terms,  the aleatory descriptor of the physical system to generate the random neural architecture of the SNN.
\paragraph{(i) Methodology proposed for supervised training for SPINN}

Classical PINNs enforce the SBVP weakly by penalizing PDE and boundary residuals at collocation points. This strategy avoids repeatedly solving the SBVP during training but requires explicit implementation or access to the differential operator and boundary conditions, and typically represents the solution in a discretized state space whose dimension is tied to the underlying mesh.
The proposed SPINN uses the SBVP only through evaluations of its solution mapping and explicitly evaluates the SBVP for each input $\xx^i$ to produce a random observable $\bfY(\xx^i) = \curN^{\SBVPp}(\xx^i\, ; U_c)$, as defined by Eqs.~\eqref{eq2.2}–\eqref{eq2.3}. It then trains a stochastic surrogate $\YY(\xx^i\,;\bftheta)=\curN^{\SNNp}(\xx^i\, ;U(\cdot\,;\bftheta),\Pi_{\Lambda_\bftheta})$,
consistent with Eq.~\eqref{eq2.4}, whose output dimension is exactly $n_{\ppout}$ and is typically much smaller than the full state-space dimension.
The formulation in observation space is useful: (i) it avoids estimating a high-dimensional map into the discretized state space; (ii) it yields better-conditioned KDEs and lower-variance estimates because the probability density function is defined on $\RR^{n_{\ppout}}$; and (iii) it allocates the statistical approximation to the observed quantities rather than to the full field.
{\color{black}
For each training input $\xx^i$, the proposed SPINN requires an ensemble of SBVP evaluations rather than evaluations of pointwise PDE residuals. However, the corresponding SBVP observables are independent of the SNN hyperparameter $\bftheta$. The $n_\psim$ realizations
$\{\bfy^\ell(\xx^i),\ell=1,\ldots,n_\psim\}$ can therefore be generated once in an offline preprocessing stage and reused for all trial-grid points and all local-optimization iterations. Consequently, updating $\bftheta$ requires recomputing only the SNN construction and its predictive conditional probability density function, not resolving the SBVP.
This formulation does not require differentiating or modifying the numerical solution procedure and enables direct probabilistic alignment in observation space. Specifically, for each $\xx^i$ in the training dataset $\curD_{\trainingp}$ of Eq.~\eqref{eq2.5}, the SNN conditional probability density function $p^\SNNpp_{\YY\vert\XX}(\cdot \, \vert \,\xx^i\,;\bftheta)$ is trained by minimizing the data NLL in Eq.~\eqref{eq2.8}, plus an SBVP $\to$ SNN cross-entropy term. This matches the densities of observables through common KDE smoothing, rather than through pointwise PDE residuals.
}
Furthermore, the realization $\bfeta^\ell$ of the normalized Gaussian $\RR^m$-valued random variable $\bfH$ is shared between the realization  $u^\ell(\cdot\, ;\bftheta)$ for the SNN and the realization $u_c^\ell(\cdot)$ for the SBVP, such that
\begin{equation*}
 u^\ell(\bfx\, ;\bftheta) = [M(\bfx\,;\bftheta)]\,\bfeta^\ell \, , \,\,  \bfx \in \curS
\quad , \quad
u_c^\ell(\bfx) = [M_c(\bfx)]\,\bfeta^\ell \, , \,\,\bfx\in\curS_c
\end{equation*}
This common-germ coupling reduces the gradient variance of the alignment term and stabilizes the optimization over $\bftheta$.
Because the SNN is random by construction, its conditional probability density function reflects both the SBVP-induced aleatory variability and the internal architectural randomness. The alignment therefore operates directly at the level of conditional probability distributions on observables.
Let $\operatorname{KL}$ denote the Kullback-Leibler divergence, $\operatorname{S}$ the Shannon entropy, and $\operatorname{CE}$ the cross-entropy; the alignment uses the CE/KL identity
\begin{equation} \label{eq2.10}
\rm{KL}\,\Big (\, p^\SBVPpp_{\bfY \vert \XX}(\cdot\, \vert \,\xx) \,\,  \Vert \,\, p^\SNNpp_{\YY \vert \XX}(\cdot\, \vert \,\xx \, ;\bftheta)\,\Big ) =
         - \rm{S}\,\Big (\, p^\SBVPpp_{\bfY \vert \XX}(\cdot\, \vert \, \xx )\,\Big ) + \rm{CE}\,\Big (\, p^\SBVPpp_{\bfY \vert \XX}(\cdot\, \vert \,\xx)\, , p^\SNNpp_{\YY \vert \XX}(\cdot\, \vert \,\xx \, ;\bftheta) \,\Big) \,
\end{equation}
which enforces that this stochastic surrogate matches the SBVP-induced conditional probability density function on the observation space, without ever requiring access to (or approximation of) the full state field or the governing operators. Here, the conditional probability density function $p^\SBVPpp_{\bfY \vert \XX}$ is independent of $\bftheta$.
In practice, we estimate $p^{\SNN}_{\bfY \vert\bfX}(\cdot \mid \xx;\bftheta)$ and $p^{\SBVP}_{\bfY \vert\bfX}(\cdot \mid \xx)$ by Gaussian KDE with a shared bandwidth per output component, and we reuse the same germ $\bfH$ across SBVP and SNN draws (common random numbers) to reduce the variance of the cross-entropy estimator and its gradient~\cite{Rubinstein2008}.
\paragraph{(ii) Definition of the supervised training for SPINN}
Taking into account the explanations presented in paragraph~(i) above, supervised training for SPINN minimizes a  data term plus an SBVP $\to$ SNN alignment term, both evaluated on the observation space and averaged over inputs,
\begin{equation} \label{eq2.11}
\bftheta^\optp ( \alpha_{\rm{top}})
= \arg\min_{ \bftheta \,\in \,\curC_{\bftheta} }\, \curJ^\SPINNpp(\bftheta\, ; \alpha_{\rm{top}})\, .
\end{equation}
with
\begin{equation} \label{eq2.12}
 \curJ^\SPINNpp(\bftheta\, ; \alpha_{\rm{top}}) =  \sum_{i=1}^{n_d} \bigg ( -\log p^\SNNpp_{\YY\vert\XX}(\yy^i \,\vert \,\xx^i\, ; \bftheta) \, + \, \alpha_{\rm{top}} \, \rm{CE}_i(\bftheta) \bigg ) \, ,
\end{equation}
where $\alpha_{\rm{top}}\geq 0$ is a  dimensionless trade-off parameter and
\begin{equation} \label{eq2.13}
\rm{CE}_i(\bftheta)
= - \int_{\RR^{n_\ppout}}  p^\SBVPpp_{\bfY \vert \XX}(\bfy\, \vert \,\xx^i)\, \log p^\SNNpp_{\YY \vert \XX}(\bfy\, \vert \,\xx^i \, ;\bftheta)   \, d\bfy  \, .
\end{equation}
{\color{black}
CE-form SPINN training admits the following finite-sample interpretation. For each input $\xx^i$, the objective combines the log-likelihood of the observed datum $\yy^i$ with weighted log-likelihood contributions evaluated at simulator-generated pseudo-observations. This is an algebraic representation of the Monte Carlo cross-entropy estimator and is not, by itself, claimed as a novel statistical operation.
Its probabilistic interpretation nevertheless differs fundamentally between the two regimes considered below. In Case~1,
$p^{\Obsp}_{\bfY\vert\XX} = p^{\SBVPpp}_{\bfY\vert\XX}$, and the observation and SBVP terms have the same population target. The additional contribution then acts mainly as finite-sample regularization and variance reduction. In Case~2,
$p^{\Obsp}_{\bfY\vert\XX} \neq  p^{\SBVPpp}_{\bfY\vert\XX}$, and the simulator-generated pseudo-observations are not additional samples from the observation-generating probability distribution. They represent a distinct SBVP-induced reference probability measure. The resulting population target is the nontrivial convex mixture defined in Eq.~\eqref{eq2.17}, which implements a controlled compromise between data fidelity and the nominal SBVP under model--data discrepancy.
}
To estimate $\rm{CE}_i(\bftheta)$, Eq.~\eqref{eq2.13} is rewritten as
\begin{equation} \label{eq2.14}
\rm{CE}_i(\bftheta)
= - E_{\bfY\vert\XX=\xx^i} \bigg\{\log p^\SNNpp_{\YY \vert \XX}(\bfY \, \vert \, \xx^i \, ;\bftheta) \bigg\}  \, ,
\end{equation}
which can be estimated by using the realizations $\bfy^\ell(\xx^i) =  \curN^\SBVPp(\xx^i \,; u_c^\ell)$ of $\bfY(\xx^i)$ defined in Section~\ref{Section2.1}-(v) as
\begin{equation} \label{eq2.15}
\widehat{\rm{CE}}_i(\bftheta)
= - \frac{1}{n_\psim} \sum_{\ell=1}^{n_\psim}  \log p^\SNNpp_{\YY \vert \XX}(\bfy^\ell(\xx^i)\, \vert \,\xx^i \, ;\bftheta)  \, ,
\end{equation}
From Eqs.~\eqref{eq2.12} to \eqref{eq2.15}, the scalar $\alpha_{\rm{top}}$ controls the balance between fidelity
to the data and alignment with the SBVP-induced conditional probability density function. It is convenient to reparameterize this weight as
\begin{equation}\label{eq2.16}
\lambda = \frac{ \alpha_{\rm top} }{ 1 + \alpha_{\rm top} } \in [0,1[ \, ,
\end{equation}
so that, in the population limit discussed in paragraph~(iii) below, the SPINN objective corresponds to an information
projection of a convex mixture
\begin{equation}\label{eq2.17}
q_{\lambda}(\cdot \,\vert\, \xx) = (1-\lambda)\,p^{\Obsp}_{\bfY\vert\XX}(\cdot \,\vert\, \xx)
+ \lambda \, p^{\SBVPpp}_{\bfY\vert\XX}(\cdot \,\vert\, \xx)\, .
\end{equation}
onto the SNN family (Proposition~\ref{Proposition1}). For $\alpha_{\rm top}=0$ (that is, $\lambda=0$), training reduces to classical
maximum likelihood on the experimental observations alone, corresponding to the optimization problem defined by Eq.~\eqref{eq2.7}. As $\alpha_{\rm top}$ increases (so that $\lambda \uparrow 1$), the SBVP conditional probability density function
$p^{\SBVPpp}_{\bfY\vert\XX}(\cdot \,\vert\, \xx)$ progressively dominates the effective training target.
Intermediate values of $\alpha_{\rm top}$ therefore implement a tunable compromise: they regularize the SNN toward physics-consistent
conditional probability distributions while preserving robustness to potential model–data mismatch in the SBVP.
{\color{black}
In Case~2, the limiting choice $\alpha_{\rm top} = 0$ and any prescribed choice $\alpha_{\rm top} > 0$ do not define the same population learning problem. By Proposition~\ref{Proposition1}, $\alpha_{\rm top} = 0$ targets $p^{\Obsp}_{\bfY\vert\XX}$, whereas
$\alpha_{\rm top} > 0$ targets the information projection of the mixture $q_{\lambda}$ in Eq.~\eqref{eq2.17}. Consequently, a comparison between separately retrained data-only and aligned models would address a distinct model-selection question concerning the chosen data--SBVP compromise; it is not required to establish the SPINN formulation or its population interpretation.
}
In practice, we do not treat $\alpha_{\rm top}$ as an additional tunable hyperparameter selected on a separate model-selection
set or via an outer optimization loop. Instead, we exploit the Monte Carlo samples already generated to evaluate
Eqs.~\eqref{eq2.12}-\eqref{eq2.15} to define a single scale-matching value $\alpha_0$. Let $\bftheta_\pref$ denote a reference hyperparameter, obtained for instance from a short trial-based optimization with $\alpha_{\rm top} = 0$ (see
Section~\ref{Section6.4}). We first compute the empirical means of the data negative log-likelihood and of the alignment term
at this reference point,
\begin{equation}\label{eq2.18}
\overline L_{\rm data}(\bftheta_\pref) = \frac{1}{n_d} \sum_{i=1}^{n_d} -\log p^\SNNpp_{\YY\vert\XX}(\yy^i \,\vert \,\xx^i\, ; \bftheta_{\pref})
\quad , \quad
\overline L_{\rm align}(\bftheta_\pref) = \frac{1}{n_d} \sum_{i=1}^{n_d} \rm{CE}_i(\bftheta_\pref) \, ,
\end{equation}
where $\rm{CE}_i$ is defined in Eq.~\eqref{eq2.13}. We then set
\begin{equation}\label{eq2.19}
\alpha_0 = \frac{\overline L_{\rm data}(\bftheta_\pref)}{\overline L_{\rm align}(\bftheta_\pref)} \, ,
\end{equation}
so that, at $\bftheta_\pref$, the two contributions in $\curJ^\SPINNpp$ have comparable magnitude.
{\color{black}
The scale-matching relation in Eq.~\eqref{eq2.19} is defined in the normalized observation coordinates fixed in Section~\ref{Section2.1}-(iii). Its role is operational: it balances the numerical magnitudes of the data-fidelity and alignment contributions at the reference point $\bftheta_\pref$. No coordinate-invariant statistical-distance interpretation of $\alpha_0$ is required by the training procedure.
}
We simply fix
\begin{equation*}
\alpha_{\rm top}:= \alpha_0 \, ,
\end{equation*}
thereby keeping the balance between data fidelity and SBVP$\to$SNN alignment controlled by the scale-matching
construction above, without any additional hyperparameter search or extra SBVP solves. The two-step training
algorithm with this fixed trade-off parameter is detailed in Section~\ref{Section6.5}.

In this work, there is no hold-out model-selection set and no distinct diagnostic phase. The diagnostic evaluation is intrinsic to training through the SBVP$\to$SNN alignment term in Eq.~\eqref{eq2.12}: by minimizing the cross-entropy against the SBVP-induced conditional probability density function, the learned SNN is continuously regularized toward physics-consistent predictions.
Model selection is therefore carried out on the training set, while generalization is assessed exclusively on the test dataset (Section~\ref{Section3}).
\paragraph{(iii) Population target and gradient identity for the alignment term}

Let $p^{\Obsp}_{\bfY\vert\XX}(\cdot \,\vert\, \xx)$ denote the (unknown) conditional probability density function of the experimental measurements generating $\curD_\trainingp$, and recall the SBVP and SNN conditionals $p^{\SBVPpp}_{\bfY\vert\XX}(\cdot \,\vert\, \xx)$ and $p^{\SNNpp}_{\YY\vert\XX}(\cdot \,\vert\, \xx\, ;\bftheta)$.
Consider the population version of Eq.~\eqref{eq2.12}, defined up to an additive constant independent of $\bftheta$, as
\begin{equation*}
\curJ^{\SPINNpp}(\bftheta\, ;  \alpha_{\rm{top}})
= E_{\XX}\bigg \{ \,\rm{CE}\,\Big ( p^{\Obsp}_{\bfY\vert\XX}(\cdot \,\vert\, \XX)\, ,\,p^{\SNNpp}_{\YY\vert\XX}(\cdot \,\vert\, \XX\, ;\bftheta)\Big )
+ \alpha_{\rm top}\,\rm{CE}\,\Big ( p^{\SBVPpp}_{\bfY\vert\XX}(\cdot \,\vert\, \XX)\, ,\, p^{\SNNpp}_{\YY\vert\XX}(\cdot \,\vert\, \XX\, ;\bftheta)\Big )\,\bigg \} \, ,
\end{equation*}
where $\rm{CE}(q,p)=-\int_{\RR^{n_\ppout}} q(\bfy)\, \log p(\bfy) \, d\bfy$ as in Eq.~\eqref{eq2.13}.

\begin{proposition}[Population target as information projection onto a convex mixture] \label{Proposition1}
Fix $\alpha_{\rm top}\ge 0$ and set $\lambda=\alpha_{\rm top}/(1+\alpha_{\rm top})\in[0,1[$.
For each $\xx$, define the convex mixture
\begin{equation*}
q_{\lambda}(\cdot \,\vert\, \xx) = (1-\lambda)\,p^{\Obsp}_{\bfY\vert\XX}(\cdot \,\vert\, \xx)
+ \lambda \, p^{\SBVPpp}_{\bfY\vert\XX}(\cdot \,\vert\, \xx)\, .
\end{equation*}
Then
\begin{equation}\label{eq2.20}
\curJ^{\SPINNpp}(\bftheta\, ;  \alpha_{\rm{top}}) = (1+\alpha_{\rm top})\,
E_{\XX}\bigg \{ \,\rm{CE}\, \Big ( q_{\lambda}(\cdot \,\vert\, \XX)\,,\,p^{\SNNpp}_{\YY\vert\XX}(\cdot \,\vert\, \XX\, ;\bftheta)\Big) \bigg \} \, .
\end{equation}
Consequently,
\begin{equation*}
\bftheta^\opt( \alpha_{\rm{top}}) \in \arg\min_{\bftheta}\, E_{\XX}\bigg \{ \rm{KL}\, \Big ( q_{\lambda}(\cdot \,\vert\, \XX)\,\, \Vert \,\, p^{\SNNpp}_{\YY\vert\XX}(\cdot \,\vert\, \XX\, ;\bftheta)\Big )\bigg \} \, ,
\end{equation*}
and if the SNN family is rich enough that $p^{\SNNpp}_{\YY\vert\XX}(\cdot \,\vert\, \xx\,;\bftheta)$ can realize $q_{\lambda}(\cdot \,\vert\, \xx)$ for a.e. $\xx$, then every global minimizer satisfies,
\begin{equation*}
p^{\SNNpp}_{\YY\vert\XX}(\cdot \,\vert\, \xx\, ;\bftheta^\opt_t) = q_{\lambda}(\cdot \,\vert\, \xx)\quad \text{for a.e.}\,  \xx\, .
\end{equation*}
In the limits $\alpha_{\rm top} = 0$ and $\alpha_{\rm top}\to\infty$, the target reduces to $p^{\Obsp}_{\bfY\vert\XX}$ and $p^{\SBVPpp}_{\bfY\vert\XX}$, respectively.
\end{proposition}

\begin{proof}
By linearity of $\rm{CE}$ in its first argument and the definition of $\lambda$, we get
\begin{equation*}
\rm{CE}\, \Big ( p^{\Obsp}_{\bfY\vert\XX}\, , p^{\SNNpp}_{\YY\vert\XX}\Big ) + \alpha_{\rm top}\rm{CE}\, \Big ( p^{\SBVPpp}_{\bfY\vert\XX}\, , p^{\SNNpp}_{\YY\vert\XX}\Big )
=(1+\alpha_{\rm top})\,\rm{CE}\, \Big( (1-\lambda)\, p^{\Obsp}_{\bfY\vert\XX} + \lambda \, p^{\SBVPpp}_{\bfY\vert\XX}\, , \, p^{\SNNpp}_{\YY\vert\XX}\Big ),
\end{equation*}
pointwise in $\xx$. Taking $E_\XX$ yields \eqref{eq2.20}.
Since $\rm{CE}(q,p)=\rm{S}(q)+\rm{KL}(q\, \Vert\,  p)$ with $\rm{S}(q)$ independent of $\bftheta$, minimizing $\curJ^{\SPINNpp}$ is equivalent to minimizing the conditional KL. The last statement follows from Gibbs’ inequality.
\end{proof}

\begin{proposition}[Score-form gradient of the alignment term]\label{Proposition2}
Assume for each $\xx$ that $\bfy\mapsto p^{\SNNpp}_{\YY\vert\XX}(\bfy \,\vert\, \xx\, ;\bftheta) > 0$ for all $\bfy$ in the support of
$p^{\SBVPpp}_{\bfY\vert\XX}(\cdot \vert\,  \xx)$, and that there exists
$g_{\xx}\in L^1\big(p^{\SBVPpp}_{\bfY\vert\XX}(\cdot \, \vert\, \xx)\big )$ with
$\Vert \nabla_{\bftheta}\log p^{\SNNpp}_{\YY\vert\XX}(\bfy \, \vert \, \xx\, ;\bftheta)\Vert \le g_{\xx}(\bfy)$
in a neighborhood of $\bftheta$, so that differentiation under the integral sign is justified.
Then for every $\xx^i$,
\begin{equation} \label{eq2.21}
\nabla_{\bftheta}\, \rm{CE}_i(\bftheta)
= -\,E_{\bfY\vert \XX=\xx^i} \bigg\{\,\nabla_{\bftheta} \log p^{\SNNpp}_{\YY\vert\XX}(\bfY \,\vert\, \xx^i\, ;\bftheta)\,
     \bigg \}.
\end{equation}
In particular, if for a given $\bftheta$, we have $p^{\SNNpp}_{\YY\vert\XX}(\cdot \,\vert\, \xx^i\, ;\bftheta) =
 p^{\SBVPpp}_{\bfY\vert\XX}(\cdot \,\vert\, \xx^i)$, then $\nabla_{\bftheta}\rm{CE}_i(\bftheta)=\bfzero$.
\end{proposition}

\begin{proof}
By definition, $\rm{CE}_i(\bftheta)=-\int p^{\SBVPpp}_{\bfY\vert\XX}(\bfy \,\vert\, \xx^i)\,
\log p^{\SNNpp}_{\YY\vert\XX}(\bfy \,\vert\, \xx^i\, ;\bftheta)\, d\bfy$. Interchange differentiation and integration under the stated regularity and identify the expectation under $p^{\SBVPpp}_{\bfY\vert\XX}(\cdot \,\vert\, \xx^i)$.
\end{proof}

\noindent\textit{Remarks}

\noindent {\textit{ (a) I-projection interpretation of Proposition~\ref{Proposition1}}}. This proposition  formalizes the interpretation already hinted by Eq.~\eqref{eq2.12}, when the data term is also written in cross-entropy form: SPINN performs the conditional I-projection of the mixture  $q_\lambda(\cdot \,\vert\,\xx)$ onto the parameterized SNN family,
that is, it selects for each $\xx$ the element of the SNN conditionals minimizing
$\KLg\,\Big (q_\lambda(\cdot \,\vert\, \xx)\,\, \Vert \,\, p^{\SNNpp}_{\YY \vert \XX}(\cdot\,\vert\,\xx\, ;\bftheta)\Big )$,
and then averages over the input distribution.
This makes explicit the role of $\alpha_{\rm top}$ as a parameter controlling the mixture between the observation-induced conditional probability distribution
($\lambda\downarrow 0$), which allows discrepancy between the SBVP model and the data-generating probability distribution,
and the SBVP-induced conditional probability distribution ($\lambda\uparrow 1$).\\

\noindent {\textit{ (b) Gradient-estimation interpretation of  Proposition~\ref{Proposition2}}}. This proposition justifies estimating $\nabla_{\bftheta}\, \rm{CE}_i$ by Monte Carlo from SBVP-generated observables alone, without SBVP sensitivities. It therefore does not require differentiating the SBVP solution mapping. In practice we use the same Gaussian random vector $\bfH$ to generate both $U_c$ and $U(\cdot;\bftheta)$ when forming the Monte Carlo sum in \eqref{eq2.15}; this “common-random-numbers” coupling empirically lowers the variance of the stochastic gradient of the alignment term.\\

\noindent {\textit{ (c) Interpretation of values of $\alpha_0$ close to one}}.  In the idealized regime where the training dataset itself is generated by numerically solving the SBVP (Case~1), the empirical distribution of the observations converges to the same conditional probability distribution as the SBVP-based reference. In this case, the SPINN objective with the alignment term and the pure data-likelihood objective corresponding to Eqs.~\eqref{eq2.7}-\eqref{eq2.8} have the same population minimizers for the SNN conditional distributions. The additional alignment contribution then acts mainly as a variance-reduction and regularization device at finite sample sizes: it stabilizes the Monte Carlo estimators entering the loss but does not inject extra physical information beyond what is already present in the SBVP-generated training data. Consistently, the scale-matching factor $\alpha_0$ introduced in paragraph~(ii) is expected to be close to one in this regime.\\

{\color{black}
\noindent {\textit{ (d) Operational role of $\alpha_0$ under model--data mismatch}}. In the more realistic configuration of Section~\ref{Section2.1}-(iii), the training dataset is built from experimental measurements or high-fidelity numerical simulations that only approximately satisfy the SBVP model (Case~2). In this setting, the alignment term remains informative because it introduces the SBVP-induced conditional probability distribution as a distinct source of probabilistic information while allowing the observations to contribute through the data-fidelity term. The role of $\alpha_0$ is solely to provide the scale-matching prescription of Eq.~\eqref{eq2.19} in the fixed normalized observation representation used throughout the calculations. The probabilistic meaning of a nonzero alignment weight under model--data mismatch is given by the mixture $q_{\lambda}$ and Proposition~\ref{Proposition1}, not by interpreting $\alpha_0$ as an invariant discrepancy between probability distributions.
}
\subsection{Method used for estimating the conditional probability density function of the random output vector $\YY(\xx\, ;\bftheta)$ of the SNN}
\label{Section2.3}
The theory presented in Section~\ref{Section4} addresses supervised learning of random neural architectures structured by latent random fields on compact, boundaryless, multiply connected manifolds.
For a fixed hyperparameter $\bftheta$, the SNN simulator involves two distinct layers of internal randomness.
The centered latent Gaussian field is defined as $U(\cdot\, ;\bftheta) = [M(\cdot\, ;\bftheta)]\,\bfH$,
where $\bfH$ is an $\RR^m$-valued centered Gaussian random vector with identity covariance matrix.
The covariance structure of $U(\cdot\, ;\bftheta)$ therefore depends only on the deterministic field
$[M(\cdot\, ;\bftheta)]$. The Monte Carlo methodology, used for estimating the conditional probability density function
under a fixed architecture realization for a given $\bftheta$, is given in Algorithm~\ref{algorithm1}.
\begin{algorithm}[!t]
\caption{Monte Carlo methodology for estimation of the conditional probability density function under a fixed architecture realization for a given $\bftheta$}
\label{algorithm1}
\begin{algorithmic}[1]
\STATE \textbf{Input:} hyperparameter $\bftheta$, input $\xx$, number of simulations $n_\psim$.
\medskip
\STATE {\textit{Step 1 (Architecture realization for fixed $\bftheta$).}}
\STATE Draw one realization $\bfeta^0$ of $\bfH$.
\STATE Construct the latent field $u^{(0)}(\cdot\, ;\bftheta) = [M(\cdot\, ;\bftheta)]\,\bfeta^0$.
\STATE Construct the random architecture (see Section~\ref{Section4.2}):
\STATE $\quad\triangleright$ draw one realization of all Poisson-process random quantities,
\STATE $\quad\triangleright$ Poisson-sampled neuron locations $\{\bfx_\bftheta^1,\ldots,\bfx_\bftheta^N \}$,
\STATE $\quad\triangleright$ index sets $(\curJ^\pin_\bftheta ,\curJ^\pout_\bftheta)$ defining input and output components,
\STATE $\quad\triangleright$ geodesic-distance-based sparsity mask (Eqs.~\eqref{eq4.14}-\eqref{eq4.19}).
\medskip
\STATE {\textit{Step 2 (Random weight generation).}}
\FOR{$\ell = 1,\dots,n_\psim$}
    \STATE Draw an independent realization $\bfeta^\ell$ of $\bfH$.
    \STATE Construct $u^\ell(\cdot\, ;\bftheta) = [M(\cdot\, ;\bftheta)]\,\bfeta^\ell$.
    \STATE Generate the corresponding sparse random weight matrix under the fixed architecture (Eqs.~\eqref{eq4.26}-\eqref{eq4.27}).
\ENDFOR

\medskip
\STATE {\textit{Step 3 (Compute $\bfbeta(\bftheta)$ for the low-rank representation of the bias vector).}}
\STATE For the prescribed input $\xx$, compute $\bfbeta(\bftheta)$ as the least-squares solution of the
random nonlinear neural equation (Section~\ref{Section4.6}).

\medskip
\STATE {\textit{Step 4 (Free activated states and outputs).}}
\FOR{$\ell = 1,\dots,n_\psim$}
    \STATE Solve the random nonlinear neural equation using the fixed-point iteration of Section~\ref{Section4.8} with sparse weight matrix associated with $u^\ell$ to obtain the free activated state.
    \STATE Compute the corresponding realization $\yy^\ell(\xx\,;\bftheta)$ of $\YY(\xx\,;\bftheta)$.
\ENDFOR

\medskip
\STATE \textbf{Output:}
\STATE The set of independent realizations $\{\yy^\ell(\xx\,;\bftheta)\}_{\ell=1}^{n_\psim}$,
used to estimate the conditional probability density function $p^\SNNpp_{\YY\vert\XX}(\cdot \, \vert \,  \xx \, ;\bftheta)$
(see Eq.~\eqref{eq4.45}).
\end{algorithmic}
\end{algorithm}
\section{Methodology for testing and intrinsic diagnostic criteria}
\label{Section3}
In order to evaluate the proper-score performance of the coupled construction, we first present the general methodology for testing and intrinsic diagnostic evaluation. Since the stochastic neural network with random neuronal architecture must be assessed together with the SBVP alignment, we use criteria adapted from~\cite{Soize2026}. For completeness and to keep the exposition self-contained, these definitions are reproduced below with notations adapted to the present framework.
All KDE bandwidths are set using the Silverman rule on the training dataset and then held fixed for evaluation on the test dataset.

Throughout this section, intrinsic diagnostic criteria refer to post-training checks that stem from the alignment-based training objective (Eq.~\eqref{eq2.12}) and are computed using the test dataset. We do not introduce a separate diagnostic dataset. Metrics such as $\criter_o$, confidence-interval discrepancies, and CRPS are evaluation criteria, not a separate model-selection phase.
\subsection{Estimation of $\NLL_\SNNp$ on the test dataset}
\label{Section3.1}
We denote by $\NLL_\SNNp$ the negative log-likelihood associated with the SNN.
Let $\{ ( \xx^i_\testp , \yy^i_\testp) , i=1,\ldots,n_\testp\}$  be the test dataset.
For the optimal value $\bftheta^\optp$ computed with Eq.~\eqref{eq2.11},  the $\NLL_\SNNp$ evaluated on the test dataset is
\begin{equation} \label{eq3.1}
\NLL_\SNNp^\test = - \sum_{i=1}^{n_\testpp} \log \, p^\SNNpp_{\YY\vert\XX}(\yy^i_\testp \, \vert
                                              \, \xx^i_\testp \, ; \bftheta^\optp) \, ,
\end{equation}
where $p^\SNNpp_{\YY\vert\XX}$ is the conditional probability density function of $\YY(\xx\, ;\bftheta)$ defined in Eq.~\eqref{eq2.4}.
\subsection{Criterion $\criter_o$ for assessing over-specialization}
\label{Section3.2}
In classical settings, overfitting arises when a model fits noise and fails to generalize. Here, however, randomness is a structural property induced by the stochastic generation of network configurations. Thus, overfitting should be read as the selection of hyperparameters that favor over-specialized, low-probability architectures rather than noise-fitting.
The optimal training negative log-likelihood
\begin{equation}\label{eq3.2}
\NLL_\SNNp^\opt = - \sum_{i=1}^{n_d} \log \, p^\SNNpp_{\YY\vert\XX}(\yy^i \, \vert \, \xx^i \, ; \bftheta^\optp) \, ,
\end{equation}
computed on the $n_d$ training samples using the negative log-likelihood definition in Eq.~\eqref{eq2.8}, evaluated at
$\bftheta^\optp$, and the test negative log-likelihood $\NLL_\SNNp^\test$, computed on $n_\testp$ samples by Eq.~\eqref{eq3.1}, together diagnose generalization.
Because the stochastic architecture governed by $\bftheta$ can approximate highly complex functions, appreciable train–test gaps may occur. We therefore quantify this form of over-specialization by
\begin{equation} \label{eq3.3}
\criter_o               = \frac{\vert \overline L_\SNNp^\opt - \overline L_\SNNp^\test\vert}{ \overline L_\SNNp^\test } \quad , \quad
\overline L_\SNNp^\opt  = \frac{1}{n_d} \NLL_\SNNp^\opt \quad , \quad
\overline L_\SNNp^\test = \frac{1}{n_\testp} \NLL_\SNNp^\test \, .
\end{equation}
This criterion measures the relative discrepancy between the (size-adjusted) training and test NLLs. Under i.i.d. sampling and a sufficiently large test set (e.g., $n_\testp \approx n_d/5$), $\criter_o \ll 1$ indicates a small train--test discrepancy in the size-normalized NLL, whereas large values signal hyperparameter-induced over-specialization to the training data.
For this class of stochastic, non-feedforward architectures, the typical scale of
train-test discrepancies differs from classical feedforward settings. As a rough
guide in our regime:

 {\tiny$\bullet$}\hspace{0.2em} $\criter_o\lesssim 0.10$: no notable over-specialization.

{\tiny$\bullet$}\hspace{0.2em} $\criter_o\in[0.10,\,0.20]$: intermediate train--test discrepancy requiring inspection of the associated predictive conditional probability distributions.

 {\tiny$\bullet$}\hspace{0.2em} $\criter_o > 0.20 $: likely over-specialization.
\subsection{Criteria  for testing accuracy via confidence intervals}
\label{Section3.3}
Accuracy is assessed by comparing confidence intervals obtained from the SNN with those estimated directly from the training dataset, for a prescribed probability level $p_c$. Let $\YY(\xx\, ; \bftheta^\optp)$ be the $\RR^{n_\ppout}$-valued random output of the SNN for input $\xx\in\RR^{n_\ppin}$. For each $k\in\{1,\ldots,n_\ppout\}$, the $p_c$-confidence intervals
\begin{equation}\label{eq3.4}
\CI^\SNNpp_k(\xx) = [\alpha^\SNNpp_k(\xx) \, , \beta^\SNNpp_k(\xx)]
\quad , \quad
\CI^\trainingpp_k(\xx) = [\alpha^\trainingpp_k(\xx) \, , \beta^{\,\trainingpp}_k(\xx)] \, ,
\end{equation}
are constructed respectively from $n_\psim$ independent realizations of $\YY(\xx\,;\bftheta^\optp)$ and from the empirical conditional distribution based on the training dataset.
Define the vectors of lower/upper bounds across outputs by
\begin{equation*}
\bfalpha^\SNNpp(\xx) = (\alpha_1^\SNNpp(\xx),\ldots,\alpha_{n_\ppout}^\SNNpp(\xx)) \quad , \quad
\bfbeta^\SNNpp(\xx) = (\beta_1^\SNNpp(\xx),\ldots,\beta_{n_\ppout}^\SNNpp(\xx)) \, ,
\end{equation*}
and similarly $\bfalpha^\trainingpp(\xx)$ and $\bfbeta^\trainingpp(\xx)$.
The relative discrepancies on the lower and upper bounds are then quantified by
\begin{equation}\label{eq3.5}
\criter_{\bfalpha} = \frac{1}{n_\testp} \sum_{i=1}^{n_\testpp}
2\,\frac{\Vert \bfalpha^\SNNpp(\xx^i_\testp) - \bfalpha^\trainingpp(\xx^i_\testp)\Vert}
{\Vert \bfalpha^\SNNpp(\xx^i_\testp) + \bfalpha^\trainingpp(\xx^i_\testp)\Vert}
\quad , \quad
\criter_{\bfbeta} = \frac{1}{n_\testp} \sum_{i=1}^{n_\testpp}
2\,\frac{\Vert \bfbeta^{\,\SNNpp}(\xx^i_\testp) - \bfbeta^{\,\trainingpp}(\xx^i_\testp)\Vert}
{\Vert \bfbeta^{\,\SNNpp}(\xx^i_\testp) + \bfbeta^{\,\trainingpp}(\xx^i_\testp)\Vert} \, ,
\end{equation}
where $\bfalpha^\SNNpp(\xx),\bfbeta^{\,\SNNpp}(\xx)$ (resp. $\bfalpha^\trainingpp(\xx),\bfbeta^{\,\trainingpp}(\xx)$) denote the lower and upper bounds of the SNN (resp. data-based) confidence intervals.\\

\noindent\textit{Remarks}.
The criteria in Eq.~\eqref{eq3.5} compare confidence intervals constructed
with equal tails (symmetric in probability), even though the regimes considered here are non-Gaussian.
Allowing a mild tail asymmetry can reduce discrepancies between $\criter_{\bfalpha}$ and $\criter_{\bfbeta}$.
One may use quantiles $(\tau_-,\,1-\tau_+)$ with $\tau_-+\tau_+=1-p_c$ and select $(\tau_-,\tau_+)$
so as to minimize the weighted sum $w_\alpha\,\criter_{\bfalpha}+w_\beta\,\criter_{\bfbeta}$.
This introduces the additional weights $w_\alpha$ and $w_\beta$ and should be viewed as an optional
calibration refinement, particularly when $\criter_{\bfbeta}\gg \criter_{\bfalpha}$.\\

\noindent\textit{Additional criterion (interval overlap)}.
Beyond the relative discrepancies in Eq.~\eqref{eq3.5}, a geometric agreement score based on the interval overlap coefficient (IOC) can be defined.
For each test input $i\in\{1,\ldots,n_\testp\}$ and output component $k\in\{1,\ldots,n_\pout\}$, we define the intersection and union lengths
\begin{align*}
L^{\cap}_{ik} & = \Big(\min\{\,\beta^{\,\SNNpp}_k(\xx^i_\testp),\,\beta^{\,\trainingpp}_k(\xx^i_\testp)\}
-\max\{\alpha^\SNNpp_k(\xx^i_\testp),\,\alpha^\trainingpp_k(\xx^i_\testp)\}\Big)_+  \, , \\
L^{\cup}_{ik} & =\max\{\,\beta^{\,\SNNpp}_k(\xx^i_\testp),\,\beta^{\,\trainingpp}_k(\xx^i_\testp)\}
-\min\{\alpha^\SNNpp_k(\xx^i_\testp),\,\alpha^\trainingpp_k(\xx^i_\testp)\} \, ,
\end{align*}
where $(u)_+=\max(u,0)$.
The (per-component) $\IOC$ is then defined as
\begin{equation*}
\IOC_{ik}=\frac{L^{\cap}_{ik}}{L^{\cup}_{ik}+\varepsilon_F},
\end{equation*}
with $\varepsilon_F = 2.2204\times 10^{-16}$ included only to avoid division by zero in the presence of degenerate intervals.
Aggregating over test points and output components yields the scalar criterion
\begin{equation}\label{eq3.6}
\criter_{\IOCp} = 1-\frac{1}{n_\testp\,n_\pout}\sum_{i=1}^{n_\testpp}\sum_{k=1}^{n_\ppout}\IOC_{ik}\in [0,1] \, .
\end{equation}
By construction, $\criter_{\IOCp} = 0$ if and only if all SNN and training data-based intervals coincide,
and $\criter_{\IOCp}\to 1$ as the intervals become disjoint.
This criterion is scale-invariant and is simultaneously sensitive to center shifts and to under/over-dispersion of the widths, thereby complementing the symmetric, tail-wise discrepancies in Eq.~\eqref{eq3.5}.\\

\noindent\textit{Indicative overlap scale}.
For brevity, let $\overline{\IOC} = 1-\criter_{\IOCp}\in[0,1]$ denote the average overlap.
The following interpretive categories are used as non-asymptotic diagnostic ranges:

{\tiny$\bullet$}\hspace{0.2em} Excellent alignment:  $\criter_{\IOCp}\le 0.15\quad(\overline{\IOC}\ge 0.85)$.

{\tiny$\bullet$}\hspace{0.2em} Good alignment:  $0.15<\criter_{\IOCp}\le 0.35\quad(0.65\le \overline{\IOC}<0.85)$.

{\tiny$\bullet$}\hspace{0.2em} Moderate alignment:  $0.35<\criter_{\IOCp}\le 0.55\quad(0.45\le \overline{\IOC}<0.65)$.

{\tiny$\bullet$}\hspace{0.2em} Weak alignment:  $\criter_{\IOCp}>0.55\quad(\overline{\IOC}<0.45)$.

\noindent When combined with the tail-wise discrepancies in Eq.~\eqref{eq3.5}, these categories help diagnose whether disagreement is driven primarily by center shifts (suggested by comparable $\criter_{\bfalpha}$ and $\criter_{\bfbeta}$) or by asymmetric tail calibration (e.g., $\criter_{\bfalpha} \ll \criter_{\bfbeta}$).
As the regimes considered are non-Gaussian while the intervals use equal tails, larger values of $\criter_{\bfbeta}$ relative to $\criter_{\bfalpha}$ often indicate upper-tail under- or over-dispersion, in which case mild tail asymmetry (cf. the remark following Eq.~\eqref{eq3.5}) typically improves $\overline{\IOC}$.
\subsection{Continuous Ranked Probability Score (CRPS)}
\label{Section3.4}
The CRPS \cite{Matheson1976,Hersbach2000,Gneiting2007} provides a complementary criterion by comparing the full predictive conditional probability distribution with observed outcomes over the test dataset. For input $\xx$ belonging to the test dataset and component $k$, the CRPS predicted by the SNN is denoted as $\CRPS^\SNNpp_k(\xx)$ and the CRPS estimated  with the training dataset is denoted as $\CRPS^\trainingpp_k(\xx)$,
\begin{equation}\label{eq3.7}
\CRPS^\SNNpp_k(\xx) = \int_{-\infty}^{+\infty} \! \left( F_k^\SNNpp(z \, \vert \, \xx) - \Unit_{]-\infty, \yy_k^\Obspp(\xx)]}(z) \right)^2 dz
\,\,\, , \,\,\,
\CRPS^\trainingpp_k(\xx) = \int_{-\infty}^{+\infty} \!\left( F_k^\trainingpp(z \, \vert \, \xx) - \Unit_{]-\infty, \yy_k^\Obspp(\xx)]}(z) \right)^2 dz
\, ,
\end{equation}
where $F_k^\SNNpp(\cdot \, \vert \, \xx)$ denotes  the conditional cumulative distribution function of $\YY_k(\xx\, ;\bftheta^\optp)$,
$F_k^\trainingpp(\cdot \, \vert \, \xx)$ denotes the CDF estimated with the training dataset, and $\yy_k^\Obsp(\xx)$ is the observed value corresponding to input $\xx$ such that
$\yy^i = \yy^\Obsp(\xx^i)$ with $(\xx^i,\yy^i)\in\curD_\trainingp$ and
$\yy^i_\testp = \yy^\Obsp(\xx^i_\testp)$ with $(\xx^i_\testp,\yy^i_\testp)\in\curD_\testp$.
For a given $\xx$, $\CRPS^\SNNpp_k(\xx)$ is predicted using a Monte Carlo approximation based on $n_\psim$ realizations (see, e.g., \cite{Gneiting2007}), which  yields
\begin{equation}\label{eq3.8}
\CRPS^\SNNpp_k(\xx) \approx \frac{1}{n_\psim} \sum_{\ell=1}^{n_\psim} \left| \yy_k^\ell(\xx) - \yy_k^\Obsp(\xx) \right|
- \frac{1}{ n_\psim (n_\psim -1)} \sum_{\ell < \ell'}  \vert \yy_k^\ell(\xx) - \yy_k^{\ell'}(\xx) \vert \, ,
\end{equation}
and $\CRPS^\trainingpp_k(\xx)$ is estimated with the conditional samples of the training dataset using the k-nearest-neighbor estimation of the conditional probability density function \cite{Mack1979}.
The global test scores are then
\begin{equation}\label{eq3.9}
\CRPS^\SNNpp_\testp = \frac{1}{n_\testp n_\ppout} \sum_{i=1}^{n_\testpp} \sum_{k=1}^{n_\ppout} \CRPS^\SNNpp_k(\xx^i_\testp)
\quad , \quad
\CRPS^\trainingpp_\testp = \frac{1}{n_\testp n_\ppout} \sum_{i=1}^{n_\testpp} \sum_{k=1}^{n_\ppout} \CRPS^\trainingpp_k(\xx^i_\testp) \, .
\end{equation}
In the present finite-parameter regime, one typically observes $\CRPS^\SNNpp_\testp \gtrsim \CRPS^\trainingpp_\testp$, although this ordering is not guaranteed in general. To evaluate model quality, we introduce the scaled CRPS, defined as
\begin{equation}\label{eq3.10}
\CRPS^\scaledpp_\testp = \CRPS^\SNNpp_\testp / \CRPS^\trainingpp_\testp \, .
\end{equation}
The quality of the probabilistic model evaluated with the CRPS criterion is analyzed in detail in Sections~\ref{Section3.5}
and~\ref{Section7}.
\subsection{Interpretation of the evaluation criteria when the training and test datasets do not almost surely satisfy the SBVP}
\label{Section3.5}
Among all diagnostic criteria introduced above, the CRPS-based score in Eq.~\eqref{eq3.7} is the most informative marginal criterion and should be regarded as the primary summary of probabilistic performance of the conditional probability prediction at the component level, since it is a proper
scoring rule that jointly reflects calibration and sharpness of each univariate predictive conditional probability distribution.
It should be interpreted together with $\NLL_\SNNp^\test$, which evaluates the joint conditional probability density function.\\

\noindent Although the focus of this subsection is the misspecified regime (Case~2),
we first recall the SBVP-consistent regime (Case~1) as a reference point for interpreting the scaled scores.
The quantity $\CRPS^\scaledpp_\testp$ in Eq.~\eqref{eq3.10} is a relative proper-score diagnostic:
it compares the SNN predictive conditional probability distributions (numerator) against the nonparametric, data-based conditional benchmark
constructed from the training dataset (denominator). Hence $\CRPS^\scaledpp_\testp = 1 + \varepsilon_\CRPSpp$ means that,
in aggregate over test inputs and output components, the SNN incurs about a $(100\,\varepsilon_\CRPSpp)\%$ larger CRPS than that
finite-sample benchmark. This ratio should be interpreted as a gap to a data-driven reference (in the scaled domain),
not as a pointwise "percent error" on the physical fields.
\paragraph{(i) Case 1 (SBVP-consistent datasets)}
When the training and test datasets almost surely satisfy the SBVP (Case~1 in Section~\ref{Section6.2}), the data-based conditional
benchmark $\CRPS^\trainingpp_\testp$ targets (up to finite-sample estimation error) the same conditional probability distribution as the SBVP.
In this idealized regime, the scale-matching factor $\alpha_0$ defined in Eq.~\eqref{eq2.19} is expected to be close to one, and
$\CRPS^\scaledpp_\testp$ can be read directly as the (dimensionless) relative deterioration of the parametric SNN surrogate
with respect to the flexible conditional estimator built from the training sample.
\paragraph{(ii) Case 2 (model-data mismatch SBVP)}
When the training and test datasets do not almost surely satisfy the SBVP (Case~2 in Section~\ref{Section6.2}), the SBVP acts as an imperfect SBVP-induced conditional probability distribution and the SPINN objective explicitly targets a compromise between data fidelity and SBVP alignment, rather than the SBVP-induced conditional probability distribution itself.
In this regime, $\CRPS^\scaledpp_\testp$ still has the same meaning as above: it quantifies how close the trained SNN
predictive conditional probability distributions are to the data-based conditional benchmark derived from the training dataset. Importantly,
$\CRPS^\scaledpp_\testp$ does not test whether the observations satisfy the SBVP-induced conditional probability distribution.
{\color{black}
In Case~2, the model--data mismatch is represented explicitly by the distinction between $p^{\Obsp}_{\bfY\vert\XX}$ and $p^{\SBVPpp}_{\bfY\vert\XX}$ in the population formulation. The factor $\alpha_0$ in Eq.~\eqref{eq2.19} is used only to scale the two empirical objective contributions in the adopted normalized coordinates, while predictive assessment is provided by the test NLL, CRPS, generalization, and calibration criteria defined above.
}
\paragraph{(iii) Practical role of $\alpha_0$ under SBVP mismatch}
{\color{black}
\noindent The factor $\alpha_0$ is the scale-matching quantity defined in Eq.~\eqref{eq2.19} at the reference point $\bftheta_\pref$ and in the normalized observation coordinates adopted in the methodology. It is used to avoid an additional outer tuning loop for $\alpha_{\rm top}$ and does not need to be interpreted as an intrinsic statistical discrepancy. In Case~1, the observation and SBVP contributions have the same population target, whereas in Case~2 a nonzero $\alpha_{\rm top}$ defines the explicit mixture target $q_{\lambda}$ in Eq.~\eqref{eq2.17}. Thus, the role of $\alpha_0$ is to set the numerical scale of this prescribed compromise, while the quality of the resulting predictive conditional probability distributions is assessed by the test criteria of Sections~\ref{Section3.1}--\ref{Section3.4}.
}
\subsection{Practical use of the intrinsic criteria in Case 2 (SBVP mismatch)}
\label{Section3.6}
In Case~2, the SBVP is an imperfect SBVP-induced conditional probability distribution and SPINN targets a compromise between data fidelity and SBVP alignment.
This compromise is controlled by $\alpha_{\rm{top}}$ through $\lambda=\alpha_{\rm{top}}/(1+\alpha_{\rm{top}})$ (see Eq.~\eqref{eq2.16}) in the mixture
$q_{\lambda}(\cdot \,\vert\, \xx) = (1-\lambda)\,p^{\Obsp}_{\bfY\vert\XX}(\cdot \,\vert\, \xx)
+ \lambda \, p^{\SBVPpp}_{\bfY\vert\XX}(\cdot \,\vert\, \xx)$ (see Eq.~\eqref{eq2.17}).
{\color{black}
The scale-matching value $\alpha_0$ (Eq.~\eqref{eq2.19}) is used as an operational default for setting $\alpha_{\rm{top}}$ without an outer model-selection loop. Its role is to balance the two objective contributions in the fixed normalized representation; the probabilistic content of the compromise is specified by $q_{\lambda}$.
}

In this misspecified regime, no single scalar criterion can certify surrogate validity. For validating such a compromise surrogate,
$\CRPS^\scaledpp_\testp$ should be reported together with the other intrinsic criteria of Sections~\ref{Section3.1} to \ref{Section3.4}:
(i) $\NLL_\SNNp^\test$ (strictly proper score for the joint conditional probability density function),
(ii) the generalization/over-specialization index $\criter_o$, and (iii) the confidence-interval agreement metrics
$\criter_{\bfalpha}$, $\criter_{\bfbeta}$, and $\criter_{\IOCp}$.
A practically useful reading is that any increase of $\CRPS^\scaledpp_\testp$ above $1$ quantifies the price paid in CRPS-based probability prediction
(relative to the data-based benchmark) for maintaining SBVP alignment, whereas the accompanying scores diagnose whether this price remains
acceptable in terms of out-of-sample likelihood, generalization, and calibration.
Accordingly, the intrinsic diagnostics of Sections~\ref{Section3.1} to \ref{Section3.4} can be used as a structured decision workflow:

\begin{enumerate}
\item \textit{Generalization criterion}: $\criter_o$ (Eq.~\eqref{eq3.3}).
First check that the model does not over-specialize to the training data. Since $\criter_o$ measures the relative
train-test discrepancy in size-adjusted NLL, it directly flags hyperparameter-induced specialization of the stochastic
architecture. As we have already indicated in Section~\ref{Section3.2}, as a practical guideline in this class of stochastic architectures,
$\criter_o\lesssim 0.10$ indicates no notable over-specialization, $\criter_o\in[0.10,0.20]$ requires additional inspection of the predictive conditional probability distributions, and $\criter_o  > 0.20$ is likely over-specialization. Configurations beyond this range should be discarded or treated with caution, regardless of their other scores.

\item \textit{Calibration of probability intervals in observation space}: $(\criter_{\bfalpha},\criter_{\bfbeta},\criter_{\IOCp})$
(Eqs.~\eqref{eq3.5}-\eqref{eq3.6}).
Next assess whether probability dispersion of the SNN conditional probability distribution is compatible with the data-driven conditional benchmark at the chosen
coverage level $p_c$ (Eq.~\eqref{eq3.4}). The tail-wise discrepancies $\criter_{\bfalpha}$ and $\criter_{\bfbeta}$ quantify relative errors
in the lower and upper bounds of equal-tail confidence intervals; comparable values suggest that disagreement is
primarily a center shift, while strong asymmetry (e.g., $\criter_{\bfbeta}\gg \criter_{\bfalpha}$) indicates tail-miscalibration.
The overlap-based score $\criter_{\IOCp}\in[0,1]$ complements these diagnostics by jointly capturing center shifts and
under/over-dispersion of interval widths. The indicative overlap scale in Section~\ref{Section3.3} can be used as a non-asymptotic diagnostic criterion for interval calibration. When tail asymmetry
dominates, the optional refinement discussed after Eq.~\eqref{eq3.5} (allowing mildly asymmetric quantiles) can be used
to reduce $\criter_{\bfbeta}$-$\criter_{\bfalpha}$ discrepancies.

\item \textit{Primary proper-score performance on test data}: $\NLL_\SNNp^\test$ (Eq.~\eqref{eq3.1}).
Among the configurations that satisfy the generalization and calibration requirements, we rank the models by
$\NLL_\SNNp^\test$, which is a strictly proper scoring rule for the joint conditional probability density function and thus provides a
statistically principled measure of out-of-sample agreement with the test distribution. In Case~2 (SBVP mismatch),
this criterion is central because it compares the surrogate predictive conditional probability distribution to the empirical test distribution
without relying on the SBVP as a data-generating model.
For clarity, the quantity $\NLL_\SNNp^\test$ is always evaluated for the final trained parameter $\bftheta^{\optp}$ delivered by the chosen optimization scheme (PGD-Adam with TRIAL step initialization), and it is the only score used to rank proper-score performance.
{\color{black}
No separate data-only retraining is part of this evaluation workflow: in Case~2, $\alpha_{\rm top} = 0$ and $\alpha_{\rm top} > 0$ correspond to different population targets, as established by Proposition~\ref{Proposition1}. The reported criteria assess the SPINN model defined by its prescribed aligned objective.
}

\item \textit{Marginal accuracy of estimated probability distributions relative to the training-based benchmark}: $\CRPS^\scaledpp_\testp$ (Eq.~\eqref{eq3.10}).
Finally use $\CRPS^\scaledpp_\testp$ as a complementary, dimensionless summary of marginal proper-score skill
(relative to the nonparametric conditional benchmark built from the training dataset). Since the CRPS criteria act
componentwise on univariate marginals and are aggregated over outputs and test inputs (Section~\ref{Section3.4} and the remark
on CRPS versus joint likelihood), $\CRPS^\scaledpp_\testp$ should be interpreted as a \emph{gap to the flexible
data-based benchmark} in a finite-sample, high-dimensional setting, rather than as an absolute percent error.
In Case~2, departures of $\CRPS^\scaledpp_\testp$ above $1$ quantify the increase in marginal CRPS associated with imposing a nonzero SBVP-alignment weight at the chosen $\alpha_{\rm{top}}$ (typically $\alpha_{\rm{top}}=\alpha_0$),
and should be interpreted jointly with the NLL and interval-calibration diagnostics above.
\end{enumerate}

\noindent Taken together, this workflow separates (1) generalization control ($\criter_o$), (2) calibration of probability intervals
$(\criter_{\bfalpha},\criter_{\bfbeta},\criter_{\IOCp})$, and (3) proper-score performance ($\NLL_\SNNp^\test$ and
$\CRPS^\scaledpp_\testp$). This separation is essential in Case~2 because SBVP alignment is intentionally used
as a prior/regularizer under mismatch rather than as a ground-truth generator of the data.
\section{Implementation details for random neural architectures on compact, boundaryless, multiply connected manifolds}
\label{Section4}
{\color{black}
For self-containment, this section restates the components of the manifold-based stochastic neural architecture that are required for the present SPINN formulation. The generic architecture generator, based on an anisotropic latent Gaussian field and a Poisson point process on a compact, boundaryless, multiply connected manifold $\curS$, together with the associated nonlinear neural equation, was introduced in~\cite{Soize2026}. Here, this construction is adapted to the SBVP-coupled setting by specifying the parameter values used in Sections~\ref{Section5} and \ref{Section6}, relaxing constraints that define some of the admissible sets of the optimization problems, and introducing the normalization and numerical safeguards required by the present implementation. All definitions and numerical procedures needed to reproduce the SPINN results are provided below. Reference~\cite{Soize2026} is retained for attribution and for its broader analysis of the generic stochastic architecture, but it is not a prerequisite for implementing the methodology developed in this paper.
}
\subsection{Latent field, weak form, finite-element discretization, and reduced representation}
\label{Section4.1}
\paragraph{(i) Definition of the latent random field}
Let $d\sigma(\bfx)$ be the intrinsic measure  on $\curS$.
Let $\{\curB(\bfx),\bfx\in\curS\}$ be a real-valued, centered Gaussian white noise (a generalized random field) on the probability space $(\Omega,\curT,\curP)$. Its covariance $C_\curB$ is a generalized function \cite{Guelfand1964,Kree1986}
satisfying $\langle C_\curB , f\otimes f\rangle = \Vert \, f\,\Vert^2_{L^2(\curS)}$ for all $f\in L^2(\curS)$, where
$\Vert \, f\,\Vert^2_{L^2(\curS)} = \int_\curS f(\bfx)^2\, d\sigma(\bfx)$.
Let $\{U_{n_h}(\bfx),\bfx\in\curS\}$ denote the latent random field, a second-order, centered, Gaussian real-valued random field on $(\Omega,\curT,\curP)$ indexed by $\curS$, defined as the unique solution to the stochastic partial differential equation,
\begin{equation} \label{eq4.1}
(\tau_0 - \langle \,\nabla_\bfx  , [K_{n_h}(\bfx)]\, \nabla_\bfx\, \rangle ) \, U_{n_h}(\bfx)
                               = \curB(\bfx)  \quad , \quad \forall \bfx\in\curS  \, ,
\end{equation}
where $\tau_0 > 0$ is fixed, and $\bfx \mapsto [K_{n_h}(\bfx)]$ is a function from $\curS$ into $\MM_3^{+0}$.
Let the operator $\curL_{n_h} = \tau_0 -\langle \nabla_\bfx , [K_{n_h}(\bfx)] \, \nabla_\bfx \rangle$ be defined,
where $\nabla_\bfx$ denotes the surface gradient acting tangentially to $\curS$. At each point $\bfx \in \curS$,  the matrix $[K_{n_h}(\bfx)]$, when projected onto the tangent space $T_\bfx \curS$, defines a symmetric, positive-definite endomorphism in $\mathrm{End}(T_\bfx \curS)$.
The field $[K_{n_h}]$ is constructed so that there exist two constants $k_\star$ and $k^\star$ with $0 < k_\star \leq k^\star < +\infty$ such that
\begin{equation} \label{eq4.2}
k_\star\, \Vert\bfu\Vert^2 \, \leq \langle \,[K_{n_h}(\bfx)]\, \bfu , \bfu \, \rangle  \leq \, k^\star\, \Vert\bfu\Vert^2
                                \quad , \quad \forall \bfu\in T_\bfx\curS \quad , \quad \forall\bfx\in\curS  \, .
\end{equation}
These bounds guarantee the uniform ellipticity of the operator $\curL_{n_h}$.
Let $\bfe^1(\bfx), \bfe^2(\bfx), \bfe^3(\bfx)$ be a positively oriented local orthonormal frame at $\bfx\in\curS$, with
$\bfe^3(\bfx) = \bfn(\bfx)$ the outward unit normal, and $\bfe^1(\bfx)$ and $\bfe^2(\bfx)$ belong to the tangent space
$T_\bfx \curS$. Let $h^{(1)}_{n_h}(\bfx)$ and $h^{(2)}_{n_h}(\bfx)$ be positive real-valued fields on $\curS$.
We define $[K_{n_h}]$ as the diagonal tensor field,
\begin{equation}\label{eq4.3}
[K_{n_h}(\bfx)] = h^{(1)}_{n_h}(\bfx)\,\, \bfe^1(\bfx) \otimes \bfe^1(\bfx) + h^{(2)}_{n_h}(\bfx)\,\, \bfe^2(\bfx) \otimes \bfe^2(\bfx)
                        \quad , \quad \forall\, \bfx \in \curS \, ,
\end{equation}
where, for $k\in \{1,2\}$,  the field $h^{(k)}_{n_h}$  is an $n_h$-finite representation of a real-valued field on $\curS$.
For $k\in\{1,2\}$ and for given $0\leq n_h$, a low-rank representation of $h^{(k)}_{n_h}(\bfx)$ is written as
\begin{equation}\label{eq4.4}
0 < c_\star^{(k)} \,\leq \, h^{(k)}_{n_h}(\bfx) = h_k + \sum_{j=1}^{n_h} \beta_j^{\,(k)}\,\hh^{(k,j)}(\bfx) \, \leq \,  c^{(k)\star} \quad , \quad \forall \, \bfx \in \curS \, ,
\end{equation}
where $h_k > 0$ and $0 < c_\star^{(k)} \,\leq  c^{(k)\star} $, which implies that Eq.~\eqref{eq4.2} holds, and where
\begin{equation}\label{eq4.5}
\int_\curS \hh^{(k,j)}(\bfx)\, \hh^{(k,j')}(\bfx) \, d\sigma(\bfx) = \delta_{jj'}
\quad , \quad
\int_\curS \hh^{(k,j)}(\bfx)\,  d\sigma(\bfx) = 0 \, .
\end{equation}
For each given $h_k \in [h_{k,\min} , h_{k,\max}]$, with $0 < h_{k,\min} < h_{k,\max} < +\infty$, the hyperparameters $\{\beta_j^{\,(k)}\}_j$ belong to an admissible set
$\curC_{\adp ,k}$ ensuring that $h^{(k)}_{n_h}(\bfx)$ satisfies Eq.~\eqref{eq4.4}
for all $\bfx\in\curS$. We then have
\begin{equation}\label{eq4.6}
\bfbeta^{\,(k)}_{n_h} = (\beta_1^{\,(k)},\ldots , \beta_{n_h}^{\,(k)})\, \in \, \curC_{\adp ,k} \, \subset \, \RR^{n_h}
\quad , \quad k \in\{1,2\} \, .
\end{equation}
The covariance operator of random field $U_{n_h}$ is $C_{U_{n_h}} =(\curL_{n_h}^{-1})^2$.\\

\noindent\textit{Additional assumption in coherence with Section~\ref{Section4.5}}.
For $n_h > 0$ and for every $k\in\{1,2\}$, it is assumed that the real-valued functions $\hh^{(k,j)}$ are continuous on $\curS$ for all $j=1,\dots,n_h$. Since $\curS$ is compact, the map $\bfx \mapsto \sum_{j=1}^{n_h} \vert \hh^{(k,j)}(\bfx)\vert$
is continuous and therefore attains its minimum on $\curS$. In the definition of $B_k$ in~Eq.\eqref{eq4.32}, the ratio involves
$(\sum_{j=1}^{n_h}\vert\hh^{(k,j)}(\bfx)\vert )^{-1}$. Hence, any point $\bfx$ such that
$\sum_{j=1}^{n_h} \vert \hh^{(k,j)}(\bfx)\vert = 0$ yields the value $+\infty$ and cannot realize the minimum.
Therefore, $B_k$ is well-defined and finite.
Consequently, the set $\curC_{\ad,k}$ defined in Eq.~\eqref{eq4.6} is a closed and bounded subset of $\RR^{n_h}$, and thus compact. Hence,
for $n_h > 0$,
\begin{equation}\label{eq4.7}
\curC_{\ad,k} \,\, \rm{is \, compact \, in} \,\,  \RR^{n_h} \,\, \rm{for} \,\, k \in\{1,2\}\, .
\end{equation}
\paragraph{(ii) Finite element discretization}
The weak form of Eq.~\eqref{eq4.1} and the existence of a unique second-order solution are studied in \cite{Soize2026}.
On the finite-element surface $\curS_h$ approximating $\curS$, with the basis $\{\varphi_q\}_{q=1}^{n_o}$, we write
$U_{n_h}(\bfx) \simeq \sum_{q=1}^{n_o} U_q \,\varphi_q(\bfx)$ and the finite-element discretization of the weak form of Eq.~\eqref{eq4.1} is written as
\begin{equation}\label{eq4.8}
( \tau_0\,[g] + [\kappa_{n_h}] ) \, \bfU = \bfB \,\, , \,\,
[g]_{qq'} = \! \int_{\curS_h}\! \varphi_{q'}(\bfx)\, \varphi_q(\bfx) \, d\sigma(\bfx) \,\,\, , \,\,\,
[\kappa_{n_h}]_{qq'} = \! \int_{\curS_h} \!\langle \, [K_{n_h}(\bfx)]\, \nabla_\bfx \varphi_{q'}(\bfx), \nabla_\bfx\varphi_q(\bfx)\, \rangle \, d\sigma(\bfx) \, ,
\end{equation}
where $\bfU=(U_1,\ldots,U_{n_o})$ and $\bfB=(B_1,\ldots,B_{n_o})$ with $B_q = \! \int_{\curS_h} \!\curB(\bfx)\, \varphi_q(\bfx)\, d\sigma(\bfx)$.
Since $\curB$ is white, $\bfB$ is centered and its covariance matrix is $E\{\bfB\otimes\bfB\} = [g] \in\MM_{n_o}^+$.
\paragraph{(iii) Spectral analysis of the covariance matrix of $\bfU$ and a reduced-order representation of the random field $U$}
The covariance matrix of the random vector $\bfU=(U_1,\ldots,U_{n_o})$ is
\begin{equation}\label{eq4.9}
[C_{\bfU}] = (\tau_0 [g] + [\kappa_{n_h}])^{-1} \, [g] \, (\tau_0 [g] + [\kappa_{n_h}])^{-T} \in\MM_{n_o}^+ \, .
\end{equation}
Let $[\lambda^m]\in\MM_m^+$ be the diagonal matrix of eigenvalues associated with $[\Phi^m]\in\MM_{n_o,m}$, whose columns span the dominant eigenspace of $[C_{\bfU}]$ (for a fixed $m\ll n_o$).
The reduced-order representation $U^m$ of $U$ is the random field
\begin{equation}\label{eq4.10}
U^m(\bfx)=\big\langle \bfpsi^m(\bfx),\bfH\big\rangle \quad , \quad
\bfpsi^m(\bfx)=[\lambda^m]^{1/2}[\Phi^m]^T\,\bfvarphi(\bfx)
\quad , \quad \bfvarphi(\bfx) = (\varphi_1(\bfx), \ldots , \varphi_{n_o}(\bfx))\in\RR^{n_o}\, ,
\end{equation}
where $\bfH$ is a centered Gaussian $\RR^m$-valued random variable whose covariance is the identity matrix.
We now write $U^m_\bftheta$ and $\bfpsi^m_\bftheta$ to denote the dependence of $U^m$ and $\bfpsi^m$ on the hyperparameters $h_1$, $h_2$, $\bfbeta^{\,(1)}_{n_h}$ and $\bfbeta^{\,(2)}_{n_h}$, which are some components of the hyperparameter vector $\bftheta$.\\

\noindent{\textit{Choice of the truncation index $m$}}.
In Eq.~\eqref{eq4.10}, the reduced latent field $U^m$ is obtained by projecting the Gaussian random vector $\bfU$ (which represents a
finite-dimensional approximation of the Gaussian random field $U_{n_h}$)  onto the dominant eigenspace of its covariance matrix $[C_\bfU]$, spanned by the eigenvectors $\{\Phi^k\}_{k=1}^m$ associated with the eigenvalues $\{\lambda_k\}_{k=1}^m$ satisfying $\lambda_1 \geq \ldots \geq \lambda_m$.
This is the usual principal component analysis (PCA) of $\bfU$, which is used to  construct a truncated Karhunen–Lo\`eve expansion of the random field $U_{n_h}$: the contribution of each eigenfunction to the pointwise variance is proportional to $\lambda_k$, so that most of the energy of $U_{n_h}$ is carried by the first eigenfunctions.
For the class of elliptic SPDEs considered here, the dominant eigenfunctions are associated with large-scale, smooth patterns on $\curS$, whereas higher-index eigenfunctions are progressively more oscillatory and correspond to shorter correlation lengths. Increasing $m$ therefore enriches the spatial spectrum of the latent field $U^m$ by gradually adding finer-scale fluctuations with decreasing variance.
Through Eq.~\eqref{eq4.25}, where the $\bftheta$-dependence {\color{black} is made explicit}, these fluctuations are sampled at the neuron locations and injected into the field-conditioned weights in Eq.~\eqref{eq4.26}, thereby enlarging the family of random architectures that can be generated for a fixed hyperparameter vector~$\bftheta$.
In practice, the choice of $m$ results from a trade-off. If $m$ is too small, the latent field is overly smooth and cannot encode sufficiently rich spatial variability in the weights. If $m$ is excessively large, many high-index modes contribute very little to
the total variance but may induce unnecessarily rough fluctuations at the discrete neuron locations, without a clear benefit for the loss function. A convenient guideline is to choose $m$ so that the cumulative spectral energy exceeds a prescribed threshold (e.g., $90\%$), ensuring that the latent field retains most of the covariance structure while keeping the dimensionality of the random
architecture under control.
\paragraph{Remark 1 (Directional correlation length)}
With the operator $\curL_{n_h}$ and the anisotropy
$[K_{n_h}(\bfx)]=h^{(1)}_{n_h}(\bfx)\,\bfe^1(\bfx)\otimes \bfe^1(\bfx)+h^{(2)}_{n_h}(\bfx)\,\bfe^2(\bfx)\otimes \bfe^2(\bfx)$ expressed in the tangent frame $\{\bfe^1,\bfe^2\}$ introduced above, the correlation lengths $\ell_1$ and $\ell_2$ along the two
principal directions scale like the square root of the corresponding directional
diffusivities divided by $\tau_0$. When $h^{(k)}_{n_h}(\bfx)$ varies mildly around positive
constants $h_k^2$, the geodesic correlation lengths satisfy the 1D–SPDE scaling
$\ell_1 \,\asymp\, h_1/\sqrt{\tau_0}$ and $\ell_2 \,\asymp\, h_2/\sqrt{\tau_0}$.
Thus, increasing $\tau_0$ shortens correlations uniformly, while increasing $h_k^2$
elongates correlations specifically along $\bfe_k$.
For example, for a torus with major radius $R$ and minor radius $r$, take constant
$h^{(1)}_{n_h}(\bfx)\equiv h_1^2$ and $h^{(2)}_{n_h}(\bfx)\equiv h_2^2$. If $\bfe^1$ is chosen azimuthal
(along the major circle) and $\bfe^2$ meridional (along the minor circle), then
$\ell_{\text{az}} \,\asymp\, h_1/\sqrt{\tau_0}$ and $\ell_{\text{mer}} \,\asymp\, h_2/\sqrt{\tau_0}$.
As a practical guideline to avoid fields that are nearly constant around a loop, one may
keep $\ell_{\text{az}}\lesssim \pi R$ and $\ell_{\text{mer}}\lesssim \pi r$;
these are modeling choices rather theoretical constraints.
\paragraph{Remark 2 (Construction of the finite orthonormal family $\{\hh^{(k,j)},j=1,\ldots,n_h\}$ in $L^2(\curS)\cap C(\curS)$)}
There are many possible ways to construct such a finite family. For instance, one may take a realization of an arbitrary random orthonormal basis. However, in the present framework, this family can be obtained as a finite truncation of a Hilbert basis in $L^2(\curS)$. A suitable choice is to use the normalized eigenfunctions associated with the $n_h$ smallest eigenvalues of the eigenvalue problem for the Laplace-Beltrami operator on $\curS$.
Since $\curS$ is a compact, boundaryless, multiply connected manifold, the Laplace-Beltrami operator $-\Delta_{\curS}$ has a discrete spectrum $0=\lambda_{\LBp,1} < \lambda_{\LBp,2} \leq \lambda_{\LBp,3} \leq \cdots \to + \infty$,
with the (one–dimensional) kernel spanned by the constant function. Let $-\Delta_{\curS}\,\rchi^{j} = \lambda_{\LBp,j}\,\rchi^{j}$ on $\curS$ be the eigenvalue problem where the eigenfunction $\rchi^{j}$, associated with the eigenvalue $\lambda_{\LBp,j}$, is normalized by
\begin{equation*}
\int_{\curS}\rchi^{j}(\bfx)\,\rchi^{j'}(\bfx)\,d\sigma(\bfx)=\delta_{jj'}\quad,\quad
\int_{\curS}\rchi^{j}(\bfx)\,d\sigma(\bfx)=0 \quad \text{for} \,\, j\ge 2 \quad , \quad \text{and} \,\, \rchi^{1} = |\curS|^{-1/2} \, .
\end{equation*}
By elliptic regularity on compact manifolds, the eigenfunctions $\rchi^{j}$ are smooth on $\curS$; in particular, they are continuous. We then define, for $k\in\{1,2\}$ and $j=1,\ldots,n_h$, $\hh^{(k,j)}(\bfx) = \rchi^{\,j+1}(\bfx)$,
which yields an $L^2(\curS)$-orthonormal, low-eigenvalue family obtained by discarding the constant eigenfunction.
On the finite-element surface $\curS_h$ with basis $\{\varphi_q\}_{q=1}^{n_o}$, we write
$\rchi^j(\bfx) \simeq \sum_{q=1}^{n_o} v^j_q\, \varphi_q(\bfx)$
and the finite-element discretization of the eigenvalue problem for the Laplace-Beltrami operator is written as
\begin{equation*}
[\kappa_\LBp]\,[v]=[g]\,[v]\,[\lambda_\LBp] \quad , \quad
[\kappa_\LBp]_{qq'} = \int_{\curS_h}\! \langle\nabla_\bfx\varphi_{q'}(\bfx),\nabla_\bfx\varphi_q(\bfx)\rangle\,d\sigma(\bfx)\, ,
\end{equation*}
where $[v] = [\bfv^1 \ldots \bfv^{n_h+1}]\in\MM_{n_o,n_h+1}$ with $\bfv^j = (v_1^j,\ldots, v_{n_o}^j)\in\RR^{n_o}$, and where
$[v]^T [g]\, [v]= [I_{n_h + 1}]$ and $[\lambda_\LBp]=\mathrm{diag}(\lambda_{\LBp,1},\ldots,\lambda_{\LBp,n_h+1})$. Therefore, we have
\begin{equation*}
\hh^{(k,j)}(\bfx) \simeq \sum_{q=1}^{n_o} v^{(j+1)}_{q}\,\varphi_q(\bfx) \quad , \quad   j\in\{1,\ldots,n_h\}
\quad , \quad k\in\{1,2\}\, ,
\end{equation*}
and the discrete functions inherit $L^2(\curS_h)$–orthonormality and the low-eigenvalue ordering.
\paragraph{Notation for the random field $U^m$ with prescribed $m$, $h_1$, $h_2$, $\bfbeta^{\,(1)}_{n_h}$, and $\bfbeta^{\,(2)}_{n_h}$}
In the following, $m$ and $n_h$ are assumed fixed. The hyperparameter vector is defined by
\begin{equation*}
\bftheta = (h_1, h_2, \zeta_s, \bfbeta^{\,(1)}_{n_h}, \bfbeta^{\,(2)}_{n_h})\, .
\end{equation*}
Since the random field $U^m$ defined on $\curS$ depends on $h_1$, $h_2$, $\bfbeta^{\,(1)}_{n_h}$, and $\bfbeta^{\,(2)}_{n_h}$,
we henceforth introduce the notation
\begin{equation*}
U^m(\cdot) \equiv U(\cdot\,;\bftheta)
\end{equation*}
to make this parametric dependence explicit. In addition, we use the same notation $U(\bfx\,;\bftheta)$ for $\bfx\in\curS$ and for $\bfx\in\curS_h$. Similarly, $\bfpsi^m(\bfx)$ in Eq.~\eqref{eq4.10} is denoted by
$\bfpsi(\bfx\,;\bftheta)$.
\subsection{Neuron sampling via an inhomogeneous Poisson process and definition of input and output sets}
\label{Section4.2}
\paragraph{(i) Neuron sampling}
To generate the positions $\{\bfx^i_\bftheta, i=1,\ldots , N\}$ of $N$ neurons on the manifold $\curS_h$, we employ a spatial point process that adapts to the local variability of the latent Gaussian field $\{U(\bfx\, ;\bftheta),\bfx\in\curS_h\}$.
We thus define an inhomogeneous Poisson point process on $\curS_h$ with intensity $\Lambda(\bfx\, ;\bftheta) \, d\sigma(\bfx)$, where $d\sigma(\bfx)$ denotes the surface measure on the manifold $\curS_h$ induced by its embedding in $\RR^3$.
This means that, for any measurable $A\subset\curS_h$, the random variable $\curN_A(\bftheta)$, representing the number of points in $A$, is Poisson distributed with mean
\begin{equation*}
E\{\curN_A(\bftheta)\}  = \int_A \Lambda(\bfx\, ;\bftheta)\,d\sigma(\bfx)\, .
\end{equation*}
Using Eq.~\eqref{eq4.10}, for all $\bfx\in\curS_h$, the variance of $U(\bfx\,; \bftheta)$ is given by
\begin{equation*}
\Lambda(\bfx;\bftheta)=\Vert \,\bfpsi(\bfx\, ;\bftheta)\, \Vert^2 \, .
\end{equation*}
The intensity measure of the inhomogeneous Poisson process is $\Lambda(\bfx\, ;\bftheta)\,d\sigma(\bfx)$.
Conditional on the event $\curN_{\curS_h}(\bftheta)=N$, the locations of the $N$ points (neurons)
are independently and identically distributed over $\curS_h$ with probability density function $p(\bfx\, ;\bftheta)$ with respect to $d\sigma(\bfx)$, given by
\begin{equation}\label{eq4.11}
p(\bfx\, ;\bftheta)=\frac{\Lambda(\bfx;\bftheta)}{\int_{\curS_h}\Lambda(\bfy;\bftheta)\, d\sigma(\bfy)}\, .
\end{equation}
The set of accepted points $\{\bfx_\bftheta^i,i=1,\ldots , N\}\subset \curS_h$ constitutes the $N$ neurons of the graph:
\begin{equation*}
\curV_\bftheta= \{\bfx_\bftheta^1,\ldots,\bfx_\bftheta^N \} \, ,
\end{equation*}
The sampling procedure for the inhomogeneous Poisson process via a rejection method is detailed in Section~3 of \cite{Soize2026}.
\paragraph{(ii) Definition of input and output sets}
For a realization $u(\cdot\,; \bftheta)$ of $U(\cdot\,; \bftheta)$, the input index set $\curJ_\bftheta^\pin$ and the output index set $\curJ_\bftheta^\pout$ are defined as the indices  of the $n_\pin$ smallest and $n_\pout$ largest values of $u(\bfx_\bftheta^j\, ; \bftheta)$, respectively. The remaining indices form the hidden set $\curJ_\bftheta^\pint$, with $N=n_\pin+n_\pint+n_\pout$, and $\curJ_\bftheta^\pin\cap \curJ_\bftheta^\pint\cap\curJ_\bftheta^\pout = \emptyset$.
The number  $n_\pint \geq  1$ of hidden neurons is given by
\begin{equation}\label{eq4.12}
n_\pint = N - n_\pin - n_\pout \, .
\end{equation}
\subsection{Intrinsic geodesics and sparsity mask}
\label{Section4.3}
Let $[D_g]\in\MM_{n_o}$ be the all–pairs shortest–path geodesic distance matrix on the finite-element graph. For $\bfx$ and $\bfy$ in $\curS_h$, the finite-element interpolation
\begin{equation} \label{eq4.13}
\hat d_g(\bfx,\bfy)=\sum_{q=1}^{n_o}\sum_{q'=1}^{n_o}\varphi_{q}(\bfx)\,[D_g]_{qq'}\,\varphi_{q'}(\bfy) \, ,
\end{equation}
extends geodesics to arbitrary points. With neuron–neuron distances
$[D_{g,\bftheta}^{(\mathrm{neu})}]_{ij}=\hat d_g(\bfx_\bftheta^i,\bfx_\bftheta^j)$ and self–tuning local scales
$\hat\sigma_{\bftheta,i}$ (distance to the $(k+1)$th nearest neighbor in geodesic metric with
$k = \left\lfloor \min(\sqrt{N},\,N-1) \right\rfloor$), the diffusion affinity is
\begin{equation} \label{eq4.14}
[w^g_\bftheta]_{ij}=\exp\Bigg(-\frac{[D_{g,\bftheta}^{(\mathrm{neu})}]_{ij}^{2}}{\hat\sigma_{\bftheta,i}\,\hat\sigma_{\bftheta,j}}\Bigg)
\quad , \quad  [w^g_\bftheta]_{ii}=0 \quad , \quad i,j\in\{1,\ldots, N\}\, .
\end{equation}
To control the sparsity of the graph, we introduce a global percentile parameter $\tau_\pprc \in [0,100]$ and a
softer input-aware percentile $\tau_{\pprc,\pin}\in[0,100]$ with $\tau_{\pprc,\pin}<\tau_\pprc$,
applied only to edges incident to inputs (indices $\curJ_\bftheta^\pin$ from Section~\ref{Section4.2}).
Let $a_\pprc^\bftheta$ denote the $\tau_\pprc$th percentile of the upper–triangular off–diagonal entries of $[w^g_\bftheta]$,
and let $a_{\pprc,\pin}^\bftheta$ denote the $\tau_{\pprc,\pin}$th percentile restricted to those entries with at least one input endpoint,
\begin{align}
a_\pprc^\bftheta = & \,\texttt{prctile}\left( \{ [w^g_\bftheta]_{ij} : 1 \le i < j \le N \}\, , \, \tau_\pprc \right) \, ,\label{eq4.15} \\
a_{\pprc,\pin}^\bftheta = & \, \texttt{prctile} \left( \{ [w^g_\bftheta]_{ij} : 1 \le i < j \le N \, ,\, (i\in\curJ_\bftheta^\pin \ \text{or}\ j\in\curJ_\bftheta^\pin) \} \, ,\, \tau_{\pprc,\pin} \right) \, . \label{eq4.16}
\end{align}
The symmetric binary mask $[M^\psp_\bftheta]\in\{0,1\}^{N\times N}$ (with zero diagonal) is then defined by
\begin{equation}\label{eq4.17}
[M^\psp_\bftheta]_{ij} \,=\,
\begin{cases}
1 & \text{if } i<j \,\, , \,\, i\notin\curJ_\bftheta^\pin \,\, , \,\, j\notin\curJ_\bftheta^\pin \,\, ,\,\, [w^g_\bftheta]_{ij} \ge a_\pprc^\bftheta \, , \\[2pt]
1 & \text{if } i<j \,\, , \,\, \big(i\in\curJ_\bftheta^\pin \ \text{or}\ j\in\curJ_\bftheta^\pin\big)\,\, , \,\, [w^g_\bftheta]_{ij} \ge a_{\pprc,\pin}^\bftheta\quad , \\[2pt]
0 & \text{otherwise},
\end{cases}
\quad [M^\psp_\bftheta]_{ii}=0\,\, , \,\, [M^\psp_\bftheta]=[M^\psp_\bftheta]^T.
\end{equation}
To prevent a collapse of the input–to–free coupling, we proceed as follows.
We fix an integer $K_{\min}\in\NN$ with $K_{\min}\le n_\pint$.
For each input index $j\in\curJ_\bftheta^\pin$, define the set
\begin{equation}\label{eq4.18}
\curI_{\min}(j)\ \subset\ \curJ_\bftheta^\pint \quad , \quad
\vert \curI_{\min}(j)\vert = K_{\min},
\end{equation}
as the $K_{\min}$ indices $i\in\curJ_\bftheta^\pint$ that maximize the geometric affinity
$[w^g_\bftheta]_{ij}$. We then enforce
\begin{equation}\label{eq4.19}
[M^\psp_\bftheta]_{ij} = [M^\psp_\bftheta]_{ji} = 1\quad\text{for all } i\in\curI_{\min}(j) \, , \, j\in\curJ_\bftheta^\pin \, ,
\end{equation}
after applying the percentile thresholds, and finally symmetrize with zero diagonal.
Equivalently,
\begin{equation}\label{eq4.20}
\sum_{i\,\in\,\curJ_\bftheta^\pint} [M^\psp_\bftheta]_{ij}\ \ge\ K_{\min}\quad , \quad \forall\, j\in\curJ_\bftheta^\pin.
\end{equation}
Finally, the graph defined by $[M^\psp_\bftheta]$ is required to be connected (otherwise $\tau_\pprc$ and/or $K_{\min}$ are adapted).
The geometry–aware sparse matrix is the Hadamard product
\begin{equation}\label{eq4.21}
[w^\psp_\bftheta]_{ij} = [M^\psp_\bftheta]_{ij}\, [w^g_\bftheta]_{ij}\, .
\end{equation}
Given the global sparsification percentile $\tau_\pprc\in[0,100]$ and a intensity multiplier $r_\psp > 1$ for input–incident edges, we define the input-aware percentile,
\begin{equation}\label{eq4.22}
\tau_{\pprc,\pin} = 100\, \Bigl(1-\min\{\,1,\, r_\psp\bigl(1-\tfrac{\tau_\pprc}{100}\bigr)\,\}\Bigr)\in[0,100]\, .
\end{equation}
We use percentile thresholds to remove connections: the fraction of all remaining connections is $s_{\rm all} = 1-\tau_\pprc/100$,
and the fraction of remaining connections to input neurons is
$s_\pin = 1-\tau_{\pprc,\pin}/100=\min\{1,\,r_\psp\,s_{\rm all}\}$.
We use the following default choices (with admissible ranges for sensitivity): $\tau_\pprc \in [50,70]$, $r_\psp \in[1.5,2]$ with a default value $r_\psp =1.8$, and
\begin{equation}\label{eq4.23}
K_{\min} = \min\{ \lfloor s_\pin n_\pint \rfloor\, , \,  \max\{5\, , \, \lceil\rho_\psp n_\pint\rceil \,\}\} \quad , \quad \rho_\psp\in [0.09\,, 0.20] \, ,
\end{equation}
with a default value $\rho_\psp=0.12$.
We always apply $[M^\psp_\bftheta]_{ii}=0$ as well as symmetry.
We check that the graph induced by $[M^\psp_\bftheta]$ is connected; otherwise we modify $\tau_\pprc$ and/or $K_{\min}$ while keeping $\tau_{\pprc,\pin}<\tau_\pprc$.
Since the $N\times N$ matrix  $[M^\psp_\bftheta]$ is symmetric and has entries in $\{0,1\}$, the  number $n_\pinc$ of inter-node connections is defined by
\begin{equation}\label{eq4.24}
n_\pinc = \sum_{1\, \le\,  i\,  < \, j \le \, N} [M^\psp_\bftheta]_{ij} \,  \leq \, \frac{N(N-1)}{2} \, .
\end{equation}
\subsection{Field-conditioned random weights and sparse matrix}
\label{Section4.4}
The attributes of the nodes are defined as the sampled latent field values
\begin{equation}\label{eq4.25}
S^\bftheta_i = U(\bfx_\bftheta^i ;\bftheta) \,\, , \,\, i=1,\ldots,N \quad , \quad
\sigma_{\bfS_\bftheta}^2 = N^{-1}\operatorname{tr}[C_{\bfS_\bftheta}] \, ,
\end{equation}
where $[C_{\bfS_\bftheta}]$ is the covariance matrix of the $\RR^N$-valued random variable $\bfS_\bftheta = (S_{\bftheta,1}, \ldots, S_{\bftheta,N})$.
For a scale $\zeta_s \in [\zeta_{s,\min}, \zeta_{s,\max}]$ with $0 < \zeta_{s,\min} < \zeta_{s,\max} < +\infty$, the dense symmetric random weights are
\begin{equation}\label{eq4.26}
 [\bfW_\bftheta]_{ij} = [w^g_\bftheta]_{ij}\,
    \exp\!\left(-\frac{(\Delta S^{\bftheta}_{ij})^2}{2\zeta_s^2\,\sigma_{\bfS_\bftheta}^2}\right), \quad [\bfW_\bftheta]_{ii} = 0 \, ,
\end{equation}
where $\Delta S^{\bftheta}_{ij} = S^\bftheta_i - S^\bftheta_j$,
and the geometry–aware sparse random matrix is the Hadamard product
\begin{equation}\label{eq4.27}
 [\bfW^\psp_\bftheta]_{ij} = [M^\psp_\bftheta]_{ij} \, [\bfW_\bftheta]_{ij} \quad ,\quad i,j\in\{1,\ldots, N\}\, .
\end{equation}
The scale parameter $\zeta_s$ in Eq.~\eqref{eq4.26} controls how strongly the latent field $U^m_\bftheta$
modulates the geometry-based affinities $[w^g_\bftheta]_{ij}$. For small values of $\zeta_s$, the
factor $\exp \, (-(\Delta S^{\bftheta}_{ij})^2 / (2\zeta_s^2\,\sigma_{\bfS_\bftheta}^2))$ decays rapidly as the
latent-field contrast $\vert\Delta S^{\bftheta}_{ij}\vert$ increases, which strongly attenuates connections between
neurons that are separated in latent space and reduces the overall Frobenius norm of $[\bfW_\bftheta]$.
Conversely, large values of $\zeta_s$ keep $[\bfW_\bftheta]_{ij}$ close to the geometry-induced weights
$[w^g_\bftheta]_{ij}$ and yield stronger, more delocalized couplings. In this sense, $\zeta_s$ acts as an
architecture-level hyperparameter that controls both the contrast and the energy of the random weight
matrix before sparsification and subsequent normalizations.
\subsection{Specification of the constraints applied to $h_1$, $h_2$, $\zeta_s$, $\bfbeta^{\,(1)}_{n_h}$, and $\bfbeta^{\,(2)}_{n_h}$
when $n_h \ge 0$}
\label{Section4.5}
We consider the cases $n_h=0$ and then $n_h\ge 1$.
\paragraph{(i) Specification of the constraints applied to $h_1$, $h_2$, $\zeta_s$ when $n_h =0$}
Taking into account the hypotheses introduced in Section~\ref{Section4.1}-(i) and Section~\ref{Section4.4}, we have
\begin{equation}\label{eq4.28}
\bftheta = (h_1\, ,h_2\, ,\zeta_s) \in \curC_\bftheta = [h_{1,\min} , h_{1,\max}]\times [h_{2,\min} , h_{2,\max}] \times
[\zeta_{s,\min}, \zeta_{s,\max}] \subset \RR^3 \, ,
\end{equation}
with $0 < h_{k,\min} < h_{k,\max} < +\infty$ for $k\in\{1,2\}$ and $0 < \zeta_{s,\min} < \zeta_{s,\max} < +\infty$.
Consequently, $\curC_\bftheta$ is compact in $\RR^3$.
\paragraph{(ii) Specification of the constraints applied to $h_1$, $h_2$, $\zeta_s$, $\bfbeta^{\,(1)}_{n_h}$, and $\bfbeta^{\,(2)}_{n_h}$
when $n_h \ge 1$}
The hyperparameter $\bftheta$ must satisfy the following constraints:
\begin{equation}\label{eq4.29}
\bftheta = (h_1\, ,h_2\, ,\zeta_s , \bfbeta^{\,(1)}_{n_h}, \bfbeta^{\,(2)}_{n_h}) \in
\curC_\bftheta = [h_{1,\min} , h_{1,\max}]\times [h_{2,\min} , h_{2,\max}] \times [\zeta_{s,\min}, \zeta_{s,\max}]  \times
\curC_{\adp,1}\times \curC_{\adp,2} \subset \RR^{3+2 n_h} \, .
\end{equation}
Due to Eq.~\eqref{eq4.7}, it is deduced that $\curC_\bftheta$ is compact in  $\RR^{3+2 n_h}$.
For $k\in\{1,2\}$, the construction of the admissible set $\curC_{\adp,k}$, defined in Eq.~\eqref{eq4.6} with Eq.~\eqref{eq4.4},
is performed by projection onto admissible bounds defined as follows,
\begin{equation}\label{eq4.30}
\curC_{\adp,k} =  [-B_k,B_k]^{n_h} \, \subset\, \RR^{n_h} \, ,
\end{equation}
where $B_k>0$ is chosen as
\begin{equation}\label{eq4.31}
B_k \,=\, \min_{\bfx\,\in\,\curS} \,
\frac{ \max \!\big( 0\, , \min( h_k - c_\star^{(k)} \, , \, c^{(k)\star} - h_k ) \big )}
     { \sum_{j=1}^{n_h} \big| \, \hh^{(k,j)}(\bfx) \, \big| } \, .
\end{equation}
If the denominator vanishes at some $\bfx$, that point imposes no constraint and can be ignored in the minimum.
In the limiting case $c_\star^{(k)} = 0$ and $c^{(k)\star} = +\infty$, Eq.~\eqref{eq4.31} reduces to
\begin{equation}\label{eq4.32}
B_k \,=\, \min_{\bfx\,\in\,\curS} \,
\frac{ h_k }
     { \sum_{j=1}^{n_h} \big| \, \hh^{(k,j)}(\bfx) \, \big| } \, .
\end{equation}
With this choice, any $\bfbeta^{\,(k)}_{n_h}\in [-B_k,B_k]^{n_h}$ satisfies the inequalities in Eq.~\eqref{eq4.4} for all $\bfx\in\curS$.
The value $B_k$ depends on the current admissible value of $h_k$; if a fixed nominal value is used in the projection step, that value must lie strictly inside the interval $(c_\star^{(k)},c^{(k)\star})$.
\subsection{Construction of the low-rank representation of the bias vector and hyperparameter}
\label{Section4.6}

A low-rank representation of the bias vector $\bfb\in\RR^N$ is defined by
\begin{equation}\label{eq4.33}
\bfb=\sum_{j=1}^{n_b}\beta_j\,\bfh^j = [h]\,\bfbeta \quad , \quad \bfbeta\in\RR^{n_b} \quad , \quad [h]\in\MM_{N,n_b} \, .
\end{equation}
Typical choices for $n_b$ satisfy $n_b\in [0.1\, , 0.70] \times n_\pint$ with the upper bound $N - n_\pint$. The vectors $\{\bfh^1,\ldots ,\bfh^{n_b}\}$ form an orthonormal set in $\RR^N$. For a fixed value of $\bftheta$, $n_\psim$ independent realizations of the sparse random weight matrix are generated, and for a given input $\xx$, the vector $\bfbeta(\bftheta)$ is obtained as the least-squares solution of the random nonlinear neural equation described in Section~\ref{Section4.7}.
\subsection{Random nonlinear neural equation and outputs}
\label{Section4.7}
Let $f(x)=\tanh(x)$ (applied componentwise) and let  $\bfA_{\bftheta}$ be the random activated vector with values in $\RR^N$ such that
\begin{equation} \label{eq4.34}
\bfA_{\bftheta} = \bff\,( \,[\bfW^{\psp}_\bftheta] \, \bfA_{\bftheta}\, +\, \bfb\, )\, .
\end{equation}
A partition of the random vector  $\bfA_{\bftheta}$ is performed into prescribed inputs $\xx\in\RR^{n_\ppin}$ at indices $\curJ_\bftheta^\pin$ and the random vector $\hat\bfA_{\bftheta}$ of the free components, with values in $\RR^{\hat n}$, where $\hat n = N - n_\pin$.
After percentile reduction  and the non-degeneracy constraint $K_{\min}$ on the input-free block,
we use selection matrices $[\hat O_\bftheta]\in \MM_{\hat n,N}$ and $[O_\bftheta^{\,\pin}]\in\MM_{n_\ppin,N}$ to partition $[\bfW^{\psp}_\bftheta]$ into the random free block $[\hat\bfW_\bftheta] = [\hat O_\bftheta]\,[\bfW^\psp_\bftheta]\, [\hat O_\bftheta]^T$ with values in $\MM_{\hat n}$ and the random input block $[\hat\bfW_\bftheta^{\pin}] =[\hat O_\bftheta]\, [\bfW^\psp_\bftheta]\, [O_\bftheta^{\,\pin}]^T$ with values in $\MM_{\hat n , n_\pin}$. The free bias is
$\hat\bfb_\bftheta=[\hat O_\bftheta]\,\bfb\in\RR^{\hat n}$.
Therefore, the random free activated state $\hat\bfA_\bftheta$ with values in $\RR^{\hat n}$ satisfies
\begin{equation}\label{eq4.35}
\hat\bfA_{\bftheta}=\hat\bff\,\big(\,[\hat\bfW_\bftheta]\,\hat\bfA_{\bftheta} + [\hat\bfW_\bftheta^{\pin}]\,\xx +
                 \hat\bfb_\bftheta\big)\, ,
\end{equation}
where $\hat\bff$ denotes the restriction of $\bff$ to the free components, and the output is
\begin{equation*}
\YY(\xx;\bftheta) = [O_\bftheta^{\,\pout}]\,\hat\bfA_{\bftheta} \, .
\end{equation*}
In this  conditional formulation, input neurons are clamped to the observed values $\xx$; hence  no update equation is solved for them and the feedback free block $[\hat\bfW_\bftheta^{\pin}]^T$ does not enter the fixed-point iteration. Only the free block $[\hat\bfW_\bftheta]$
must be controlled for contraction.
\paragraph{(i) Input-block and free-bias scaling to preserve input sensitivity}
To avoid a numerically vanishing input influence  after sparsification and to prevent bias–dominated saturation, we introduce
a two–fold soft rescaling together with a non–degeneracy constraint.
First, we balance the input block against the free block as
\begin{equation}\label{eq4.36}
\gamma_\bftheta = \min \Big(\gamma_{\max}\, ,\, \frac{\Vert [\hat\bfW_\bftheta]\Vert_F}{\sqrt{n_\pin}\,\Vert [\hat\bfW_\bftheta^{\pin}]\Vert_F}\Big) \quad , \quad
[\widehat{\bfW}_\bftheta^{\pin}] =  \gamma_\bftheta \,[\hat\bfW_\bftheta^{\pin}] \, .
\end{equation}
Second, we softly rescale the bias to keep the pre–activation in a non–saturated regime,
\begin{equation}\label{eq4.37}
\gamma_{b,\bftheta} = \min \Bigl(\gamma_{b,\max}\, ,\, \frac{\Vert [\hat\bfW_\bftheta]\Vert_F + \Vert [\widehat{\bfW}^{\pin}_\bftheta]\Vert_F}{\sqrt{\hat n}\,\Vert \hat\bfb_\bftheta\Vert_2}\Bigr)
\quad , \quad
\widehat{\bfb}_\bftheta \,=\, \gamma_{b,\bftheta} \,\hat\bfb_\bftheta \, .
\end{equation}
Finally, we enforce a non–degeneracy condition that prevents the input block from collapsing relative to the
free block and bias contributions,
\begin{equation}\label{eq4.38}
\Vert [\widehat{\bfW}^{\pin}_\bftheta]\Vert_F \;\ge\; \rho_{\rm in}\,
\Bigl(\,\Vert [\hat\bfW_\bftheta]\Vert_F \,+\, \tfrac{1}{\sqrt{\hat n}}\,\Vert \widehat{\bfb}_\bftheta\Vert_2\,\Bigr)
\quad , \quad \rho_{\rm in}\in(0,1)\, .
\end{equation}
Therefore, we replace $[\hat\bfW_\bftheta^{\pin}]$ by $[\widehat{\bfW}_\bftheta^{\pin}]$  and $\hat\bfb_\bftheta$ by $\widehat{\bfb}_\bftheta$ in \eqref{eq4.35}. This preserves sparsity and topology while balancing the scales of the three contributions.
Two steps can be used to set the values of parameters $\rho_{\rm in}$, $\gamma_{\max}$, and $\gamma_{b,\max}$ for a fixed value of the parameter $c_{\targetp}$ (see (ii) below).
Step (a): set $\rho_{\rm in}=0$ to disable Eq.~\eqref{eq4.38}. Choose a conservative bias scale $\gamma_{b,\max}\in[0.05,0.2]$ (default $0.1$). Tune $\gamma_{\max}$ by a geometric search over $\gamma_{\max}\in[10^{-4},10]$ and select $\gamma_{\max}^\star$ so that the observed response (for instance, standard deviations) reasonably matches the target.
Step (b): fix $\gamma_{\max}=\gamma_{\max}^\star$ and re–enable Eq.~\eqref{eq4.38} by increasing $\rho_{\rm in}$ from $0$ in the interval $]0,0.7]$.
\paragraph{(ii) Balancing of the free block  for numerical contraction}
A degree–symmetric normalization, followed by a light two–pass $\ell^2$ equilibration is applied,
so that $([I_{\hat n}]-[\hat\bfW_\bftheta])$ is numerically well conditioned, almost surely. With the diagonal degree matrix
$[D_{\targetp}]$ defined by $[D_{\targetp}]_{ii}=\sum_{j=1}^{\hat n}\!\big|[\hat\bfW_\bftheta]_{ij}\big|$ for $i\in\{ 1,\ldots,\hat n\}$, we define the random matrix $[\widetilde{\bfW}_\bftheta]$ with values in $\MM_{\hat n}$, such that
\begin{equation}\label{eq4.39}
[\widetilde{\bfW}_\bftheta] = [D_{\targetp}]^{-1/2}\,[\hat\bfW_\bftheta]\,[D_{\targetp}]^{-1/2}\, .
\end{equation}
Then, given a target $c_{\targetp} > 0$, the equilibration consists of two successive passes of diagonal rescaling of rows and columns.
For $i,j \in \{1,\ldots,\hat n\}$, diagonal random matrices $[D_r]$ and $[D_c]$ are constructed from the rows $[\widetilde{\bfW}_\bftheta]_{i,:}$ and the columns $[D_r \widetilde{\bfW}_\bftheta]_{:,j}$, respectively, and the scalings toward $c_{\targetp}$ are performed such that
\begin{equation}\label{eq4.40}
[D_r]_{ii} =  \frac{c_{\targetp}}{\Vert\,  [\widetilde{\bfW}_\bftheta]_{i , :}\, \Vert} \quad , \quad
[D_c]_{jj} = \frac{c_{\targetp}}{\Vert \, [D_r\,\widetilde{\bfW}_\bftheta]_{: , j}\, \Vert} \quad , \quad
[\widehat{\bfW}_\bftheta] \;=\; [D_r]\,[\widetilde{\bfW}_\bftheta]\,[D_c]\,.
\end{equation}
This two–pass procedure preserves sparsity and prepares the next steps.
Typical values for the balancing target are $c_{\targetp}\in [0.40, 0.60]$, with default $c_{\targetp}=0.5$.
Lower values reduce conditioning factor and therefore variability, whereas higher values strengthen
coupling at the cost of larger fluctuations.
\paragraph{(iii) Input sensitivity rescaling from a local linearization}
At a reference point $\bar\xx$ (empirical mean of the inputs), we define the random linearized initialization in the almost surely sense,
\begin{equation}\label{eq4.41}
\hat\bfA^{(0)}_\bftheta = ([I_{\hat n}] -[\widehat{\bfW}_\bftheta])^{-1}\,\big([\widehat{\bfW}_\bftheta^{\pin}]\,\bar\xx \,+\, \widehat{\bfb}_\bftheta\big) \quad , \quad
\bfZ_0(\bar\xx) =  [\widehat{\bfW}_\bftheta]\,\hat\bfA^{(0)}_\bftheta +\, [\widehat{\bfW}_\bftheta^{\pin}]\,\bar\xx \,+\, \widehat{\bfb}_\bftheta \, .
\end{equation}
The Jacobian of $\hat\bff(\bfZ_0(\xx))$ with respect to $\xx$, evaluated in $\xx = \bar\xx$, is then the random matrix $[\bfF_\bftheta]$ with values in $\MM_{\hat n , n_\ppin}$,
\begin{equation}\label{eq4.42}
[\bfF_\bftheta] =  \mathrm{diag}\big( \hat\bff{\,'}(\bfZ_0(\bar\xx))\big)\,([I_{\hat n}] -[\widehat{\bfW}_\bftheta])^{-1}\,[\widehat{\bfW}_\bftheta^{\pin}] \, ,
\end{equation}
where $\hat\bff{\,'}(\cdot) = 1 - \tanh^2(\cdot)$
and we rescale columnwise the input block $[\widehat{\bfW}_\bftheta^{\pin}]$ to impose a common target $g_\hi>0$ on the
$\ell^2$ norm of the columns of $[\bfF_\bftheta]$, replacing the real-valued random variable $\Xi_j$ by its projection onto the deterministic interval $[s_{\min} , s_{\max}]$ with $s_{\min} > 0$,
\begin{equation}\label{eq4.43}
\Xi_j = \min\!\Big\{\,\max\!\Big(\frac{g_\hi}{\Vert [\bfF_\bftheta]_{:,j}\Vert},\, s_{\min}\Big),\; s_{\max}\Big\}
\quad , \quad
[\widehat{\bfW}_\bftheta^{\pin}] \;\leftarrow\; [\widehat{\bfW}_\bftheta^{\pin}]\,[\mathrm{diag}(\Xi_1,\ldots,\Xi_{n_\pin})] \,.
\end{equation}
This final step of scaling  and rescaling preserves the sparsity pattern.
Set the sensitivity target $g_\hi$ close to one so that the linearized Jacobian remains in the near-linear regime implied by
$\Vert \hat\bfA^{(0)}_\bftheta\Vert_\infty \leq 0.35$. A practical  range is $g_\hi \in [0, 1.5]$ with default $0$.
Then enforce $ 0 < s_{\min} \leq 1 \leq s_{\max}$ and clamp the per-input factors with $s_{\min}\in [0.1 , 1.0]$ (default $1.0$) and
$s_{\max}\in [ 1.0 , 2.0 ]$ (default $1.0$; choose $> 1$ only if amplification of weak inputs is allowed). Setting
$g_\hi=0$ and $s_{\min} = s_{\max} = 1$ disable this rescaling from a local linearization.
\paragraph{(iv) Final form of the nonlinear neural equation to be solved}
For each $(\xx^i,\bftheta)$, we have to solve, almost surely, the random nonlinear neural equation
\begin{equation}\label{eq4.44}
\hat\bfA_{\bftheta} = \hat\bff\,\big(\, [\widehat{\bfW}_\bftheta]\,\hat\bfA_{\bftheta} \,+\, [\widehat{\bfW}_\bftheta^{\pin}]\,\xx^i \,+\, \widehat{\bfb}_\bftheta\big)\,,
\end{equation}
using the fixed-point iteration presented in Section~\ref{Section4.8}.
\paragraph{(v) Parameter summary and trial-step strategy}
{\color{black}
The quantities introduced in paragraphs~(i)--(iii) must be distinguished from the generative hyperparameter
$\bftheta = (h_1, h_2, \zeta_s, \bfbeta^{\,(1)}_{n_h},\bfbeta^{\,(2)}_{n_h})$. They are construction, normalization, or numerical-control parameters rather than edge-level or weight-level trainable parameters. They comprise the balancing target $c_{\targetp}$, the input soft rescaling $\gamma_{\max}$, the bias soft rescaling $\gamma_{b,\max}$, the nondegeneracy ratio $\rho_{\rm in}$, the minimum input coupling $K_{\min}$, computed from Eq.~\eqref{eq4.23} as a function of $\rho_\psp$, and the optional input-sensitivity scheme $(g_\hi,s_{\min},s_{\max})$. The status of all learned, selected, prescribed, and safeguard quantities is summarized in Section~\ref{Section6.parameterclassification}.
}
The trial step is used to construct an optimal initial point for the projected gradient descent with Adam updates. In this step, for each value of $\bftheta$ in the grid defined for the trial step, $\bfbeta^{\,(1)}_{n_h}$ and $\bfbeta^{\,(2)}_{n_h}$ are chosen either as  $\bfzero_{n_h}$ or as an independent realization of the same normalized Gaussian random vectors. In the latter, $\bfbeta^{\,(1)}_{n_h}$ and $\bfbeta^{\,(2)}_{n_h}$ vary for each grid point.
The optimal values of  the latent-field hyperparameters $(h_1,h_2,\zeta_s)$ are sought via a grid search. In addition, only the most application-dependent hyperparameters are varied, while the structural ones are kept fixed for stability:
\begin{itemize}
  \item \textit{Varied in trial step (application-dependent)}.
          (i) Preserving input sensitivity: $\gamma_{\max}\in[10^{-4},10]$ for the input soft rescaling,
          $\gamma_{b,\max}\in[0.05,0.2]$ for the bias soft rescaling, and $\rho_{\rm in} \in ]0.0001,0.7]$ for the non-degeneracy ratio.
          (ii) Balancing of the free block for numerical contraction: $c_\targetp \in [0.4 , 0.7]$.

  \item \textit{Kept fixed in trial step (stability safeguards)}.
           Minimal input coupling $K_{\min}$ generated through $\rho_\psp$ fixed in $[0.09,0.20]$ (e.g., $\rho_\psp= 0.20$).
           Input sensitivity rescaling from a local linearization: often this rescaling is not necessary, if used the parameters are fixed, $g_\hi\in[0.9,1.5]$, $s_{\min}\in [0.001,0.1]$, and $s_{\max}\in [1.0,2.0]$.
\end{itemize}
This design keeps the trial budget bounded while capturing the application-dependent effects that most impact identifiability and contraction of the fixed-point solver.\\

\noindent{\textit{Remark}}.
The balancing target $c_\targetp$ enters only through the degree-symmetric normalization
and equilibration in Eqs.~\eqref{eq4.39}-\eqref{eq4.40}, which are applied to the free block $[\hat\bfW_\bftheta]$
constructed from the field-conditioned weights in Eq.~\eqref{eq4.26}. For a fixed realization of the sparse
architecture, this step rescales rows and columns so that $\Vert [\hat\bfW_\bftheta]\Vert_F \simeq c_\targetp < 1$,
thereby acting as a global gain on the recurrent dynamics in the nonlinear neural equation~\eqref{eq4.44}
and directly influencing the overall level of statistical fluctuations of $\YY(\xx\, ;\bftheta)$.

In contrast, the parameter $\zeta_s$ operates upstream, at the level of the latent-field-induced
architecture: it shapes the magnitude and heterogeneity of the entries of $[\bfW_\bftheta]$ and hence of
$[\hat\bfW_\bftheta]$ before balancing. The variance of the SNN output is therefore driven jointly
by these two parameters: $\zeta_s$ controls the structure and contrast of the random weights,
while $c_\targetp$ fixes the effective amplification factor of the random dynamics.
In the numerical implementation below, $\zeta_s$ is treated as a trainable hyperparameter,
whereas $c_\targetp$ is kept within a narrow range for stability and identifiability.
\subsection{Solving the nonlinear neural equation using fixed-point iteration}
\label{Section4.8}
For each realization of Eq.~\eqref{eq4.44}, we have to solve an equation in $\bfa\in\RR^{\hat n}$ of the form
\begin{equation*}
\bfa = \hat\bff([\widehat\bfw]\, \bfa + \widehat\bb) \quad , \quad [\widehat\bfw] \in \MM_{\hat n} \quad , \quad \widehat\bb\in\RR^{\hat n}
\quad , \quad \hat\bff =\tanh(\cdot)\, .
\end{equation*}
We use the relaxed update,
\begin{equation*}
\bfa^{(k+1)} = (1-\alpha_\prelax)\,\bfa^{(k)} + \alpha_\prelax\,\hat\bff\big ([\widehat\bfw]\, \bfa^{(k)} + \widehat\bb\big)
\quad , \quad \alpha_\prelax \in ]\, 0\, , 1] \, .
\end{equation*}
Let $\bfa^{(0)}$ be  the solution of the linearized problem $([I_{\hat n}]-[\widehat\bfw])\,\bfa^{(0)}=\widehat\bb$ and uniformly scale it to keep the iterate in a quasi–linear regime, $\Vert \bfa^{(0)}\Vert_\infty \leq a_\ast$ with $a_\ast = 0.35$.
With a relative tolerance $\mathrm{tol}$ (set to $10^{-5}$, for instance), we stop at the first $k$ such that
$\max \{ \mathrm{step}_k \, , \mathrm{res}_k \}  \, \leq \,\mathrm{tol}$, where
\begin{equation*}
\mathrm{step}_k =
\frac{ \Vert\bfa^{(k+1)}-\bfa^{(k)}\Vert } { \max\{\Vert\bfa^{(k+1)}\Vert,\,\Vert\bfa^{(k)}\Vert,\,1\} }
\quad , \quad
\mathrm{res}_k =\frac{ \Vert\bfr^{(k)}\Vert } { \max\{\Vert\bfa^{(k)}\Vert,\,\Vert\hat\bff([\widehat\bfw]\bfa^{(k)}+\widehat\bb)\Vert,\,1\} } \, ,
\end{equation*}
where $\bfr^{(k)} = \hat\bff( [\widehat\bfw]\, \bfa^{(k)} + \widehat\bb) - \bfa^{(k)}$.
Upon termination (including the case where $k$ reaches the iteration cap), we select the consistent activated state
\begin{equation*}
\bfa^\star =  \hat\bff\big([\widehat\bfw]\,\bfa^{(k)}+\widehat\bb\big)
\end{equation*}
for computing the outputs, so as to avoid mixing pre- and post-activation states at exit.
It should be noted that an accelerated convergence method without additional Jacobian or "matrix-vector" products, based on Anderson acceleration~\cite{Anderson1965,Walker2011,Toth2015}, has been tested in this framework and did not yield tangible gains.
\paragraph{Remark on the sufficient contraction condition}
Consider the relaxed fixed-point map
\begin{equation*}
\curT(\bfa) = (1-\alpha_\prelax) \,\bfa + \alpha_\prelax\,\tanh\!\big(\, [\widehat\bfw] \,\bfa + \widehat\bb\big).
\end{equation*}
Since $\tanh$ is $1$-Lipschitz, for all $\bfa$ and $\bfa'$,
\begin{equation*}
\Vert\curT(\bfa)-\curT(\bfa')\Vert
\le \big(1-\alpha_\prelax + \alpha_\prelax\Vert\, [\widehat\bfw]\, \Vert_2\big)\,\Vert\bfa-\bfa'\Vert\, ,
\end{equation*}
where $\Vert\cdot\Vert_2$ is the spectral norm (operator 2-norm).
Hence, if $\Vert \,[\widehat\bfw] \,\Vert_2 < 1$ (as targeted by the degree-symmetric normalization and the two-pass $\ell^2$ equilibration yielding $\Vert\, [\widehat\bfw]\, \Vert_2 \simeq c_{\targetp} < 1$), then $\curT$ is a contraction with ratio $q = 1-\alpha_\prelax(1- \Vert\, [\widehat \bfw]\, \Vert_2)\in(0,1)$ and the iteration converges to the unique fixed point for any initialization.
\subsection{Supervised learning as likelihood identification}
\label{Section4.9}
Let $\bftheta\in\curC_{\bftheta}$. Given the training dataset $\{(\xx^i,\yy^i), i=1,\ldots, n_d\}$ with one observation per input and the realizations $\{\bfy^\ell(\xx^i), \ell=1,\ldots,n_\psim\}$ of $\bfY(\xx^i)$, we have to compute $p^\SNNpp_{\YY\vert\XX}(\yy^i \,\vert \,\xx^i; \bftheta)$ (see Eqs.~\eqref{eq2.8} and \eqref{eq2.12}), and
$p^\SNNpp_{\YY\vert\XX}(\bfy^\ell(\xx^i) \,\vert \,\xx^i; \bftheta)$ (see Eq.~\eqref{eq2.15}).
The intractable probability density function $p^\SNNpp_{\YY\vert\XX}(\cdot \,\vert \,\xx\, ; \bftheta)$ is estimated by Monte Carlo over the random architecture/weights generated at $\bftheta$, using a Gaussian-kernel probability density function estimator \cite{Bowman1997,Gentle2009,Givens2013}. With realizations $\{\yy^\ell(\xx\,;\bftheta),\ell=1,\ldots , n_\psim\}$ and Silverman bandwidth $s_\pSB = ( 4 /( n_\psim \,(2 + n_\ppout) )  )^{ 1/(4+n_\ppout)}$, we construct the estimation $\hat p^\SNNpp_{\YY\vert\XX}(\yy \,\vert \,\xx\,; \bftheta)$ as
\begin{equation}\label{eq4.45}
\hat p^\SNNpp_{\YY\vert\XX}(\yy \,\vert \,\xx\,; \bftheta) = \frac{1}{n_\psim}\sum_{\ell=1}^{n_\psim}
                     \frac{1}{ (\sqrt{2\pi}\, s_\pSB)^{\,n_\ppout}\,\sigma_{\yy,1}(\xx\,;\bftheta)\times\ldots\times\sigma_{\yy,n_\ppout}(\xx\,;\bftheta) } \,
                     \exp \Bigg ( -\frac{1}{2\, s_\pSB^2} \sum_{k=1}^{n_\ppout}
                       \Bigg (\frac{ \yy_k - \yy_k^\ell(\xx\, ;\bftheta) } { \sigma_{\yy,k}(\xx\,;\bftheta)} \Bigg )^2 \Bigg ) \, ,
\end{equation}
where $\sigma_{\yy,k}(\xx\,;\bftheta)$ is the empirical standard deviation estimated from the realizations
$\{\yy^\ell(\xx\,;\bftheta)\}_\ell$.
{\color{black}
The bandwidth factor in Eq.~\eqref{eq4.45} is prescribed by the Silverman rule and is not selected from test performance. Moreover, the reference choice $n_\psim = 100$ is not an unchecked Monte Carlo value: Section~\ref{Section6.6}-(v) reports the complete trial-plus-PGD--Adam calculation for $n_\psim \in \{100,140,200,300\}$. The principal marginal proper-score criterion, CRPS, is evaluated directly from the Monte Carlo output sample through Eq.~\eqref{eq3.8} and does not require the joint KDE. The present paper uses Eq.~\eqref{eq4.45} as the specified numerical density estimator for the likelihood-based terms; it does not claim a general optimality result for high-dimensional KDE.
}
Solving the optimization problem defined by Eqs.~\eqref{eq2.11}-\eqref{eq2.12}
(or Eqs.~\eqref{eq2.7}-\eqref{eq2.8} for $\alpha_{\rm{top}} = 0$) is presented in Section~\ref{Section6.4},  which also addresses the selection of the trade-off parameter in SPINN.
\paragraph{Remark 1. On the positivity for the conditional probability distribution and for its Gaussian-KDE estimation}
For each $\bftheta\in\curC_\bftheta$ and each $\xx\in\RR^{n_\ppin}$, Eq.~\eqref{eq2.6} postulates that the conditional probability distribution $P^{\,\SNNpp}_{\YY\,\vert\,\XX}(d\yy \, \vert \, \xx\, ;\bftheta)$ admits a probability density function
$\yy\mapsto p^{\SNNpp}_{\YY\, \vert\,\XX}(\yy \, \vert \, \xx\, ;\bftheta)$ with respect to $d\yy$.
Let $S^{\SNNpp}_{\xx,\bftheta}$ denote the support of the conditional probability distribution $P^{\,\SNNpp}_{\YY\,\vert\,\XX}(d\yy \, \vert \, \xx\, ;\bftheta)$. Then the conditional probability density function vanishes outside its support,
\begin{equation*}
p^{\SNNpp}_{\YY\mid\XX}(\yy\mid \xx;\bftheta) = 0 \quad \text{for } \, \, \yy\notin S^{\SNNpp}_{\xx,\bftheta} \, .
\end{equation*}
By contrast, the Gaussian-KDE estimator used in Eq.~\eqref{eq4.45} is strictly positive on all of
$\RR^{n_\ppout}$ provided that the Silverman bandwidth $s_\pSB $ and the componentwise empirical scales
$\sigma_{\yy,k}(\xx\,;\bftheta)$, are strictly positive,
\begin{equation*}
\hat p^\SNNpp_{\YY\vert\XX}(\yy \,\vert \,\xx\,; \bftheta) \, > \, 0 \quad \text{for} \, \, \yy\in\RR^{n_\ppout} \, .
\end{equation*}
\paragraph{Remark 2. Scope of the conditional probability density function assumption and role of the Monte Carlo--Gaussian-KDE estimator}
Equation~\eqref{eq2.6} is used in this paper as a modeling assumption for the conditional probability distribution associated with the SNN output. However, the computational method defined in Section~\ref{Section4} does not require an explicit characterization of the corresponding conditional probability density function.
Indeed, for fixed $(\xx,\bftheta)$, all statistical quantities implemented in the present work are evaluated from the Monte Carlo sample $\{\yy^\ell(\xx\,;\bftheta),\ell=1,\ldots , n_\psim\}$
generated by Algorithm~\ref{algorithm1}, through the Gaussian-KDE reconstruction given in Eq.~\eqref{eq4.45}. In particular, the positivity property used in practice is the strict positivity of the Gaussian-KDE estimator (under strictly positive bandwidth/empirical scales), not a global positivity statement for the exact conditional probability density function.
A derivation of Eq.~\eqref{eq2.6} directly from the full random construction of Section~\ref{Section4} would require additional assumptions on the random implicit input-output map (including the architecture-generation and nonlinear fixed-point steps). Such assumptions are mathematically independent of the simulation-based estimation procedure and are therefore left outside the scope of the present paper.
\subsection{Mathematical results related to this random neural network}
\label{Section4.10}
We give two mathematical results related to the random neural network.
The first result is reproduced from \cite{Soize2026}, where its proof and related comments can be found.
The second result is adapted to the present notation and is proved below.
Let $L^2_0(\curS)\subset L^2(\curS)$ denote the Hilbert space of square-integrable functions on $\curS$, with zero mean.
Let $\HH_S$ be the set of self-adjoint, positive-definite, Hilbert-Schmidt operators on $L^2_0(\curS)$, corresponding to covariance operators of second-order, centered, Gaussian real-valued random fields indexed by $\curS$.
%
\begin{proposition}[Expressive stochastic representation of the Gaussian latent field model within the diagonal model]
\label{Proposition3}
For $n_h\ge 1$ and for each $k \in\{1,2\}$, let $\{ \hh^{(k,j)} \}_{j \,\geq \,1} \subset L^2_0(\curS)\cap C(\curS)$ be a Hilbert basis of $L^2_0(\curS)$ and let $\bfbeta^{\,(k)}_{n_h}\in\curC_{\adp ,k}$ be such that Eq.~\eqref{eq4.6} and thus  Eq.~\eqref{eq4.4} hold. Hence, the diagonal tensor fields $[K_{n_h}]=\mathrm{diag} \big(h^{(1)}_{n_h},h^{(2)}_{n_h}\big)$ are uniformly elliptic on $\curS$.
Let $[K_\star] = \mathrm{diag}\big(h^{(1)}_\star,h^{(2)}_\star\big)$ be any uniformly elliptic diagonal tensor field on $\curS$ with $h^{(k)}_\star$ belonging to the $L^\infty(\curS)$-closure of $\mathrm{span}\{\hh^{(k,j)}\}_{j\ge 1}$, and let $C_\star\in\HH_S$ denote the covariance operator of the SPDE solution associated with $[K_\star]$. Then there exist two sequences of parameter vectors $\{\bfbeta^{\,(1)}_{n_h}\}_{n_h}$ and $\{\bfbeta^{\,(2)}_{n_h}\}_{n_h}$ (defining $[K_{n_h}]$ via Eqs.~\eqref{eq4.3} and \eqref{eq4.4}) such that the corresponding covariance operators $\{C_{U_{n_h}}\}_{n_h}\subset\HH_S$ satisfy
\begin{equation*}
\lim_{n_h \to \infty} \Vert \, C_\star - C_{U_{n_h}}  \Vert_{{}_{\HH_S}} = 0.
\end{equation*}
In particular, the model class generated by diagonal tensor fields of the form \eqref{eq4.3} is dense in $\HH_S$ within the diagonal subclass.
\end{proposition}
The training of the hyperparameter $\bftheta$ is performed using projected gradient descent with Adam updates and trial-based initialization~\cite{Kingma2015,Conn2009}, which generates the optimization sequence $\{\bftheta^{\,(n)}\}_{n\geq 1}$.
For each fixed $\bftheta^{\,(n)}$, the sparse random weight is given by  Eq.~\eqref{eq4.27}, where $[M^\psp_{\bftheta^{\,(n)}}]$ is  specified in Eq.~\eqref{eq4.17}, and $[\bfW_{\bftheta^{\,(n)}}]$ from  Eq.~\eqref{eq4.26} is rewritten component-wise as $[\bfW_{\bftheta^{\,(n)}}]_{ij} = [w^g_{\bftheta^{\,(n)}}]_{ij}\, F^{\,(n)}_{ij}$, where $[\bfF^{\,(n)}]$ is a  random symmetric $(N\times N)$ real matrix  with entries defined by
\begin{equation}\label{eq4.46}
F^{\,(n)}_{ij} = \exp\,\left(-\frac{(\Delta S^{\bftheta^{\,(n)}}_{ij})^2}{2(\zeta_s^{(n)})^2\,\sigma_{\bfS_{\bftheta^{\,(n)}}}^2}\right) \quad , \quad
F^{\,(n)}_{ij} = F^{\,(n)}_{ji} \quad , \quad F^{\,(n)}_{ij} \in \,]0 , 1] \,\, a.s.
\end{equation}
In Eq.~\eqref{eq4.46}, $\zeta_s^{(n)}$ denotes the component of $\bftheta^{\,(n)}$ corresponding to  $\zeta_s$ in $\bftheta$
(see Eq.~\eqref{eq4.29}).
Note that it is sufficient to prove the expressive stochastic encoding of the probability distribution of the random weight matrix
$[\bfW_{\bftheta^{\,(n)}}]$. This is because, in the limit case of sparsity, that is, when the  global percentile threshold parameter $\tau_\pprc$, defined in Section~\ref{Section4.3}, is set to $0$, we have $[\bfW^\psp_{\bftheta^{\,(n)}}] = [\bfW_{\bftheta^{\,(n)}}]$.
Let $[\bfF]$ be any symmetric second-order random matrix of size $(N\times N)$ with entries in $]\, 0 ,1]$ a.s.,
such that for each $(i,j)$, $F_{ij}\in L^2(\Omega,\sigma(\Xi_{ij}))$ where $\sigma(\Xi_{ij})$ is the $\sigma$-algebra generated by  $\Xi_{ij}$. Then, we have to demonstrate that the stochastic model defined by Eq.~\eqref{eq4.46} constitutes  an expressive stochastic  encoding of the probability distribution of $[\bfF]$.  In the proposition below, $\Delta S^{\bftheta^{\,(n)}}_{ij}$, defined in
Section~\ref{Section4.4}, is denoted simply by $Y_{ij}^{(n)}$.
%
\begin{proposition}[Expressive stochastic encoding of the probability distribution of the random weight matrix]
\label{Proposition4}
Let $Y_{ij}^{(n)} = U( \bfx^i_{\bftheta^{\,(n)}} ; \bftheta^{\,(n)} )$ $-  U( \bfx^j_{\bftheta^{\,(n)}} ; \bftheta^{\,(n)} )$ be the second-order, centered, Gaussian real-valued random variable (see Section~\ref{Section4.1}) and let $\sigma^{(n)}_{ij}$ be its standard deviation. Therefore, the real-valued random variable $\Xi_{ij} = Y_{ij}^{(n)} / \sigma^{(n)}_{ij}$ is second-order, centered, Gaussian with unit variance, $E\{\Xi_{ij}^2\}=1$,  and the collection $\{\Xi_{ij}\}_{i,j=1}^N$ is jointly dependent.
For each $(i,j)$, the normalized Hermite polynomials $\{\curH_\alpha(\Xi_{ij})\}_{\alpha\geq 0}$ form a complete orthonormal system of $L^2(\Omega,\sigma(\Xi_{ij}))$.
Let $b^{(n)}_{ij}$ and $a^{(n)}_{ij}$ be positive scalars such that
\begin{equation*}
b^{(n)}_{ij} = a^{(n)}_{ij} \, / \, (\zeta_s^{(n)} )^2   > 0 \quad , \quad
a^{(n)}_{ij} =  (\sigma^{(n)}_{ij})^2 \, / \, (2 \,\sigma^2_{\bfS_{\bftheta^{\,(n)}}} )  > 0 \, .
\end{equation*}
Let $\curH_\alpha(\xi)$ denote the normalized Hermite polynomials on $\RR$.
Then each random variable $F^{(n)}_{ij}$ defined by Eq.~\eqref{eq4.46}, is written as $F^{(n)}_{ij} = {\exp}(-b^{(n)}_{ij}\, \Xi_{ij}^2)$, and admits the polynomial chaos expansion,
\begin{equation}\label{eq4.47}
F^{\,(n)}_{ij} = \sum_{\alpha=0}^{\infty} f_{ij,\, \alpha}^{(n)} \,\,\curH_\alpha(\Xi_{ij}) \, ,
\end{equation}
where the odd coefficients vanish and the even coefficients are explicitly given by
\begin{equation}\label{eq4.48}
 f_{ij,\, 2\alpha}^{(n)} = \frac{ \sqrt{ (2\alpha)! } }{\alpha!} \, \frac{ (-b^{\,(n)}_{ij})^\alpha }
                                                         { ( 1+2 b^{\,(n)}_{ij} )^{\,\alpha +1/2} } \quad , \quad
 f_{ij,\, 2\alpha+1}^{(n)} = 0 \, .
\end{equation}
Let $\ell^2_{\rm even}(\NN)$ be the closed subspace of $\ell^2(\NN)$ defined by
\begin{equation}\label{eq4.49}
\ell^2_{\rm even}(\NN) = \{ \bfc = \{c_\alpha \}_{\alpha\in\NN } \in \ell^2(\NN)\,:\, c_{2\alpha+1}=0\, ,\, \forall \alpha\in\NN \} \, .
\end{equation}
If the sequence $b^{\,(n)}_{ij} > 0$ is such that the set of coefficient families
$\{\{f^{(n)}_{ij,\alpha}\}_{\alpha \in \NN}\}_{1\leq i,j\leq N}$ indexed by $n \in \NN$ is tail-dense in the finite product of $\ell^2_{\rm even}(\NN)$, i.e.,
\begin{equation}\label{eq4.50}
\forall \{\bfc_{ij}=\{c_{ij,\alpha}\}_{\alpha\in\NN}\}_{i,j=1}^N \subset \ell^2_{\rm even}(\NN)\, ,\, \bfc_{ij}=\bfc_{ji}\, ,\, \forall \varepsilon > 0\, ,\, \forall n_0\in\NN\, ,\, \exists \, n \geq n_0 \, : \sum_{i,j=1}^N\sum_{\alpha=0}^{\infty} ( f^{(n)}_{ij,\alpha} - c_{ij,\alpha} )^2 < \varepsilon^2 \, ,
\end{equation}
then the stochastic model defined by Eq.~\eqref{eq4.46}
constitutes an expressive stochastic encoding of the probability distribution within the entrywise even subclass $\{[F]:\,[F]=[F]^T\, ,\,F_{ij}\in L^2(\Omega,\sigma(\Xi_{ij})) \, , \, F_{ij}\in\,]0,1] \,\, a.s.\, , \, F_{ij}=g_{ij}(\Xi_{ij}^2) \,\, a.s. \hbox{ for some measurable function } g_{ij} \, , \, \forall (i,j)\}$.
In particular, the sequence  $\{[\bfF^{(n)}]\}_{n\geq 1}$, generated by the stochastic model, is such that
\begin{equation}\label{eq4.51}
\exists\,\{n_q\}_{q\geq 1}\ \text{strictly increasing such that}\ \
\lim_{q\to+\infty}E\{ \, \Vert \, [\bfF^{(n_q)}] - [\bfF] \,\Vert_F^2 \}= 0 \, .
\end{equation}
\end{proposition}
%
\begin{proof}
The proof is developed in three steps.\\

\noindent\textit{Step 1}. Let $b^{\,(n)}_{ij}$ and $\Xi_{ij}$ be defined in Proposition~\ref{Proposition4}.
Then, Eq.~\eqref{eq4.46} can be rewritten as
\begin{equation}\label{eq4.52}
F^{\,(n)}_{ij} = \exp \left ( - b^{(n)}_{ij} \, \Xi_{ij}^2 \right )  \, .
\end{equation}

\noindent\textit{Step 2}. For $\alpha\in\NN$, let $\curH_\alpha(\xi)$ be the normalized Hermite polynomials on $\RR$, defined by the orthonormality relation,
\begin{equation*}
E\{\curH_\alpha (\Xi_{ij})\, \curH_\beta (\Xi_{ij})\} = \int_\RR \curH_\alpha (\xi)\, \curH_\beta (\xi)
 \, \frac{1}{\sqrt{2\pi}} \exp(-\xi^2/2) \, d\xi  = \delta_{\alpha\beta} \quad , \quad \curH_0(\xi) = 1\, .
\end{equation*}
Since $\Xi_{ij}$ is standard Gaussian, and $F^{(n)}_{ij} \in L^2(\Omega,\sigma(\Xi_{ij}))$, the Hermite chaos expansion of $F^{(n)}_{ij}$ is given (see for instance \cite{Cameron1947,Xiu2002,Soize2004a}), by Eq.~\eqref{eq4.47}
where the coefficients, calculated by $f_{ij,\, \alpha}^{(n)} = E\{ F^{\,(n)}_{ij}\,\curH_\alpha(\Xi_{ij})\}$, are given by
Eq.~\eqref{eq4.48}. The vanishing of the odd coefficients follows from the parity of $F^{(n)}_{ij}$ as a function of $\Xi_{ij}$.\\
Since $0< F^{\,(n)}_{ij}\le 1$ a.s. by Eqs.~\eqref{eq4.46} and \eqref{eq4.52}, we have
\begin{equation}\label{eq4.53}
E\{ \, \Vert \, [\bfF^{\,(n)}] \,  \Vert^2_F \, \} = \sum_{i,j=1}^N  E \{ ( F^{\,(n)}_{ij})^2\} \,  < +\infty \, ,
\end{equation}
which shows that $[\bfF^{\,(n)}]$ is a second-order random matrix.
The random matrix $[\bfF^{\,(n)}]$ admits the matrix-valued chaos expansion,
\begin{equation}\label{eq4.54}
[\bfF^{\,(n)}] = \sum_{\alpha=0}^{\infty} [f^{(n)}_\alpha] \circ [\bfcurH_\alpha] \, ,
\end{equation}
where $\circ$ denotes the Hadamard (entrywise) product,  $[f^{(n)}_\alpha]\in\MM_N$ has entries
$[f^{(n)}_\alpha]_{ij} =  f_{ij\, , \, \alpha}^{(n)}$, and $[\bfcurH_\alpha]$ is the random matrix with entries
$[\bfcurH_\alpha]_{ij} = \curH_\alpha(\Xi_{ij})$. Due to Eq.~\eqref{eq4.53}, the  series defined by Eq.~\eqref{eq4.54} converges in $L^2(\Omega, \MM_N)$.\\

\noindent\textit{Step 3}. Let $[\bfF]$ be any target random symmetric matrix in the entrywise even subclass introduced in Proposition~\ref{Proposition4}. For each $(i,j)$, the Hermite expansion of $F_{ij}$ is written as
\begin{equation*}
F_{ij} = \sum_{\alpha=0}^{\infty} c_{ij,\alpha}\,\curH_\alpha(\Xi_{ij}) \quad , \quad
c_{ij,\alpha} = E\{F_{ij}\,\curH_\alpha(\Xi_{ij})\} \, .
\end{equation*}
Since $F_{ij}=g_{ij}(\Xi_{ij}^2)$ a.s., the coefficients satisfy $c_{ij,2\alpha+1}=0$ for all $\alpha\in\NN$, and thus
$\bfc_{ij}=\{c_{ij,\alpha}\}_{\alpha\in\NN}$ belongs to $\ell^2_{\rm even}(\NN)$ defined by Eq.~\eqref{eq4.49}.
Moreover, the symmetry of $[\bfF]$ yields $\bfc_{ij}=\bfc_{ji}$. Applying the tail-density condition defined by Eq.~\eqref{eq4.50}, with $\varepsilon=1/q$ and with the lower-index bound chosen recursively, yields a strictly increasing sequence $\{n_q\}_{q\geq 1}$ such that
\begin{equation*}
\sum_{i,j=1}^N\sum_{\alpha=0}^{\infty} ( f^{(n_q)}_{ij,\alpha} - c_{ij,\alpha} )^2 < \frac{1}{q^2} \, .
\end{equation*}
By Parseval's identity applied entrywise and by summation over the finite set of pairs $(i,j)$, we obtain
\begin{equation*}
E\{ \, \Vert \, [\bfF^{(n_q)}] - [\bfF] \,\Vert_F^2 \}
=
\sum_{i,j=1}^N\sum_{\alpha=0}^{\infty} ( f^{(n_q)}_{ij,\alpha} - c_{ij,\alpha} )^2
< \frac{1}{q^2} \, .
\end{equation*}
Passing to the limit $q\to+\infty$ proves Eq.~\eqref{eq4.51}.
\end{proof}
\section{Illustration  of a physical and abstract setting}
\label{Section5}
In this section, we present the strong formulation of the SBVP in Cartesian coordinates, construct its weak formulation, and study
the existence and uniqueness of a second-order random solution.

\subsection{Geometry definition and Gaussian random field $U_c$ for the SBVP}
\label{Section5.1}
\paragraph{(i) Geometry definition}
We work in $\RR^3$ with a Cartesian coordinate system with origin $O$ and orthonormal basis $\{\bfe_1,\bfe_2,\bfe_3\}$.
An arbitrary point is denoted $\bfx = x_1\bfe_1 + x_2\bfe_2 + x_3\bfe_3$, with coordinates $(x_1,x_2,x_3)$.
We consider the SBVP posed on the surface of an infinitely long circular cylinder of radius $r$ in $\RR^3$, with spatial periodicity defined by a fundamental cell along the axial direction $\bfe_3$ of period $L= 2\pi R$ with $R$ finite.
Let $\mathscr{C}=\{\bfx\in\RR^3 \mid x_1^2 + x_2^2 = r^2\}$ denote the cylinder surface.
The origin $O$ of the Cartesian system is located on the cylinder axis at the left cross-section of the cell, which is the circle
$\Sigma_0 = \{\bfx\in\mathscr{C} \mid x_3 = 0\}$. The right cross-section of the cell is the circle $\Sigma_L = \{\bfx\in\mathscr{C} \mid x_3 = L\}$.
Let $\curS_c$ be the open surface strip (open relative to $\mathscr{C}$) defined by
$\curS_c = \{\bfx\in\mathscr{C} \mid 0 < x_3 < L\}$,
whose (relative) boundary is written as $\partial\curS_c = \Sigma_0\cup \Sigma_L$.
The domain of the cell of the periodic circular cylinder is then $\curS_{c0} = \curS_c\cup \Sigma_0$, whose (relative) boundary is $\Sigma_L$. The (relative) closure of $\curS_c$ is $\thinoverline{\curS}{\!_c} = \curS_c\cup \partial\curS_c$.
Let $\bfn$ be the unit normal to $\mathscr{C}$ (pointing radially outward).
We denote by $ds(\bfx)$ the intrinsic surface measure on $\mathscr{C}$ (and hence on $\curS_c$).
Throughout Section~\ref{Section5}, $\nabla_{\bfx}$ denotes the surface (tangential) gradient on $\mathscr{C}$.
Let $\bfnu$ be the conormal along $\Sigma_0\cup\Sigma_L$ (tangent to $\mathscr{C}$, normal to the boundary curves), with
$\bfnu = -\bfe_3$ on $\Sigma_0$ and $\bfnu = + \bfe_3$ on $\Sigma_L$.
We define the closed cell by
$\curS_{c0} := \curS_c \cup \Sigma_0 \cup \Sigma_L = \thinoverline{\curS}{\!_c}$.
Set $\curS = \curS_{c0}/\!\sim$ with $(x_1,x_2,0)\sim(x_1,x_2,L)$.
\paragraph{(ii) Definition of a Gaussian random field $U_c$}
We define a  Gaussian random field $U_c$ that is used to construct the random coefficient field of the random elliptic operator in the SBVP.
Let $\{ U_c(\bfx),\bfx\in \thinoverline{\curS}{\!_c}\}$ be the second-order, centered,  Gaussian real-valued random field, indexed by $\thinoverline{\curS}{\!_c}$, with almost surely continuous trajectories, normalized so that $E\{U_c(\bfx)^2\} = 1$. This field satisfies $U_c\vert_{\Sigma_L} = U_c\vert_{\Sigma_0}$ almost surely.
Define the cylindrical coordinates $\bfx = (x_1,x_2,x_3) = (r \cos\psi\, , r \sin\psi , x_3)$ with $x_3=R\phi$ to parameterize the cell $\curS_{c0}$ on the cylindrical surface, with $\phi \in [0, 2\pi[$ and $\psi$ in $[0, 2\pi[$ (note that for $\phi\in [0, 2\pi]$, one recovers $\thinoverline{\curS}{\!_c}$). The covariance function $C_{U_c}(\bfx,\bfx') = E\{U_c(\bfx)\,U_c(\bfx')\}$, parameterized in $(\phi,\psi)$, and ensuring the periodicity condition
$C_{U_c}\vert_{\Sigma_L} = C_{U_c}\vert_{\Sigma_0}$, is written as
\begin{equation} \label{eq5.1}
C_{U_c}\big((\phi,\psi),(\phi',\psi')\big) = \exp\bigg(-\frac{\alpha(\phi,\phi')}{\lambda_\phi}\bigg) \times \exp\bigg(-\frac{\beta(\psi,\psi')}{\lambda_\psi}\bigg)\, ,
\end{equation}
where $\lambda_\phi > 0$ and $\lambda_\psi > 0$ are the axial and circumferential correlation lengths, respectively. The functions $\alpha$ and $\beta$ are defined by
\begin{equation} \label{eq5.2}
\alpha(\phi,\phi') = R \times {\min}\big (\, \vert \phi - \phi' \vert \, , 2\pi - \vert \phi - \phi' \vert\,\big) \quad , \quad
\beta(\psi,\psi') = r \times {\min}\big ( \, \vert \psi - \psi' \vert \, , 2\pi - \vert \psi - \psi' \vert\,\big) \, .
\end{equation}
{\color{black}
The covariance function defined by Eqs.~\eqref{eq5.1}-\eqref{eq5.2} is adopted for the numerical illustration. It is part of the specification of the physical stochastic model and is not a hyperparameter of the SNN. More generally, any continuous positive-definite covariance function compatible with the periodic geometry could be used in the same construction. Changing its form or the correlation lengths $\lambda_\phi$ and $\lambda_\psi$ changes the spatial dependence of $U_c$, and therefore the probability laws of $K_c$, of the stochastic solution, and of the observable $\bfY(\xx)$ that SPINN is trained to represent. In an SBVP-consistent setting, such a change defines a different target problem and requires retraining, but does not by itself constitute model misspecification. If the covariance model used in the SBVP alignment term differs from the data-generating covariance model, it contributes to the SBVP--data mismatch and can affect predictive accuracy through the cross-entropy term.
}
\paragraph{(iii) Discrete covariance of $U_c$ on the periodic surface and nodal Gaussian random vector}
Consistently with Section~\ref{Section1}-(i), we use a finite-dimensional representation of the Gaussian field through the shared Gaussian germ, written as $\{U_c(\bfx) = [M_c(\bfx)]\, \bfH\, , \bfx\in \thinoverline{\curS}{\!_c}\}$.
Since the objective here is to construct the nodal random vector $\bfU_c$, we bypass an explicit continuous Karhunen-Lo\`eve construction and directly construct $\bfU_c$ via a principal component analysis of the discrete covariance matrix of $U_c$.
Let $\{\bfx_c^k= (x_{c,1}^k,x_{c,2}^k,x_{c,3}^k) \}_{k=1}^{n_{oc}}$ be the nodes of a structured finite-element mesh of $\thinoverline{\curS}{\!_c}$ (see Section~\ref{Section6.1}), with  periodic pairing of boundary nodes across $\Sigma_0$ and $\Sigma_L$.
It  is assumed that
(1) the number of nodes on $\Sigma_0$ is $n_{\Sigma_0}$, yielding $n_{oc} = n_o +n_{\Sigma_0}$;
(2) the nodes on $\Sigma_0$ are $\{\bfx_c^k\}_{k=1}^{n_{\Sigma_0}}$;
(3) the nodes on $\Sigma_L$ are $\{\bfx_c^k\}_{k=n_o +1}^{n_{oc}}$;
 and
(4) for each $\bfx_c^k\in\Sigma_0$ (with $k=1,\ldots,n_{\Sigma_0}$) there exists a unique $\bfx_c^{k'}\in\Sigma_L$ (with $k'= n_o +1,\ldots, n_o+n_{\Sigma_0}=n_{oc}$) such that $(x_{c,1}^k,x_{c,2}^k,L)$ corresponds to $(x_{c,1}^{k'}, x_{c,2}^{k'},0)$ under the periodic identification.
To guarantee exact periodicity under truncation of the principal component analysis, we first merge each such pair of boundary nodes and assemble the covariance on the quotient mesh, with $(x_1,x_2,0)\sim( x_1,x_2,L)$.
The total number of distinct nodes after merging is thus $n_o$.
Define the discrete covariance matrix
$[C_{U_c}]_{jk}  =  C_{U_c} (\bfx_c^j,\bfx_c^k)$ for $1 \leq j \, , \, k \leq n_o$,
where $C_{U_c}$ is given by Eqs.~\eqref{eq5.1}–\eqref{eq5.2}.
By construction, the matrix $[C_{U_c}]$ is symmetric positive-definite and its rows/columns automatically enforce periodicity.
We now solve the symmetric eigenvalue problem
$[C_{U_c}]\,\bfu^{\,j}  = \mu_j \, \bfu^{\,j}$ with $j=1,\ldots, n_o$,
with eigenvalues ordered $\mu_1\geq \mu_2\geq \ldots \geq \mu_{n_o} > 0$ and eigenvectors $\{\bfu^{\,j}\}$ orthonormal in $\RR^{n_o}$.
Fix a truncation index $m <  n_o$ and define the rectangular matrix
$[\curM_c]  =  [\, \mu_1^{1/2} \bfu^{\,1}\;\; \mu_2^{1/2}\bfu^{\,2}\;\; \ldots \;\; \mu_m^{1/2}\bfu^{\,m}\,\big]\;\in \MM_{n_o,m}$.
Let $\bfH=(H_1,\ldots,H_m)$ be a normalized Gaussian $\RR^m$-valued random variable.
We then define the nodal Gaussian random variable $\bfU_c$ with values in $\RR^{n_o}$ by
\begin{equation}\label{eq5.3}
\bfU_c  = \sum_{j=1}^m \mu_j^{1/2}\,\bfu^{\,j}\,H_j  =  [\curM_c]\,\bfH \quad , \quad E\{\bfH\} = \bfzero_m \quad, \quad E\{\bfH\otimes\bfH\} = [I_m]\, ,
\end{equation}
whose components $U_{c,1},\ldots U_{c,n_o}$ represent the discretized samples of the Gaussian field $\{U_c(\bfx),\bfx\in\curS_{c0}\}$ at the mesh nodes of the quotient mesh. The covariance of $\bfU_c$ is the optimal rank-$m$ approximation (in Frobenius norm) of the full discrete covariance,
$E\{\bfU_c\otimes\bfU_c\} = \sum_{j=1}^m \mu_j \,\bfu^{\,j} \otimes \bfu^{\,j}=: [C^{(m)}_{U_c}]$.
Because the periodic identification was imposed before the eigen-decomposition, all eigenvectors $\bfu^{\,j}$ live in the periodic subspace; hence, the truncated representation~\eqref{eq5.3} preserves exact periodicity of $\bfU_c$ at the boundary nodes for any $m$.
\subsection{Strong formulation of the SBVP in Cartesian coordinates}
\label{Section5.2}
We consider the following elliptic SBVP, expressed in Cartesian components of the surface gradient/divergence operator on $\curS_c$ (this formulation is chosen to prepare for discretization with the finite-element method). Since the operator has constants in its null space under periodic boundary conditions, the strong formulation is written for the mean-zero representative with the deterministic scalar
\[
\lambda_f = \frac{1}{|\curS_c|}\int_{\curS_c} f_{\mathrm{src}}(\bfy)\,ds(\bfy) .
\]
The problem is
\begin{align}
- & \nabla_\bfx\!\cdot\! \big( K_c(\bfx)\, \nabla_\bfx Z(\bfx) \big)  = f_{\mathrm{src}}(\bfx)-\lambda_f
                   \quad \text{in} \quad  \curS_c \, ,                                                    \label{eq5.4}\\
 & Z\vert_{\Sigma_L}  = Z\vert_{\Sigma_0}\, ,                                                              \label{eq5.5}\\
 &  \big (K_c\, \nabla_{\bfx} Z \big ) \! \cdot \bfnu \, \vert_{\Sigma_L} =
     -\,\big (K_c\, \nabla_{\bfx} Z\big)\! \cdot \bfnu \,\vert_{\Sigma_0}\, .                             \label{eq5.6}
\end{align}
Here $\nabla_\bfx$ and $(\nabla_\bfx \cdot)$ denote the surface gradient and divergence on $\curS_c$, respectively.
The positive-valued random field $\{ K_c(\bfx), \bfx\in\thinoverline{\curS}{\!_c}\}$ is scalar and satisfies the periodicity
constraint $K_c\vert_{\Sigma_L} = K_c\vert_{\Sigma_0}$ almost surely. The randomness enters through $K_c(\bfx)= \xi(U_c(\bfx))$.
The random field $\{Z(\bfx),\bfx\in\thinoverline{\curS}{\!_c}\}$ is the $\RR$-valued random state variable indexed by $\thinoverline{\curS}{\!_c}$ to be solved.
For each fixed $\bfx$ in $\thinoverline{\curS}{\!_c}$, the random coefficient is written as
$K_c(\bfx) = \xi(U_c(\bfx))$ where $\xi:\RR \to \RR^+$ is a deterministic mapping ensuring positivity and also that the random field $\{Z(\bfx), \bfx\in\thinoverline{\curS}{\!_c}\}$ is second order, which means that for all $\bfx\in\thinoverline{\curS}{\!_c}$, we have $E\{Z(\bfx)^2\} < +\infty$. Finally, $f_{\mathrm{src}}$ is a deterministic forcing term.
The randomness is entirely in $K_c$ via the random field $U_c$ defined in Section~\ref{Section5.1}.
\paragraph{(i) Model for the deterministic forcing function}
The model of the deterministic forcing function $f_{\mathrm{src}}$, defined on the surface $\curS_c$ of the cylindrical cell, is localized along the axial coordinate $x_3$, and written as,
\begin{equation} \label{eq5.7}
f_{\mathrm{src}}(\bfx) = \frac{a_f}{\sqrt{2\pi}\,\epsilon_f}
\exp\!\left( -\,\frac{(x_3 - L/2)^2}{2\,\epsilon_f^2} \right) \quad , \quad  \forall\, \bfx=(x_1,x_2,x_3)
                                \in\thinoverline{\curS}{\!_c} \, ,
\end{equation}
where $a_f > 0$ is the amplitude and $\epsilon_f > 0$ the axial localization width. This profile is constant along the circumferential direction and varies in $x_3$, reaching its maximum at $x_3 = L/2$ (the center of the cell). Note that, if $\epsilon_f \to 0^+$, then the forcing becomes increasingly localized around the circle of abscissa $x_3 = L/2$ on the cylindrical surface.
It can be seen that $f_{\mathrm{src}}$ is a continuous function on $\thinoverline{\curS}{\!_c}$, and we have
$f_{\mathrm{src}} \vert_{\Sigma_L} = f_{\mathrm{src}} \vert_{\Sigma_0}$.
\paragraph{(ii) Construction of the random field $K_c$}
For this construction, based on the maximum entropy principle under the constraints defined by the available information,
we follow the theoretical probabilistic construction presented in \cite{Soize2006,Soize2017b,Soize2017e,Soize2021d} to obtain a coercive and continuous stochastic elliptic operator.
Let $\VV$ be the Gamma positive-valued random variable whose PDF is written as
\begin{equation} \label{eq5.8}
p_{\VV}(\vv\, ; \delta^{-2} ) =  \ioneone_{\RR^+}(\vv)\,\frac{1}{\Gamma(\delta^{-2}) } \, \vv^{ \frac{1}{\delta^2} -1} \, \exp (-\vv) \, ,
\end{equation}
where $\Gamma(z) = \int_0^{+\infty} t^{\, z-1} \, e^{-t}\, dt$, $\Gamma(z+1) = z\,\Gamma(z)$, and
which is such that its mean is $E\{\VV\} = \delta^{-2}$ and its standard deviation is $\sigma_{\VV} = \delta^{-1}$.
It is assumed that $\delta > 0$.
Let $\vv\mapsto \qq = F_\VV(\vv\,;\delta^{-2})$ be the cumulative distribution function (CDF) of $\VV$, and let $\qq \mapsto \vv= F_\VV^{-1}(\qq\,;\delta^{-2})$ be the inverse function of $F_\VV$. Let $\UU$ be the normalized Gaussian real-valued random variable ($E\{\UU\}=0$ and $E\{\UU^2\}=1$ with PDF $p_\UU(\uu) = (2\pi)^{-1/2} \, \exp(-\uu^2/2)$, and let $\uu\mapsto F(\uu)$ be the CDF of $\UU$. Let $\uu\mapsto h(\uu\, ; \delta^{-2})$ be the function from $\RR$ into $\RR^+$ such that
\begin{equation} \label{eq5.9}
 h(\uu\, ; \delta^{-2}) = F_\VV^{-1}(F(\uu)\,;\delta^{-2}) \quad , \quad \forall\, \uu \in \RR \, .
\end{equation}
Therefore, the random variable $\VV$ can be written as $\VV =  h(\UU\, ; \delta^{-2})$. Using this equation, we construct the random field $\{V_c(\bfx),\bfx\in \thinoverline{\curS}{\!_c}\}$ such that, for each fixed $\bfx$ in $\thinoverline{\curS}{\!_c}$,
\begin{equation} \label{eq5.10}
V_c(\bfx) = h( U_c(\bfx)\, ; \delta^{-2} )\, ,
\end{equation}
where, for all $\bfx$ fixed in $\thinoverline{\curS}{\!_c}$,  we have $E\{U_c(\bfx)\}=0$, $E\{U_c(\bfx)^2\}=1$,
$E\{V_c(\bfx)\}=\delta^{-2}$ and the standard deviation of $V_c(\bfx)$ is $\sigma_{V_c}(\bfx) = \delta^{-1}$.
These quantities are independent of $\bfx$. The positive-valued random field $\{K_c (\bfx) ,\bfx\in\thinoverline{\curS}{\!_c}\}$ is then defined by
\begin{equation} \label{eq5.11}
K_c(\bfx) = (\delta^2\, V_c(\bfx) + \varepsilon) / (1+\varepsilon)\quad , \quad \varepsilon > 0 \quad , \quad \bfx\in\thinoverline{\curS}{\!_c} \, .
\end{equation}
{\color{black}
\paragraph{Correlation transfer through the nonlinear mapping}
For two points $\bfx,\bfx'\in\thinoverline{\curS}{\!_c}$, let $\rho = C_{U_c}(\bfx,\bfx')$, and let $(\UU_1,\UU_2)$ be a standard bivariate Gaussian random vector with correlation $\rho$. Since the transformation from $V_c$ to $K_c$ in Eq.~\eqref{eq5.11} is affine with a positive slope, $\operatorname{Corr}\left(K_c(\bfx),K_c(\bfx')\right) = \operatorname{Corr}\left(V_c(\bfx),V_c(\bfx')\right) =
T_\delta(\rho)$, where $T_\delta(\rho) = \delta^2 E\left\{ h(\UU_1\,;\delta^{-2})\,h(\UU_2\,;\delta^{-2}) \right\} -\delta^{-2}$.
The mapping $h(\cdot\,;\delta^{-2})$ defined by Eq.~\eqref{eq5.9} is strictly increasing. It therefore preserves the Gaussian copula, and hence rank-based dependence, but it does not exactly preserve the Pearson correlation. A converged Gauss--Hermite quadrature of the preceding expectation shows that
\begin{equation*}
\max_{0\, \leq\, \rho\, \leq \, 1}\vert \, T_\delta(\rho) - \rho \, \vert \simeq
8.0\times10^{-4} \, , \,\, 1.25\times10^{-3}\, , \,\,
1.79\times10^{-3}\, , \, \text{and} \, 2.43\times10^{-3}
\end{equation*}
for $\delta=0.12$, $0.15$, $0.18$, and $0.21$, respectively. Consequently, the correlation distortion induced by Eq.~\eqref{eq5.9} is negligible over the ranges $\delta\in[0.12,0.18]$ and $\delta\in[0.15,0.21]$ used in Cases~1 and 2.
}
\paragraph{Second-order moments}
For all $\bfx$ fixed in $\thinoverline{\curS}{\!_c}$, the second-order moments of the random variable $K_c(\bfx)$ are
\begin{equation} \label{eq5.12}
E\{K_c(\bfx)\} = 1 \quad , \quad E\{K_c(\bfx)^2\} = 1 + \delta^2/(1+\varepsilon)^2  < +\infty\quad , \quad
\bfx\in\thinoverline{\curS}{\!_c}\, ,
\end{equation}
which shows that $K_c$ is a second-order random field indexed by $\thinoverline{\curS}{\!_c}$.
The standard deviation of the random variable $K_c(\bfx)$, equal to its coefficient of variation, is
$\sigma_{K_c}(\bfx) = \delta/(1+\varepsilon)$. \\

\noindent\textit{Almost-sure periodicity, continuity, and inequality}. Using Eqs.~\eqref{eq5.9} to \eqref{eq5.11}, and the almost sure continuity of the random field $\{ U_c(\bfx),\bfx\in \thinoverline{\curS}{\!_c}\}$, the positive-valued random field $\{K_c(\bfx),\bfx\in\thinoverline{\curS}{\!_c}\}$:

\noindent (1) is such that $K_c \vert_{\Sigma_L} = K_c\vert_{\Sigma_0}$ almost surely,

\noindent (2) has almost surely continuous trajectories on $\thinoverline{\curS}{\!_c}$,

\noindent (3) satisfies almost surely the following inequality,
\begin{equation} \label{eq5.13}
0 < \, \varepsilon/(1+\varepsilon) \, \leq \, K_c(\bfx) \quad , \quad \forall\, \bfx\in \, \thinoverline{\curS}{\!_c}\, .
\end{equation}
In particular, we set $k_{c,{\min}} := \varepsilon/(1+\varepsilon)$, so that $K_c(\bfx) \ge k_{c,{\min}}$ for all $\bfx \in \thinoverline{\curS}{\!_c}$.
\paragraph{Generator of realizations of random field $K_c$}
Any realization $\{k_c^\ell(\bfx),\bfx\in\thinoverline{\curS}{\!_c}\}$ of the second-order positive-valued random field
$\{K_c(\bfx),\bfx\in\thinoverline{\curS}{\!_c}\}$ is such that, for all $\bfx \in\thinoverline{\curS}{\!_c}$,
$k_c^\ell(\bfx) = (\delta^2\, \vc_c^\ell(\bfx) + \varepsilon) / (1+\varepsilon)$
where $\{ \vc_c^\ell(\bfx),\bfx\in\thinoverline{\curS}{\!_c}\}$ is the realization of the random field
$\{V_c(\bfx),\bfx\in\thinoverline{\curS}{\!_c}\}$, generated from the
realization $\{u_c^\ell(\bfx)\,,\bfx\in\thinoverline{\curS}{\!_c}\}$ of the Gaussian random field
$\{U_c(\bfx),\bfx\in\thinoverline{\curS}{\!_c}\}$, such that, for all $\bfx\in\thinoverline{\curS}{\!_c}$,
$\vc_c^\ell(\bfx) = h( u_c^\ell(\bfx)\, ; \delta^{-2} )$.
\paragraph{(iii) Definition of the SBVP and SNN input vector $\xx$}
The input vector $\xx \in \RR^{n_\ppin}$ of the SBVP encodes deterministic parameters that control both the external forcing and the stochastic properties of the PDE coefficient. In the present model, we set $n_\pin = 2$ and $\xx = (\xx_1, \xx_2)$ with
$\xx_1 = a_f > 0$ is the amplitude of the deterministic forcing function $f_{\mathrm{src}}(\bfx)$ defined in Eq.~\eqref{eq5.7}.
Varying $\xx_1$ changes the magnitude of the external excitation applied to the SBVP.
The parameter $\xx_2 = \delta > 0$ is the constant coefficient of variation of the Gamma-distributed intermediate field $V_c$. The PDF of the random variable $V_c(\bfx)$, for $\bfx$ fixed in $\curS_c$, is defined by Eq.~\eqref{eq5.8} and is used to construct the stochastic coefficient $K_c(\bfx)$, defined by Eq.~\eqref{eq5.11}. This parameter controls the amplitude of the statistical fluctuations of $K_c(\bfx)$ around its unit mean.
The input vector $\xx$ is thus a compact descriptor of the excitation strength and of the stochastic contrast of the medium, both of which influence the SBVP output $\bfY(\xx)$ defined by Eqs.~\eqref{eq2.2}-\eqref{eq2.3} (see paragraph (iv) below) and  hence the SNN output $\YY(\xx\, ;\bftheta)$ defined by Eq.~\eqref{eq2.4}.
\paragraph{(iv) Definition of the observation operator $\curO$ for the SBVP}
From Eqs.~\eqref{eq2.2}-\eqref{eq2.3}, the observation operator $\curO$ is such that
$\bfY(\xx) = \curO(Z(\cdot\, ;\xx))$. This operator $\curO$  extracts task-specific information from the random field $\{Z(\bfx\, ;\xx),\bfx\in\thinoverline{\curS}{\!_c}\}$, such as pointwise values, averages, or spectral coefficients. For instance, for the point evaluation and for prescribed locations $\{\bfx^{(j)}_\obsp\}_{j=1}^{n_\ppout}$ on the domain $\curS_c$, the random observation vector
$\bfY(\xx)$ with values in $\RR^{n_\ppout}$, can be written as,
\begin{equation} \label{eq5.14}
\bfY(\xx) = ( Y_1(\xx), \ldots , Y_{n_\ppout}(\xx)) \quad , \quad Y_{\! j}(\xx) = Z\big(\bfx^{(j)}_\obsp\, ;\xx\big) \,\, , \,\, j=1,\dots,n_\pout \, .
\end{equation}
This corresponds, for instance, to direct measurements at fixed sensors.
\subsection{Weak formulation of the SBVP}
\label{Section5.3}
Let $T:\Sigma_0\to\Sigma_L$ be the periodic identification map defined by $T(x_1,x_2,0)=(x_1,x_2,L)$.
Let $\curV$ be the subspace of  the Sobolev space $\curH^1(\curS_c)$, defined by
\begin{equation} \label{eq5.15}
\curV = \{ \vc\in \curH^1(\curS_c) : \vc\vert_{\Sigma_L}\circ T = \vc\vert_{\Sigma_0} \text{ in } \curH^{1/2}(\Sigma_0) \}\, .
\end{equation}
The weak formulation of the SBVP, defined by Eqs.~\eqref{eq5.4}-\eqref{eq5.6}, is written as:  find the random field $Z$, defined on probability space $(\Omega,\curT,\curP)$, with values in $\curV$ such that, almost surely,
\begin{equation}\label{eq5.16}
B(Z,\vc) = f(\vc), \quad \forall \vc \in \curV,
\end{equation}
where for all $z$ and $\vc$ in $\curV$,
\begin{equation}\label{eq5.17}
B(z,\vc) = \int_{\curS_c} K_c(\bfx)\, \nabla_\bfx z(\bfx)
                                       \cdot \nabla_\bfx \vc(\bfx) \, ds(\bfx) \quad , \quad
                                      f(\vc) = \int_{\curS_c} f_{\mathrm{src}}(\bfx) \, \vc(\bfx) \, \, ds(\bfx) \, .
\end{equation}
For deterministic problems, the variational formulation of the periodic elliptic problem and the corresponding existence and uniqueness result are classical. The functional analytic framework relies on Sobolev spaces and the Lax-Milgram theorem; see \cite{Lions2012,Brezis2011}. The periodic setting and the associated cell problems are standard in the theory of periodic structures and homogenization; see \cite{Bensoussan1978}. Further background on elliptic operators and variational methods can be found in \cite{Dautray2013}. The stochastic elliptic case is addressed, for instance, in \cite{Babuska2005,Frauenfelder2005,Wan2009a,Nouy2014,Soize2021d}.
\subsection{Existence and uniqueness of a second-order random solution}
\label{Section5.4}
Because  constants belong to the null space of the divergence-form operator in \eqref{eq5.4}, we classically introduce the subspace
 $\curV_0$ of $\curV$, defined by
\begin{equation}\label{eq5.18}
\curV_0 =  \{ \vc\in \curV \quad \text{such that} \quad  \int_{\curS_c} \vc(\bfx)\,ds(\bfx) = 0 \}\, .
\end{equation}
On $\curV_0$, we use the seminorm $\Vert \vc\Vert _{1,\curS_c} = \Vert \nabla_\bfx \vc \Vert_{L^2(\curS_c)}$, which is a norm equivalent to $\Vert \vc \Vert_{\curH^1(\curS_c)}$ by the Poincar\'e inequality,
\begin{equation}\label{eq5.19}
\Vert\vc\Vert_{L^2(\curS_c)} \, \leq \,  c_P\, \Vert \nabla_\bfx \vc\Vert_{L^2(\curS_c)} \quad , \quad \forall\, \vc\in \curV_0 \, ,
\end{equation}
for some deterministic constant $c_P > 0$ depending only on $(r,L)$. Consequently, for all $\vc\in\curV_0$,
\begin{equation}\label{eq5.20}
\Vert \nabla_\bfx \vc \Vert_{L^2(\curS_c)} \, \leq \,  \Vert \vc \Vert_{\curH^1(\curS_c)}
\, \leq \, (1+c_P^2)^{1/2} \, \Vert \nabla_\bfx \vc \Vert_{L^2(\curS_c)} \, \leq \, (1+c_P) \, \Vert \nabla_\bfx \vc \Vert_{L^2(\curS_c)}\, ,
\end{equation}
so $\Vert \nabla_\bfx \cdot \Vert_{L^2(\curS_c)}$ defines a norm on $\curV_0$ equivalent to $\Vert \cdot \Vert_{\curH^1(\curS_c)}$.\\

\begin{proposition}[Existence and uniqueness of a second-order random solution] \label{Proposition5}
There exists a unique random field $Z$ with values in $\curV_0$ such that, almost surely,
\begin{equation}\label{eq5.21}
B\big(Z,\vc\big) \, = \, f(\vc) \quad , \quad \forall\, \vc\in\curV_0 \, .
\end{equation}
Moreover, we have almost surely,
\begin{equation}\label{eq5.22}
\Vert \nabla_\bfx Z\Vert_{L^2(\curS_c)}\, \leq \,
                  \frac{c_P}{k_{c,{\min}}}\, \Vert f_{\mathrm{src}}\Vert_{L^2(\curS_c)} \quad , \quad
                  \Vert Z\Vert_{\curH^1(\curS_c)} \, \leq \,
                   \frac{(1+c_P)c_P}{k_{c,{\min}}}\, \Vert f_{\mathrm{src}}\Vert_{L^2(\curS_c)} \, ,
\end{equation}
where $k_{c,{\min}} = \varepsilon / (1+\varepsilon) \, > \, 0$. Consequently,
\begin{equation}\label{eq5.23}
E \left \{\, \Vert Z\Vert_{\curH^1(\curS_c)}^2 \,\right\} \, \leq \,
 \frac{(1+c_P)^2 c_P^2}{k_{c,{\min}}^2}\, \Vert f_{\mathrm{src}}\Vert_{L^2(\curS_c)}^2 \, < \, +\infty \, .
\end{equation}
Since $\curV_0$ is a separable Hilbert space and the random bilinear form $B(\omega\, ;\cdot,\cdot)$ is jointly measurable,
the solution map $\omega\mapsto Z(\cdot\, ;\omega)\in \curV_0$ is $\curT$-measurable. Together with Eq.~\eqref{eq5.23},
this implies that $Z\in L^2(\Omega\, ;\curV_0)$, making it a second-order solution.
\end{proposition}

\begin{proof}
Since $\{K_c(\bfx),\bfx\in\thinoverline{\curS}{\!_c}\}$ has almost surely continuous trajectories on the compact set $\thinoverline{\curS}{\!_c}$, we have
$K_{c,{\max}} = \sup_{\bfx \,\, \in \,\,\thinoverline{\curS}{\!_c}} K_c(\bfx) \, < \, +\infty \,\, \text{a.s.}$,
and taking into account Eq.~\eqref{eq5.13}, we have
$0 \, < \,k_{c,{\min}} \, \leq \, \inf_{\bfx\,\,\in \,\,\thinoverline{\curS}{\!_c}} K_c(\bfx) \,\, \text{a.s.}$
Hence, for all $z$ and $\vc$ in $\curV_0$, we have almost surely,
\begin{align*}
B(\vc ,\vc) & = \int_{\curS_c} K_c(\bfx)\, \Vert\nabla_\bfx \vc\Vert^2 \, ds(\bfx) \, \geq \,
                             k_{c,{\min}}\, \Vert\nabla_\bfx \vc \Vert_{L^2(\curS_c)}^2  \, ,\\
\vert B(z,\vc)\vert &\le \, K_{c,{\max}}\, \Vert \nabla_\bfx z\Vert_{L^2(\curS_c)}\,
                                 \Vert\nabla_\bfx \vc \Vert_{L^2(\curS_c)}\,  .
\end{align*}
Moreover, for $f$ defined in \eqref{eq5.17} and any $\vc\in\curV_0$,
\begin{equation*}
\vert f(\vc)\vert \, \leq \, \Vert f_{\mathrm{src}} \Vert_{L^2(\curS_c)}\,\Vert \vc \Vert_{L^2(\curS_c)}
\, \leq\, c_P\, \Vert f_{\mathrm{src}} \Vert_{L^2(\curS_c)}\, \Vert \nabla_\bfx \vc \Vert_{L^2(\curS_c)} \,.
\end{equation*}
Therefore $f\in \curV_0'$ and $B$ is continuous and coercive almost surely on $(\curV_0,\Vert\cdot\Vert_{1,\curS_c})$.
The Lax-Milgram lemma on $\curV_0$ shows that there exists, almost surely,  a unique random solution in $\curV_0$. The second equation in Eq.~\eqref{eq5.22} yields Eq.~\eqref{eq5.23}, which proves that the unique random solution is second order.\\
\end{proof}
\paragraph{Remarks}
\begin{itemize}
\item[(a)] If Eq.~\eqref{eq5.4} were written with the unprojected right-hand side $f_{\mathrm{src}}$ and tested on the full periodic space $\curV$, a necessary and sufficient compatibility condition would be $\int_{\curS_c} f_{\mathrm{src}}\, ds=0$; in that case solutions would be defined only up to an additive constant. In the present model defined in Eq.~\eqref{eq5.7}, this compatibility condition is not satisfied. The formulation in Eq.~\eqref{eq5.21} on $\curV_0$ is therefore the weak form of the projected problem in Eq.~\eqref{eq5.4}, with the additive constant fixed by the zero-mean constraint.
\item[(b)] Regarding Eq.~\eqref{eq5.14}, point evaluation is not continuous in $\curH^1(\curS_c)$ on a two-dimensional manifold. Without assumptions that ensure a continuous version of $Z$, observations can be modeled as localized averages, which are continuous linear functionals on $\curH^1(\curS_c)$ and remain second-order by Eq.~\eqref{eq5.23}.
\end{itemize}
\subsection{Finite element discretization and stochastic solver}
\label{Section5.5}
The finite-element discretization of the weak formulation stated in Proposition~\ref{Proposition5}, together with the Monte Carlo solution procedure, is detailed in Section~\ref{Section6}. We briefly clarify here how the random vector $\bfU_c$, introduced in Section~\ref{Section5.1}-(iii), enters the discretization.
The surface $\thinoverline{\curS}{\!_c}$ is discretized using three-node flat triangular elements with linear ($P_1$) interpolation functions in an isoparametric finite-element framework. For each element, the contribution to the bilinear form $B(z,\vc)$ (see Eq.~\eqref{eq5.17}) is evaluated by standard $P_1$ interpolation, using the nodal values of the random coefficient field $K_c(\bfx)$ at the three vertices of the element.
Since $K_c(\bfx)$ is obtained from $U_c(\bfx)$ through Eqs.~\eqref{eq5.10}–\eqref{eq5.11}, the element matrices depend only on the values of $U_c$ at these nodes. At the global level, these nodal values constitute the components of the random vector $\bfU_c$ defined in Section~\ref{Section5.1}-(iii).
\section{Numerical application and implementation details for the SNN approach}
\label{Section6}
This section describes in detail the complete setup of the numerical experiment, including the geometry, the finite-element discretizations of the SBVP and SNN domains, the specification of input and output random variables, the construction of training and test datasets under intentional misspecification, and the parameterization of the supervised learning procedure. The mathematical model of the SBVP is  defined in Section~\ref{Section5}, and that of the SNN in Section~\ref{Section4}.

In order to  analyze the role played by the alignment term with respect to the data-fidelity term in Eq.~\eqref{eq2.12}, we define two numerical experiments, denoted \textbf{Case 1} and \textbf{Case 2}. The first one addresses the situation in  which the training and test datasets almost surely satisfy the SBVP, and, therefore, is equivalent to the optimization problem defined by Eq.~\eqref{eq2.7}. For the second one, the training and test datasets do not almost surely satisfy the SBVP. Consequently, the alignment term is taken into account,
and the corresponding optimization problem is defined by Eq.~\eqref{eq2.11}.
\subsection{Parameterization of the SBVP, finite-element mesh of the SBVP domain $\protect\thinoverline{\curS}{\!_c}$}
\label{Section6.1}
\paragraph{(i) Geometry and finite-element mesh}
Using the notation introduced in Section~\ref{Section5.1}, the domain $\thinoverline{\curS}{\!_c} = \curS_c\cup\Sigma_0\cup\Sigma_L$ is defined by
\begin{equation}\label{eq6.1}
\thinoverline{\curS}{\!_c} = \{\bfx\in\RR^3 \, \vert \, x_1^2 + x_2^2 = r^2 \,\, , \,\, 0\leq x_3 \leq L\} \, ,
\end{equation}
with $L= 2\pi R$, $R=2$, and $r=0.7$.
The surface $\thinoverline{\curS}{\!_c}$ is discretized by flat 3-node triangular finite elements with linear shape functions in the isoparametric formulation \cite{Zienkiewicz2005}, yielding the geometrical approximation $\thinoverline{\curS}{\!_{c}}^h$. The nodes $\{\bfx_c^k\}_{k=1}^{n_{oc}}$ are ordered according to the convention of  Section~\ref{Section5.1}-(iii).
The generatrix $0\leq x_3\leq L$ is divided into $n_g= 80$ subdivisions and the circular section of radius $r$ into $n_c= 24$ subdivisions, producing a structured mesh with $n_{oc} = (80+1)\times 24 = 1944$ nodes and $n_\pelem = 3840$ triangular elements (Fig.~\ref{figure1a}).
Let $\{ \varphi^c_k, k=1,\ldots , n_{oc}\}$ denote the associated finite-element basis functions. The finite-element approximation of the Gaussian random field $\{U_c(\bfx),\bfx\in\thinoverline{\curS}{\!_{c}}\}$ is then
\begin{equation}\label{eq6.2}
U_c^h(\bfx) = \sum_{k=1}^{n_{oc}} U_{c,k}\, \varphi^c_k(\bfx) \quad , \quad \forall\, \bfx\in\thinoverline{\curS}{\!_{c}}^h \, ,
\end{equation}
where $U_{c,k} = U_c(\bfx_c^k)$.
Due to periodic identification across $\Sigma_0$ and $\Sigma_L$, paired boundary nodes share the same value almost surely. Accordingly, the reduced nodal vector $\bfU_c\in\RR^{n_o}$ on the quotient mesh (obtained by merging the paired nodes) is given by $\bfU_c=[\curM_c]\bfH$, as defined in Section~\ref{Section5.1}-(iii), where $\bfH$ is of dimension $m=90$. The values at the duplicated end rings in $\thinoverline{\curS}{\!_{c}}^h$ are assigned consistently from the corresponding entries of this reduced nodal vector $\bfU_c = (U_{c,1},\ldots , U_{c,n_o})$.
\begin{figure}[H]
    \centering
    \begin{subfigure}[b]{0.30\textwidth}
        \centering
        \includegraphics[width=\textwidth]{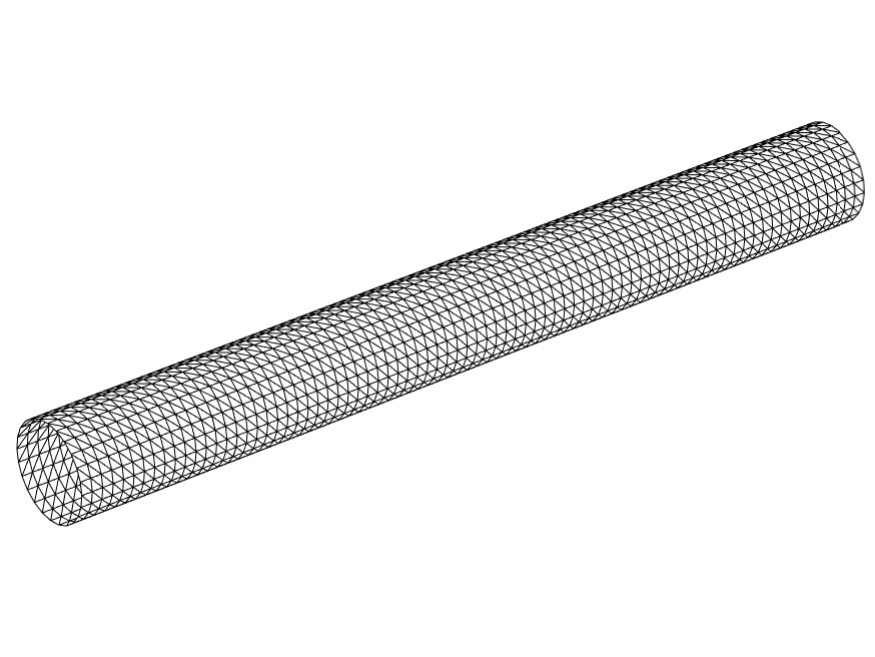}
        \caption{Finite element mesh of $\protect\thinoverline{\curS}{\!_c}$.}
        \label{figure1a}
    \end{subfigure}
    \begin{subfigure}[b]{0.30\textwidth}
        \centering
        \includegraphics[width=\textwidth]{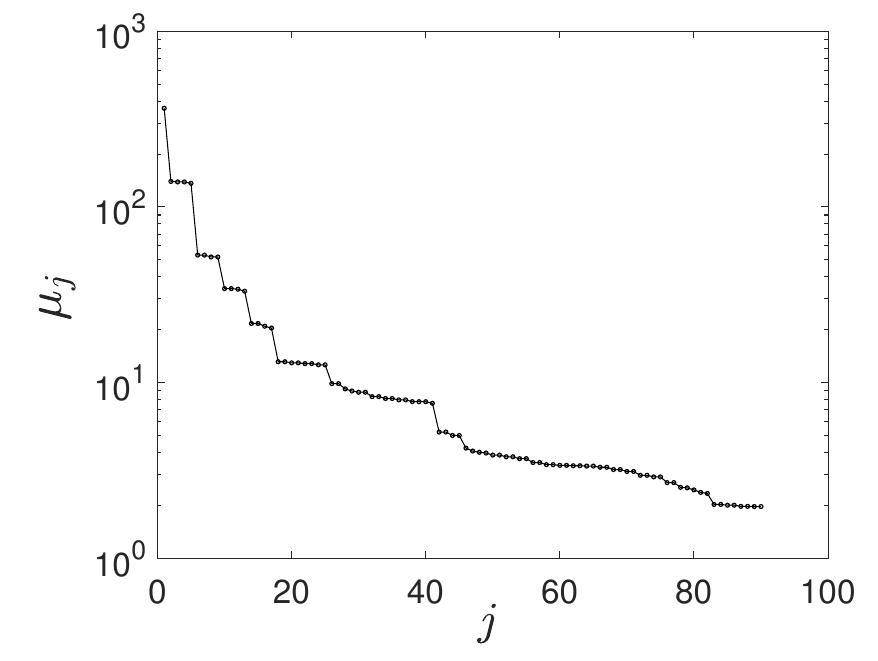}
        \caption{Spectrum of the covariance matrix $[C_{U_c}]$.}
        \label{figure1b}
    \end{subfigure}
    \begin{subfigure}[b]{0.30\textwidth}
        \centering
        \includegraphics[width=\textwidth]{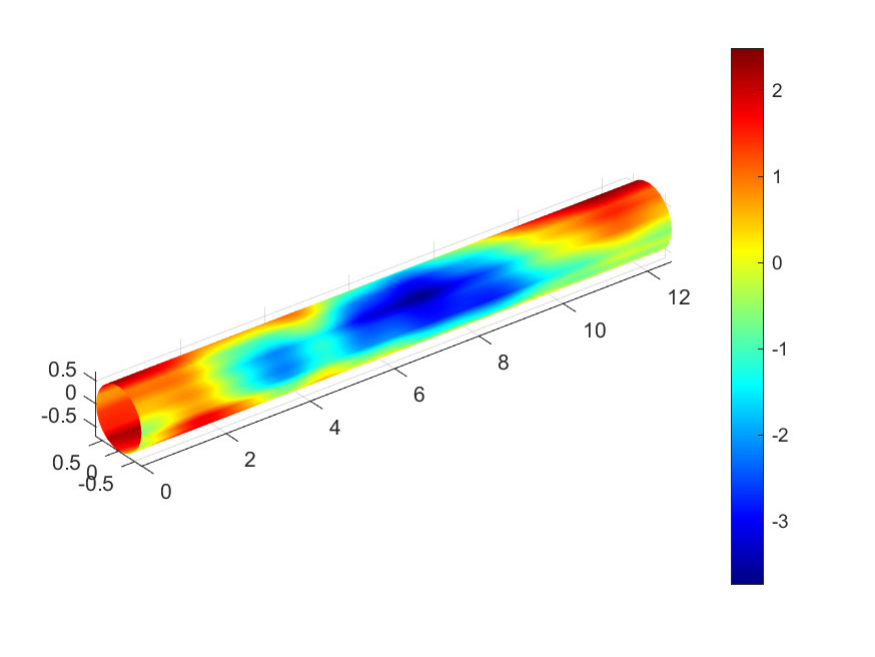}
        \caption{Realization of random field $U_c^h$.}
        \label{figure1c}
    \end{subfigure}
\caption{SBVP: (a) Structured mesh $\protect\thinoverline{\curS}^{\, h}_{c}$ approximating $\protect\thinoverline{\curS}{\!_c}$.
(b) Graph of the eigenvalues $\mu_1,\ldots,\mu_m$  of the covariance matrix $[C_{U_c}]$, plotted on a $\log_{10}$ scale. (c) A realization of the random field $\{U_c^h(\bfx),\bfx\in\protect\thinoverline{\curS}^{\, h}_{c}\} $.}
    \label{figure1}
\end{figure}

\paragraph{(ii) Definition of the covariance structure of $U_c$}
The covariance function $C_{U_c}$ of the random field $U_c$ is defined by Eqs.~\eqref{eq5.1}-\eqref{eq5.2}, with correlation lengths
$\lambda_\phi =\pi R/2 = 3.1416$ and $\lambda_\psi = \pi r/2 = 1.0996$, to be compared to the axial period  $L=12.57$ and the circumferential perimeter $2\pi r = 4.40$ of the cell. The parameter $\varepsilon$ controlling the lower bound of $K_c$ (see Eq.~\eqref{eq5.13}) is fixed to $10^{-6}$. Fig.~\ref{figure1b} displays  the eigenvalues $\mu_1,\ldots,\mu_m$ of the covariance matrix $[C_{U_c}]$ (see Section~\ref{Section5.1}-(iii)). Fig.~\ref{figure1c} shows one realization of the random field $\{U_c^h(\bfx),\bfx\in\protect\thinoverline{\curS}^h_c\}$ defined by Eq.~\eqref{eq6.2}.
\paragraph{(iii) Definition of the input}
The input vector $\xx \in \RR^{n_\ppin}$, with $n_\pin=2$, is defined as follows:
$\xx_1 = a_f$ is the amplitude of the deterministic forcing function $f_{\mathrm{src}}$ given in Eq.~\eqref{eq5.7}, with
$\epsilon_f = \pi R/8 = 0.7854$, and $\xx_2=\delta$, introduced in Section~\ref{Section5.2}(ii)--(iii), is the dimensionless parameter that controls the amplitude of the statistical fluctuations of $K_c(\bfx)$ around its unit mean.
The probability  distribution $P_\XX(d\xx)$ of the $\RR^{n_\ppin}$-valued random variable $\XX = (\XX_1,\XX_2)$
assumes independent components: $\XX_1$ is uniformly distributed on $\epsilon_f \sqrt{2\pi}\times  [0.94\, , 1.092]$, that is on
$[1.85 , 2.15]$, and $\XX_2$ is uniformly distributed on $[0.12\, , 0.18]$.
\paragraph{(iv) Definition of the output}
The output is  the observation vector defined by Eq.~\eqref{eq5.14}, with $n_\ppout = 80$ and nodes $\{\bfx^{(j)}_\obsp\}_{j=1}^{n_\ppout}$ of coordinates $\bfx^{(j)}_\obsp = (r, 0, (j-1) L/80)$, corresponding to equally spaced nodes along the generatrix line $\psi=0$ of the cylinder (note that $x^{(n_\ppout)}_{3,\obsp} < L$).
\paragraph{(v) Discrete enforcement of periodicity and the mean-zero constraint}
Consider the triangular mesh of $\thinoverline{\curS}{\!_c}$ introduced in paragraph (i), with $n_{oc}=(n_g+1)\, n_c$ nodes (including both end rings $\Sigma_0$ and $\Sigma_L$) and $n_\pelem=2\,n_g\,n_c$ elements. Let $n_o=n_g\,n_c$ be the number of reduced nodes after periodic identification.
Let $[\SSeclair]\in\MM_{n_{oc},n_o}$ be the sparse binary identification matrix that maps the reduced vectors to the full vectors by duplicating the $\Sigma_0$ trace on $\Sigma_L$. For any reduced vectors $\bfz_h$ and $\bfv_h$ in $\RR^{n_o}$, set
$\bfz_h^\pfull = [\SSeclair] \, \bfz_h$ and $\bfv_h^\pfull = [\SSeclair]\, \bfv_h$.
Let $[B_h]\in\MM_{n_{oc}}$ be the matrix obtained by assembling the finite-element discretization of the bilinear form in Eq.~\eqref{eq5.17} on the full mesh for a fixed realization of $K_c$, and let $\bff_h\in\RR^{n_{oc}}$ be the corresponding load vector. The reduced (periodic) matrix and vector are
\begin{equation}\label{eq6.3}
[ B_h^\pred ] = [\SSeclair]^T \, [B_h]\, [\SSeclair] \in\MM_{n_o}\quad, \quad \bff_h^{\,\pred} = [\SSeclair]^T \, \bff_h \in\RR^{n_o}\, .
\end{equation}
\paragraph{(vi) Consistency with the weak form and periodic boundary conditions}
By construction, $\bfz_h^\pfull\vert_{\Sigma_L} = \bfz_h^\pfull\vert_{\Sigma_0}$ and likewise for $\bfv_h^\pfull$, which is the nodal counterpart of the trace condition~\eqref{eq5.5}. Moreover $K_c\vert_{\Sigma_L} = K_c\vert_{\Sigma_0}$ (Section~\ref{Section5.2}-(ii)) and the conormals satisfy $\bfnu\vert_{\Sigma_L}=-\bfnu\vert_{\Sigma_0}$ (Section~\ref{Section5.1}-(i)).
\paragraph{(vii)  Element metrics, orientation invariance, and positivity}
For each triangle $\TT = \{ \bfx^1 , \bfx^2 , \bfx^3 \}\subset\thinoverline{\curS}{\!_c}^h$, the P1 shape-function gradients are computed tangentially as
\begin{equation*}
\nabla_{\bfx}{\bfvarphi^c} = [F]\,([F]^T\, [F])^{-1}\,\nabla_{\hat\bfx}{\bfvarphi^c}
\quad , \quad
[F] = [\,\bfx^2-\bfx^1 \, , \,  \bfx^3-\bfx^1\,]\in \MM_{3,2}
\quad , \quad
|\TT|=\tfrac12 \Vert (\bfx^2-\bfx^1)\times(\bfx^3-\bfx^1)\Vert \, .
\end{equation*}
Hence the local contributions are $K_c\vert_{\TT}\, \langle \nabla_\bfx\varphi_i^c \, ,  \nabla_\bfx\varphi_j^c\rangle \,\vert \TT \vert$ (with, e.g., $K_c\vert_\TT$ evaluated at the barycenter), which depend only on the induced metric $[F]^T \, [F]$ and the area $\vert \TT \vert$ and are therefore orientation-invariant. With $K_c(\bfx)\ge k_{c,\min}=\varepsilon/(1+\varepsilon)>0$ a.s. (Eq.~\eqref{eq5.13}), $[B_h^\pred]$ is symmetric positive semidefinite with kernel $\operatorname{span}\{\bfun\}$ on the periodic surface.
\paragraph{(viii)  Mean-zero constraint and uniqueness in $\curV_0$}
To approximate  the continuous space $\curV_0$ defined by Eq.~\eqref{eq5.18},
let $\bfm\in\RR^{n_{oc}}$ be the vector of nodal area weights on the full  cylinder mesh: each $m_i$ equals the sum, over triangles  incident to node $i$, of one third of each triangle’s area. Its  reduced version is
$\bfm^\pred = [\SSeclair]^T\, \bfm\in\RR^{n_o}$. The discrete constraint
$(\bfm^\pred)^T \, \bfz_h = 0$ represents $\int_{\curS_c} z(\bfx) \, ds(\bfx) = 0$ and removes the constant null function. Using a single Lagrange multiplier $\lambda\in\RR$, which discretely approximates the constant $\lambda_f$ introduced in Eq.~\eqref{eq5.4}, the linear system solved at each realization reads
\begin{equation}\label{eq6.4}
\begin{bmatrix}
[B_h^{\,\pred}] & \bfm^\pred\\[1pt]
(\bfm^\pred)^T & 0
\end{bmatrix}
\begin{bmatrix}
\bfz_h\\ \lambda
\end{bmatrix}
=
\begin{bmatrix}
\bff_h^{\,\pred}\\ 0
\end{bmatrix}\, .
\end{equation}
Because $[B_h^\pred]$ is positive definite on the subspace $\{\bfz_h\in\RR^{n_o} : (\bfm^\pred)^T \bfz_h = 0\}$ (see Eqs.~\eqref{eq5.19}, \eqref{eq5.20}), then Eq.~\eqref{eq6.4} has a unique solution $(\bfz_h,\lambda)$.
\paragraph{(ix)  Right-hand side and quadrature}
Either a consistent forcing vector $\bff_h$ assembled by element quadrature for the second term of~\eqref{eq5.17}, or a lumped surface quadrature $\bff_h = [\mathrm{diag}(\bfm)]\, \bff_{\mathrm{src}}^{\,h}$ (with $\bff_{\mathrm{src}}^{\,h}$ the nodal evaluation of $f_{\mathrm{src}}$), is admissible. Both choices preserve periodicity under~\eqref{eq6.4}; the lumped version is first-order accurate and numerically robust for representing Eq.~\eqref{eq5.7}.
\subsection{Definition of the training and test datasets for \textbf{Case 1} and \textbf{Case 2}}
\label{Section6.2}
Two cases are defined to assess SPINN.\\

\noindent\textbf{Case 1}. In this case, the training and test datasets consist of points that almost surely satisfy the SBVP. Consequently, the optimization problem defined by Eq.~\eqref{eq2.11} is equivalent to the one defined by Eq.~\eqref{eq2.7}, implying that $\alpha_{\mathrm{top}} \simeq 1$.\\

\noindent\textbf{Case 2}. This case corresponds to the second regime introduced in Section~\ref{Section1}-(vi), in which the training and test datasets consist of points that do not almost surely satisfy the SBVP. We then construct a statistical surrogate that achieves a compromise between data fidelity and alignment with the SBVP. Since experimental values are not available,  the training and test datasets are generated as realizations of a perturbation of the SBVP, denoted as $\SBVP^\MSP$.\\

\noindent{\textit{SBVP} for generating the physics dataset}. As described in Section~\ref{Section5}, the SBVP is fully specified by the geometric parameters $r$ and $L$; by the probability distribution $P_\XX(d\xx)$ of the input random variables $\XX_1 = a_f$ (controlling the magnitude of the forcing function) and $\XX_2 =\delta$ (controlling the fluctuation level of $K_c$); by the $\RR^{n_\ppout}$-valued output; by the parameter $\epsilon_f$ (controlling the localization of the forcing function); by the parameter $\varepsilon$ (controlling the lower bound of $K_c$); and by the spatial correlation lengths $\lambda_\phi$ and $\lambda_\psi$. All numerical values of these parameters are defined in Section~\ref{Section6.1}-(i)--(iv).\\

\noindent{\textit{SBVP}$\,^\MSP$ for generating the training and test datasets}. These datasets are generated with an intentional misspecification (MSP) of certain SBVP parameters; the resulting stochastic model is  denoted by $SBVP\,^\MSP$. The modified parameters are as follows:

\noindent (a) The support of $P_{\XX_1}$ is shifted from $[1.85 , 2.15]$ to
$[1.50 , 1.80]$, yielding the uniform distribution $P^\MSP_{\XX_1}$ of $a_f^\MSP$.

\noindent (b) The support of $P_{\XX_2}$ is shifted from
$[0.12 , 0.18]$ to $[0.15 , 0.21]$, yielding the uniform distribution $P^\MSP_{\XX_2}$ of $\delta^\MSP$.

\noindent (c) The correlation lengths are perturbed as  $\lambda_\phi^\MSP = 3.5$ and $\lambda_\psi^\MSP = 1.2$.\\

\noindent\textit{Training and test datasets}.
For both cases, the training and test datasets are drawn from the same reference dataset $\curD_D = \{ ( \xx^i , \yy^i) , i=1,\ldots,n_D\}$, consisting of $n_D = n_d +n_\testp$ samples, with $n_d=200$ and $n_\testp=40$. For each input $\xx^i$, a single realization $\yy^i$ of the observation vector is computed.
\subsection{Parameterization of the SNN and finite-element mesh of the SNN domain $\protect{\curS}$}
\label{Section6.3}
\paragraph{(i) Geometry and finite-element mesh of $\curS$}
The SNN manifold $\curS$ is a two-dimensional torus in $\RR^3$ parameterized by $\bfx(\uc,\vc) = (x_1(\uc,\vc),x_2(\uc,\vc),x_3(\uc,\vc))$, with
\begin{align*}
x_1(\uc,\vc) &= (R + r \cos \vc)\cos \uc \, ,\\
x_2(\uc,\vc) &= (R + r \cos \vc )\sin \uc \, ,\\
x_3(\uc,\vc) &= r \sin \vc  \, ,
\end{align*}
for $\uc$ and $\vc$ in $[0, 2\pi]$, where the major radius is $R = 2$ and the minor radius is $r = 0.7$. The surface $\curS$ is meshed using flat 3-node triangles and linear shape functions. There are $80$ subdivisions along the major circle and $24$ subdivisions along the minor circle. This results in a regular mesh with $n_o= 1920$ nodes and $n_\pelem = 3840$three-node triangular elements, forming a structured mesh $\curS_h$ that approximates the compact, boundaryless, multiply connected manifold $\curS$, as shown in Fig.~\ref{figure2a}.
\begin{figure}[H]
    \centering
    \begin{subfigure}[b]{0.30\textwidth}
        \centering
        \includegraphics[width=\textwidth]{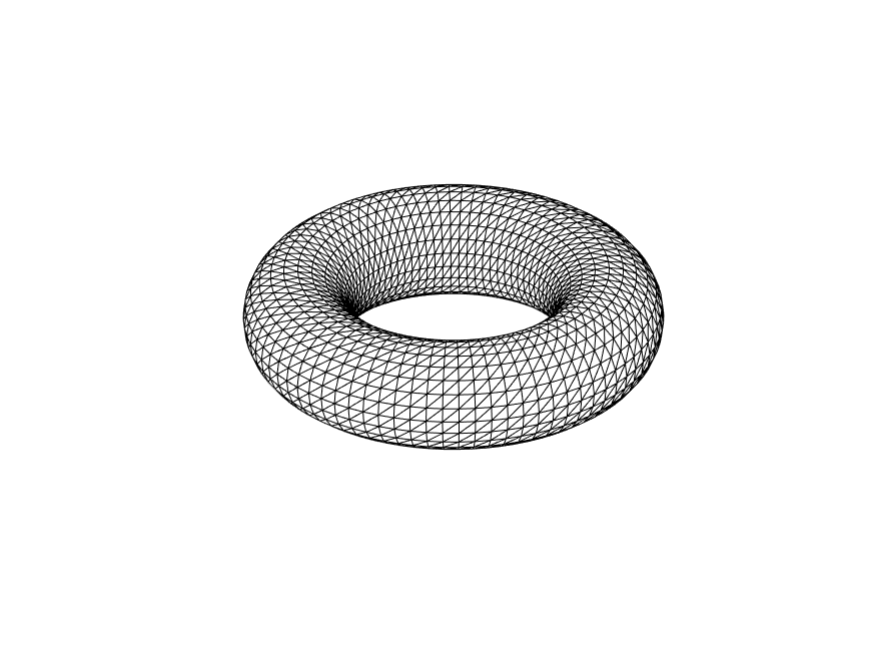}
        \caption{Finite element mesh of $\curS$.}
        \label{figure2a}
    \end{subfigure}
    \begin{subfigure}[b]{0.30\textwidth}
        \centering
        \includegraphics[width=\textwidth]{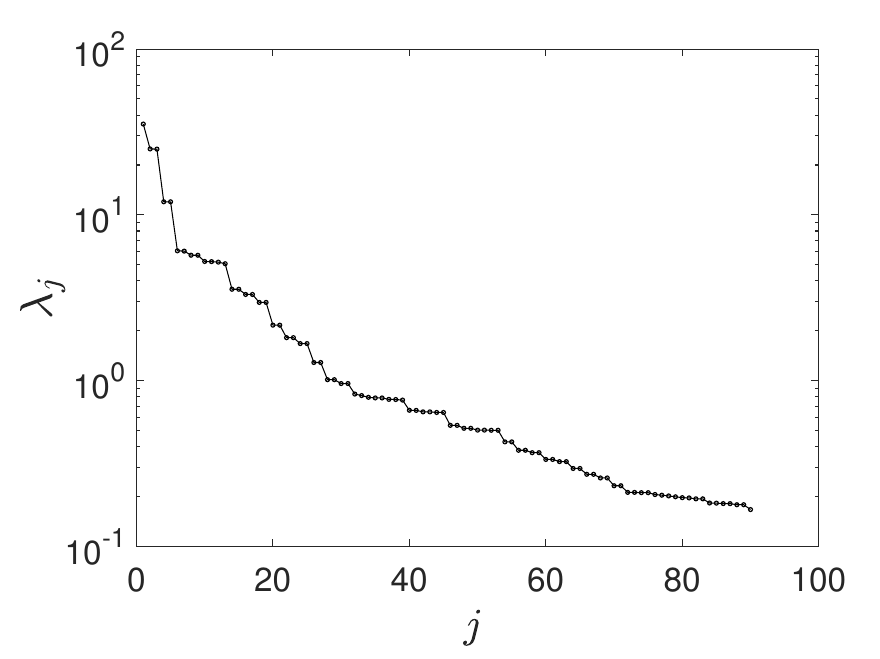}
        \caption{Spectrum of the covariance matrix $[C_U]$.}
        \label{figure2b}
    \end{subfigure}
    \begin{subfigure}[b]{0.30\textwidth}
        \centering
        \includegraphics[width=\textwidth]{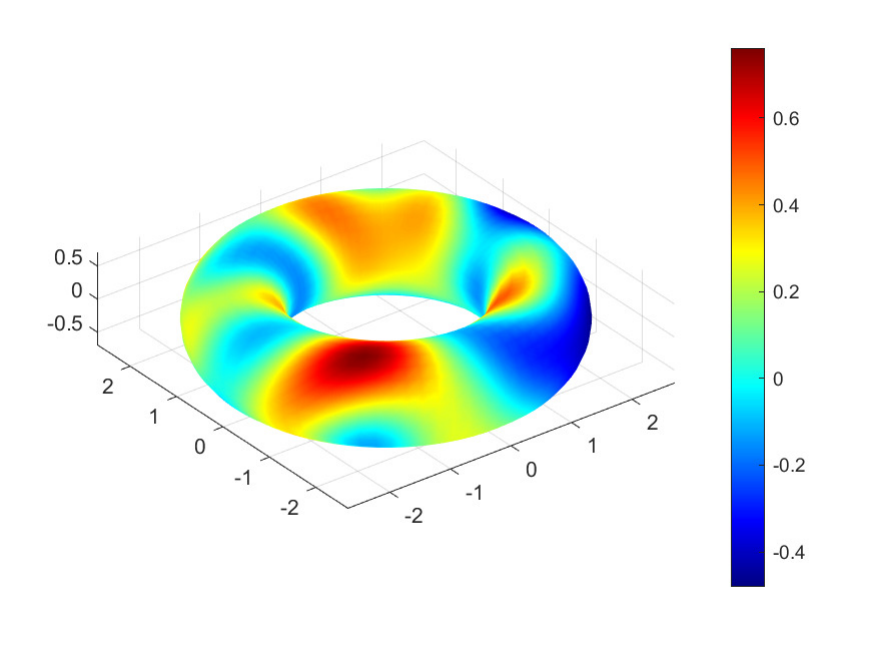}
        \caption{Realization of random field $U^h$.}
        \label{figure2c}
    \end{subfigure}
    \caption{SNN. (a) Structured mesh $\curS_h$ approximating $\curS$.
     (b) $\log_{10}$-scale plot of the dominant eigenvalues $\lambda_1,\ldots,\lambda_m$ of $[C_{\bfU}]$ for prescribed $(\tau_0,h_1,h_2,n_h)$ and one realization of $(\bfbeta^{(1)},\bfbeta^{(2)})$.
    (c) Corresponding realization of $\{U^h_\bftheta(\bfx),\,\bfx\in\curS_h\}$.}
    \label{figure2}
\end{figure}
\noindent\textit{Remark}. For the present numerical example, we used the same geometric parameters and the same number of subdivisions to generate the finite-element meshes of the cylinder and the torus; this was purely for convenience and is not required. In general, the proposed method introduces a manifold $\curS_c$ for the SBVP and a manifold $\curS$ for the SNN, with no \textit{a priori} correspondence
between them.
\paragraph{(ii) Latent Gaussian random field $U(\cdot\,;\bftheta)$ on $\curS$}
Let $\{ \varphi_k, k=1,\ldots , n_{o}\}$ denote the finite-element basis functions. The finite element approximation of the latent Gaussian random field $\{ U(\bfx\,;\bftheta),\bfx\in\curS\}$ is given by
\begin{equation}\label{eq6.5}
U^h(\bfx\,;\bftheta) = \sum_{k=1}^{n_{o}} U_k(\bftheta)\, \varphi_k(\bfx) \quad , \quad \forall\, \bfx\in \curS_h\, ,
\end{equation}
where $U_k(\bftheta) = U(\bfx^k\, ;\bftheta)$, $\{\bfx^k\}_{k=1}^{n_{o}}$ are the mesh nodes, and the nodal vector $\bfU(\bftheta)= (U_1(\bftheta),\ldots ,U_{n_{o}}(\bftheta))= [\curM(\bftheta)]\, \bfH$, where $\bfH$ is of dimension $m=90$.
As an illustration, for $\tau_0=1$, $h_1=0.66$, $h_2= 0.72$, and $n_h=32$, and for a given realization of $\bfbeta^{\,(1)}$ and $\bfbeta^{\,(2)}$ satisfying Eq.~\eqref{eq4.4} (see Section~\ref{Section4.1}),
Fig.~\ref{figure2b} displays  the eigenvalues $\lambda_1,\ldots,\lambda_m$ of the covariance matrix $[C_\bfU]$
(see Eq.~\eqref{eq4.9}), while Fig.~\ref{figure2c} shows a corresponding  realization of the latent random field
$\{ U^h(\bfx\, ;\bftheta), \bfx\in \curS_h \}$ (see Eq.~\eqref{eq6.5}).
\subsection{Random realizations of the SBVP and the SNN}
\label{Section6.4}
As required by Eq.~\eqref{eq2.15}, we also require realizations of the SBVP.
We generate $n_\psim$ realizations  $\{\bfeta^\ell, \ell=1,\ldots , n_\psim\}$ of the $\RR^m$-valued
random variable $\bfH$. Using these realizations, for each $\xx^i$ in the dataset $\curD_D$ we compute $n_\psim$ realizations $\bfy^\ell(\xx^i)\in\RR^{n_\ppout}$ of the SBVP random output $\bfY(\xx^i)$. During the trial step, following Algorithm~\ref{algorithm1}, for each hyperparameter value $\bftheta$ in the grid, we construct one realization of the SNN random architecture associated with a realization of $\bfH$. For this architecture, we then construct $n_\psim$ realizations of the sparse random weight matrix using the samples $\{\bfeta^\ell, \ell=1,\ldots , n_\psim\}$ of $\bfH$. Consequently, we obtain $n_\psim$ realizations $\yy^\ell(\xx^i\,;\bftheta)\in\RR^{n_\ppout}$ of the random output $\YY(\xx^i\,;\bftheta)$.
\subsection{Training algorithm for a fixed trade-off parameter}
\label{Section6.5}

We adapt the two-step training procedure of \cite{Soize2026} to the SPINN objective $\curJ^\SPINNpp(\bftheta\,; \alpha_{\rm top})$
in Eqs.~\eqref{eq2.12}-\eqref{eq2.13}, for a fixed value of the trade-off parameter $\alpha_{\rm top}$. As explained in
Section~\ref{Section2.2}-(ii) and in the remarks of Section~\ref{Section2.2}-(iii), we set
\begin{equation} \label{eq6.6}
\alpha_{\rm top} = \alpha_0 \, ,
\end{equation}
where $\alpha_0$ is the scale-matching value defined in Eq.~\eqref{eq2.19}, and we keep this value fixed throughout the
training and testing phases.
Before the main SPINN training run, we determine $\alpha_0$ as follows. We construct a reference hyperparameter
pair $\bftheta_\pref$ by a short trial-based optimization with $\alpha_{\rm top} = 0$
(data-only training corresponding to the optimization problem defined by Eqs.~\eqref{eq2.7}-\eqref{eq2.8}). Using the same
Monte Carlo samples and KDE estimates as in Eqs.~\eqref{eq2.12}-\eqref{eq2.15}, we compute $\overline L_{\rm data}(\bftheta_\pref)$
and $\overline L_{\rm align}(\bftheta_\pref)$ such that
\begin{equation*}
\overline L_{\rm data}(\bftheta_\pref) = \frac{1}{n_d} \sum_{i=1}^{n_d} -\log p^\SNNpp_{\YY\vert\XX}(\yy^i \,\vert \,\xx^i\, ; \bftheta_{\pref})
\quad , \quad
\overline L_{\rm align}(\bftheta_\pref) = \frac{1}{n_d} \sum_{i=1}^{n_d} \rm{CE}_i(\bftheta_\pref) \, ,
\end{equation*}
and define $\alpha_0$ (see Eq.~\eqref{eq2.19}) by
\begin{equation}\label{eq6.7}
\alpha_0 = \frac{\overline L_{\rm data}(\bftheta_\pref)}{\overline L_{\rm align}(\bftheta_\pref)} \, ,
\end{equation}
This construction requires no additional SBVP solves beyond those already needed to evaluate the SPINN loss.
\paragraph{Step 1: Trial-based model selection at fixed $\alpha_{\rm top} = \alpha_0$}
The trial-based model selection is performed solely on the training objective
$\curJ^\SPINNpp(\bftheta\,; \alpha_0)$ (see Eq.~\eqref{eq2.12}), with no hold-out model-selection split; generalization is
assessed on the test dataset using the diagnostics of Section~\ref{Section3}.
For the fixed value $\alpha_{\rm top} = \alpha_0$:
\begin{enumerate}
\item We construct a grid with $n_\pgrid$ points over
      $(h_1,h_2,\zeta_s,c_\targetp)\in [h_{1,\min},h_{1,\max}] \times [h_{2,\min},h_{2,\max}] \times [\zeta_{s,\min},\zeta_{s,\max}]
      \times [c_{\targetp,\min},c_{\targetp,\max}]$.
      For each grid point $(h_1^j,h_2^j,\zeta_s^j,c_\targetp^j)$, the associated vectors
      $\bfbeta^{\,(1),j}_{n_h}$ and $\bfbeta^{\,(2),j}_{n_h}$ are obtained by projecting
      (using the method presented in Section~\ref{Section4.5}) two independent realizations of a normalized Gaussian
      $\RR^{n_h}$-valued random variable, so that the constraints in Eq.~\eqref{eq4.4} are satisfied.
      This yields the grid-associated points
      $\{\bftheta_\ptrial^j  , j=1,\ldots , n_\pgrid\}$, where
      $\bftheta^j_\ptrial = ( h_1^j,h_2^j, \zeta_s^j, \bfbeta^{\,(1),j}_{n_h}, \bfbeta^{\,(2),j}_{n_h} )$.
       It should be noted that the SNN is not strongly sensitive to the values of $h_1$ and $h_2$ themselves; rather,
      the sensitivity is driven by $h^{(1)}_{n_h}(\bfx)$ and $h^{(2)}_{n_h}(\bfx)$ and is therefore governed by
      $\bfbeta_{n_h}^{\,(1)}$ and $\bfbeta_{n_h}^{\,(2)}$ (see Eq.~\eqref{eq4.3}). By contrast, the SNN exhibits a strong
      sensitivity to $\zeta_s$ and $c_\targetp$. Accordingly, we can identify values $h_{1,\ptrial}^\optp$ and $h_{2,\ptrial}^\optp$ from a few preliminary pretrial calculations and then keep $h_1$ and $h_2$ fixed at these values. The trial step is then
      performed over $(\zeta_s,c_\targetp)\in[c_3,d_3]\times[c_4,d_4]$, which reduces the number of grid points. Under this
      strategy, $\bftheta_\ptrial^j=(h_{1,\ptrial}^\optp,h_{2,\ptrial}^\optp,\zeta_s^j,\bfbeta^{\,(1),j}_{n_h},\bfbeta^{\,(2),j}_{n_h})$.
\item For each $\bftheta_\ptrial^j$, evaluate the SPINN objective
      ${\curJ}^\SPINNpp(\bftheta_\ptrial^{\,j}\, ;\alpha_0)$ on the training dataset, using the Monte Carlo
      approximation in Eqs.~\eqref{eq2.12}-\eqref{eq2.15}.
\item Retain the trial minimizer
\begin{equation}\label{eq6.8}
\bftheta_\ptrial^\optp    = \arg\min_{j\, \in \, \{1,\ldots,n_\ppgrid\}}\, {\curJ}^{\SPINNpp}(\bftheta_\ptrial^{\,j}\, ;\alpha_0)
\end{equation}
and define
\begin{equation}\label{eq6.9}
{\color{black}
J_{\min,\ptrial} = {\curJ}^{\SPINNpp}(\bftheta_\ptrial^\optp\,;\alpha_0) \, .
}
\end{equation}
\end{enumerate}
\paragraph{Step 2: Local refinement for the fixed trade-off parameter $\alpha_0$}
With $\alpha_{\rm{top}}$ fixed to $\alpha_0$:
\begin{enumerate}
\item Initialize $\bftheta_\pinit = \bftheta_\ptrial^\optp$.
\item Freeze the input–output neuron selections to those at the trial optimum (to ensure differentiability),
      and freeze the bias coefficients $\bfbeta$ at $\bfbeta_{\ptrial}^\optp$.
\item Consider the optimization in $\bfbeta^{\,(1)}_{n_h}$ and $\bfbeta^{\,(2)}_{n_h}$, while keeping the other parameters
      $h_1$, $h_2$, $\zeta_s$, and $c_\targetp$ at their trial-optimal values
      $h_{1,\ptrial}^\optp$, $h_{2,\ptrial}^\optp$, $\zeta_{s,\ptrial}^\optp$, and $c_\targetp^\optp$.
\item Minimize the cost function
      $(\bfbeta^{\,(1)}_{n_h},\bfbeta^{\,(2)}_{n_h})\mapsto \curJ^\SPINNpp(\bfbeta^{\,(1)}_{n_h},\bfbeta^{\,(2)}_{n_h}\, ;\alpha_0)$,
            by projected gradient descent with Adam updates, under the admissibility constraints defined in Eq.~\eqref{eq4.4} following the method presented in Section~\ref{Section4.5}. This yields the refined optimum
      $\bftheta^\optp = ( h_{1,\ptrial}^\optp, h_{2,\ptrial}^\optp, \zeta_{s,\ptrial}^\optp ,\bfbeta^{\,(1), \optp}_{n_h}, \bfbeta^{\,(2),\optp}_{n_h})$.
\end{enumerate}

\noindent{\textit{Remarks on the PGD-Adam solver used in Step 2}}.
In Step~2, the projected gradient descent with Adam updates is stabilized by a short, nearly monotone backtracking line search. After each Adam update proposes an increment, the resulting candidate is first projected onto the admissible box constraints for the two PGD coefficient blocks, and the training objective is evaluated at this projected iterate. If the objective does not increase, the projected full step is accepted.
When the projected full step increases the objective, the algorithm keeps the same Adam search direction and reduces the step length geometrically (up to a small, fixed number of trials, using repeated halving), while reapplying the box projection at every trial and reevaluating the objective. The first trial that restores a strict decrease is accepted, and the best candidate encountered during the backtracking loop is tracked throughout. If no non-increasing candidate is found, an additional safeguard prevents stagnation: the best candidate is accepted only if its objective increase remains below a very small relative tolerance. Otherwise, the Adam update is rejected and the iterate is left unchanged.
This procedure leaves the Adam moment recursion unchanged, while providing a lightweight globalization mechanism that enforces feasibility by projection and promotes a largely monotone decrease of the Step~2 training objective.
{\color{black}
\subsection{Status of the parameters and numerical controls}
\label{Section6.parameterclassification}
The quantities entering the SPINN implementation have distinct roles and should not all be interpreted as learned hyperparameters.\\

\noindent\textit{Quantities identified or selected from the training objective.}
The generative SNN hyperparameter is $\bftheta = (h_1, h_2, \zeta_s,$ $\bfbeta^{\,(1)}_{n_h}, \bfbeta^{\,(2)}_{n_h})$.
The values $h_1$ and $h_2$ are identified in the pretrial analysis, $\zeta_s$ and $c_{\targetp}$ are selected in the trial step, and the two coefficient vectors are subsequently refined by PGD--Adam. The bias coefficients $\bfbeta$ are obtained from the least-squares construction, whereas $\alpha_0$ is computed from the scale-matching relation in Eq.~\eqref{eq6.7}; neither is manually tuned on the test dataset.\\

\noindent\textit{Prescribed model-order and resolution parameters.}
The quantities $N$, $n_h$, $m$, $n_b$, $n_\psim$, the finite-element meshes, the number of Poisson candidate points, and the sparsification settings define the stochastic hypothesis class, its numerical resolution, and its computational budget.
They are not continuous trainable hyperparameters. The principal resolution quantities $n_h$, $n_\psim$, and $N$ are examined in the convergence study, while $m$ is selected from spectral-energy and Poisson-density stabilization.\\

\noindent\textit{Rule-based density-estimation quantities.}
The KDE bandwidth factor is prescribed by the Silverman rule in Eq.~\eqref{eq4.45}, and the componentwise scales are empirical standard deviations computed from the generated ensemble. No density-estimation parameter is selected using test performance.\\

\noindent\textit{Numerical safeguards.}
The input and bias rescalings, the minimum input-connectivity condition, the contraction control, the fixed-point relaxation and tolerance, the admissible-set projection, and the Adam backtracking procedure are introduced to ensure nondegeneracy, feasibility, and numerical convergence. They are not trainable architectural parameters.\\

\noindent The Gaussian and Poisson germs are Monte Carlo random variables used to generate architecture and weight realizations. They are not manually selected parameters. The values adopted for the reference configuration are specified in the following subsection.
}
\subsection{Definition and  detailed analysis of a reference configuration based on \textbf{Case 1}}
\label{Section6.6}
\paragraph{(i) Definition of  a reference configuration and SNN parameters}.
We consider the Case~1 configuration, for which $n_h=32$, $N=200$, and $n_\psim = 100$ are fixed.
We then specify the values of all remaining parameters for this reference configuration, for which a complete analysis is presented. In addition, this reference configuration is used  to carry out  a convergence analysis with respect to  $n_h$, $n_\psim$, and $N$.
Since $n_\pin$ and $n_\pout$ are fixed by the physical problem under consideration, varying the total number of neurons $N$ is equivalent to varying the number of internal neurons $n_\pint = N - n_\pin - n_\pout$.
These parameters are defined as follows:
\begin{enumerate}\itemsep0.2em
\item[(1)] The parameter $\tau_0$ in Eq.~\eqref{eq4.1} is set to $1$, and the parameters $h_1$, $h_2$, $\bfbeta^{\,(1)}_{n_h}$, and $\bfbeta^{\,(2)}_{n_h}$ (components of $\bftheta$), controlling the latent Gaussian random field $U_\bftheta$, satisfy the constraints defined in  Eq.~\eqref{eq4.29} with $n_h =32$.
\item[(2)] As noted above, the dimension $m$ of the normalized Gaussian random vector $\bfH$ is set to  $90$. For the optimal value of the hyperparameter $\bftheta_{\ptrial}^\optp$ obtained in the trial step, this choice
    corresponds to a relative spectral energy exceeding $91\%$ and yields stabilization of the Poisson density for $m \geq 90$.
\item[(3)] The total number of neurons, $N=200$, is decomposed into $n_\pin = 2$ input neurons and $n_\pout = 80$ output neurons; therefore, $n_\pint =  N - n_\pin - n_\pout = 118$ hidden neurons. To compute realizations of the random neuron locations,
    $2n_\pelem = 7680$ candidate points are first generated on the surface mesh $\curS_h$. These points form the support on which we construct the Poisson point process with intensity density $\Lambda_\bftheta(\bfx)$, as defined in Section~\ref{Section4.2}.
\item[(4)] The graph sparsification is not specified heuristically. It results from percentile-based diffusion masking, modulated by geodesic distances and field-affinity scores, as presented in Section~\ref{Section4.3}. The binary sparsity mask for the random neural graph topology is constructed using a global percentile threshold $\tau_\pprc = 60\%$ (see Eq.~\eqref{eq4.15}). With $r_\psp=1.8$, the softer input-aware percentile is set to $\tau_{\pprc,in} = 28\%$ (see Eq.~\eqref{eq4.22}). This yields a  total of $8067$ undirected edges. Setting $\rho_\psp=0.2$,  the minimum coefficient required to prevent collapse of the input-to-free coupling is
    $K_{\min} \approx 24$ (see Eq.~\eqref{eq4.23}).
\item[(5)] The parameters  involved in solving the random nonlinear neural equation, introduced in Section~\ref{Section4.7}, are fixed in the PGD-Adam step at their optimal values obtained in the trial step.
    The input-block and free-bias scaling (see Section~\ref{Section4.7}-(i)) is applied.
    Before performing the trial step, a preanalysis is carried out to estimate the values of $\rho_\pin$, $\gamma_{\max}$, and $\gamma_{b,\max}$; these parameters are then fixed in the trial step with $\rho_\pin = 0.0006$, $\gamma_{\max} = 1.8313$, and $\gamma_{b,\max} = 0.1615$.  The balancing of the free block for numerical contraction (see Section~\ref{Section4.7}-(ii))  is also applied. The optimal value of the parameter $c_\targetp$ is estimated in the trial step simultaneously with the trial-optimal value of the hyperparameter $\bftheta = ( h_1,h_2,\zeta_s,\bfbeta_{n_h}^{\,(1)}, \bfbeta_{n_h}^{\,(2)})$.
    The input sensitivity rescaling from a local linearization (see Section~\ref{Section4.7}-(iii)) was not necessary and therefore is not applied. The nonlinear neural equation is solved as explained in Section~\ref{Section4.8} with relaxation factor  $\alpha_\prelax = 0.8$ and a relative tolerance $\mathrm{tol} = 10^{-5}$.
\end{enumerate}

\paragraph{(ii) Illustration of a realization of the latent random field, the neuron locations, and the connectivity}.
Fig.~\ref{figure3a} displays one realization of the random neuron locations on $\curS_h$, superimposed on the associated realization of the latent random field $\{U(\bfx\,;\bftheta),\bfx\in \curS_h\}$, as defined in  Eq.~\eqref{eq6.5}.
Fig.~\ref{figure3b} shows the realization of the random sparse connectivity on $\curS_h$, associated with the realization of the latent random field $U(\cdot\,;\bftheta)$ and the corresponding realization of random neuron locations shown in  Fig.~\ref{figure3a}. The sparse connectivity corresponds to a realization of the random edge-mask matrix defined by Eq.~\eqref{eq4.14}, with edges drawn along the geodesics defined by Eq.~\eqref{eq4.13} used to construct the matrix. The geodesic paths between neurons shown in Fig.~\ref{figure3b} are not perfectly represented due to the triangulation of the manifold and the fact that the neurons do not coincide with the mesh nodes of the surface $\curS_h$. This is a visualization artifact; {\color{black} however, the connectivity generation is correct}.

\begin{figure}[H]
    \centering
    \begin{subfigure}[b]{0.49\textwidth}
        \centering
        \includegraphics[width=\textwidth]{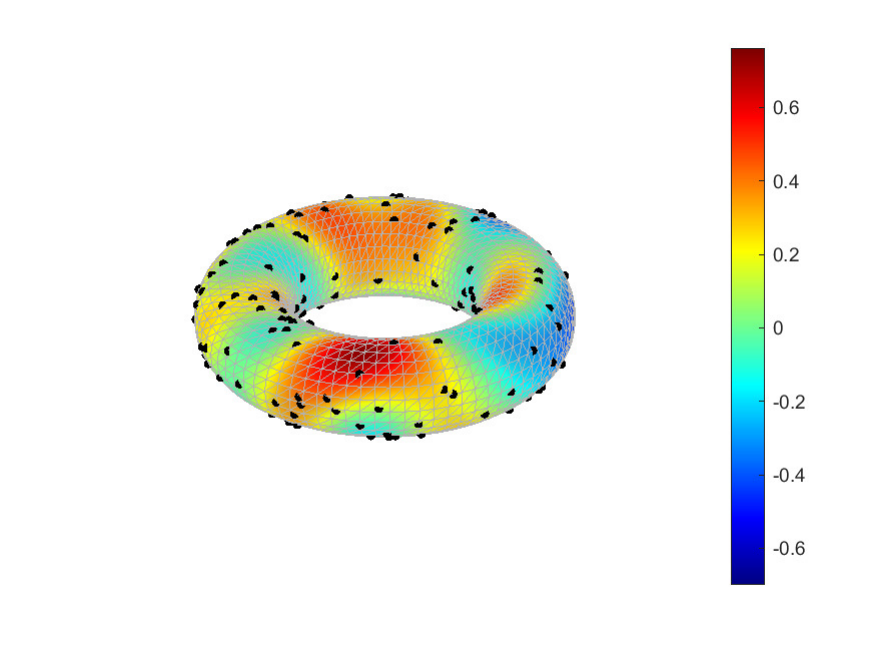}
        \caption{Realization of the random neuron locations on $\curS_h$.}
        \label{figure3a}
    \end{subfigure}
    \begin{subfigure}[b]{0.49\textwidth}
        \centering
        \includegraphics[width=\textwidth]{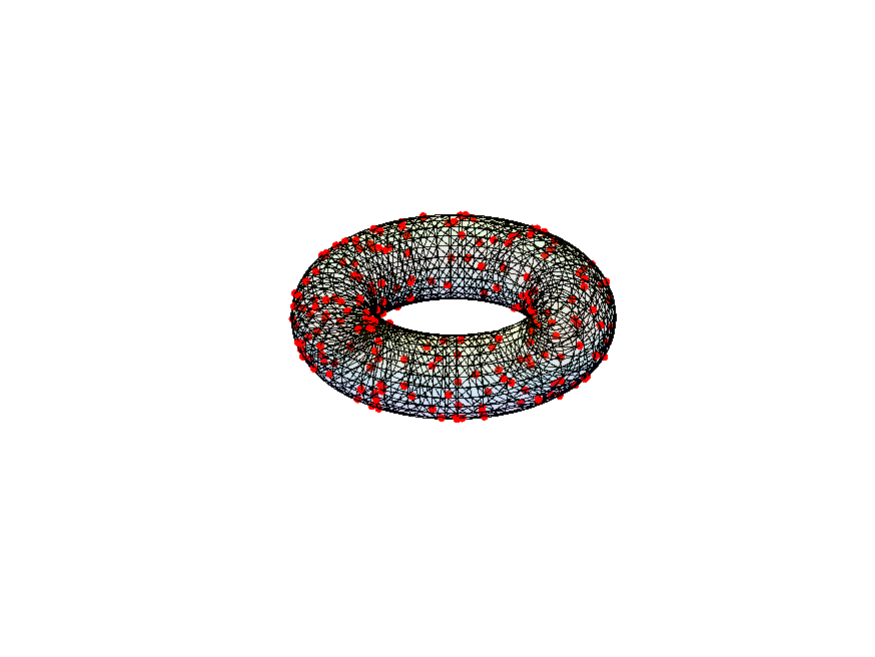}
        \caption{Realization of the intensity map corresponding to figure (a).}
        \label{figure3b}
    \end{subfigure}
    \caption{SNN Case~1: (a) realization of the random neuron locations on $\curS_h$, superimposed on the associated realization of the latent random field. (b) Realization of the sparse connectivity corresponding to the edge-mask matrix, but drawn along the geodesics.}
    \label{figure3}
\end{figure}
\paragraph{(iii) Trial-based global optimization problem}.
For the grid, we fix $h_{1,\ptrial}^\optp = 0.66$ and  $h_{2,\ptrial}^\optp= 0.72$, and we consider $\zeta_s \in [0.03, 0.04]$
and $c_\targetp \in [0.5,0.7]$. The grid  contains $18\times 18 = 324$ points $\bftheta^j$. These bounds were obtained by successive refinements: we partitioned a larger domain into subdomains, explored them sequentially, and solved the optimization problem on each subdomain using the trial-based global optimizer. Consequently, the resulting trial-based solution is already a good approximation   of a high-quality solution.
The value of $\alpha_0$ (see Eqs.~\eqref{eq6.6}--\eqref{eq6.7}) is $\alpha_0 = 1.0458$, which is close to $1$ and therefore consistent with Case~1, as explained in Case~1 of Section~\ref{Section6.2}. This deviation from $1$ is due to the fact that $n_\psim=100$ is relatively small.
Fig.~\ref{figure4} shows the graphs of $j\mapsto \curJ^\SNNpp(\bftheta_\ptrial^j)$ (see Eq.~\eqref{eq2.8}),
$j\mapsto \widehat{\rm{CE}}(\bftheta_\ptrial^j) = (1/n_d)\sum_{i= 1}^{n_d} \widehat{\rm{CE}}_i((\bftheta_\ptrial^j))$ (see Eq.~\eqref{eq2.15}),
and $j\mapsto \curJ^\SPINNpp(\bftheta_\ptrial^j\, ; \alpha_0)$ (see Eq.~\eqref{eq2.12}) over the grid.
\begin{figure}[H]
    \centering
    \begin{subfigure}[b]{0.30\textwidth}
        \centering
        \includegraphics[width=\textwidth]{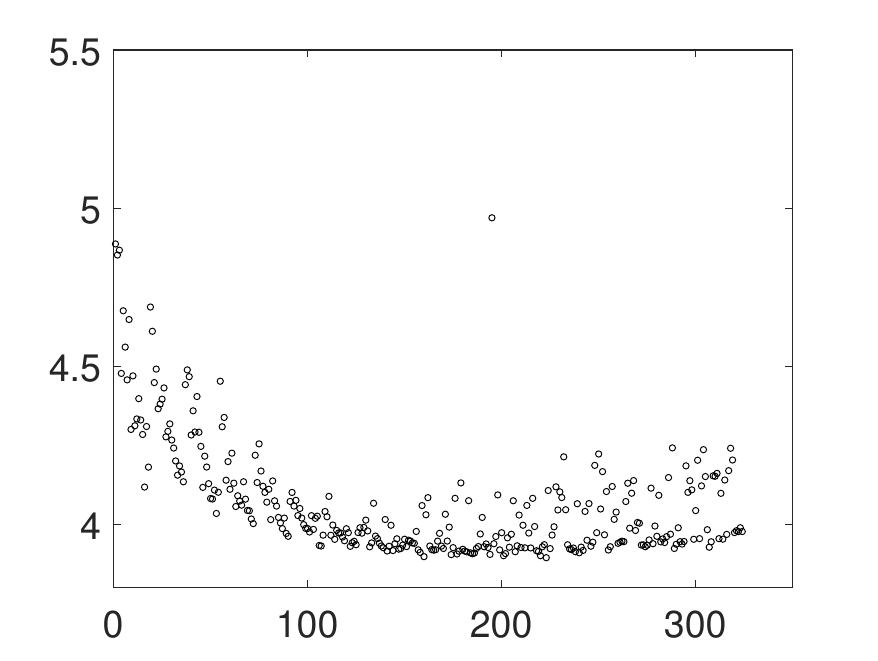}
        \caption{Graph of $j\mapsto \curJ^\SNNpp(\bftheta_\pptrial^j)$.}
        \label{figure4a}
    \end{subfigure}
    \begin{subfigure}[b]{0.30\textwidth}
        \centering
        \includegraphics[width=\textwidth]{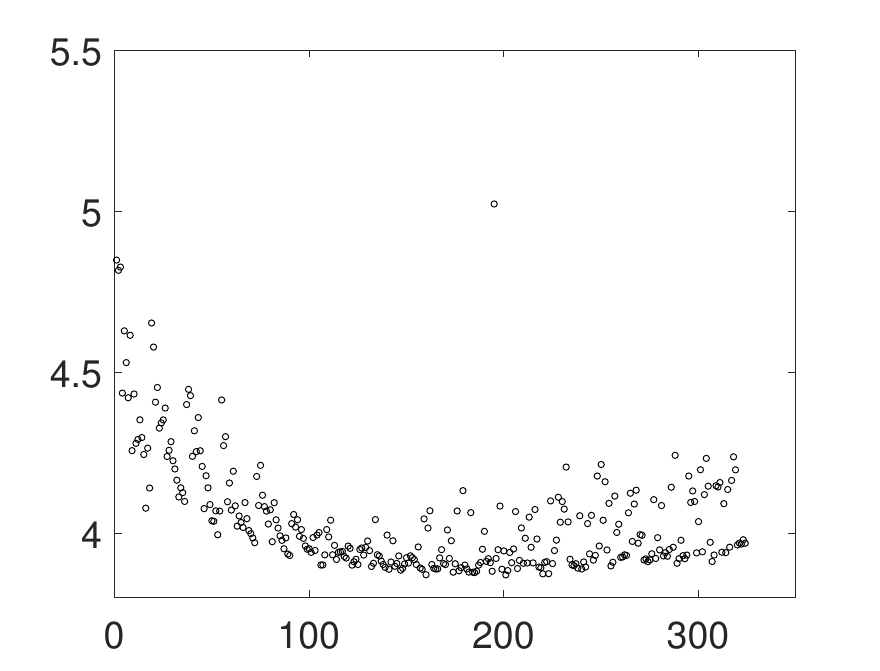}
        \caption{Graph of $j\mapsto \widehat{\rm{CE}}(\bftheta_\pptrial^j)$.}
        \label{figure4b}
    \end{subfigure}
    \begin{subfigure}[b]{0.30\textwidth}
        \centering
        \includegraphics[width=\textwidth]{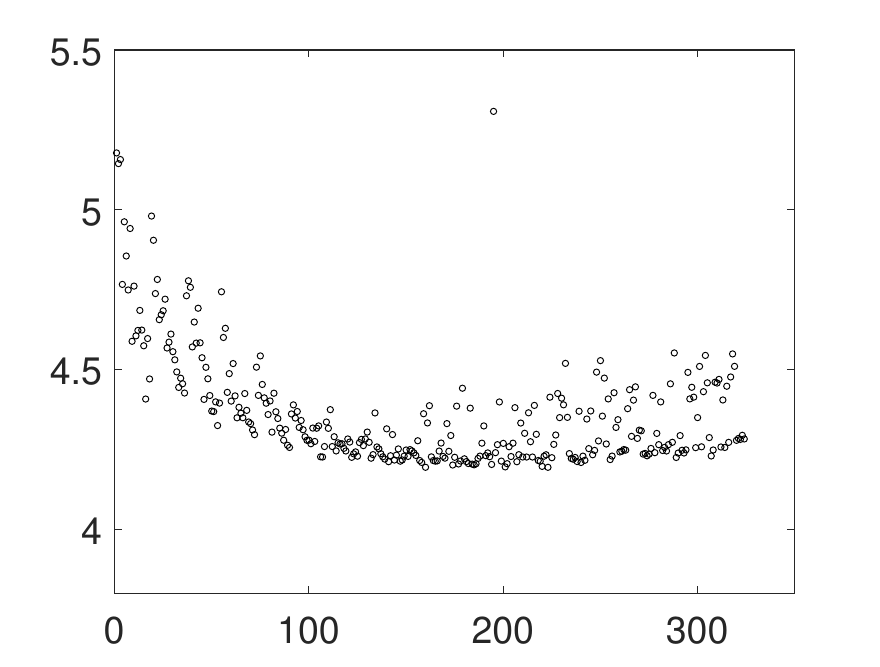}
        \caption{Graph $j\mapsto \curJ^\SPINNpp(\bftheta_\pptrial^j\, ; \alpha_0)$.}
        \label{figure4c}
    \end{subfigure}
    \caption{Case~1: $\log_{10}$-scale plots of the graphs of functions $\curJ^\SNNpp$, $\widehat{\rm{CE}}$, and
    $\curJ^\SPINNpp(\cdot, ; \alpha_0)$ on the grid points $j\in\{1,\ldots, 324\}$.}
    \label{figure4}
\end{figure}
\noindent Solving Eq.~\eqref{eq6.8} yields $\zeta_s^\optp = 0.03353$ and $c_\targetp^\optp = 0.64118$. The graphs of functions $k\mapsto \beta_{k,n_h,\ptrial}^{\, (1),\optp}$ and  $k\mapsto \beta_{k,n_h,\ptrial}^{\, (2),\optp}$ from $\{1,\ldots , n_h\}$ into $\RR$, and of $i\mapsto \beta_{i,\ptrial}^\optp$ from $\{1,\ldots , n_b\}$  into $\RR$ (with $n_b=80$), are shown in Fig.~\ref{figure4}. We have $J_{\min ,\ptrial}  = {\curJ}^{\SPINNpp}(\bftheta_\ptrial^\optp\, ;\alpha_0) = 1.56787\times 10^{4}$.
\begin{figure}[H]
    \centering
    \begin{subfigure}[b]{0.30\textwidth}
        \centering
        \includegraphics[width=\textwidth]{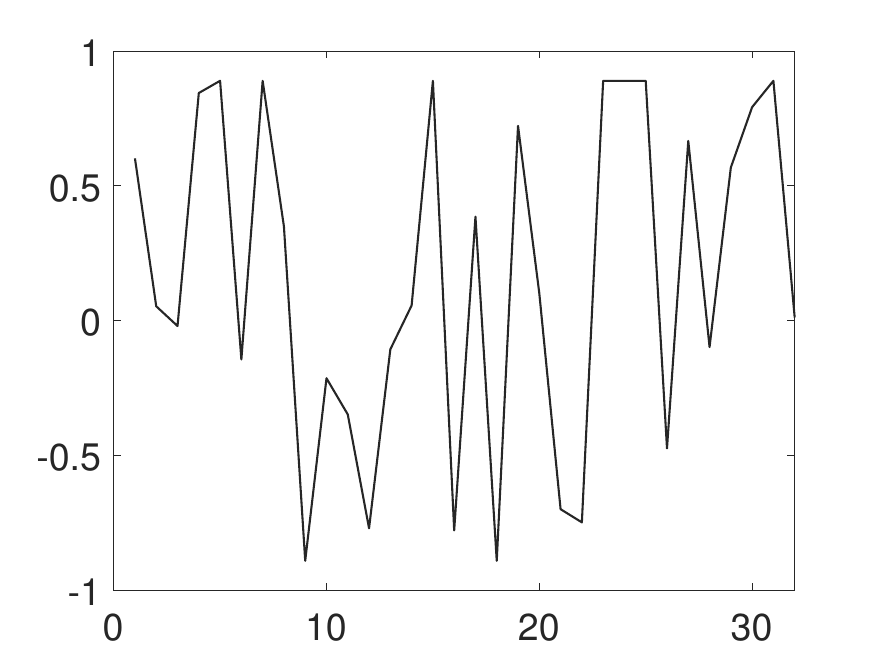}
        \caption{Graph of $k\mapsto \beta_{k,n_h,\pptrial}^{\, (1),\optp}$.}
        \label{figure5a}
    \end{subfigure}
    \begin{subfigure}[b]{0.30\textwidth}
        \centering
        \includegraphics[width=\textwidth]{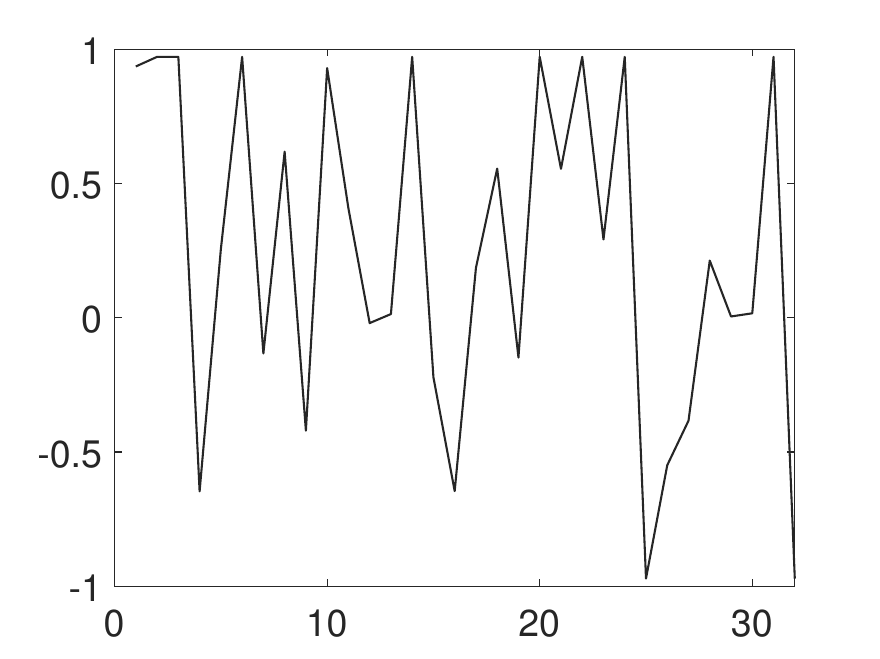}
        \caption{Graph of $k\mapsto \beta_{k,n_h,\pptrial}^{\, (2),\optp}$.}
        \label{figure5b}
    \end{subfigure}
    \begin{subfigure}[b]{0.30\textwidth}
        \centering
        \includegraphics[width=\textwidth]{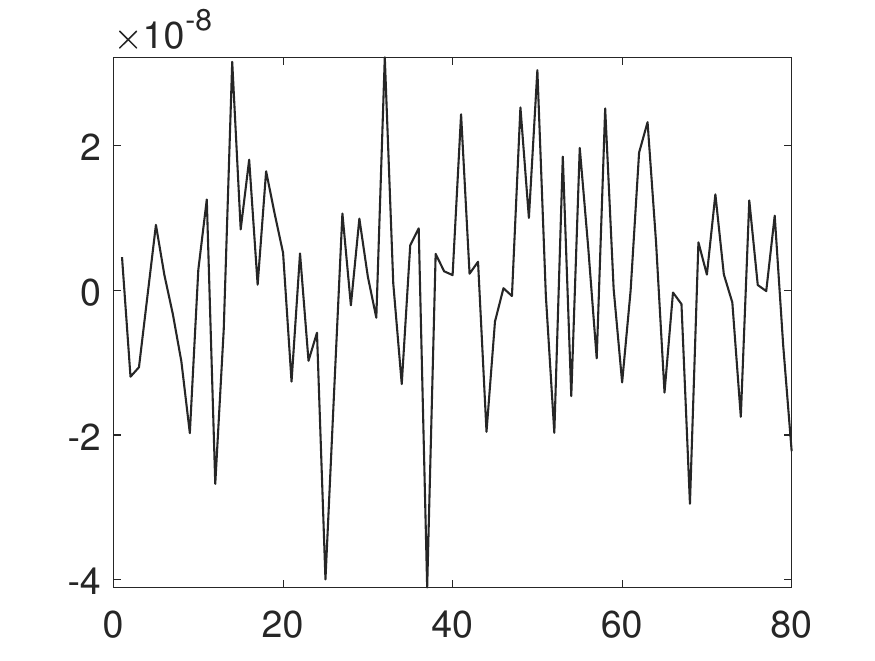}
        \caption{Graph $i\mapsto \beta_{i,\pptrial}^\optp$ from $\{1,\ldots , n_b\}$.}
        \label{figure5c}
    \end{subfigure}
    \caption{Case~1: Plots of the functions $k\mapsto \beta_{k,n_h,\pptrial}^{\, (1),\optp}$ and
    $k\mapsto \beta_{k,n_h,\pptrial}^{\, (2),\optp}$ on $\{ 1,\ldots, 32 \}$, and of
    $i\mapsto \beta_{i,\pptrial}^\optp$ on $\{1,\ldots , 80\}$.}
    \label{figure5}
\end{figure}
\paragraph{(iv) Local optimization using the projected gradient descent with Adam updates (PGD-Adam)}.
The PGD-Adam local optimizer is initialized with the trial-based global optimal solution
$\bftheta^\optp_\ptrial$ computed in paragraph (iii) above.
The step size for the Adam algorithm is set to $5\times 10^{-5}$. The gradient of the function
$( \bfbeta_{n_h}^{\,(1)}, \bfbeta_{n_h}^{\,(2)} ) \mapsto \curJ^\SPINNpp(\bfbeta_{n_h}^{\,(1)}, \bfbeta_{n_h}^{\,(2)})$
on $\RR^{n_h}\times  \RR^{n_h}$ is computed using a central finite-difference approximation with a relative  directional step size of $0.015$.
Fig.~\ref{figure6a} shows $\log_{10}(\curJ^\SPINNpp(n))$ as a function of the iteration number $n$ of the PGD-Adam optimizer, where $\curJ^\SPINNpp(n)$ denotes the value of $\curJ^\SPINNpp$ at iteration $n$.
Figs.~\ref{figure6b} and \ref{figure6c} show the plots of the graph $k\mapsto \beta_{k,n_h}^{\, (1),\opt}$ and $k\mapsto \beta_{k,n_h}^{\, (2),\opt}$ for $k\in\{1,\ldots, 32\}$.
{\color{black}
Since the trial-based solution used for initialization is already close to a high-quality solution, the PGD--Adam refinement produces only a small change in the identified parameters. Consequently, for $\ell\in\{1,2\}$, the components of $\bfbeta_{n_h}^{\,(\ell),\opt}$ remain very close, or equal, to those of $\bfbeta_{n_h,\pptrial}^{\,(\ell),\opt}$. In fact,
$J_{\min} = {\curJ}^{\SPINNpp}(\bftheta^\optp\, ; \alpha_0) = 1.59523 \times 10^{4}$, $J_{\min,\ptrial} = 1.56787 \times 10^{4}$.
The relative difference between these two objective values is about $1.7\%$. In this configuration, the local refinement therefore does not improve the training objective, although it produces only a limited modification of the trial-based solution.
}
\begin{figure}[H]
    \centering
    \begin{subfigure}[b]{0.30\textwidth}
        \centering
        \includegraphics[width=\textwidth]{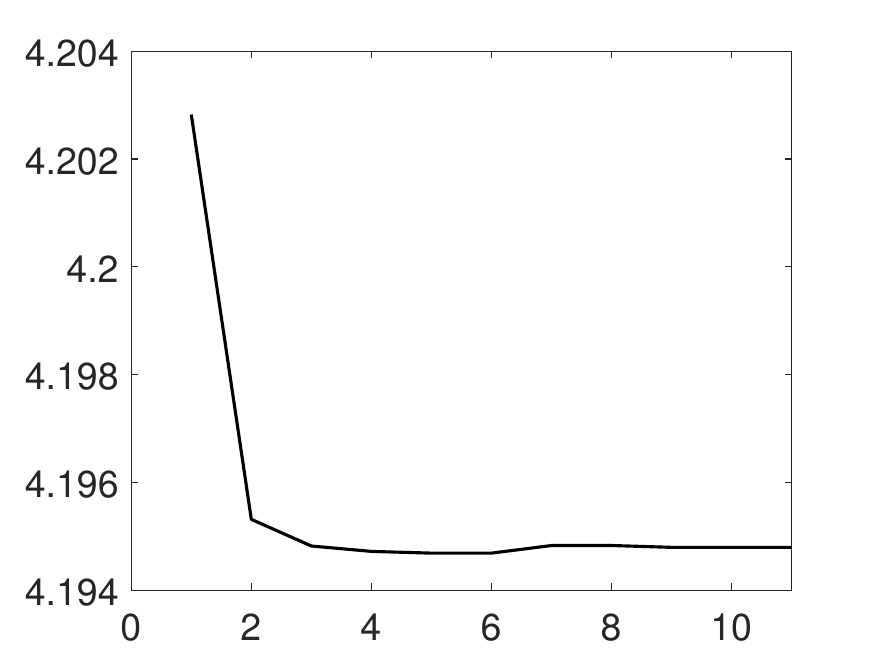}
        \caption{Graph of $n\mapsto \log_{10}(\curJ^\SPINNpp(n))$.}
        \label{figure6a}
    \end{subfigure}
    \begin{subfigure}[b]{0.30\textwidth}
        \centering
        \includegraphics[width=\textwidth]{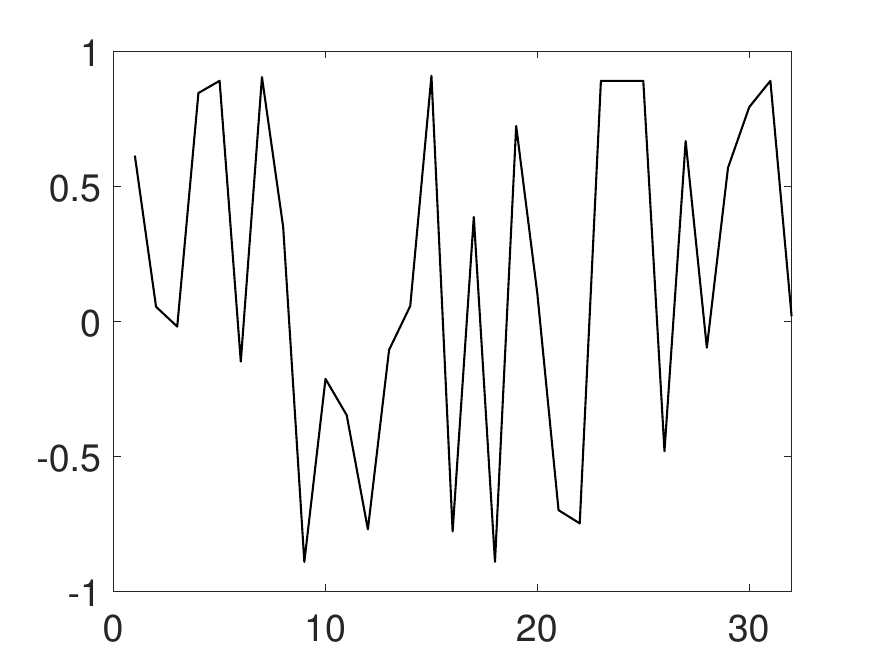}
        \caption{Graph of $k\mapsto \bfbeta_{k,n_h}^{\, (1),\optp}$.}
        \label{figure6b}
    \end{subfigure}
    \begin{subfigure}[b]{0.30\textwidth}
        \centering
        \includegraphics[width=\textwidth]{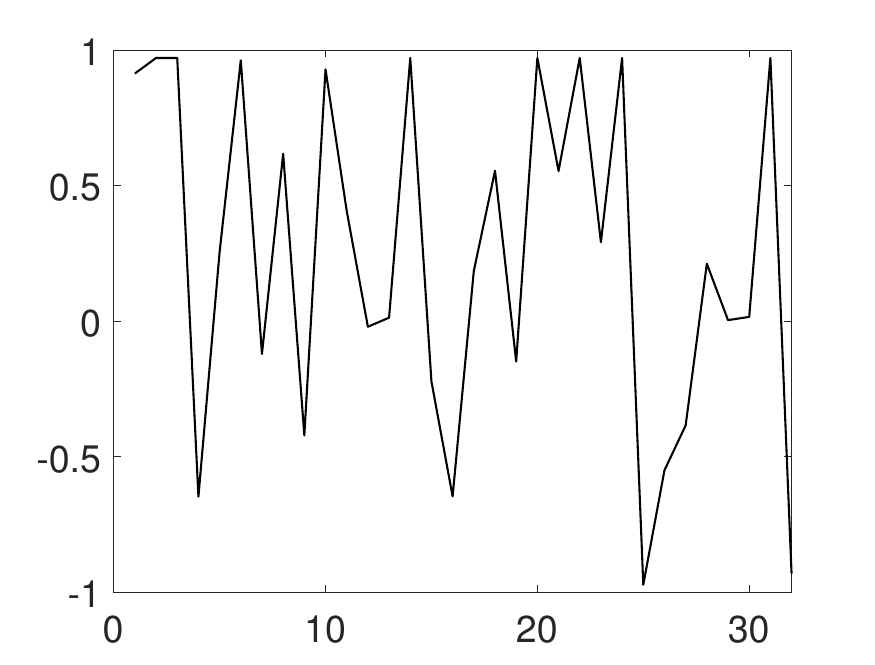}
        \caption{Graph $k\mapsto \bfbeta_{k,n_h}^{\, (2),\optp}$.}
        \label{figure6c}
    \end{subfigure}
    \caption{Case~1: Plots of the functions $n\mapsto \log_{10}(\curJ^\SPINNpp(n))$,
    $k\mapsto \bfbeta_{k,n_h}^{\, (1),\optp}$, and
    $k\mapsto \bfbeta_{k,n_h}^{\, (2),\optp}$ on $\{ 1,\ldots, 32 \}$.}
    \label{figure6}
\end{figure}
\noindent Table~\ref{table1} reports, for the intrinsic criteria defined in Section~\ref{Section3},
the results obtained after local refinement and those obtained at the trial step for the same value $\alpha_0$.
The locally refined solution (PGD) improves upon the trial-step solution (TRIAL),
while remaining close to it, consistent with the discussion above.
For this reference configuration, $\criter_o=0.03338$ indicates no notable over-specialization
(Section~\ref{Section3.2}). Moreover, $\criter_{\IOCp}=0.2273$ indicates good alignment
(see the rule-of-thumb scale in Section~\ref{Section3.3}).
Finally, $\CRPS^\scaledpp_\testp=1.0630$ (i.e., a $6.3\%$ increase relative to the training-based benchmark)
indicates a limited degradation in marginal skill under the present finite-parameter configuration.
The resulting configuration is of sufficient quality for the subsequent analyses; its accuracy is discussed further in
Section~\ref{Section7.3}. In particular, the best performance in our runs is not obtained for $N=200$
(the present reference configuration) but for $N=300$, for which $\CRPS^\scaledpp_\testp=1.018$
(i.e., a $1.8\%$ increase).
\begin{table}[h]
  \caption{Case~1: Evaluation-criteria values of the reference configuration ($N=200$, $n_h=32$, and $n_\psim=100$).}
  \label{table1}
  \begin{center}
  \resizebox{\textwidth}{!}{%
  \begin{tabular}{|c|c|c|c|c|c|c|c|c|c|}
    \hline
      Eqs. & \eqref{eq3.2}&\eqref{eq3.1}&\eqref{eq3.3}&\eqref{eq3.5}&\eqref{eq3.5}&\eqref{eq3.6}&\eqref{eq3.9}&\eqref{eq3.9}&\eqref{eq3.10}\\
    \hline
           & $\scriptstyle\NLL_\SNNp^\opt$ & $\scriptstyle\NLL_\SNNp^\test$  &
           $\scriptstyle\criter_o$ & $\scriptstyle\criter_{\bfalpha}$ &  $\scriptstyle\criter_{\bfbeta}$ & $\scriptstyle\criter_{\IOCp}$ & $\scriptstyle\CRPS^\SNNpp_\testp$ & $\scriptstyle\CRPS^\trainingpp_\testp$ & $\scriptstyle\CRPS^\scaledpp_\testp$\\
    \hline
      $\TRIAL$  & 7839.4 & 1521.5 & 0.03046 & 0.05574  & 0.05459 & 0.2350          & 0.043313 & 0.040216 & 1.0770          \\
      $\PGD$    & 7852.6 & 1519.8 & 0.03338 & 0.05479  & 0.05274 & \textbf{0.2273} & 0.042751 & 0.040216 & \textbf{1.0630} \\
    \hline
  \end{tabular}
  } 
\end{center}
\end{table}

\noindent To examine the results in more detail, we consider $3$ selected output components with indices $k \in\{14,35,41\}$ among the $80$ outputs (see Section~\ref{Section6.1}-(iv)). The corresponding observation points are
$\bfx^{(k)}_\obsp = (r, 0, (k-1) L/80)$.
For these components of the output vector, we compare the confidence interval (see Eq.~\eqref{eq3.4})
$\CI^\SNNpp_k(\xx^i_\testp) = [\alpha^\SNNpp_k(\xx^i_\testp) \, , \beta^\SNNpp_k(\xx^i_\testp)]$ estimated by the optimal SNN
with
$\CI^\trainingpp_k(\xx^i_\testp) = [\alpha^\trainingpp_k(\xx^i_\testp) \, , \beta^\trainingpp_k(\xx^i_\testp)]$ estimated from the training dataset,
where $i\in\{6,24\}$ corresponds to the test-dataset input realizations $\xx^6_\testp=(2.1317 , 0.1668)$ and
$\xx^{24}_\testp = (1.9054,0.1291)$. The results are given in Table~\ref{table2}.
\begin{table}[h]
  \caption{Case~1: For $k = 14$, $35$, and $41$, conditional confidence intervals $\CI^\trainingpp_k(\xx^i_\testp)$ and $\CI^\SNNpp_k(\xx^i_\testp)$ for  inputs $i\in\{6,24\}$ from the test dataset, with $N=200$, $n_h=32$, and $n_\psim=100$.}
  \label{table2}
  \begin{center}
  \begin{tabular}{|c||c|c||c|c|}
    \hline
        $\scriptstyle k$ & $\scriptstyle \CI^\trainingpp_k(\xx^6_\testp)$  & $\scriptstyle \CI^\SNNpp_k(\xx^6_\testp)$
         & $\scriptstyle \CI^\trainingpp_k(\xx^{24}_\testp)$  & $\scriptstyle \CI^\SNNpp_k(\xx^{24}_\testp)$ \\
    \hline
     14  & $[-0.8472,-0.5951]$   & $[-0.8097,-0.5537]$ & $[-0.7526,-0.5485]$   & $[-0.8241,-0.5810]$\\
     35  & $[ 1.0775, 1.4817]$   & $[ 1.0125, 1.4362]$ & $[ 1.0061, 1.3164]$   & $[ 1.0106, 1.4174]$\\
     41  & $[ 1.3910, 1.8954]$   & $[ 1.2778, 1.8351]$ & $[ 1.2755, 1.6860]$   & $[ 1.2772, 1.8219]$\\
    \hline
  \end{tabular}
\end{center}
\end{table}
For $k \in\{14,35,41\}$ and $i\in\{6,24\}$, Figs.~\ref{figure7} and \ref{figure8}  show the probability density functions
$y\mapsto p_{YY_k\vert\XX}^\SNNpp(y \vert\xx_\testp^i;\bftheta^\optp)$ (estimated with the optimal SNN),
which are compared with the densities
$y\mapsto p_k^\trainingpp(y \vert\xx_\testp^i)$ (estimated from the training dataset).
\begin{figure}[H]
    \centering
    \begin{subfigure}[b]{0.30\textwidth}
        \centering
        \includegraphics[width=\textwidth]{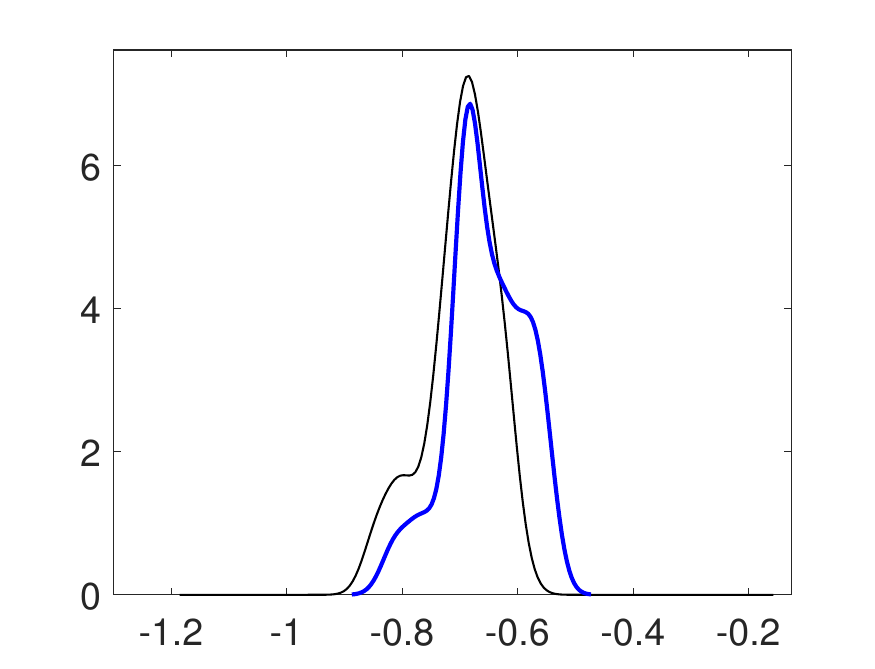}
        \caption{Graphs of the probability density functions for $k=14$.}
        \label{figure7a}
    \end{subfigure}
    \begin{subfigure}[b]{0.30\textwidth}
        \centering
        \includegraphics[width=\textwidth]{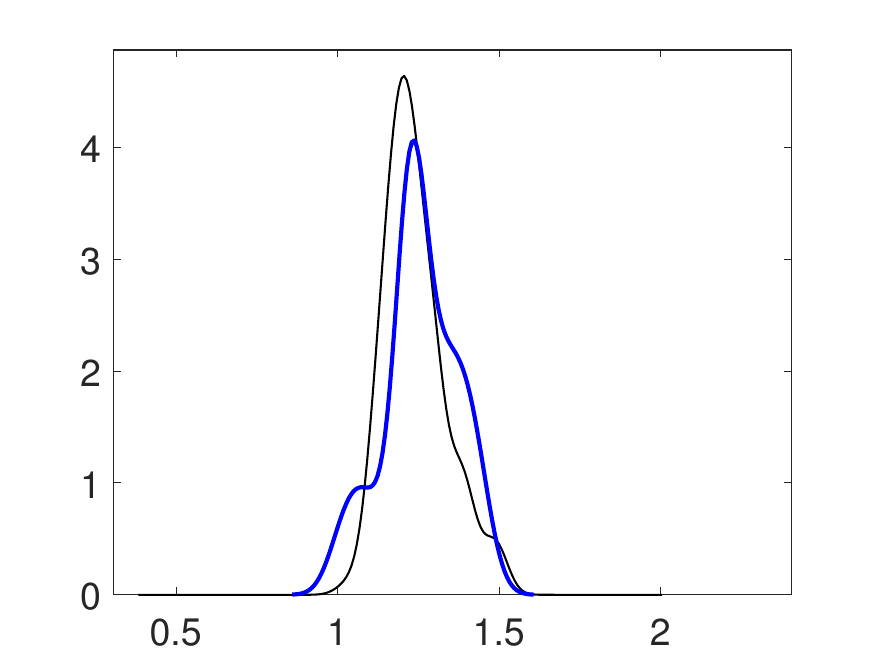}
        \caption{Graphs of the probability density functions for $k=35$.}
        \label{figure7b}
    \end{subfigure}
    \begin{subfigure}[b]{0.30\textwidth}
        \centering
        \includegraphics[width=\textwidth]{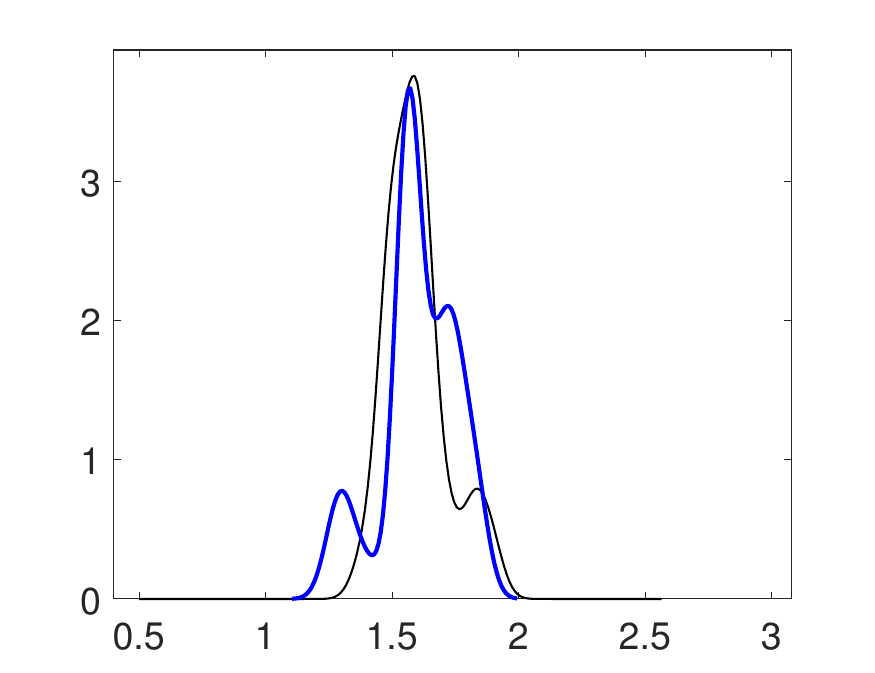}
        \caption{Graphs of the probability density functions for $k=41$.}
        \label{figure7c}
    \end{subfigure}
    \caption{Case~1: Plots of the probability density functions $y\mapsto p_{YY_k\vert\XX}^\SNNpp(y \vert\xx_\testp^6;\bftheta^\optp)$
                   (thin black line) and  $y\mapsto p_k^\trainingpp(y \vert\xx_\testp^6)$ (thick blue line)}
    \label{figure7}
\end{figure}
\begin{figure}[H]
    \centering
    \begin{subfigure}[b]{0.30\textwidth}
        \centering
        \includegraphics[width=\textwidth]{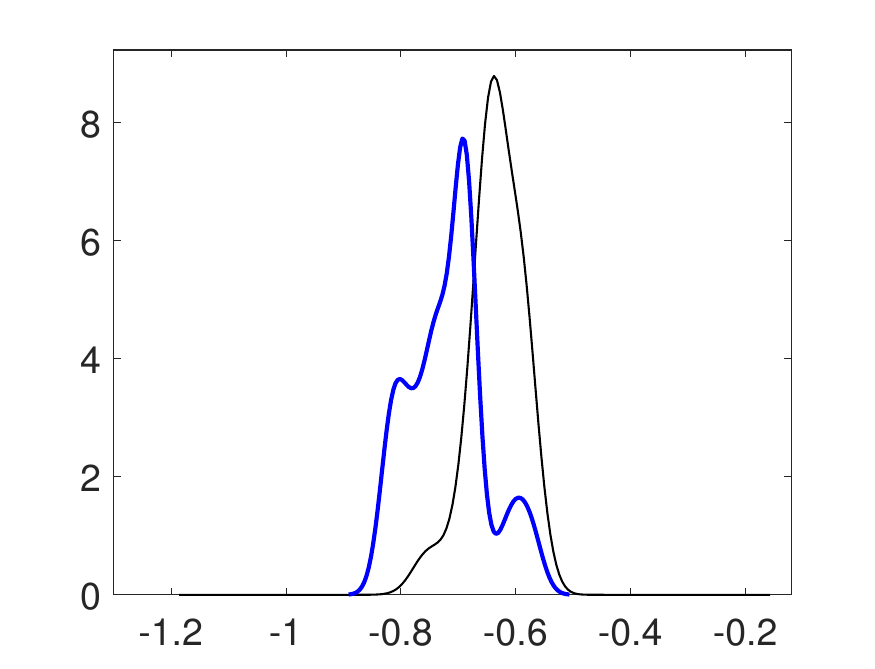}
        \caption{Graphs of the probability density functions for $k=14$.}
        \label{figure8a}
    \end{subfigure}
    \begin{subfigure}[b]{0.30\textwidth}
        \centering
        \includegraphics[width=\textwidth]{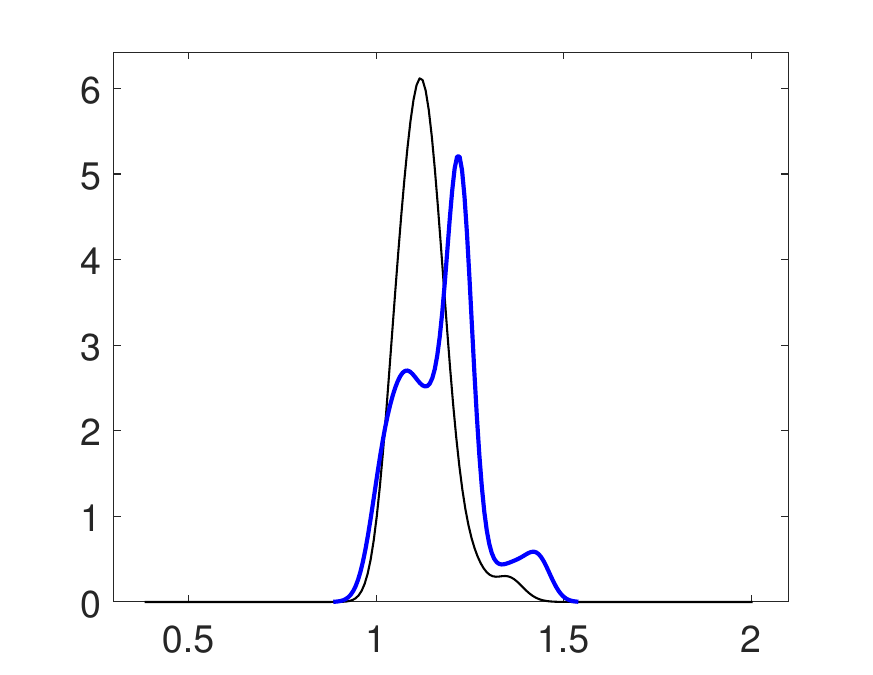}
        \caption{Graphs of the  probability density functions for $k=35$.}
        \label{figure8b}
    \end{subfigure}
    \begin{subfigure}[b]{0.30\textwidth}
        \centering
        \includegraphics[width=\textwidth]{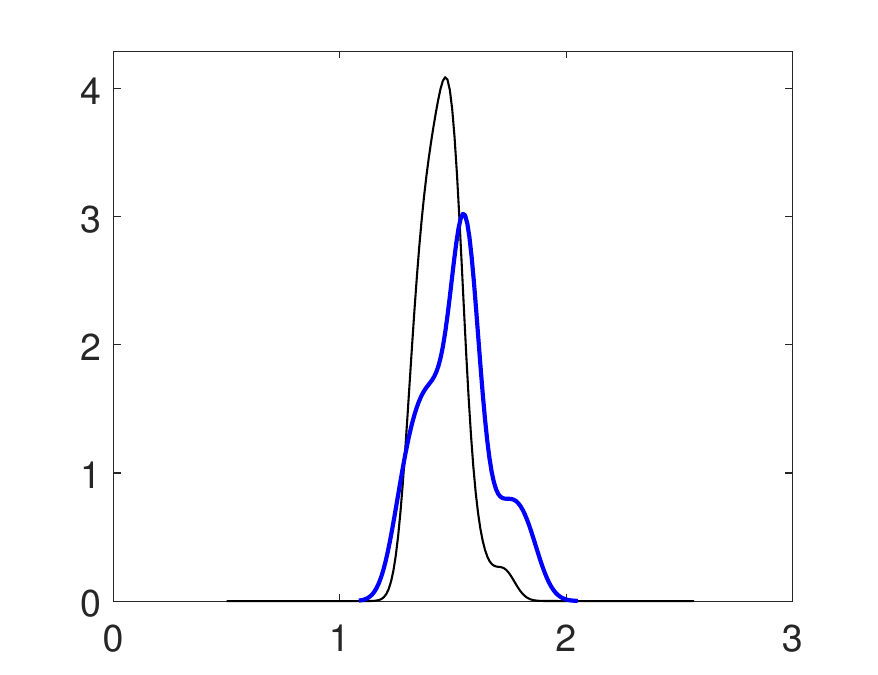}
        \caption{Graphs of the  probability density functions for $k=41$.}
        \label{figure8c}
    \end{subfigure}
    \caption{Case~1: Plots of the probability density functions $y\mapsto p_{YY_k\vert\XX}^\SNNpp(y \vert\xx_\testp^{24};\bftheta^\optp)$
                   (thin black line) and  $y\mapsto p_k^\trainingpp(y \vert\xx_\testp^{24})$ (thick blue line)}
    \label{figure8}
\end{figure}
\paragraph{(v) Convergence analysis with respect to the reference configuration}
The reference-configuration parameters $n_h$, $n_\psim$, and $N$ are defined as
$n_h^\rc=32$, $n_\psim^\rc = 100$, and $N^\rc = 200$. We carry out a convergence analysis with respect to this reference configuration by considering $n_h\in\{4,8,16,32,64,128\}$, $n_\psim\in\{100,140,200,300\}$, and $N\in\{130,200,250,300\}$.
These quantities are the principal controls of latent-field resolution, Monte Carlo resolution, and neural-architecture size, respectively.
{\color{black}
In particular, the dependence on $n_\psim$ is a convergence study of the complete implemented training pipeline, not a post-processing variation: for each value of $n_\psim$, the full trial step is followed by the PGD--Adam refinement. This directly assesses the Monte Carlo resolution entering the Gaussian-KDE likelihood calculations used for training.
}
All other SNN parameters are set to the values used in the reference-configuration analysis (see paragraphs~(iii) and (iv) above), and their status is summarized in {\color{black} Section~\ref{Section6.parameterclassification}}.
The parameter $\alpha_0$, which controls the alignment term in SPINN, is recomputed at every point of the convergence analysis.
The complete trial step is performed, followed by the PGD--Adam step, according to the methodology presented in paragraphs~(iii) and (iv).
{\color{black}
The trial procedure should not be interpreted as a single arbitrary neural-network initialization. For the reference calculation, it evaluates $18 \times 18 = 324$ candidates in $(\zeta_s,c_{\targetp})$. Each candidate is associated with projected coefficient vectors
$\bfbeta^{\,(1)}_{n_h}$ and $\bfbeta^{\,(2)}_{n_h}$ and with the corresponding generated architecture, and every objective evaluation uses $n_\psim$ independent latent-field and random-weight realizations. The best trial candidate provides the initialization of the subsequent local refinement.
Independent repetitions of the complete trial-plus-PGD calculation using new outer architecture germs are not included in the present study. Such repetitions would quantify the realization-to-realization variability of the finite-sample architecture construction and of the sample-average objective. The convergence analysis with respect to $n_\psim$ controls the principal inner Monte Carlo resolution but is not a substitute for a complete outer-germ replication study. This additional robustness assessment is left for future large-scale investigations.
}
Fig.~\ref{figure9} summarizes the convergence analysis by plotting the functions $n_h\mapsto\CRPS^\scaledpp_\testp(n_h,n_\psim^\rc,N^\rc)$,
$n_\psim\mapsto \NLL_\SNNp^\opt(n_h^\rc,n_\psim,N^\rc)$, and $N\mapsto\CRPS^\scaledpp_\testp(n_h^\rc,n_\psim^\rc,N)$.
Tables~\ref{table3} and \ref{table4} report the evaluation criteria for the PGD--Adam optimal solution as a function of
$N\in\{130,200,250,300\}$, for $n_h = n_h^\rc = 32$ and $n_\psim = n_\psim^\rc = 100$. The best solution with respect to $\CRPS^\scaledpp_\testp$ is obtained for $N = 300$, for which $\CRPS^\scaledpp_\testp=1.0180$ (see Section~\ref{Section7}).
\begin{figure}[H]
    \centering
    \begin{subfigure}[b]{0.30\textwidth}
        \centering
        \includegraphics[width=\textwidth]{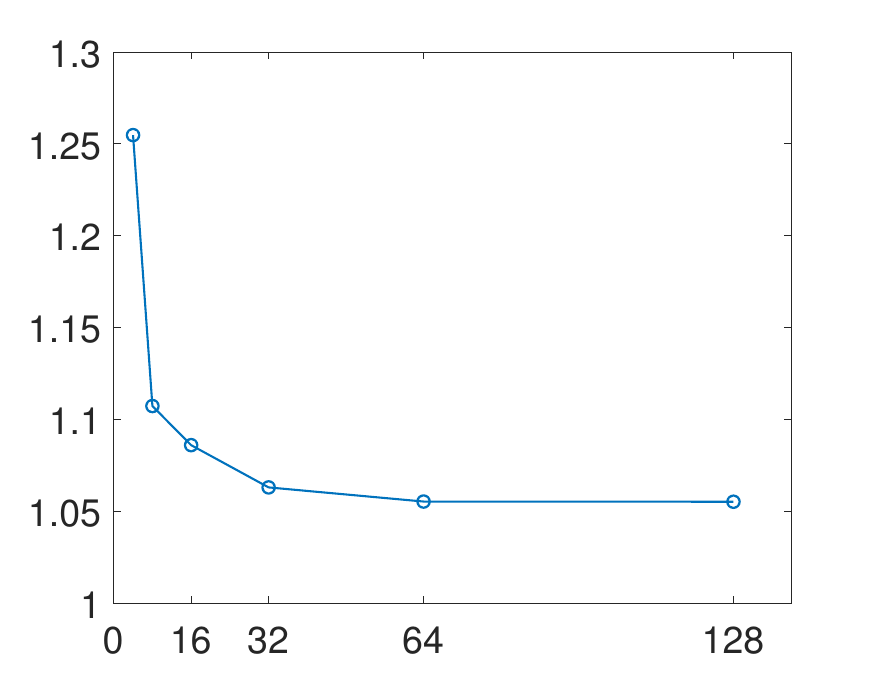}
        \caption{Graph of $n_h\mapsto\CRPS^\scaledpp_\testp(n_h,n^\rc_\psim,N^\rc)$.}
        \label{figure9a}
    \end{subfigure}
    \begin{subfigure}[b]{0.30\textwidth}
        \centering
        \includegraphics[width=\textwidth]{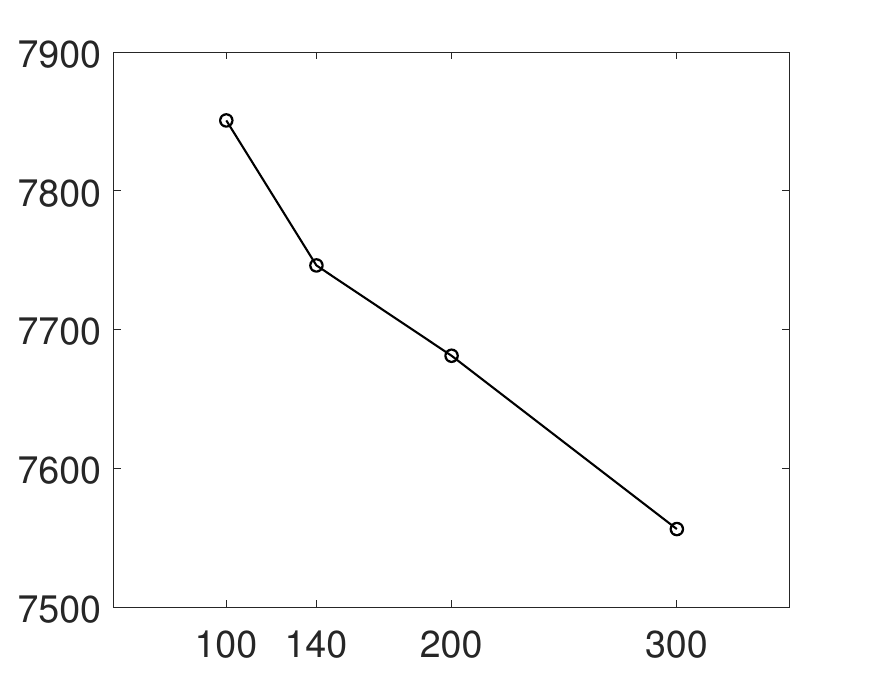}
        \caption{Graph of $n_\psim\mapsto \NLL_\SNNp^\opt(n_h^\rc,n_\psim,N^\rc)$.}
        \label{figureçb}
    \end{subfigure}
    \begin{subfigure}[b]{0.30\textwidth}
        \centering
        \includegraphics[width=\textwidth]{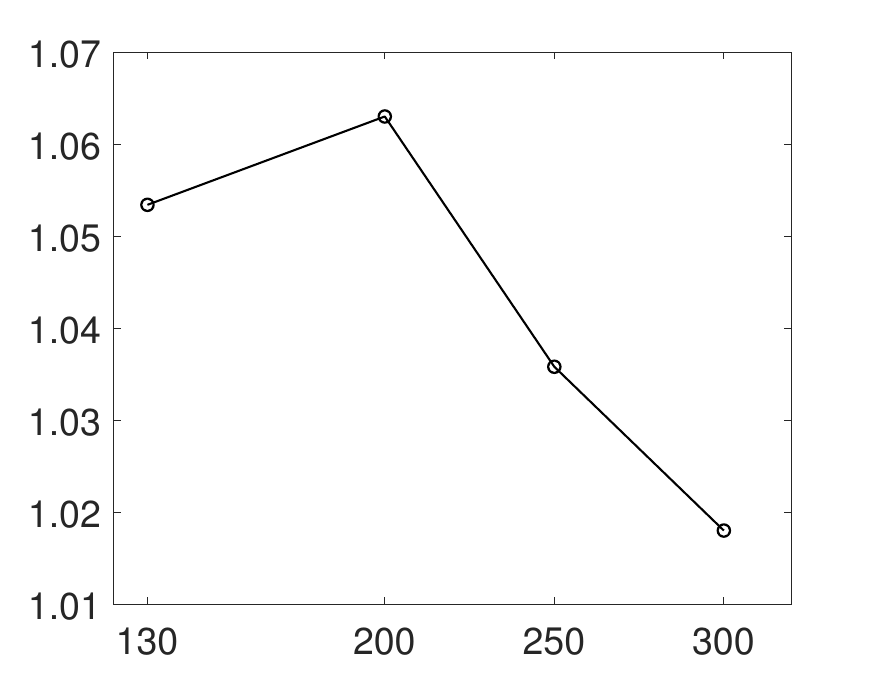}
        \caption{Graph of $N\mapsto\CRPS^\scaledpp_\testp(n_h^\rc,n_\psim^\rc,N)$.}
        \label{figure9c}
    \end{subfigure}
    \caption{Case~1: Convergence analysis with respect to $n_h$, $n_\psim$, and $N$, with $n_h^\rc=32$, $n_\psim^\rc=100$, and $N^\rc=200$. Plots of the following functions:(a) $n_h\mapsto\CRPS^\scaledpp_\testp(n_h,n_\psim^\rc,N^\rc)$,
             (b) $n_\psim\mapsto \NLL_\SNNp^\opt(n_h^\rc,n_\psim,N^\rc)$, and
             (c) $N\mapsto\CRPS^\scaledpp_\testp(n_h^\rc,n_\psim^\rc,N)$. }
    \label{figure9}
\end{figure}
\begin{table}[h]
  \caption{Case~1:Neuron-architecture parameters and $\alpha_0$ values for the alignment term, as a function of $N\in\{130, 200, 250, 300\}$,  for $n_h= n_h^\rc=32$ and $n_\psim = n_\psim^\rc=100$.}
  \label{table3}
  \begin{center}
  \small
  \setlength{\tabcolsep}{6pt}        
  \renewcommand{\arraystretch}{1.0}  
  \begin{tabular}{|c|c|c|c|c|}
    \hline
      Eqs. & \eqref{eq4.12}&\eqref{eq4.24}&\eqref{eq4.33}&\eqref{eq6.7}\\
    \hline
      $ N$ & $ n_\pint$ & $ n_\pinc$ & $ n_b$ &  $\alpha_0$ \\
    \hline
      130 & 48  & 3426     & 32  & 1.0717 \\
      200 & 118 & 8067     & 80  & 1.0458 \\
      250 & 168 & 12625    & 114 & 1.0718 \\
      300 & 218 & 13156    & 148 & 1.0688 \\
    \hline
  \end{tabular}
  \end{center}
\end{table}
\begin{table}[h]
  \caption{Case~1: Evaluation-criteria values for the PGD-Adam optimal solution as a function of $N\in\{130, 200, 250, 300\}$,  for $n_h= n_h^\rc=32$ and $n_\psim = n_\psim^\rc=100$.}
  \label{table4}
  \begin{center}
  \resizebox{\textwidth}{!}{%
  \begin{tabular}{|c|c|c|c|c|c|c|c|c|c|}
    \hline
      Eqs. & \eqref{eq3.2}&\eqref{eq3.1}&\eqref{eq3.3}&\eqref{eq3.5}&\eqref{eq3.5}&\eqref{eq3.6}&\eqref{eq3.9}&\eqref{eq3.9}&\eqref{eq3.10}\\
    \hline
      $N$  & $\scriptstyle\NLL_\SNNp^\opt$ & $\scriptstyle\NLL_\SNNp^\test$  &
           $\scriptstyle\criter_o$ & $\scriptstyle\criter_{\bfalpha}$ &  $\scriptstyle\criter_{\bfbeta}$ & $\scriptstyle\criter_{\IOCp}$ & $\scriptstyle\CRPS^\SNNpp_\testp$ & $\scriptstyle\CRPS^\trainingpp_\testp$ & $\scriptstyle\CRPS^\scaledpp_\testp$\\
    \hline
      130  & 7765.0 & 1481.9 & 0.04796 & 0.05168  & 0.04783 & 0.2214 & 0.042364 & 0.040216 & 1.0534 \\
      200  & 7852.6 & 1519.8 & 0.03338 & 0.05479  & 0.05274 & 0.2273 & 0.042751 & 0.040216 & 1.0630 \\
      250  & 8002.6 & 1495.8 & 0.07003 & 0.05281  & 0.04728 & 0.2290 & 0.041656 & 0.040216 & 1.0358 \\
      300  & 7931.6 & 1492.2 & 0.06308 & 0.05292  & 0.04879 & 0.2242 & 0.040939 & 0.040216 & \textbf{1.0180} \\
    \hline
  \end{tabular}
  } 
\end{center}
\end{table}
\noindent Table~\ref{table1} indicates that the reference configuration achieves a well-balanced outcome across the intrinsic
diagnostics: the generalization index remains small ($\criter_o=0.03338$ for PGD), the confidence-interval
discrepancies are comparable on the lower and upper bounds ($\criter_{\bfalpha}=0.05479$ and $\criter_{\bfbeta}=0.05274$),
and the interval-overlap score is in the ``good'' range ($\criter_{\IOCp}=0.2273$). Consistently, the marginal
probability-distribution skill relative to the training-conditional benchmark is close to one
($\CRPS^\scaledpp_\testp=1.0630$), and the gap between the TRIAL and PGD solutions remains
modest, supporting the interpretation that the selected hyperparameters deliver a stable and reproducible
fit in the SBVP-consistent regime. At a more granular level, Table~\ref{table2} confirms this picture on representative
outputs and test inputs: the SNN confidence intervals closely track the training-based intervals, with only
moderate center shifts and comparable widths across the reported components. The PDF comparisons in
Figs.~\ref{figure7} and \ref{figure8} provide a complementary pointwise check of marginal agreement at these locations;
a broader interpretation of the CRPS levels and of the conditional probability density functions is deferred to Sections~\ref{Section7.3} and \ref{Section7.4}.
\subsection{Analysis of a configuration based on \textbf{Case 2}}
\label{Section6.7}
The Case~2 configuration is defined in Section~\ref{Section6.2}, for which the training and test datasets are based on SBVP$\,^\MSP$.
All parameters of the SNN model are those used for Case~1 and are defined in Section~\ref{Section6.6}-(i). We then have $n_h=32, N = 200, and n_\psim = 100$.
For this analysis, we follows the same plan (with simplification to reduce the paper length) as the one used for the Case~1 results presented in Section~\ref{Section6.6}.
\paragraph{(i) Trial-based global optimization problem}.
The Case~1 methodology  is reused.
For the grid, we fix $h_{1,\ptrial}^\optp = 0.66$ and  $h_{2,\ptrial}^\optp= 0.72$, and we consider $\zeta_s \in [0.03, 0.05]$
and $c_\targetp \in [0.5,0.7]$. The grid contains $36\times 18 = 648$ points $\bftheta^j$.
The value of $\alpha_0$ (see Eqs.~\eqref{eq6.6} and \eqref{eq6.7}) is $\alpha_0 = 0.35428$.
{\color{black}
This value is the operational scale-matching coefficient used in the normalized training objective; it should be read as the relative numerical scaling prescribed by Eq.~\eqref{eq2.19}, not as an invariant measure of distributional discrepancy.
}
Fig.~\ref{figure10} shows the graphs of
$j\mapsto \curJ^\SNNpp(\bftheta_\ptrial^j)$ (see Eq.~\eqref{eq2.8}),
$j\mapsto \widehat{\rm{CE}}(\bftheta_\ptrial^j) = (1/n_d)\sum_{i= 1}^{n_d} \widehat{\rm{CE}}_i((\bftheta_\ptrial^j))$ (see Eq.~\eqref{eq2.15}),
and $j\mapsto \curJ^\SPINNpp(\bftheta_\ptrial^j\, ; \alpha_0)$ (see Eq.~\eqref{eq2.12}) over the grid.
\begin{figure}[H]
    \centering
    \begin{subfigure}[b]{0.30\textwidth}
        \centering
        \includegraphics[width=\textwidth]{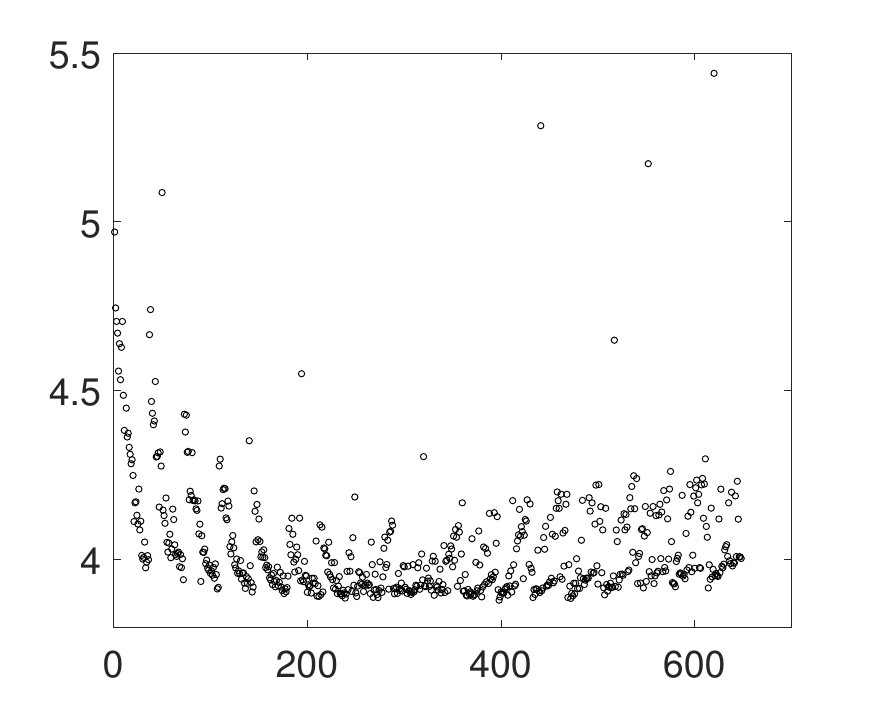}
        \caption{Graph of $j\mapsto \curJ^\SNNpp(\bftheta_\pptrial^j)$.}
        \label{figure10a}
    \end{subfigure}
    \begin{subfigure}[b]{0.30\textwidth}
        \centering
        \includegraphics[width=\textwidth]{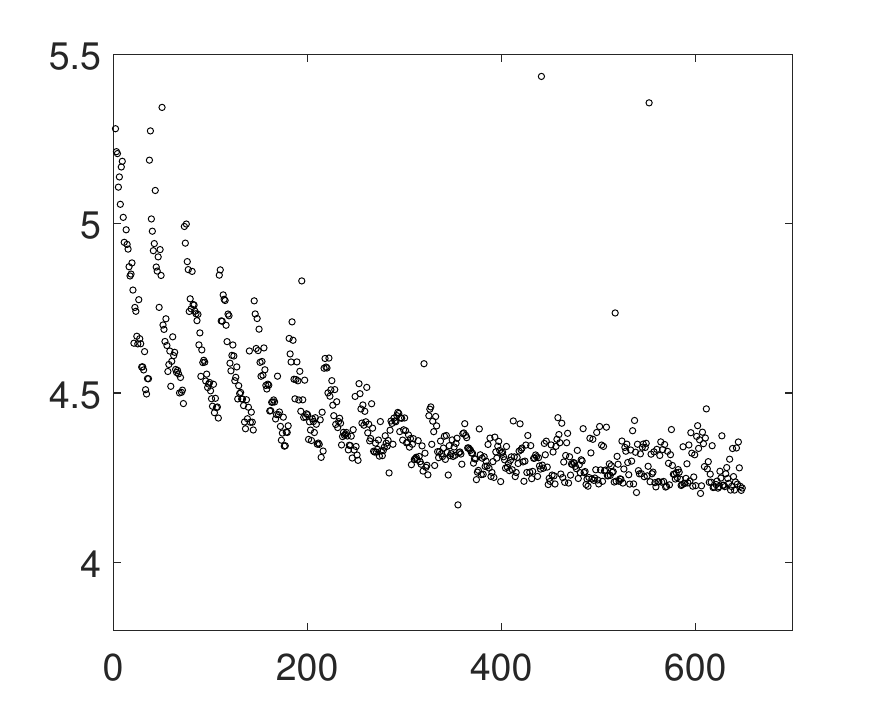}
        \caption{Graph of $j\mapsto \widehat{\rm{CE}}(\bftheta_\pptrial^j)$.}
        \label{figure10b}
    \end{subfigure}
    \begin{subfigure}[b]{0.30\textwidth}
        \centering
        \includegraphics[width=\textwidth]{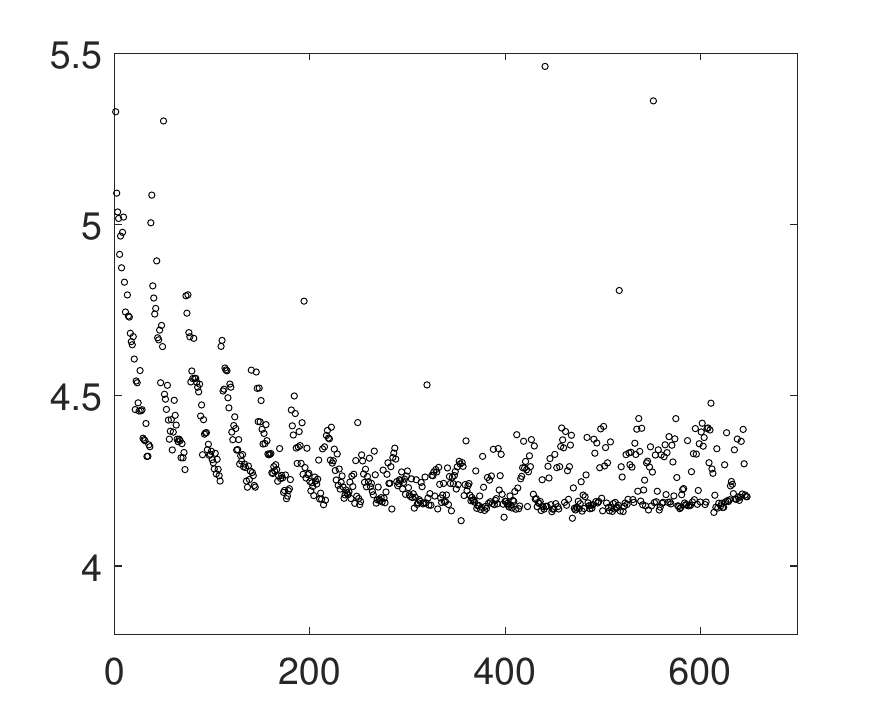}
        \caption{Graph $j\mapsto \curJ^\SPINNpp(\bftheta_\pptrial^j\, ; \alpha_0)$.}
        \label{figure10c}
    \end{subfigure}
    \caption{Case~2: $\log_{10}$-scale plots of the graphs of functions $\curJ^\SNNpp$, $\widehat{\rm{CE}}$, and
    $\curJ^\SPINNpp(\cdot, ; \alpha_0)$ on the grid points $j\in\{1,\ldots, 648\}$.}
    \label{figure10}
\end{figure}
\noindent Solving Eq.~\eqref{eq6.8} yields $\zeta_s^\optp = 0.04714$ and $c_\targetp^\optp = 0.60588$.  We have $J_{\min ,\ptrial}  = {\curJ}^{\SPINNpp}(\bftheta_\ptrial^\optp\, ;\alpha_0) = 1.3567\times 10^{4}$.
\paragraph{(ii) Local optimization using the projected gradient descent with Adam updates (PGD-Adam)}.
The parameters for the PGD-Adam local optimizer are the same as those used in Section~\ref{Section6.6}-(iv), except for the PGD-Adam step size, which is set to $0.1$.
Fig.~\ref{figure11a} shows $\log_{10}(\curJ^\SPINNpp(n))$ as a function of the PGD-Adam iteration number $n$, where $\curJ^\SPINNpp(n)$ denotes the value of $\curJ^\SPINNpp$ at iteration $n$.
Figs.~\ref{figure11b} and \ref{figure11c} show the graphs $k\mapsto \beta_{k,n_h}^{\, (1),\opt}$ and $k\mapsto \beta_{k,n_h}^{\, (2),\opt}$ for $k\in\{1,\ldots, 32\}$.
\begin{figure}[H]
    \centering
    \begin{subfigure}[b]{0.30\textwidth}
        \centering
        \includegraphics[width=\textwidth]{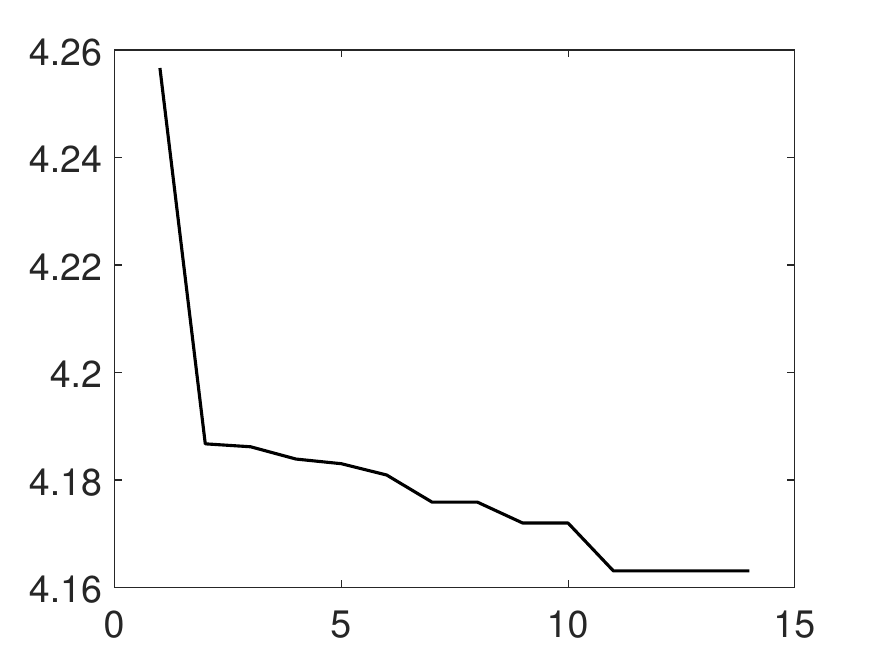}
        \caption{Graph of $n\mapsto \log_{10}(\curJ^\SPINNpp(n))$.}
        \label{figure11a}
    \end{subfigure}
    \begin{subfigure}[b]{0.30\textwidth}
        \centering
        \includegraphics[width=\textwidth]{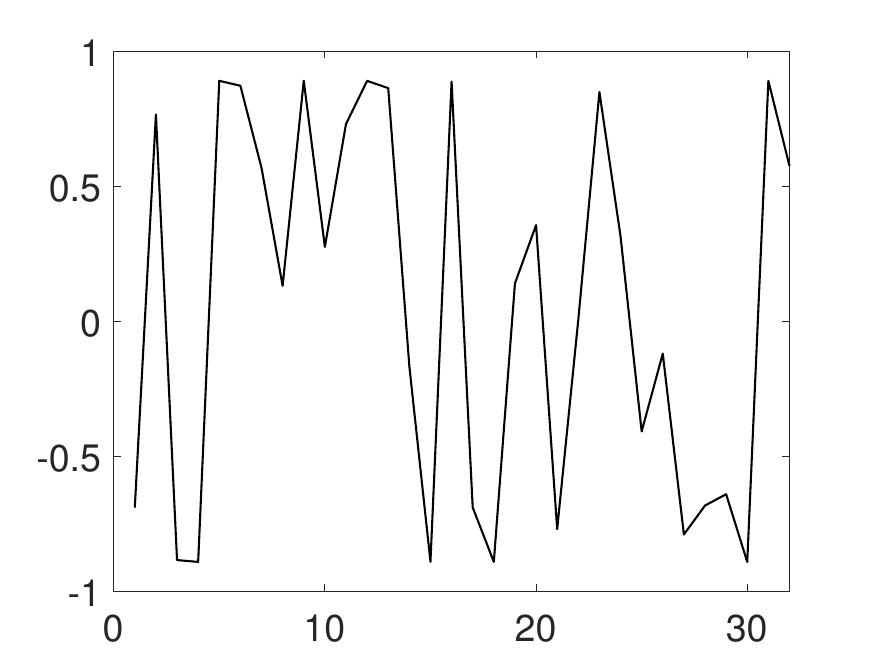}
        \caption{Graph of $k\mapsto \bfbeta_{k,n_h}^{\, (1),\optp}$.}
        \label{figure11b}
    \end{subfigure}
    \begin{subfigure}[b]{0.30\textwidth}
        \centering
        \includegraphics[width=\textwidth]{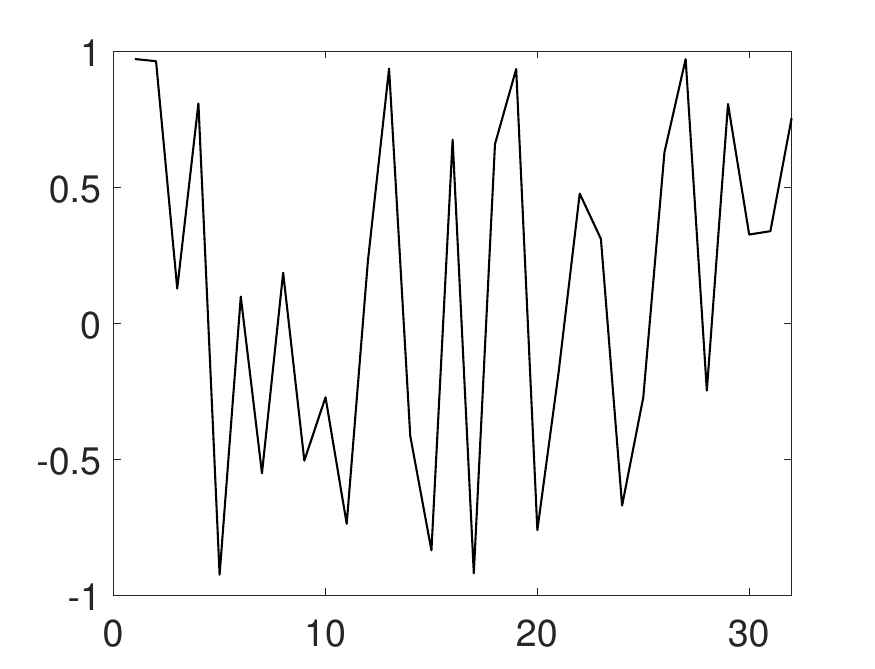}
        \caption{Graph $k\mapsto \bfbeta_{k,n_h}^{\, (2),\optp}$.}
        \label{figure11c}
    \end{subfigure}
    \caption{Case~2: Plots of the functions $n\mapsto \log_{10}(\curJ^\SPINNpp(n))$,
    $k\mapsto \bfbeta_{k,n_h}^{\, (1),\optp}$, and
    $k\mapsto \bfbeta_{k,n_h}^{\, (2),\optp}$ on $\{ 1,\ldots, 32 \}$.}
    \label{figure11}
\end{figure}
\noindent Table~\ref{table5} presents the  results obtained with the local refinement and those obtained with the trial-based  solution for $\alpha_0 = 0.35428$.
The observed overfitting ratio is $Cr_o = 0.00615$, which is well below the $0.10$ threshold and therefore provides strong evidence against overspecialization: the generalization gap between the training and test datasets is negligible. The scaled score
$\CRPS^\scaledpp_\testp = \CRPS^\SNNpp_\testp / \CRPS^\trainingpp_\testp =  1.3478$ corresponds to a relative increase of about $35\%$ in CRPS with respect to the training-conditional nonparametric reference, that is, a nontrivial loss in accuracy of estimated probability distributions in this mismatch regime. Overall, these values support the interpretation that the current setting generalizes well but incurs a substantial CRPS penalty, consistent with a moderate SBVP mismatch at $\alpha_0=0.35428$.
\begin{table}[h]
  \caption{Case~2: Evaluation-criteria values of the reference configuration ($N=200$, $n_h=32$, and $n_\psim=100$).}
  \label{table5}
  \begin{center}
  \resizebox{\textwidth}{!}{%
  \begin{tabular}{|c|c|c|c|c|c|c|c|c|c|}
    \hline
      Eqs. & \eqref{eq3.2}&\eqref{eq3.1}&\eqref{eq3.3}&\eqref{eq3.5}&\eqref{eq3.5}&\eqref{eq3.6}&\eqref{eq3.9}&\eqref{eq3.9}&\eqref{eq3.10}\\
    \hline
           & $\scriptstyle\NLL_\SNNp^\opt$ & $\scriptstyle\NLL_\SNNp^\test$  &
           $\scriptstyle\criter_o$ & $\scriptstyle\criter_{\bfalpha}$ &  $\scriptstyle\criter_{\bfbeta}$ & $\scriptstyle\criter_{\IOCp}$ & $\scriptstyle\CRPS^\SNNpp_\testp$ & $\scriptstyle\CRPS^\trainingpp_\testp$ & $\scriptstyle\CRPS^\scaledpp_\testp$\\
    \hline
      $\TRIAL$  & 8332.1 & 2189.9 & 0.24217 & 0.08323  & 0.07075 & 0.2703 & 0.05872 & 0.04139 & 1.4185 \\
      $\PGD$    & 8554.8 & 1700.5 & 0.00615 & 0.08557  & 0.07109 & 0.2657 & 0.05579 & 0.04149 & 1.3478 \\
    \hline
  \end{tabular}
  } 
\end{center}
\end{table}
\noindent As in Table~\ref{table2} devoted to Case~1,
we examine the same three selected output components
(indices $k \in\{14,35,41\}$) and the same test-dataset input realizations (indices $i\in\{6,24\}$) for Case~2.
Table~\ref{table6} reports the confidence interval (see Eq.~\eqref{eq3.4})
$\CI^\SNNpp_k(\xx^i_\testp) = [\alpha^\SNNpp_k(\xx^i_\testp) \, , \beta^\SNNpp_k(\xx^i_\testp)]$ estimated by the optimal SNN
together with $\CI^\trainingpp_k(\xx^i_\testp) = [\alpha^\trainingpp_k(\xx^i_\testp) \, , \beta^\trainingpp_k(\xx^i_\testp)]$ estimated from the training and test datasets.
\begin{table}[h]
  \caption{Case 2: For $k = 14$, $35$, and $41$, conditional confidence intervals $\CI^\trainingpp_k(\xx^i_\testp)$ and $\CI^\SNNpp_k(\xx^i_\testp)$ for  inputs $i\in\{6,24\}$ from the test dataset, with $N=200$, $n_h=32$, and $n_\psim=100$.}
  \label{table6}
  \begin{center}
  \begin{tabular}{|c||c|c||c|c|}
    \hline
        $\scriptstyle k$ & $\scriptstyle \CI^\trainingpp_k(\xx^6_\testp)$  & $\scriptstyle \CI^\SNNpp_k(\xx^6_\testp)$
         & $\scriptstyle \CI^\trainingpp_k(\xx^{24}_\testp)$  & $\scriptstyle \CI^\SNNpp_k(\xx^{24}_\testp)$ \\
    \hline
     14  & $[-0.7260,-0.4726]$   & $[-0.7333,-0.4249]$ & $[-0.6375,-0.4233]$   & $[-0.7347,-0.4264]$\\
     35  & $[ 0.8761, 1.2625]$   & $[ 0.7741, 1.2307]$ & $[ 0.8054, 1.0856]$   & $[ 0.7620, 1.2102]$\\
     41  & $[ 1.1512, 1.6467]$   & $[ 1.0590, 1.5940]$ & $[ 1.0288, 1.4066]$   & $[ 0.9874, 1.5006]$\\
    \hline
  \end{tabular}
\end{center}
\end{table}
\noindent For $k \in\{14,35,41\}$ and $i\in\{6,24\}$, Figs.~\ref{figure7} and \ref{figure8}  show the probability density functions
$y\mapsto p_{YY_k\vert\XX}^\SNNpp(y \vert\xx_\testp^i;\bftheta^\optp)$ (estimated with the optimal SNN),
which are compared with the densities $y\mapsto p_k^\trainingpp(y \vert\xx_\testp^i)$ (estimated from the training dataset).
\begin{figure}[H]
    \centering
    \begin{subfigure}[b]{0.30\textwidth}
        \centering
        \includegraphics[width=\textwidth]{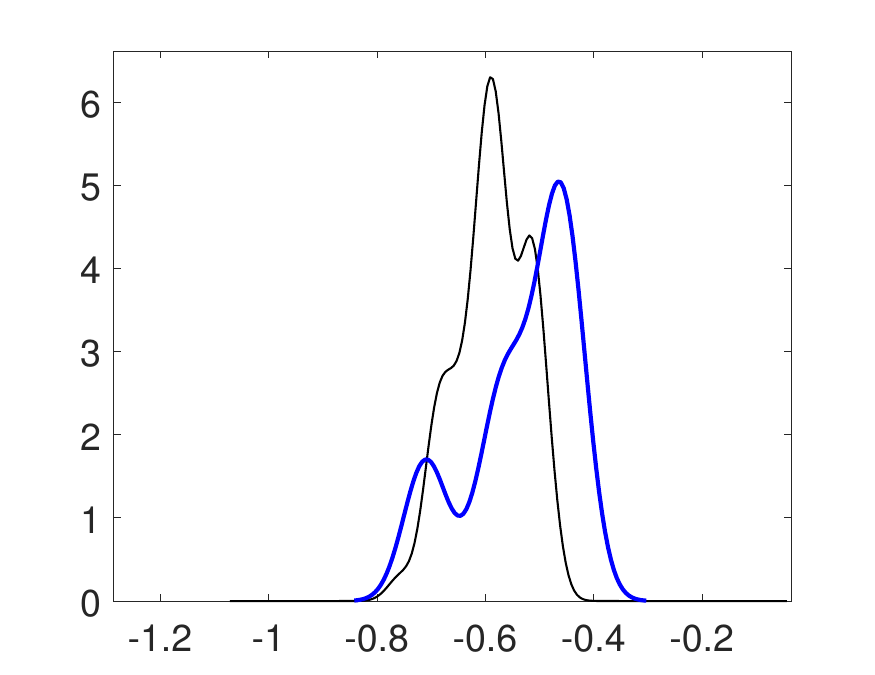}
        \caption{Graphs of the probability density functions for $k=14$.}
        \label{figure12a}
    \end{subfigure}
    \begin{subfigure}[b]{0.30\textwidth}
        \centering
        \includegraphics[width=\textwidth]{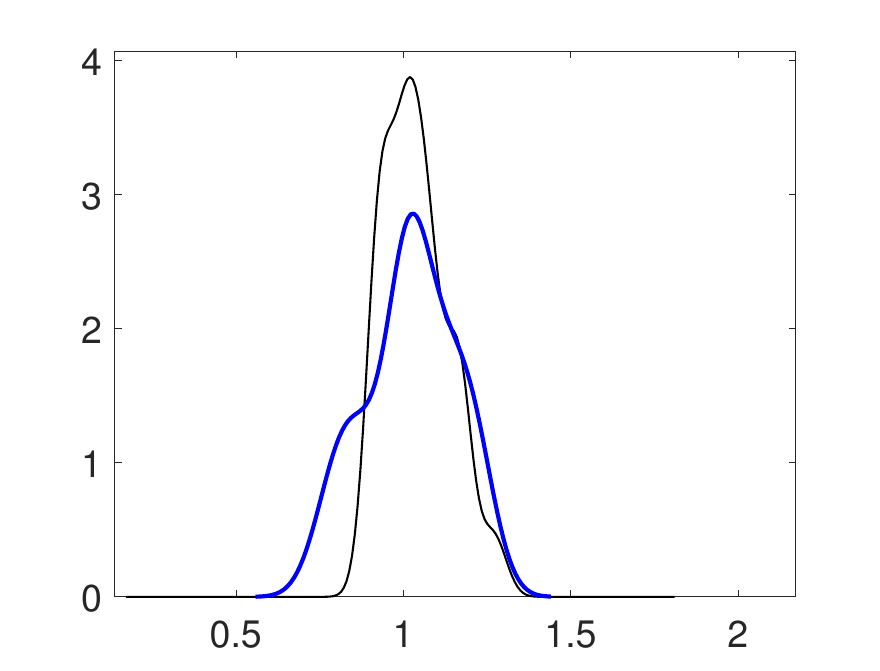}
        \caption{Graphs of the probability density functions for $k=35$.}
        \label{figure12b}
    \end{subfigure}
    \begin{subfigure}[b]{0.30\textwidth}
        \centering
        \includegraphics[width=\textwidth]{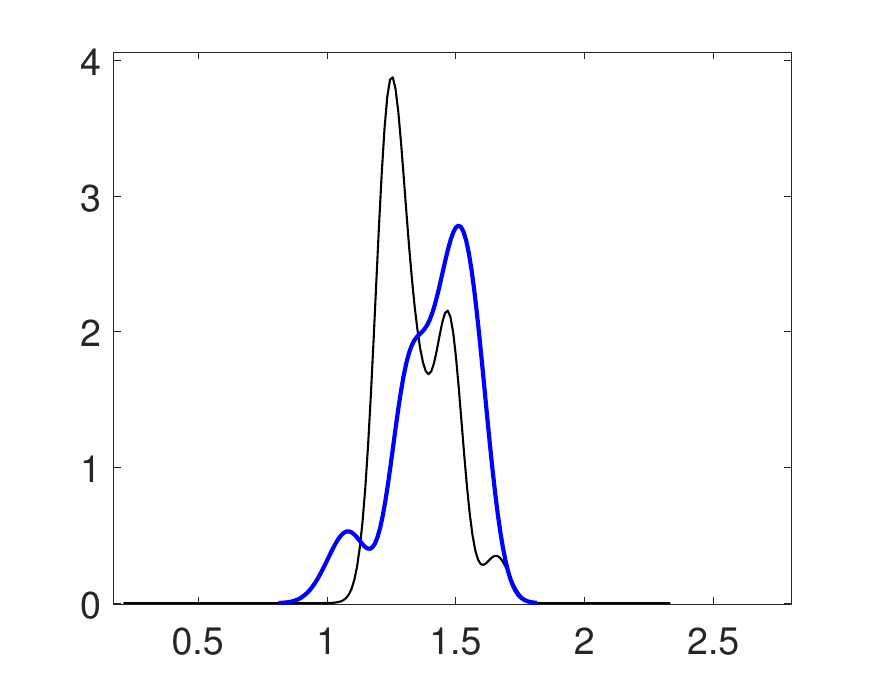}
        \caption{Graphs of the probability density functions for $k=41$.}
        \label{figure12c}
    \end{subfigure}
    \caption{Case 2: Plots of the probability density functions $y\mapsto p_{YY_k\vert\XX}^\SNNpp(y \vert\xx_\testp^6;\bftheta^\optp)$
                   (thin black line) and  $y\mapsto p_k^\trainingpp(y \vert\xx_\testp^6)$ (thick blue line), which should not be expected to closely align.}
    \label{figure12}
\end{figure}
\begin{figure}[H]
    \centering
    \begin{subfigure}[b]{0.30\textwidth}
        \centering
        \includegraphics[width=\textwidth]{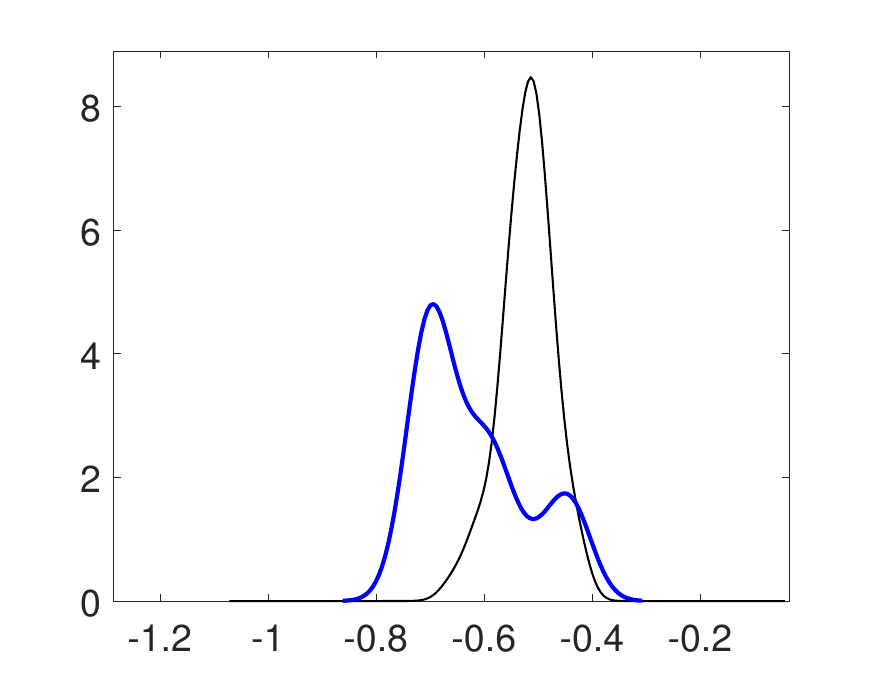}
        \caption{Graphs of the probability density functions for $k=14$.}
        \label{figure13a}
    \end{subfigure}
    \begin{subfigure}[b]{0.30\textwidth}
        \centering
        \includegraphics[width=\textwidth]{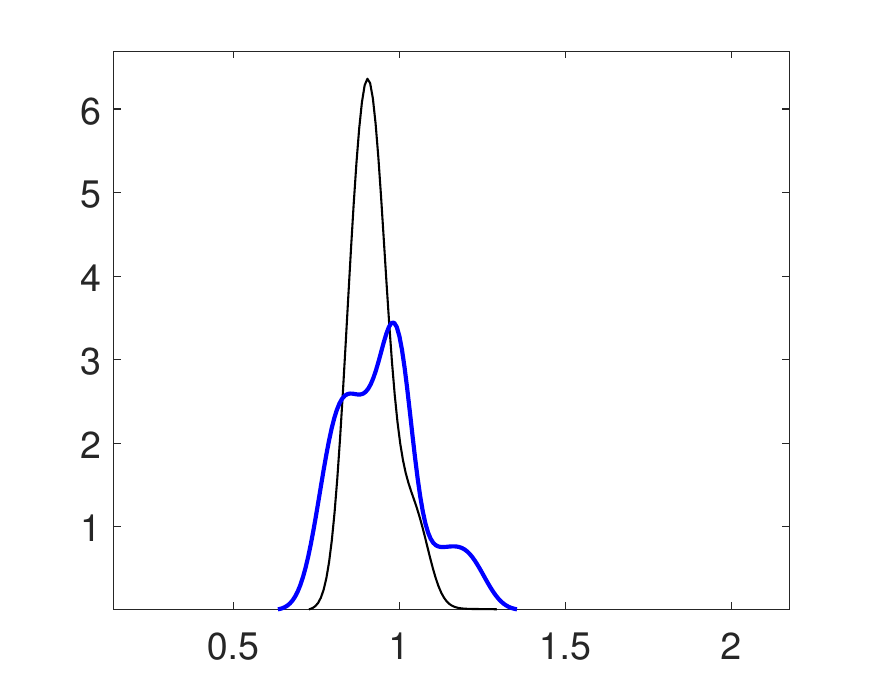}
        \caption{Graphs of the  probability density functions for $k=35$.}
        \label{figure13b}
    \end{subfigure}
    \begin{subfigure}[b]{0.30\textwidth}
        \centering
        \includegraphics[width=\textwidth]{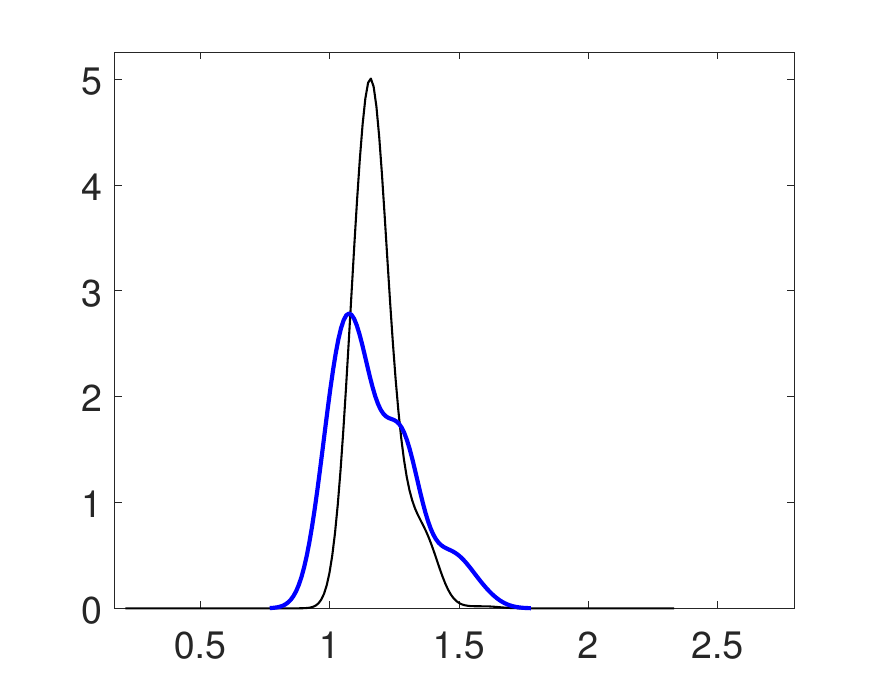}
        \caption{Graphs of the  probability density functions for $k=41$.}
        \label{figure13c}
    \end{subfigure}
    \caption{Case 2: Plots of the probability density functions $y\mapsto p_{YY_k\vert\XX}^\SNNpp(y \vert\xx_\testp^{24};\bftheta^\optp)$
                   (thin black line) and  $y\mapsto p_k^\trainingpp(y \vert\xx_\testp^{24})$ (thick blue line), which should not be expected to closely align.}
    \label{figure13}
\end{figure}
\noindent Since $\alpha_0 = 0.35428$ is not close to $1$ (in contrast to Case~1), the "train" and the "SNN" quantities  in Tables~\ref{table5} and \ref{table6} and  the probability density functions in Figs.~\ref{figure12} and  \ref{figure13} should not be compared with the expectation of close agreement.
\paragraph{(iii) Sensitivity analysis of $\alpha_0$ with respect to the number of neurons}
For this mismatch case, we focus on the architectural sensitivity to the number of neurons $N$.
To isolate this effect, we fix $n_h=32$ and $n_\psim =100$ and, for each $N$, we compute $\alpha_0(N)$ using the same estimation protocol as the one used in (i)-(ii). The parameter $\alpha_0$, which controls the alignment term in SPINN, is computed at each point in the convergence analysis. Consequently, the reported dependence $N \mapsto \alpha_0(N)$ should be interpreted as a relative sensitivity analysis at a fixed Monte Carlo budget, rather than as an attempt to eliminate Monte Carlo error in the absolute value of $\alpha_0$.
We therefore restrict the convergence analysis to $N \in \{130, 200, 250, 300\}$. Note that for $N=130$, the percentile threshold $\tau_\pprc$ (see Eq.~\eqref{eq4.15}) was reduced from $60\%$ to  $50\%$, increasing $n_\pinc$ from $3414$ to $4267$, in order to maintain sensitivity to the inputs given that $n_\pint= 48$ is really small.
Tables~\ref{table7} and \ref{table8} report the evaluation metrics for the PGD-Adam optimal solution as a function of $N$.
{\color{black}
It can be seen that $\alpha_0$ belongs to $[0.32\, , 0.44]$ over the tested neuron counts. Within the fixed normalized representation, these values keep the data-fidelity and alignment contributions at comparable numerical scales; no coordinate-invariant discrepancy interpretation is assigned to this range.
}
\begin{table}[h]
  \caption{Case~2: Neuron-architecture parameters and $\alpha_0$ values for the alignment term, as a function of $N\in\{130, 200, 250, 300\}$,  for $n_h =32$ and $n_\psim =100$.}
  \label{table7}
  \begin{center}
  \small
  \setlength{\tabcolsep}{6pt}        
  \renewcommand{\arraystretch}{1.0}  
  \begin{tabular}{|c|c|c|c|c|}
    \hline
      Eqs. & \eqref{eq4.12}&\eqref{eq4.24}&\eqref{eq4.33}&\eqref{eq6.7}\\
    \hline
      $ N$ & $ n_\pint$ & $ n_\pinc$ & $ n_b$ &  $\alpha_0$ \\
    \hline
      130 & 48  & 4267     & 32  & 0.3247 \\
      200 & 118 & 8053     & 80  & 0.3543 \\
      250 & 168 & 12607    & 114 & 0.4372 \\
      300 & 218 & 18120    & 148 & 0.3621 \\
    \hline
  \end{tabular}
  \end{center}
\end{table}
\begin{table}[h]
  \caption{Case~2: Evaluation-criteria values for the PGD-Adam optimal solution as a function of $N\in\{130, 200, 250, 300\}$,  for $n_h= 32$ and $n_\psim =100$.}
  \label{table8}
  \begin{center}
  \resizebox{\textwidth}{!}{%
  \begin{tabular}{|c|c|c|c|c|c|c|c|c|c|}
    \hline
      Eqs. & \eqref{eq3.2}&\eqref{eq3.1}&\eqref{eq3.3}&\eqref{eq3.5}&\eqref{eq3.5}&\eqref{eq3.6}&\eqref{eq3.9}&\eqref{eq3.9}&\eqref{eq3.10}\\
    \hline
      $N$  & $\scriptstyle\NLL_\SNNp^\opt$ & $\scriptstyle\NLL_\SNNp^\test$  &
           $\scriptstyle\criter_o$ & $\scriptstyle\criter_{\bfalpha}$ &  $\scriptstyle\criter_{\bfbeta}$ & $\scriptstyle\criter_{\IOCp}$ & $\scriptstyle\CRPS^\SNNpp_\testp$ & $\scriptstyle\CRPS^\trainingpp_\testp$ & $\scriptstyle\CRPS^\scaledpp_\testp$\\
    \hline
      130  & 8155.2 & 1653.7 & 0.01372 & 0.08581  & 0.07546 & 0.2705 & 0.05332 & 0.04139 & 1.2881 \\
      200  & 8554.8 & 1700.0 & 0.00615 & 0.08557  & 0.07109 & 0.2657 & 0.05579 & 0.04139 & 1.3478 \\
      250  & 7923.0 & 1570.2 & 0.00918 & 0.08086  & 0.06821 & 0.2593 & 0.04967 & 0.04139 & 1.2000 \\
      300  & 8984.9 & 1752.8 & 0.02524 & 0.09526  & 0.07900 & 0.2805 & 0.05970 & 0.04139 & 1.4423 \\
    \hline

  \end{tabular}
  } 
\end{center}
\end{table}
\paragraph{(iv) Additional comments on the results obtained for the Case~2}
{\color{black}
Tables~\ref{table7} and \ref{table8} show that the scale-matching factor $\alpha_0$ remains $O(1)$ over the tested neuron counts, with values in $[0.32,\,0.44]$. In the adopted normalized coordinates, the SBVP-alignment and data-fidelity contributions therefore remain numerically commensurate, and the prescribed choice $\alpha_{\rm{top}}=\alpha_0$ defines the mixture compromise analyzed in Proposition~\ref{Proposition1}.
}
Focusing on the completed runs, the generalization diagnostic stays well below the usual $0.10$ threshold
(\textit{e.g.}, $\criter_o \ll 0.10$ for $N\in\{130,\,200,\,250,\,300\}$), indicating that increasing architectural capacity does not trigger
a drift toward over-specialized stochastic configurations.
In addition, the calibration diagnostics $(\criter_{\bfalpha},\criter_{\bfbeta},\criter_{\IOCp})$ exhibit limited variations across $N$,
which is consistent with stable uncertainty-quantification behavior under misspecification.
In terms of the joint and marginal proper scores, the quantities $\NLL_\SNNp^\test$ and $\CRPS^\scaledpp_\testp$ can be interpreted jointly as follows: over the tested range, increasing $N$ does not produce monotone decreases of these scores; an intermediate capacity (here $N=250$) yields the smallest combined NLL--CRPS discrepancy, whereas increasing to $N=300$ can increase both the NLL and the CRPS gap despite stable generalization diagnostics.
\section{Discussion}
\label{Section7}
{\color{black}
\subsection{Positioning, target problem class, and scope of practical claims}
\label{Section7.1}
SPINN is positioned as a probabilistic learning framework for SBVPs with random operators. It targets problems in which deterministic or controlled inputs produce random, potentially high-dimensional and strongly non-Gaussian observables; only one realization of the output may be available for each observational input; and an SBVP simulator is available without requiring differentiation, solver modification, state-field reconstruction, or PDE-residual evaluation. The formulation also explicitly covers model-misspecified settings in which the probability distribution induced by the nominal SBVP does not coincide with the data-generating conditional probability distribution.
The construction has two distinct roles. First, the non-feedforward SNN defines an intrinsically stochastic, distribution-valued hypothesis class whose random topology, connectivity, and weights are controlled by a low-dimensional generative hyperparameter. Second, SPINN structurally couples this hypothesis class to the SBVP random operator through the shared Gaussian germ and the cross-entropy alignment term. In the SBVP-consistent regime, the alignment mainly acts as a finite-sample regularization and variance-reduction mechanism. In the SBVP-misspecified regime, it implements an explicit compromise between data fidelity and the conditional probability distribution induced by the nominal SBVP.
The complexity of the latent-field, point-process, random-connectivity, fixed-point, Monte Carlo, and density-estimation construction is therefore constitutive of the proposed stochastic-operator representation. It is not introduced solely to improve the score of a conventional deterministic regression model. In the reference configuration, the probability law of an SNN with $N=200$ neurons is governed by $n_{\bftheta}=3+2n_h=67$ optimized generative hyperparameters for $n_h=32$, while the $n_b=80$ bias coefficients are obtained through the least-squares construction of Section~\ref{Section4.6}, rather than through independent optimization of all network weights.
The numerical study establishes feasibility in a regime involving an $80$-dimensional non-Gaussian observable, $n_d=200$ training pairs, and one observed realization per input. The reported proper scores and calibration diagnostics show that the resulting conditional probability distributions remain close to the data-based benchmark in the SBVP-consistent case and remain controlled under deliberate SBVP misspecification. These results do not constitute a claim of universal superiority in predictive accuracy, wall-clock efficiency, scalability, or simulator-sample economy relative to all simpler probabilistic surrogates.
{\color{black}
Accordingly, the numerical example is a validation of the new stochastic-operator construction within its declared target class. Establishing empirical dominance over unrelated probabilistic surrogate families is a different comparative objective and is not among the claims of the present paper.
}
The data economy considered here concerns the number of physical observations, not the number of SBVP evaluations used in the alignment term. Similarly, the present network-size studies demonstrate feasibility over the investigated range but do not constitute a general scalability result. Simpler conditional density estimators can be preferable when only point predictions, low-order moments, or an unconstrained predictive probability distribution are required. SPINN is intended for the more specific setting in which architecture-level stochasticity and structural probabilistic coupling to the random physical operator are themselves part of the modeling objective. Accordingly, performance is assessed in observation space using strictly proper scoring rules (test NLL and CRPS), calibration discrepancies, and the generalization index $\criter_o$ (Section~\ref{Section3}).
}
{\color{black}
\subsection{Computational cost, online prediction, and scope of cross-paradigm comparisons}
\label{Section7.2}
The SPINN computational workflow has distinct offline-training and online-prediction stages. The SBVP ensembles
$\{\bfy^\ell(\xx^i),\ell=1,\ldots,n_\psim\}$ used in the cross-entropy term do not depend on $\bftheta$ and can be generated once for the $n_d$ training inputs and then reused throughout hyperparameter selection and optimization. The repeated cost during optimization is therefore associated with the SNN construction and the evaluation of its predictive conditional probability distributions.
Let $n_e$ denote the number of active edges in the sparse random neural graph, let $\hat n = N - n_\pin$, and let $n_{\rm fp}$ denote the average number of fixed-point iterations. For one value of $\bftheta$, random-architecture construction includes the reduced spectral representation of the latent field, evaluation of the Poisson intensity on the $M$ candidate points, selection of $N$ neuron locations, construction of the geodesic connectivity and sparsity mask, and generation of $n_\psim$ sparse random weight matrices. After the mesh-dependent geometric quantities have been precomputed, the non-spectral parts of this construction scale at most quadratically with $N$, together with a contribution of order $\curO (n_\psim n_e)$ for generating the weight ensemble.
For $n_x$ input points, the iterative sparse-matrix contribution to the generation of the SNN output ensemble scales as
$\curO(n_x n_\psim n_{\rm fp} n_e)$, in addition to the sparse linearized initialization used by the fixed-point solver. Direct Gaussian-KDE evaluation of the data NLL scales as $\curO (n_x n_\psim n_\pout)$, whereas the Monte Carlo cross-entropy estimator, which evaluates an SNN KDE with $n_\psim$ centers at $n_\psim$ SBVP-generated pseudo-observations, scales as
$\curO(n_x n_\psim^2 n_\pout)$. The latter term is the dominant density-estimation contribution in the present implementation.

For the reference configuration, $n_d = 200$, $n_\psim = 100$, $n_\pout = 80$, and $n_e = 8067$. One evaluation of the complete training objective therefore involves $n_d n_\psim = 20\,000$ nonlinear SNN solves. The direct cross-entropy KDE contains
$ n_d n_\psim^2 n_\pout = 200 \times 100^2 \times 80 = 1.6\times 10^8$ scalar kernel-component contributions. The final trial grid contains $n_\pgrid = 324$ points. In the PGD--Adam refinement, the optimized vector $(\bfbeta_{n_h}^{(1)},\bfbeta_{n_h}^{(2)})$ has dimension $2 n_h = 64$. Consequently, the central finite-difference gradient requires $2 (2 n_h) = 128$ objective evaluations per gradient, before the additional objective evaluations that can be introduced by the backtracking safeguard. This decomposition identifies the direct KDE implementation and the finite-difference gradient as the principal computational bottlenecks. The current implementation prioritizes methodological validation over large-scale computational optimization; accelerated KDE evaluation and differentiated or adjoint gradients are natural routes toward larger-scale implementations.
Component-wise wall-clock timings were not recorded during the calculations in a form that separates architecture generation, fixed-point solution, KDE evaluation, and optimization overhead. Retrospective hardware-specific timings would consequently not provide a reliable or reproducible cost decomposition. For this reason, the present paper reports operation counts and algorithmic scaling rather than claiming implementation-independent wall-clock performance.
After $\bftheta^\optp$ has been learned, neither SBVP simulation nor hyperparameter reoptimization is required for a new input. The random architecture and the ensemble of sparse random weight matrices can be constructed and stored offline. For one new input, distribution-valued prediction then requires $n_\psim$ fixed-point solutions; this corresponds to $100$ nonlinear SNN solves in the reference configuration. KDE is required only when an explicit value of the predictive probability density function is requested. Samples, empirical moments, quantiles, confidence intervals, and empirical proper scores can be computed directly from the output ensemble.

No numerical speedup relative to another surrogate is claimed. A speedup depends on the requested probabilistic output, the prescribed accuracy, the number of generated samples, the numerical implementation, and the hardware. In particular, comparing a distribution-valued prediction based on $n_\psim$ stochastic realizations with a single deterministic forward evaluation would not define a homogeneous computational metric.
The absence of cross-paradigm benchmarks is also deliberate. A fixed-topology stochastic network, a polynomial-chaos expansion, a conditional normalizing flow, and a probabilistic neural operator define different hypothesis classes, stochastic coordinates, capacity controls, and training objectives. Their use of the SBVP information is also different from the shared-germ structural coupling introduced here. A controlled comparison would require method-specific choices and tuning of topology, expansion order, latent dimension, regularization, conditioning variables, and model size, together with repeated realizations and initializations. Such an investigation is valuable but constitutes a separate benchmarking study.
{\color{black}
A single additional ``simple'' probabilistic surrogate would therefore not provide a neutral ablation of the proposed shared-germ SBVP--SNN coupling, because both the hypothesis class and the use of physical information would change simultaneously. The nonparametric conditional estimator already constructed from the training data and used in $\CRPS^\scaledpp_\testp$ provides the data-based probabilistic reference for the reported predictive-distribution scores without changing the scope of the study.
}
The present numerical claims are therefore restricted to feasibility, calibration, generalization, and proper-score performance of the proposed SBVP--SNN coupling. The conditional nonparametric estimator constructed from the same training data provides the internal probabilistic reference used in $\CRPS^\scaledpp_\testp$. This benchmark evaluates the quality of the predicted conditional probability distributions without reproducing the proposed architecture-level coupling. SPINN is not claimed to be universally more accurate or less expensive than simpler probabilistic surrogates; its contribution is to provide an additional stochastic-neural modeling route in which the random physical operator and the random architecture are structurally coupled through the same aleatory germ.
}
\subsection{Comments on CRPS levels in the Case~1 and Case~2 regimes}
\label{Section7.3}
The CRPS-based test scores, defined in Eq.~\eqref{eq3.9}, were introduced in Section~\ref{Section3} and interpreted in Section~\ref{Section3.5}. Here we provide additional comments on the levels reported in Sections~\ref{Section6.6}-\ref{Section6.7} for Case~1 (SBVP-consistent datasets, meaning that the training and test datasets almost surely satisfy the SBVP up to numerical discretization) and Case~2 (SBVP-misspecified datasets, meaning that the training and test datasets do not almost surely satisfy the SBVP).

These scores provide a strictly proper summary of probabilistic performance of the conditional probability prediction, since the CRPS is a strictly proper scoring rule that jointly rewards calibration and sharpness of the univariate predictive conditional probability distributions \cite{Matheson1976,Hersbach2000,Gneiting2007}. In the numerical configuration of Section~\ref{Section6}, with $n_d = 200$ training pairs, $n_\testp = 40$ test pairs, $N=200$, $n_\pin = 2$ inputs, $n_\pout = 80$ outputs, $n_h = 32$, and $n_\psim = 100$ Monte Carlo realizations per test input, the global test score is
$\CRPS^\scaledpp_\testp = {\CRPS^\SNNpp_\testp}/{\CRPS^\trainingpp_\testp} = 1.0630$ (i.e., $6.3\%$) for Case~1 (decreasing to $1.018$ for $N=300$, i.e., $1.8\%$), and $\CRPS^\scaledpp_\testp = 1.3478$ for Case~2.
{\color{black}
The latter corresponds to an increase of approximately $35\%$ in CRPS relative to the training-conditional nonparametric benchmark, indicating a loss of marginal accuracy of the estimated probability distributions in this mismatch regime. The Case~2 misspecification simultaneously modifies the supports of the two input variables and the correlation lengths of $U_c$. Consequently, this degradation quantifies the effect of the combined SBVP misspecification and cannot be attributed solely to the correlation-length perturbation.
}
{\color{black}
The value $\alpha_0 = 0.3543$ is the scale-matching coefficient used in the normalized training objective and is not interpreted here as a coordinate-invariant discrepancy measure.
}
\begin{enumerate}
\item First, the random operator encoded by the SNN is strongly constrained at the architectural level. The total number of neurons is fixed to $N = 200$, with $n_\pint = 118$ hidden units for an $80$-dimensional output, and the random connectivity matrix is sparsified so that only about $8067$ inter-node connections remain active for Case~1 and $8053$ for Case~2, corresponding to an overall sparsification rate of about $60\,\%$ and a sparsification of about $28\,\%$ for edges incident to the input neurons. The entire probability distribution of this random architecture is parameterized by a low-dimensional hyperparameter $\bftheta\in \curC_{\bftheta} \subset \RR^{n_{\bftheta}}$ with $n_{\bftheta} = 3 + 2 \, n_h  = 67$ for the present choice $n_h = 32$ (note that $n_b = 80$ bias coefficients enter the least-squares solve).
    This is markedly different from a conventional feedforward network with $\curO(10^4)$ independently trained weights, and naturally induces an approximation error between the SNN predictive conditional probability distribution and the empirical conditional probability distribution of the SBVP outputs.
\item Second, even when $\bftheta$ is optimized for a given $\alpha_{\rm{top}}$ as described in Section~\ref{Section6.5}, the information available to identify the conditional probability distributions remains intrinsically limited by finite-sample effects. For each of the $n_d = 200$ training inputs, only a single realization of the random SBVP operator is observed, and the two-dimensional input space is sampled at finitely many locations. The benchmark $\CRPS^\trainingpp_\testp$ is therefore itself a finite-sample quantity, obtained through nonparametric estimation of a probability density function based on one observation per input. For the SBVP-consistent regime (Case~1), it would therefore be unrealistic to expect that a structurally constrained parametric random operator with $147$ degrees of freedom can, on a test dataset of size $n_\testp = 40$, systematically match or outperform the CRPS level of a purely data-driven estimator that directly exploits the full training sample without architectural constraints.
\item Third, the SNN predictive conditional probability distribution is evaluated by Monte Carlo using $n_\psim = 100$ independent realizations of the random operator at each test input. The CRPS functional in Eq.~\eqref{eq3.9} is computed by replacing the predictive conditional probability distribution with its empirical distribution based on these $n_\psim$ samples. Even if the SNN predictive conditional probability distribution coincided with the data-based conditional probability distribution, the empirical CRPS computed from $n_\psim=100$ samples would still exhibit a sampling fluctuation of order $n_\psim^{-1/2}$. This unavoidable Monte Carlo variability compounds the structural approximation error induced by the finite size and sparsity of the network.
\end{enumerate}
For Case~1, beyond the architectural and Monte Carlo effects discussed above, it is important to recall that the data-based benchmark $\CRPS^\trainingpp_\testp$ is itself a finite-sample quantity obtained from a nonparametric estimate of the conditional distribution of the SBVP outputs given the inputs. In practice, the conditional samples used to approximate this empirical probability distribution are generated by the multivariate $k$-nearest neighbor conditional probability density function estimator of Mack and Rosenblatt~\cite{Mack1979}, as described in Section~\ref{Section3}. Although this estimator is consistent under suitable regularity assumptions in the asymptotic regime $n_d \to \infty$, it is here deployed in a demanding high-dimensional configuration with an $80$-dimensional output space and only $n_d = 200$ training pairs.
In this setting, the $k$-NN procedure necessarily balances bias and variance through the choice of $k$ and reconstructs conditional distributions by borrowing information from neighboring inputs in the two-dimensional input space, under an implicit smoothness hypothesis. Consequently, $\CRPS^\trainingpp_\testp$ should not be interpreted as an oracle CRPS associated with the true conditional probability distribution of the (possibly misspecified) SBVP model, but rather as the performance of a specific nonparametric estimator applied to a single finite training sample with one SBVP realization per input. The reported ratio $\CRPS^\scaledpp_\testp$ must therefore be understood as the discrepancy between a random operator with a $147$-dimensional parameterization, consisting of $67$ generative hyperparameters and $80$ bias coefficients constructed by least squares, and a flexible nonparametric benchmark in a finite-sample, high-dimensional regime, rather than as evidence of an intrinsic limitation of the SNN framework.

Taken together, the three effects described above, (1) a strongly constrained random architecture with $N = 200$ neurons and a $147$-dimensional parameterization, consisting of $67$ generative hyperparameters and $80$ bias coefficients constructed by least squares, (2) a finite training sample with a single SBVP realization per input, and (3) a Monte Carlo approximation of the predictive conditional probability distributions based on $n_\psim = 100$ samples—provide a coherent explanation for the observed level of $\CRPS^\scaledpp_\testp$ in the finite-parameter SBVP-consistent regime. The resulting scores remain small and are of the same order as those reported in recent work on probabilistic neural operators for parametric PDEs and related uncertainty-quantification problems, where strictly proper scoring rules are used to assess predictive conditional probability distributions over function-valued outputs \cite{Berrisch2024,Mouli2024,Pacchiardi2024}.
In this sense, for Case~1, the SNN-based random operators achieve a CRPS level that is fully consistent with the combined effects of architectural sparsity, limited sample size, and finite Monte Carlo resolution imposed in the present numerical study.
\paragraph{Remark on CRPS versus joint likelihood}
For the SBVP-consistent regime, the CRPS-based scores in Eqs.~\eqref{eq3.7}-\eqref{eq3.9} are used here as intrinsic evaluation criteria for the probability predictions of the SNN, not as training losses. This distinction is important in the present setting. By construction, the quantities
$\CRPS^\SNNpp_k(\xx)$ and $\CRPS^\trainingpp_k(\xx)$ entering Eq.~\eqref{eq3.9} act componentwise on the univariate marginals, and the global test scores $\CRPS^\SNNpp_\testp$ and $\CRPS^\trainingpp_\testp$ are obtained by averaging over output indices $k = 1,\dots,n_\pout$ and over test inputs.
While this makes CRPS a marginal diagnostic in its definition, this does not mean that marginal agreement
is disconnected from the joint predictive conditional probability distribution in the present SPINN/SNN construction. Indeed, the SNN joint conditional probability density function
$p^\SNNpp_{\YY\vert\XX}(\cdot \,\vert \,\xx\, ; \bftheta)$ on $\RR^{n_\ppout}$ is generated by a low-dimensional
admissible hyperparameter $\bftheta$ that simultaneously governs (i) the latent-field-driven random weight matrix and (ii) the random architecture (including input/output selections) through the same underlying latent mechanism (Sections~\ref{Section4.2} to \ref{Section4.5}).
Consequently, within this constrained model class, the full conditional dependence structure across components is
not freely adjustable independently of the marginals; it is a deterministic functional of $\bftheta$.

Let us now comment on the marginal-to-joint identifiability property (within the SPINN/SNN family).
For a fixed input $\xx \in \RR^{n_\ppin}$, let $\curL(\YY_k(\xx\, ;\bftheta)\, \vert \, \XX =\xx)$
be the conditional probability distribution of the $\RR$-valued random variable $\YY_k(\xx\, ;\bftheta)$  given the input $\XX=\xx$, and let us assume that
the mapping
$$\bftheta  \longmapsto \{ \, \mathcal{L}\,\big( \YY_k(\xx;\bftheta)\, \vert \,\XX = \xx \big ) \,:\, k=1,\ldots,n_\pout \}$$
is injective on the admissible set $\curC_{\bftheta}$. Under this property, if all componentwise
conditional marginals of the trained SNN coincide with the corresponding target conditional marginals at input $\xx$,
then the corresponding conditional joint probability distributions on $\RR^{n_\ppout}$ necessarily coincide as well.
Therefore, in this constrained setting, componentwise CRPS provides a meaningful intrinsic probe of the learned
conditional probability distribution, even though it is computed marginally.

In contrast, the SPINN training objective $\curJ^\SPINNpp(\bftheta\, ; \alpha_{\rm{top}})$ in Eqs.~\eqref{eq2.12}–\eqref{eq2.13} is formulated directly in terms of the joint conditional probability density function
$p^\SNNpp_{\YY\vert\XX}(\cdot \,\vert \,\xx\, ; \bftheta)$ on $\RR^{n_\ppout}$ and of the SBVP-induced joint conditional probability density function
$p^\SBVPpp_{\bfY \vert \XX}(\cdot\, \vert \,\xx)$ on $\RR^{n_\ppout}$, so that both the negative log-likelihood and the SBVP$\to$SNN alignment term act on the full probability distributions represented by these probability density functions.
\subsection{Comments on the conditional probability density functions for selected outputs}
\label{Section7.4}
To complement the discussion on the global intrinsic diagnostics introduced in Sections~\ref{Section3.5} and \ref{Section7.3}, we now examine the conditional probability density functions of selected output components. For a few representative components and test inputs, the conditional probability density functions are shown in Figs.~\ref{figure7}-\ref{figure8} for Case~1, and Figs.~\ref{figure12}-\ref{figure13} for Case~2.
In Case~2 (the SBVP-mismatched regime), the SPINN optimum targets a compromise
between data fidelity and SBVP alignment; therefore, the SNN conditional probability density functions should not be expected to match
the training-conditional benchmark estimated from the data. We therefore focus the discussion on Case~1, where
the training-conditional benchmark targets (up to finite-sample estimation error) the same conditional probability distribution as the
SBVP.
For a fixed test input $\xx_\testp$ and an output index $k \in \{1,\dots,n_\pout\}$, we estimate the one-dimensional conditional probability density functions
\begin{equation*}
y \mapsto p^\SBVPpp_{Y_k \vert \XX}(y\, \vert \,\xx_\testp)
\quad , \quad
y \mapsto p^\SNNpp_{\YY_k \vert \XX}(y\, \vert \,\xx_\testp \, ;\bftheta^\optp) \, ,
\end{equation*}
using the training dataset (via the conditional probability density function estimator described in Section~\ref{Section3}) and the $n_\psim$
realizations of the trained SNN random output $\YY(\xx_\testp \, ;\bftheta^\optp)$, respectively (see
Section~\ref{Section2.1}, paragraphs~(iv) and~(v)).
During training, the bandwidths entering the likelihood and alignment terms are fixed by Eq.~\eqref{eq4.45} and kept constant. In the numerical post-processing reported here, the conditional probability density functions are recomputed specifically for visualization as follows.

For the training-conditional reference (SBVP-based reference), the conditional probability density function
$p^\SBVPpp_{Y_k \vert \XX}(\cdot\, \vert \,\xx_\testp)$
is obtained using a nonparametric conditional probability density function estimator for the scalar quantity of interest and the vector-valued input with Silverman-type bandwidths. For the SNN-based surrogate, the conditional probability density function
$p^\SNNpp_{\YY_k \vert \XX}(\cdot \, \vert \,\xx_\testp \, ;\bftheta^\optp)$
is computed using GKDE from the $n_\psim$ realizations $\{\yy^\ell_k(\xx_\testp \, ;\bftheta^\optp)\}_\ell$ and using the Silverman rule for the bandwidth.

For the inspected pairs $(k,\xx_\testp)$ extracted from the test dataset, the two estimated probability density functions are similar when plotted using the above procedure, that is, when both conditional probability density functions are computed with Gaussian kernels and Silverman-type bandwidths. The SNN reproduces the non-Gaussian shapes induced by the SBVP random operator with satisfactory fidelity. This approximate visual agreement is consistent with the test-set scores reported in Section~\ref{Section6.6}-(iv) and discussed in Section~\ref{Section7.3}, which already indicate that the SNN surrogate calibrated within the SPINN framework produces predictive conditional probability distributions at the aggregate level.

A different behavior is observed for certain output components. For these components, and for some test inputs $\xx_\testp$, the conditional probability density function estimated from the SNN realizations appears more rugged than its SBVP counterpart, even though the same Silverman-type bandwidth rule is used for both estimators. From a purely graphical point of view, it can be observed that a moderate increase in the bandwidth used for the SNN-based KDE would be sufficient to improve the comparison. Empirically, this phenomenon appears more frequently for components whose conditional standard deviation is smaller than for pairs for which the comparison is reasonably good.

This behavior can also be interpreted qualitatively in light of the SPINN construction. For any fixed input $\xx$, both $\bfY(\xx)$ and $\YY(\xx \, ;\bftheta^\optp)$ are generated by transporting the same Gaussian random vector $\bfH$ through two distinct stochastic mappings: the random operator of the SBVP and the random neural architecture of the SNN. These mappings are highly nonlinear and produce non-Gaussian conditional probability distributions, but they do not induce exactly the same fine-scale structure in probability space. When the physical variability of a given observable $Y_k(\xx)$ is comparatively small, the corresponding conditional probability density function is supported on a narrow region of $\RR$; in such a regime, modest differences in the way the random neural architecture transports the Gaussian random vector to $\YY_k(\xx \, ;\bftheta^\optp)$ become more visible.
\subsection{Discussion: a generalized-posterior extension for $\bftheta$ without per-query optimization}
\label{Section7.5}
In the present paper, supervised calibration is performed by point estimation: a single admissible optimizer
$\bftheta^\optp \in \curC_{\bftheta}$ is computed by minimizing $\curJ^{\SNNpp}$ in Eq.~\eqref{eq2.7}--\eqref{eq2.8}, or
$\curJ^{\SPINNpp}$ in Eqs.~\eqref{eq2.12}--\eqref{eq2.15} (which reduces to $\curJ^{\SNNpp}$ when $\alpha_{\rm top}=0$).
As a discussion-only extension (not a contribution of the present work), one may instead construct a probability measure
$\pi(d\bftheta \,\vert\, \curD_{\trainingp})$ on $\curC_{\bftheta}$ from the training dataset $\curD_{\trainingp}$ defined in Eq.~\eqref{eq2.5}.
Such a formulation can be relevant in finite-sample regimes, as it enables posterior averaging of conditional probability distributions without reoptimizing $\bftheta$ for each new input. All probability density functions and scores in this subsection are evaluated in the scaled domain defined in Section~\ref{Section2.1}-(iii).
This extension is not used in the theoretical developments or in the numerical experiments of the present paper.
\paragraph{(i) Setting and objective}
For each $\bftheta\in\curC_{\bftheta}$, the stochastic neural network induces a random output
$\YY(\xx\, ;\bftheta)$ and a conditional probability distribution (Markov kernel)
\begin{equation} \label{eq7.1}
K_{\bftheta}(d\yy \, \vert \, \xx) \equiv p^\SNNpp_{\YY\vert\XX}(\yy \,\vert \,\xx\, ; \bftheta)\, d\yy\, ,
\end{equation}
where $\yy\mapsto p^\SNNpp_{\YY\vert\XX}(\yy \,\vert \,\xx\, ; \bftheta)$  denotes the conditional probability density function on
$\RR^{n_\ppout}$ defined in Eq.~\eqref{eq2.8}.
In practice, for fixed $(\xx,\bftheta)$, we can generate $n_\psim$ realizations
$\{ \yy^\ell(\xx\, ;\bftheta) ,\ell=1,\ldots , n_\psim\}$ of $\YY(\xx\, ;\bftheta)$ by forward simulation
(see Section~\ref{Section2.1}-(iv) and Section~\ref{Section4}).
The objective is to construct the probability distribution $\pi(d\bftheta \,\vert\, \curD_{\trainingp})$ using only forward simulation, and to use it for posterior averaging of conditional probability distributions without per-query optimization.
Once $\pi$ is available, predictions can be performed by mixture-of-kernels inference:
\begin{equation}\label{eq7.2}
K_{\pi}(d\yy \, \vert \, \xx) = \int_{\curC_{\bftheta}} K_{\bftheta}(d\yy \, \vert \, \xx) \, \pi(d\bftheta \, \vert \, \curD_\trainingp ) \, .
\end{equation}
\paragraph{(ii) Likelihood-based posterior (standard construction)}
For each $\bftheta\in\curC_{\bftheta}$ and each training input $\xx^i$, the conditional probability density function
$\yy \mapsto p^\SNNpp_{\YY\vert\XX}(\yy \,\vert \,\xx^i; \bftheta)$ associated with the Markov kernel
$K_{\bftheta}(d\yy \, \vert \, \xx^i)$ is approximated using the Monte Carlo GKDE formula in Eq.~\eqref{eq4.45}.
The corresponding  Bayesian posterior is then defined by
\begin{equation}\label{eq7.3}
\pi(d\bftheta \, \vert \, \curD_\trainingp) = \frac{1}{Z(\curD_\trainingp)}\, \pi_0(d\bftheta)\,
\prod_{i=1}^{n_d} p^\SNNpp_{\YY\vert\XX}(\yy^i \,\vert \,\xx^i \, ; \bftheta) \, ,
\end{equation}
where $\pi_0(d\bftheta)$ is a prior supported on $\curC_{\bftheta}$ and
$Z(\curD_\trainingp)$ is the normalizing constant.
Because $p^\SNNpp_{\YY\vert\XX}$ is typically accessible only through simulation-based probability density function estimation (here via Monte Carlo GKDE), each likelihood evaluation entails an inner simulation loop and a bandwidth-dependent reconstruction. In non-Gaussian regimes, this may lead to a rugged and Monte Carlo–noisy log-likelihood, so generic MCMC sampling can be computationally demanding. Standard mitigation strategies include sequential Monte Carlo, tempering, and pseudo-marginal methods for noisy likelihoods; see, e.g., \cite{Gordon1993,DelMoral2006,Andrieu2009,Jarvenpaa2024}.\\

\noindent\textit{SPINN-consistent variant including the alignment term}.
To remain consistent with the SPINN training objective $\curJ^\SPINNpp$ (see Eq.~\eqref{eq2.12}),
one may instead define the following generalized posterior,
\begin{equation}\label{eq7.4}
\pi^\SPINNpp(d\bftheta \, \vert \, \curD_\trainingp)\, \propto \,
\exp \bigl(- \alpha_{\rm{top}} \sum_{i=1}^{n_d} \rm{CE}_i(\bftheta)\bigr )\,\,
\pi_0(d\bftheta)\, \prod_{i=1}^{n_d} p^\SNNpp_{\YY\vert\XX}(\yy^i \,\vert \,\xx^i \, ; \bftheta) \, ,
\end{equation}
where $\rm{CE}_i(\bftheta)$ is the cross-entropy alignment term defined in Eq.~\eqref{eq2.13} (and approximated as in Eq.~\eqref{eq2.15}).\\

\noindent Nevertheless, in high-dimensional settings with small datasets, posterior sampling can be performed using the non-Gaussian probabilistic learning on manifolds (PLoM) methodology of \cite{Soize2020b}, which is designed to address small-data Bayesian inference with non-Gaussian priors and likelihoods.
\paragraph{(iii) Gibbs posterior / generalized Bayes (score-based, simulation-only)}
Let $S(\bftheta\, ; \xx,\yy)$ be a proper scoring rule that can be evaluated from forward quantities
(possibly through Monte Carlo) \cite{Gneiting2007}. Define the empirical risk
\begin{equation}\label{eq7.5}
\widehat{R}_{n_d}(\bftheta) = \frac{1}{n_d}\sum_{i=1}^{n_d} S(\bftheta; \xx^i,\yy^i) \, .
\end{equation}
The Gibbs posterior (generalized Bayes) is then defined by \cite{Catoni2007,Bissiri2016},
\begin{equation}\label{eq7.6}
\pi_{\tau}(d\bftheta \, \vert \, \curD_\trainingp)\, \propto \,
\exp \bigl(-\tau\, n_d\, \widehat{R}_{n_d}(\bftheta)\bigr )\,\pi_0(d\bftheta) \, ,
\end{equation}
where $\tau > 0$ controls concentration.
In practice, $\tau$ is a user-chosen inverse-temperature that sets the concentration of $\pi_\tau$ around low-risk
values of $\widehat{R}_{n_d}$. A robust strategy is to tune $\tau$ by proper-score calibration: for a grid of candidate values
$\{ \tau_g \}$, construct an approximation of $\pi_{\tau_g}(d\bftheta \, \vert \, \curD_\trainingp)$ (e.g., using (iv) below with
fixed particle budget $J$), evaluate the resulting mixture predictor $K_{\pi_{\tau_g}}(\,\cdot\,|\,\xx)$ in Eq.~\eqref{eq7.2} on a
held-out subset (or via $K$-fold cross-validation) using a proper score, and select the minimizer.
Alternatively, within (iv) below,  one may use an adaptive tempering schedule for $\tau$ (increasing $\tau$ in
increments chosen to keep the effective sample size above a prescribed fraction of $J$), and take the final value
as the operational $\tau$.\\

\noindent\textit{Option A (data fidelity term only)}: negative log-likelihood via an estimated conditional probability density function
(the logarithmic score is proper) \cite{Gneiting2007}.
For each point $(\xx^i,\yy^i)$ in the training dataset,
$p^\SNNpp_{\YY\vert\XX}(\yy^i \,\vert \,\xx^i; \bftheta)$  is estimated by Eq.~\eqref{eq4.45} using
the realizations $\{ \yy^\ell(\xx^i\, ;\bftheta) ,\ell=1,\ldots , n_\psim\}$ of $\YY(\xx^i\, ;\bftheta)$.
We define
\begin{equation}\label{eq7.7}
S_{\!\SNNpp}(\bftheta\, ;\xx^i,\yy^i) = -\log p^\SNNpp_{\YY\vert\XX}(\yy^i \,\vert \,\xx^i; \bftheta) \, .
\end{equation}
This corresponds to the logarithmic score (proper) evaluated through a simulation-based estimate
of $p^\SNNpp_{\YY\vert\XX}(\cdot\mid\xx^i;\bftheta)$.\\

\noindent\textit{Option B (SPINN-consistent loss)}: include the alignment term in the score.
For each point $(\xx^i,\yy^i)$, we have
\begin{equation}\label{eq7.8}
S_{\!\SPINNpp}(\bftheta\, ;\xx^i,\yy^i)
= -\log p^\SNNpp_{\YY\vert\XX}(\yy^i \,\vert \,\xx^i; \bftheta) + \alpha_{\rm{top}}\, \rm{CE}_i(\bftheta)\, ,
\end{equation}
where $\rm{CE}_i(\bftheta)$ is evaluated as in Eq.~\eqref{eq2.15}.
This yields a generalized Bayes posterior probability distribution whose concentration is governed jointly by $\tau$ and $\alpha_{\rm{top}}$.
Option~B is the only choice among A and C (see below), which explicitly retains the SPINN alignment contribution
through $\rm{CE}_i(\bftheta)$ at the score level.\\

\noindent\textit{Option C (sample-only proper score)}: the energy score (no logarithms of probability density functions) \cite{Gneiting2007}.
A Monte Carlo approximation is
\begin{equation}\label{eq7.9}
\widehat{S}_{\!\ESpp}(\bftheta\, ; \xx^i,\yy^i) =
\frac{1}{n_\psim} \sum_{\ell=1}^{n_\psim} \, \Vert\yy^\ell(\xx^i\, ;\bftheta) - \yy^i \Vert
- \frac{1}{n_\psim(n_\psim -1)} \sum_{1\,\le\,\ell \, < \, \ell'\, \le \, n_\psim}\,
\Vert \yy^\ell(\xx^i\, ;\bftheta)-\yy^{\ell'}(\xx^i\, ;\bftheta) \Vert \, .
\end{equation}
For $n_\ppout=1$ and with $\Vert\cdot\Vert$ replaced by $\vert \cdot \vert$, Eq.~\eqref{eq7.9} reduces to the standard
Monte Carlo identity used for the CRPS approximation in Eq.~\eqref{eq3.8}.
This score is often numerically robust.
\paragraph{(iv) Particle approximation without per-query optimization}
Here, $S(\bftheta;\xx^i,\yy^i)$ is taken as $S_{\!\SNNpp}$ (Option~A), $S_{\!\SPINNpp}$ (Option~B),
or $\widehat{S}_{\!\ESpp}$ (Option~C), consistently with the choice made above.
A discrete particle representation can be used to construct an explicit approximation of
$\pi_{\tau}(d\bftheta \, \vert \, \curD_\trainingp)$ without solving an optimization problem; this corresponds to
sequential importance sampling / sequential Monte Carlo ideas \cite{Doucet2001,DelMoral2006}.
Draw $J$ realizations $\bftheta^j\sim \pi_0(d\bftheta)$ and initialize $w_0^j=1/J$.
For each data point $(\xx^i,\yy^i)$ in the training dataset, update
\begin{equation}\label{eq7.10}
w_i^j \, \propto\,  w_{i-1}^j \,
\exp \bigl( - \tau\, S(\bftheta^j \, ; \xx^i , \yy^i )\bigr) \quad , \quad  j \in\{ 1 ,\ldots ,J \}\, ,
\end{equation}
followed by normalization $\sum_{j=1}^J w_i^j =1$ (implemented in log-scale for numerical stability).
The resulting approximation reads
\begin{equation}\label{eq7.11}
\pi_{\tau}(d\bftheta \, \vert \, \curD_\trainingp) \, \approx \,
\sum_{j=1}^J \, w_{n_d}^j\,\delta_{\bftheta^j}(d\bftheta) \, ,
\end{equation}
where $\delta_{\bftheta^j}(d\bftheta)$ is the Dirac measure at point $\bftheta^j$.
A particle approximation can be implemented using standard sequential Monte Carlo strategies. Weight degeneracy may be monitored via the effective sample size, with resampling and resample–move rejuvenation steps when necessary. Classical Metropolis–Hastings moves targeting the intermediate tempered probability distributions can be employed, including pseudo-marginal variants when the score is evaluated through Monte Carlo. Detailed algorithmic implementations follow established methodologies in sequential Monte Carlo and noisy-likelihood inference \cite{Doucet2001,DelMoral2006,Gilks2001,Andrieu2009}.

When $S$ is evaluated through Monte Carlo (Options A-C), the above move leaves invariant the Gibbs posterior
associated with the implemented (Monte Carlo) score; pseudo-marginal variants can be used when an unbiased
estimator of the relevant unnormalized target is available \cite{Andrieu2009}.
Predictions follow from Eq.~\eqref{eq7.2} by drawing an index $j^\star$ with probabilities $\{w_{n_d}^j\}_{j=1}^J$,
setting $\bftheta=\bftheta^{j^\star}$, and then sampling $\YY \sim K_{\bftheta}(\cdot \, \vert \, \xx)$.
\section{Conclusions}
\label{Section8}

{\color{black}
This work has introduced a stochastic physics-informed neural network (SPINN) for stochastic boundary-value problems (SBVPs) with random operators. The formulation couples the SBVP random operator and a manifold-based stochastic neural architecture through the shared Gaussian germ $\bfH$. The latent Gaussian field and the inhomogeneous Poisson point process then generate random neuron locations, geodesically local sparse connectivity, and a sparse random weight matrix. Physics is thus encoded at the level of the stochastic architecture and its induced conditional distributions, rather than through a residual penalty applied to a deterministic network.

Calibration is performed directly in observation space and requires only evaluations of the SBVP solution mapping, without state-field reconstruction, solver modification, or differentiation. The resulting non-feedforward SNN produces distribution-valued outputs for deterministic inputs, while a low-dimensional hyperparameter controls the probability law of its architecture and weights. This provides a compact representation mechanism for non-Gaussian random operators without optimizing a large collection of deterministic weights.
}
The supervised objective combines data fidelity, through the negative log-likelihood, with a cross-entropy alignment term relative to the SBVP-induced conditional probability distribution. In the SBVP-consistent regime (Case~1), this alignment mainly acts as a finite-sample variance-reduction and regularization mechanism. For the reference configuration ($n_d=200$, $n_\testp=40$, $n_h=32$, $n_\psim=100$), the scaled CRPS is $\CRPS^\scaledpp_\testp=1.0630$ for $N=200$ and decreases to $\CRPS^\scaledpp_\testp=1.018$ for $N=300$, showing that the SNN conditional distributions remain close to the data-driven conditional benchmark despite the low-dimensional hyperparameterization and the absence of weight-level optimization.

In the SBVP-misspecified regime (Case~2), the same formulation implements an explicit compromise between data fidelity and SBVP-induced probabilistic structure. {\color{black}The operational scale-matching value is $\alpha_0 = 0.35428$ for the reference configuration; for a nonzero alignment weight, Proposition~\ref{Proposition1} shows that the population target is the mixture $q_\lambda$, whereas
$\alpha_{\rm top} = 0$ would define the distinct data-only target $p^{\Obsp}_{\bfY\vert\XX}$.}
The scaled CRPS is $\CRPS^\scaledpp_\testp = 1.3478$ for $N = 200$. Across the tested network sizes, the proper scores are not monotone with respect to $N$; an intermediate capacity yields the smallest joint likelihood--CRPS discrepancy, while the generalization and calibration diagnostics remain stable. These results support the use of the cross-entropy alignment term as a controlled probabilistic regularization mechanism under model--data mismatch.

The numerical illustration is deliberately posed on a nontrivial geometry: a stochastic elliptic problem on a periodic cylindrical surface with a positive non-Gaussian random coefficient field and a rigorously defined weak formulation. It shows that the proposed framework can combine well-posed stochastic mechanics, manifold-based random neural architectures, and proper-score evaluation of conditional probability distributions within a single non-intrusive computational methodology.

{\color{black}
Finally, Section~\ref{Section7.5} outlines a generalized-posterior extension for the hyperparameter $\bftheta$. Although not used in the main numerical experiments, this extension provides a route to posterior averaging of conditional probability distributions without per-query reoptimization and is included as a perspective for future developments.
}
\section*{Data Availability}
No external datasets were used in this work. The numerical data analyzed in the paper were generated from the SBVP and misspecified SBVP models described in the paper. They can be regenerated from the model definitions and parameter values reported in the manuscript.
\section*{CRediT authorship contribution statement}
Christian Soize: Conceptualization, Methodology, Software, Computation, Writing.
\section*{Declaration of competing interest}
The author declares that they have no known competing financial interests or personal relationships that could have appeared to influence the work reported in this paper.


\end{document}